\documentclass[dvipsnames]{ut-thesis}
\usepackage[colorlinks, citecolor=teal, linkcolor=teal]{hyperref} % for links
\usepackage{graphicx} % for embedding graphics
\usepackage{booktabs} % for pretty tables
\usepackage{lipsum} % for gibberish text
\usepackage{caption}
\usepackage{subcaption}
\usepackage{pgfplots}
\usepackage{tikz}
\usepackage{natbib}
\usepackage{bibentry}
\nobibliography*
\usepackage{booktabs}

\usepackage{graphicx}       % For graphics inclusion
\usepackage{lmodern}        % Better fonts
\usepackage{microtype}

\usepackage{amsmath}
\usepackage{amssymb}
\usepackage{mathtools}
\usepackage{amsthm}
\usepackage{enumitem}
\usepackage{stmaryrd} % for \llbracket and \rrbracket
\usepackage{xspace}
\usepackage{multirow}
\usepackage{wrapfig}
\usepackage{nicefrac}
\usepackage[Lenny]{fncychap}
\ChNameVar{\Large}
\ChTitleVar{\Huge}

\usepackage{todonotes}

\usepackage{amsfonts}
\usepackage{dsfont}
\pgfplotsset{compat=1.15}
\usetikzlibrary{shapes, arrows.meta, positioning, fit, calc, intersections, pgfplots.fillbetween}

\usepackage{listings}

\lstdefinestyle{modelStyle}{
    backgroundcolor=\color{white},
    basicstyle=\ttfamily\footnotesize,
    breaklines=true,
    frame=single,
    rulecolor=\color{black},
    keywordstyle=\color{blue},
    commentstyle=\color{gray},
    stringstyle=\color{red},
    numbers=left,
    numberstyle=\tiny\color{gray},
    captionpos=b,
    language=Python
}
\usepackage{algorithm}
\usepackage{algpseudocode}

\DeclareMathOperator*{\argmin}{arg\,min}

\usepackage[most]{tcolorbox}

\tcolorboxenvironment{theorem}{
  enhanced jigsaw,
  colback=red!5!white,
  boxrule=0pt,                     % no visible full border
  borderline west={3pt}{0pt}{red!75!black},  % thick left border only
  sharp corners,
  before skip=10pt,
  after skip=10pt,
  breakable,
}

\tcolorboxenvironment{definition}{
  enhanced jigsaw,
  colback=green!5!white,
  boxrule=0pt,                     % no visible full border
  borderline west={3pt}{0pt}{green!60!black},  % thick left border only
  sharp corners,
  before skip=10pt,
  after skip=10pt,
  breakable,
}

\tcolorboxenvironment{lemma}{
  enhanced jigsaw,
  colback=cyan!5!white,
  boxrule=0pt,                     % no visible full border
  borderline west={3pt}{0pt}{cyan!60!black},  % thick left border only
  sharp corners,
  before skip=10pt,
  after skip=10pt,
  breakable,
}

\tcolorboxenvironment{remark}{
  enhanced jigsaw,
  colback=blue!5!white,
  boxrule=0pt,                     % no visible full border
  borderline west={3pt}{0pt}{blue!80!black},  % thick left border only
  sharp corners,
  before skip=10pt,
  after skip=10pt,
  breakable,
}

\tcolorboxenvironment{takeaway}{
  enhanced jigsaw,
  colback=yellow!10!white,
  boxrule=0pt,                     % no visible full border
  borderline west={3pt}{0pt}{yellow!80!black},  % thick left border only
  sharp corners,
  before skip=10pt,
  after skip=10pt,
  breakable,
}

\tcolorboxenvironment{contriback}{
  enhanced jigsaw,
  colback=black!5!white,
  colframe=black,
  sharp corners,
  before skip=10pt,
  after skip=10pt,
  breakable,
}

\tcolorboxenvironment{paperref}{
  enhanced jigsaw,
  colback=teal!5!white,
  colframe=teal,
  sharp corners,
  before skip=10pt,
  after skip=10pt,
  breakable,
}

\tcolorboxenvironment{assumption}{
  enhanced jigsaw,
  colback=purple!5!white,
  colframe=purple!70!black,
  boxrule=0.8pt,
  sharp corners,
  before skip=10pt,
  after skip=10pt,
  breakable,
}

\tcolorboxenvironment{corollary}{
  enhanced jigsaw,
  colback=red!5!white,
  boxrule=0pt,                     % no visible full border
  borderline west={3pt}{0pt}{red!75!black},  % thick left
  sharp corners,
  before skip=10pt,
  after skip=10pt,
  breakable,
}

\tcolorboxenvironment{proposition}{
  enhanced jigsaw,
  colback=teal!5!white,
  boxrule=0pt,                     % no visible full border
  borderline west={3pt}{0pt}{teal!60!black},  % thick left border only
  sharp corners,
  before skip=10pt,
  after skip=10pt,
  breakable,
}

\tcolorboxenvironment{proof}{
  enhanced jigsaw,
  colback=gray!5!white,
  boxrule=0pt,                     % no visible full border
  borderline west={3pt}{0pt}{black!50},  % thick left border only
  sharp corners,
  before skip=-10pt,
  after skip=10pt,
  breakable,
}

\newtheorem{theorem}{Theorem}[section]
\newtheorem{definition}[theorem]{Definition}
\newtheorem{lemma}[theorem]{Lemma}
\newtheorem{remark}[theorem]{Remark}
\newtheorem*{takeaway}{Takeaway}
\newtheorem*{contriback}{Contribution Acknowledgment}
\newtheorem*{paperref}{Paper Reference}

\newtheorem{corollary}[theorem]{Corollary}

\usepackage{bm}

\newcommand{\boxedblock}[1] {
\begin{tcolorbox}[colback=teal!5,colframe=teal!90!black]
\centering\emph{#1}
\end{tcolorbox}
}

\makeatletter
\newcommand{\algcolor}[2]{%
  \hskip-\ALG@thistlm\colorbox{#1}{\parbox{\dimexpr\linewidth-2\fboxsep}{\hskip\ALG@thistlm\relax #2}}%
}
\newcommand{\algemph}[1]{\algcolor{yellow!20!white}{#1}}
\makeatother

\author{Stephan Rabanser}
\title{Uncertainty-Driven Reliability: Selective Prediction and Trustworthy Deployment in Modern Machine Learning}
\degree{Doctor of Philosophy}
\department{Computer Science}
\gradyear{2025}
\begin{document}
  \frontmatter
    \maketitle
    \begin{abstract}

\noindent Machine Learning (ML) systems are increasingly deployed in high-stakes domains where reliability is paramount. As these systems transition from research prototypes to real-world decision-makers, their ability to recognize and respond to uncertainty becomes essential. This thesis investigates how uncertainty estimation can enhance the safety and trustworthiness of ML, with a particular focus on selective prediction—a paradigm where models abstain from predicting when confidence is low.

We first show that a model's training trajectory contains rich uncertainty signals that can be exploited without altering its architecture or loss. By ensembling predictions from intermediate checkpoints, we propose a lightweight, post-hoc abstention method that works across tasks, avoids the cost of deep ensembles, and achieves state-of-the-art selective prediction performance. Crucially, this approach is fully compatible with differential privacy (DP), allowing us to study how privacy noise affects uncertainty quality. We find that while many methods degrade under DP, our trajectory-based approach remains robust, and we introduce a framework for isolating the privacy-uncertainty trade-off. Next, we then develop a finite-sample decomposition of the selective classification gap -- the deviation from the oracle accuracy-coverage curve -- identifying five interpretable error sources and clarifying which interventions can close the gap. This explains why calibration alone cannot fix ranking errors, motivating methods that improve uncertainty ordering. Finally, we show that uncertainty signals can be adversarially manipulated to hide errors or deny service while maintaining high accuracy, and we design defenses combining calibration audits with verifiable inference.

Together, these contributions chart a path toward more reliable ML by studying how uncertainty can be estimated, evaluated, and safeguarded. The resulting systems not only make accurate predictions—but also know when to say “I do not know”.
    \end{abstract}
    \begin{acknowledgements}
      \noindent I would first like to express my deepest gratitude to my advisor, \textbf{Nicolas Papernot}, for his unwavering support, insightful guidance, and for fostering an incredibly vibrant and social lab culture. Working under your supervision has been a truly formative experience that has shaped both my academic and personal growth. Nicolas consistently empowered me to follow my curiosity, giving me the freedom to explore research directions I was genuinely passionate about. Your openness to new ideas created an environment where I felt trusted and encouraged to take intellectual risks. What I appreciated most was that our mentorship was a true dialogue—Nicolas was not only generous in sharing his expertise but also genuinely curious to learning from me, which made our collaboration feel deeply mutual and respectful. Beyond research, the lab culture you cultivated stands out as something truly special. Nicolas built a space that was intellectually stimulating and also socially connected—a place where collaboration, support, and friendship were the norm. When I share stories of how our lab operates with students from other institutions, they often remark on how rare and enviable such a community is. I feel incredibly fortunate to have been part of it.

I am also profoundly thankful to my supervisory committee professors—\textbf{Rahul G. Krishnan}, \textbf{Roger Grosse}, \textbf{David Duvenaud}, and \textbf{Zachary C. Lipton}—for their continued high-quality feedback and for challenging me to refine my research directions. Their expertise and generous mentorship have been integral to the progress of my work. I would also like to thank Professor \textbf{Aaron Roth} for serving as my external examiner on this thesis, and Professors \textbf{Chris J. Maddison} and \textbf{Roman Genov} for completing my final examination committee.

I am equally indebted to many current/former \textbf{members of the CleverHans Lab}, in particular \textbf{Mohammad Yaghini}, \textbf{Jonas Guan}, \textbf{Nathalie Dullerud}, \textbf{Sierra Wyllie}, \textbf{Anvith Thudi}, \textbf{Ilia Shumailov}, \textbf{David Glukhov}, \textbf{Nick Jia}, \textbf{Emmy Fang}, \textbf{Adam Dziedzic}, and \textbf{Franziska Boenisch}. Their friendship, thoughtful discussions, and sharp insights played a key role in shaping many of my projects. Beyond my immediate lab, the broader \textbf{Vector Institute community} provided a stimulating research environment, filled with researchers tackling interesting and highly related challenges. It also allowed me to become friends with many more outstanding researchers, including \textbf{Claas Voelcker}, \textbf{Vahid Balazadeh}, \textbf{Viet Nguyen}, \textbf{Andrew Jung}, and \textbf{Aroosa Ijaz}. I will genuinely miss this office space and the incredible people who made it such a great place to work at from the very beginning.

During the past five years, I was lucky to have had the chance to spend most of my summers doing internships in very exciting places. My gratitude goes to the \textbf{Amazon AWS Forecasting team} for hosting me for a total of four internships during my overall studies (one of them during my PhD)—particularly my managers, \textbf{Tim Januschowski} and \textbf{Jan Gasthaus}. Their mentorship and real-world ML deployment challenges were crucial stepping stones to my PhD journey. Similarly, I owe thanks to the \textbf{Google Research team} in Zurich—particularly my managers, \textbf{Nathalie Rauschmayr} and \textbf{Ace Kulshrestha}—whose support and guidance facilitated rich exploration into model cascading. I would also like to extend a heartfelt thanks to \textbf{David Krueger} for hosting me at the University of Cambridge during the summer of 2023, offering me another invaluable opportunity to broaden my research perspectives.

I am sincerely grateful to all my coauthors with whom I worked on a bunch of exciting projects: \textbf{Abhradeep Thakurta}, \textbf{Ace Kulshrestha}, \textbf{Adrian Weller}, \textbf{Adam Dziedzic}, \textbf{Akram Bin Sediq}, \textbf{Ali Shahin Shamsabadi}, \textbf{Angéline Pouget}, \textbf{Anvith Thudi}, \textbf{Armin Ale}, \textbf{Christopher A. Choquette-Choo}, \textbf{Congchao Wang}, \textbf{Emmy Fang}, \textbf{Federico Tombari}, \textbf{Hamza Sokun}, \textbf{Israfil Bahceci}, \textbf{Kimia Hamidieh}, \textbf{Krishnamurthy (Dj) Dvijotham}, \textbf{Mark R. Thomas}, \textbf{Mohammad Yaghini}, \textbf{Muhammad Ahmad Kaleem}, \textbf{Murat A. Erdogdu}, \textbf{Nathalie Rauschmayr}, \textbf{Nicolas Papernot}, \textbf{Olive Franzese}, \textbf{Petra Poklukar}, \textbf{Sean Augenstein}, \textbf{Somesh Jha}, \textbf{Wittawat Jitkrittum}, and \textbf{Xiao Wang}. The work I am discussing in this thesis would not have been possible without you and I am very thankful to have had the chance to work with all of you. 

Special thanks go out to my labmate, flatmate, and dear friend, \textbf{Mohammad Yaghini}. Moving in together at the start of our PhDs was a great choice. I already know that I will miss our joint late-night gaming sessions, the food we cooked together, and our long conversations at the kitchen table. Maybe one day we will even solve world politics together. Thanks for celebrating my victories with me, but also, and arguably even more importantly, for being there for me when I struggled at work and in my personal life. Similar thanks go out to my longtime friend, \textbf{Lukas Prantl}, for making sure I stay grounded in what really matters by ensuring that I continue to play Age of Empires II with him on a (semi-)regular basis. 

Lastly, I am profoundly grateful to my immediate family—my dad \textbf{Martin}, my mom \textbf{Verena}, and my sister \textbf{Monika}—for their unconditional kindness and support. Your belief in me has fueled my determination every step of the way. More importantly, you have shown me what a loving and caring family looks like, and I hope that one day I will be able to build the same kind of warm, supportive, and joyful home for the people I love. 

    \end{acknowledgements}
    \tableofcontents
    \listoftables
    \listoffigures
  \mainmatter
    \chapter{Introduction}
\label{ch:introduction}

\noindent Machine learning (ML) has become a central pillar of modern technology, powering applications in a multitude of fields such as healthcare~\citep{kotropoulos2009linear,sousa2009ordinal, challen2019artificial,guan2020bounded, mozannar2020consistent}, finance~\citep{9260038}, and transportation~\citep{ghodsi2021generating, tselentis2023usefulness}. These systems enable healthcare providers to diagnose diseases more accurately, financial institutions to tailor credit offers, and autonomous vehicles to navigate roads safely. Yet, as ML models increasingly transition from research environments to real-world settings, their reliability is under intensified scrutiny. As the stakes grow, so does the potential harm if models behave unexpectedly~\citep{amodei2016concrete, wiens2019no}: incorrect diagnoses can lead to life-threatening medical decisions, biased lending systems can exacerbate social inequities, and errors in autonomous driving technology can pose immediate safety hazards. These risks are further magnified by the complexity of large-scale data pipelines~\citep{pervaiz2019examining} and the often opaque nature of modern ML algorithms, which can make errors difficult to predict or detect. In addition, various industries are grappling with evolving regulatory and ethical standards~\citep{lo2020ethical, yaghini2024regulation} that demand greater accountability and transparency. 

Motivated by these risks, researchers and practitioners have begun to prioritize methods that enhance model trustworthiness and transparency. This has led to the emergence of the field of \textit{trustworthy machine learning}~\citep{li2023trustworthy}. Trustworthy machine learning aims to ensure that models not only excel at predictive performance but also uphold principles of reliability, fairness, privacy, accountability, as well as safety and security. It encompasses subfields dedicated to these goals, such as \textit{robust machine learning}~\citep{szegedy2013intriguing, papernot2017practical, rahimian2022frameworks}, which focuses on creating models resilient to variations in data distributions and adversarial inputs; \textit{fairness and bias mitigation}~\citep{hardt2016equality, mehrabi2021survey}, which aims to detect and reduce inequities that arise when data or algorithms systematically disadvantage certain groups; \textit{explainable or interpretable machine learning}~\citep{chen2019looks, rudin2019stop}, which seeks methods and frameworks to make complex model reasoning understandable to humans; and \textit{privacy-preserving techniques}~\citep{dwork2014algorithmic, abadi2016deep}, including differential privacy and federated learning, that protect sensitive information while enabling effective model training. Taken together, these subfields represent a concerted effort within the ML community to manage risks, fulfill ethical obligations, and promote transparent, accountable systems that reliably support high-stakes decisions in real-world contexts.

While each of these subfields provides essential guarantees to establish different notions of trust, they all share a dependence on a more fundamental capability: \textit{knowing the limits of a model's predictions}. A system cannot be considered truly robust, safe, or accountable if it cannot also signal when its conclusions are likely to be wrong. Consequently, a key challenge in making ML systems more reliable lies in \textit{quantifying uncertainty}~\citep{gal2016uncertainty, hullermeier2021aleatoric, gawlikowski2023survey}. Classical supervised learning methods often assume that models should output a single best prediction without any mechanism for indicating when this prediction might be wrong. In reality, however, uncertainty is an integral component of most prediction tasks. For instance, models trained on limited or noisy data can become overconfident in their predictions~\citep{guo2017calibration}, leading to harmful outcomes when deployed. This becomes especially clear when we consider high-stakes scenarios. Consider an autonomous vehicle navigating in dense fog or heavy rain: if the visual input is degraded, a well-calibrated uncertainty estimate can prompt the vehicle to slow down or safely stop rather than make an overconfident and potentially dangerous decision. In healthcare, a diagnostic model uncertain about whether a tumor is benign or malignant can flag the case for further review by a radiologist instead of issuing a definitive but incorrect diagnosis. In financial services, a model assessing loan applications can defer borderline cases to a human analyst if uncertainty is high, reducing the risk of unjust denials due to atypical applicant profiles. Consequently, techniques for quantifying and managing uncertainty are crucial to safe and responsible model deployment. These techniques can inform decisions about when a model is trustworthy enough to act on, when human oversight or additional information might be needed, and how to handle model predictions that are at high risk of being incorrect.

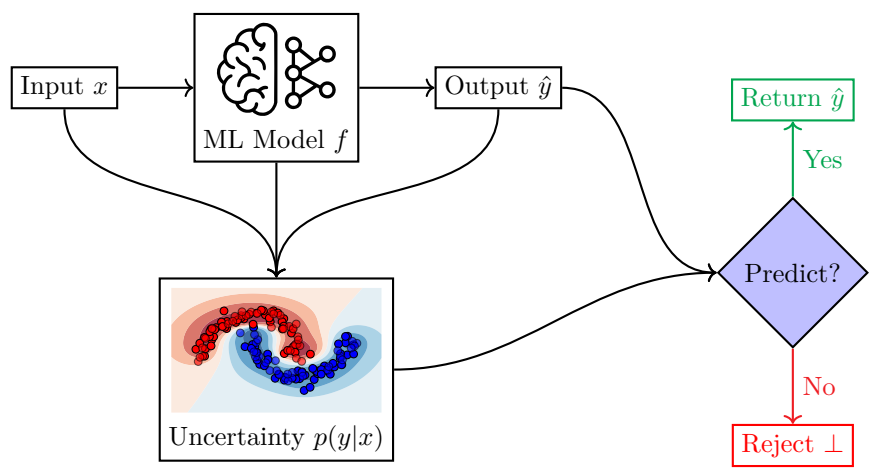
\begin{figure}[t]
    \centering
\begin{tikzpicture}[node distance=1cm,auto]
    \node[draw, thick, align=center] (x) {Input $x$};
    \node[right=of x, draw, thick, align=center] (f) {\includegraphics[width=45pt]{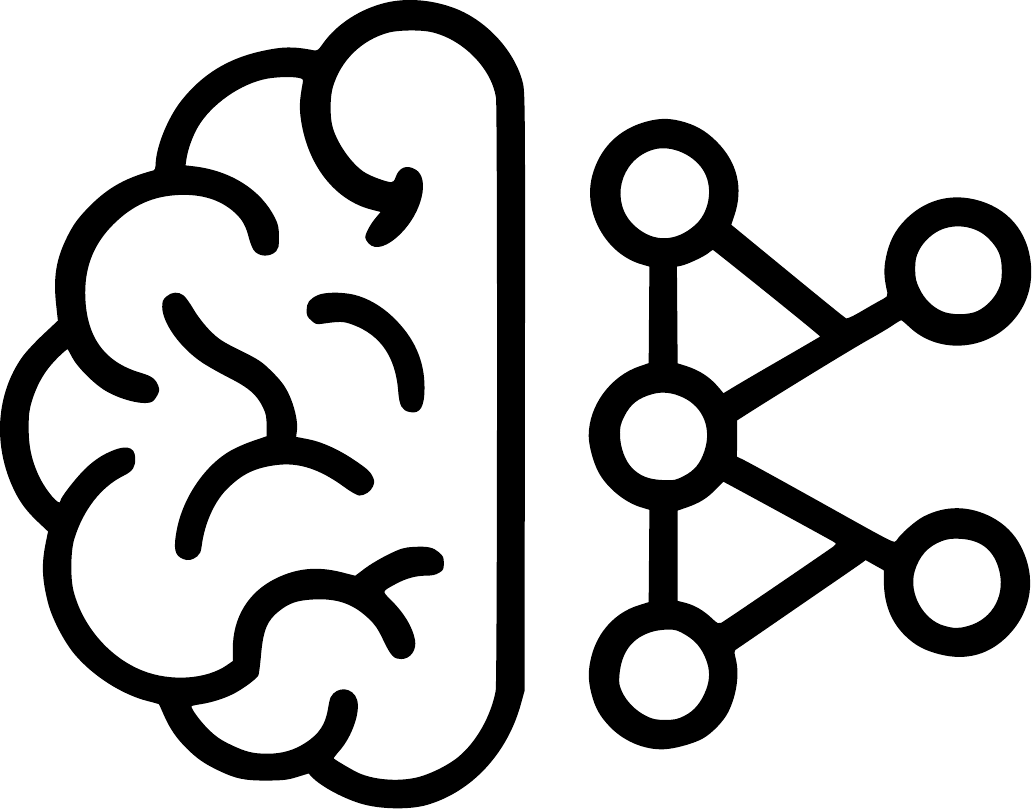}\\ML Model $f$};
    \node[right=of f, draw, thick, align=center] (y) {Output $\hat{y}$};
    \draw[->, thick] (x) -- (f);
    \draw[->, thick] (f) -- (y);
    \node[below=of f, draw, thick, align=center, yshift=-15pt] (unc) {\includegraphics[width=80pt]{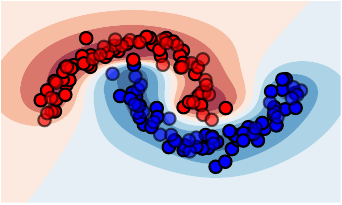}\\Uncertainty $p(y|x)$};
    \draw[->, thick] (x) to [out=270, in=90] (unc);
    \draw[->, thick] (f) to [out=270, in=90] (unc);
    \draw[->, thick] (y) to [out=270, in=90] (unc);
    \node[right=of y, diamond, draw, thick, align=center, fill = blue!25, xshift=30pt, yshift=-70pt] (pred) {Predict?};
    \draw[->, thick] (unc) to [out=0, in=180] (pred);
    \draw[->, thick] (y) to [out=0, in=180] (pred);
    \node[above=of pred, draw, Green, thick, align=center] (acc) {Return $\hat{y}$};
    \node[below=of pred, draw, red, thick, align=center] (rej) {Reject $\bot$};
    \draw[->, thick, Green] (pred) -- node[right] {Yes} (acc);
    \draw[->, thick, Red] (pred) -- node[right] {No} (rej);
    \end{tikzpicture}
    \vspace{10pt}
    \caption[Overview of the selective prediction pipeline.]{\textbf{Overview of the selective prediction pipeline}. Given an input \(x\), the model \(f\) produces a prediction \(\hat{y}\) along with an uncertainty estimate \(p(y|x)\). A selection function decides whether to return \(\hat{y}\) or reject with \(\perp\) based on the uncertainty.}
    \label{fig:selp_overview}
\end{figure}

Once reliable uncertainty estimates are available, the key challenge becomes translating them into principled actions. Decision theory offers a rich toolkit—ranging from Bayesian risk minimization to distribution-robust optimization—for choosing actions that balance expected utility against worst-case losses when outcomes are uncertain~\citep{berger2013statistical,rahimian2022frameworks}. In the supervised-learning setting, one of the most effective ways to operationalize this balance is to endow a model with an explicit \emph{reject option}: when the predicted probability of error exceeds a user-defined risk tolerance, the model abstains and defers the case to a downstream process such as a human expert, a more powerful model, or additional data collection~\citep{bartlett2008classification,geifman2017selective}. This paradigm—variously termed selective prediction, confidence-based abstention, or classification with a reject/abstain option—has seen growing adoption in high-stakes domains because it couples predictive performance with an actionable safeguard. Central to our work is precisely this concept of \textit{selective prediction} (SP, see Figure~\ref{fig:selp_overview})~\citep{chow1957optimum,el2010foundations}—an approach that allows the system to self-diagnose when it is likely to err and to withhold a decision rather than commit to a potentially harmful action. By integrating uncertainty estimates with abstention mechanisms, SP provides a formal pathway for aligning model behavior with real-world risk constraints, thereby transforming raw uncertainty into concrete, auditable decision rules that enhance safety and trustworthiness.

\vspace{15pt}
\section{Thesis Contributions}
\label{sec:contrib}

The overarching theme of this thesis is to embed uncertainty quantification at the heart of trustworthy machine learning (ML) systems. Central to our approach is the concept of selective prediction~\citep{chow1957optimum,el2010foundations}—which gives a model the option to abstain from making predictions it deems too risky. In this thesis, we make multiple contributions to this field. In short, we (i)~propose a new selective classification method based on ensembling intermediate models obtained from a single training run; (ii)~investigate the connection between selective prediction and differential privacy; (iii)~derive and analyze the key quantities that govern optimal selective prediction performance; and (iv)~study the effect of adversaries on the selective prediction pipeline. Below is an extended roadmap of how these ideas interconnect to form the core narrative of this thesis.

Before diving into the main contributions, we begin by laying the groundwork for what it means to \emph{measure} and \emph{manage} uncertainty in ML. This includes surveying classical approaches to calibrating confidence, examining how confidence estimates can fail in practice, and discussing why selectively rejecting uncertain inputs is an effective tool for mitigating high-impact errors. We highlight how uncertainty quantification (UQ) underpins reliable decision-making by offering interpretable confidence estimates, supporting risk-aware actions, and enabling resource-efficient deployment. Selective prediction is introduced as a natural downstream application of UQ—transforming uncertainty into a concrete decision to abstain when confidence is low. We emphasize the importance of calibration, robustness to distribution shift, and cost-aware thresholding, as well as the challenges that arise in deploying selective systems under real-world constraints. This background frames selective prediction not just as a technique, but as a paradigm for aligning model behavior with reliability objectives in safety-critical settings.

One of the most successful methods for uncertainty quantification we discuss is \emph{deep ensembling}~\citep{lakshminarayanan2017simple}, which estimates predictive uncertainty by training multiple models on different data subsets or with varying initializations. While ensembles are highly effective, their reliability typically improves with the number of members—making them computationally expensive at scale and difficult to deploy in resource-constrained or production environments. Motivated by this limitation, we ask whether a model’s \emph{training trajectory} can serve as a lightweight and scalable alternative for boosting selective prediction performance. Rather than modifying the model’s architecture or primary loss function, we demonstrate that signals extracted from intermediate checkpoints during training allow the model to self-diagnose its own likely points of failure. This yields a simple yet powerful abstention mechanism that can be applied post-hoc to any trained model with saved checkpoints—sidestepping the cost of full ensembles while retaining much of their reliability benefit. Moreover, this method is fully complementary to existing selective prediction approaches: it can be used to enhance any existing selector by providing an additional, trajectory-informed signal without interfering with the model’s original design or training objective.

Having demonstrated its practical utility, our training-dynamics-based approach also enables synergies with broader areas of trustworthy machine learning, making it especially well-suited for settings that demand the responsible deployment of AI. Because our approach merely observes the training process without interfering with it, it inherits a crucial compatibility with privacy-preserving learning: under the post-processing property of differential privacy (DP)~\citep{dwork2006calibrating}, operations applied after training do not compromise privacy guarantees. This makes our checkpoint-based method an attractive solution for \emph{uncertainty quantification in privacy-sensitive contexts}. We therefore turn next to exploring how selective prediction methods behave under strict privacy constraints. Many real-world applications demand that model training preserve the confidentiality of sensitive data, and DP is a gold standard for ensuring this. In this part of the thesis, we show how conventional selective prediction techniques—especially those reliant on ensembling—can amplify privacy risks or degrade under DP noise. In contrast, our training-dynamics-based approach proves particularly robust. This investigation also surfaces a deeper challenge: classical performance metrics often entangle improvements in overall model utility---a quantity which changes under varying privacy levels---with improvements in model uncertainty. To isolate uncertainty quality, we propose a new evaluation framework that enables a more faithful measurement of selective prediction performance in the presence of DP noise, capturing trade-offs between privacy, utility, and abstention coverage more effectively.

Studying privacy is a compelling example of responsible AI deployment—one that comes with an inherent utility trade-off and underscores the importance of evaluating selective prediction in scenarios where the base model's accuracy varies. The empirical and conceptual shortcomings revealed in this study therefore motivate a deeper interrogation of what actually determines selective prediction quality and performance. To that end, we study the \emph{selective-classification gap}, which corresponds to the deviation between a model's realized accuracy–coverage curve---the most commonly used evaluation metric in selective prediction---and a perfect-ordering upper bound established in our privacy work. We show that this gap can always be decomposed into five interpretable components: Bayes noise, approximation error, ranking error, statistical noise, and a miscellaneous term that captures optimization mishaps and distribution shift. The resulting finite-sample “error budget” clarifies which levers matter most at different coverage levels: increasing capacity or distilling from a stronger teacher shrinks the approximation term, additional or repeated labels dampen Bayes noise, larger validation sets reduce statistical fluctuations, and shift-aware training tackles the miscellaneous residue. Crucially, the analysis shows that monotone post-hoc calibration—long a default remedy for miscalibration—cannot touch the ranking term because it preserves the ordering of scores; progress therefore requires methods that can \emph{re-rank} predictions. By quantifying why and where models fall short of the oracle, this framework establishes principled targets for future algorithmic improvements.

Ultimately, claims of responsible AI deployment cannot simply be asserted by platform developers—they must be subject to external auditability. This naturally raises a critical question: can uncertainty quantification itself be gamed, especially in adversarial settings where the model operator cannot be fully trusted? In particular, the selective-classification gap analysis assumes that deviations from optimal performance stem from benign sources—such as modeling limitations or statistical fluctuations—not from active manipulation. Yet, in practice, the very mechanisms designed to improve reliability can themselves be exploited. This realization naturally leads us to investigate whether uncertainty, long viewed as a desirable and protective property in machine learning systems, might also be used \emph{maliciously}. Surprisingly, we find that confidence estimates can be easily deliberately manipulated to increase uncertainty in targeted input regions or for specific user groups, thereby covertly denying service while maintaining strong overall performance. These subtle attacks are difficult to detect because the system continues to perform well on aggregate metrics, even as it systematically disadvantages certain users. To counter this threat, we propose a verifiable inference mechanism that audits reported confidence and ensures that abstentions reflect genuine uncertainty. This work highlights a crucial lesson: secure and trustworthy ML requires not only sound estimation procedures, but also safeguards that ensure the integrity and authenticity of uncertainty itself.

\section{Thesis Statement}
\label{sec:contrib}

This leads us to the central thesis statement:

\boxedblock{
This thesis advances selective prediction—a model's ability to abstain when uncertain—as a cornerstone of trustworthy machine learning. Crucially, we observe that uncertainty can be effectively inferred from the inherent dynamics of the training process, rather than being retrofitted through costly modifications to the model or its training procedure. This approach provides a more direct path to robust selective prediction, which strengthens our ability to trust machine learning outputs. Consequently, because it preserves the integrity of the original training pipeline, this method is fundamentally more compatible with the principles of responsible ML deployment, enhancing crucial aspects such as differential privacy and model robustness against untrusted parties.
}

\section{Thesis Outline}
\label{sec:outline}

We summarize the contributions this thesis makes below:
\begin{itemize}
    \item \textbf{In Chapter~\ref{ch:background}}, we present background concepts crucial for understanding reliable decision-making in machine learning. We cover established approaches to modeling uncertainty, describe how predictive confidence is estimated, and discuss the rationale for selective prediction as a way to ensure safety by rejecting high-risk inputs. We also examine how confidence calibration underlies many of these techniques, enabling models to produce uncertainty estimates aligned with true error rates.

    \item \textbf{In Chapter~\ref{ch:sptd}}~\citep{rabanser2022selective}, we investigate how one can implement selective prediction by examining a model’s evolution throughout training, focusing on signals that emerge when predictions change or stabilize. We show that, by studying these signals, it is possible to determine which data points the final model is likely to misclassify. This perspective allows for a mechanism that rejects uncertain inputs without modifying the model’s architecture or primary optimization objective. The approach is adaptable to various tasks and domains, since it only requires access to intermediate training outputs.

    \item \textbf{In Chapter~\ref{ch:sptd_dp}}~\citep{rabanser2023training}, we delve into the interplay between selective prediction and differential privacy. We analyze how injecting random noise during training—an essential requirement for protecting sensitive data—can adversely affect a model's confidence estimates, reducing the reliability of conventional selective predictors. We examine why certain methods become less effective or more vulnerable under privacy constraints and propose refinements that leverage multiple training snapshots or other strategies to maintain robust performance. We also suggest novel evaluation metrics that better reflect how both privacy and reliability goals are being balanced.

    \textbf{In Chapter \ref{ch:sc_bounds}}~\citep{rabanser2025what}, we take a principled look at what governs how close a selective predictor can come to its oracle ideal. We formalize the \emph{selective-classification gap}—the deviation between a model’s accuracy–coverage curve and the perfect-ordering frontier—and derive the first finite-sample decomposition that attributes this gap to five interpretable error sources: Bayes noise, approximation limits, ranking imperfections, statistical fluctuations, and miscellaneous factors such as optimization error or distribution shift. The resulting “error budget’’ clarifies which levers (e.g., capacity, ranking-aware calibration, additional labels, or shift-robust training) most effectively shrink the gap at different coverage levels. We further prove that monotone post-hoc calibration cannot close the ranking term, motivating algorithms that can re-order confidence scores, and we validate the theory on tasks ranging from synthetic two-moons to large-scale vision benchmarks.

    \item \textbf{In Chapter~\ref{ch:conf_guard}}~\citep{rabanser2025confidential}, we highlight a potential adversarial threat in selective prediction systems: malicious entities can artificially suppress model confidence in targeted regions or for specific groups while retaining high performance elsewhere. This tactic allows an attacker to deny services or introduce bias under the pretense of uncertainty. We describe how such manipulation can be carried out, and we develop a strategy to detect and deter artificially induced low confidence. Central to this solution is the need for mechanisms that can verify confidence values or otherwise ensure that the model's abstentions are genuinely driven by uncertainty rather than deliberate manipulation.

    \item \textbf{In Chapter~\ref{ch:conclusion}}, we conclude by summarizing the contributions of the thesis and highlighting future directions. We reflect on open research questions in areas such as adaptive calibration, reliability and uncertainty in large models, new forms of adversarial threats, and privacy-preserving learning. We also consider the broader implications of adopting selective prediction at scale, including ethical considerations in sensitive or high-stakes applications.
\end{itemize}

\paragraph{Non-Thesis Research} Publications on which I have worked during my Ph.D. which are excluded from this thesis include (in chronological order):

\begin{itemize}
    \item \citet{dziedzic2022p}: \bibentry{dziedzic2022p}.\\[10pt]

    $p$-DkNN is a method for producing statistically grounded uncertainty estimates in neural networks by analyzing the similarity structure of intermediate representations and computing $p$-values for predictions. This allows us to improve the trade-off between rejecting OOD inputs and maintaining accuracy on in-distribution data. The approach is scalable, theoretically connected to Neyman–Pearson classification, and robust enough to make adversarial attacks require semantically meaningful changes.
    
    \item \citet{rabanser2022intrinsic}: \bibentry{rabanser2022intrinsic}.\\[10pt]
    In this work we introduce intrinsic anomaly detection for multivariate time series, focusing on identifying unexpected changes in a system’s state relative to its environment. We formalize the problem, release relevant datasets, and propose unsupervised methods—combining domain-adversarial and representation learning—that distinguish true anomalies from environment-driven changes.
    
    \item \citet{franzese2023robust}: \bibentry{franzese2023robust}.\\[10pt]
    Collaborative ML can improve models from distributed data but is vulnerable to malicious servers reconstructing client data and malicious clients corrupting updates. This work introduces a peer-to-peer learning framework that secures against untrusted servers and is robust to malicious clients by adapting robust aggregation methods to this setting. The approach is computationally efficient and scales to millions of parameters and hundreds of peers.

    \item \citet{pouget2025suitability}: \bibentry{pouget2025suitability}.\\[10pt]
    The suitability filter is a framework for detecting when a deployed model’s accuracy on unlabeled user data has degraded beyond an acceptable margin, without needing ground-truth labels. It works by comparing model output–based suitability signals from test and user data via statistical hypothesis testing. Experiments show it reliably flags performance drops caused by covariate shift in safety-critical settings.
    
    \item \citet{rabanser2025gatekeeper}: \bibentry{rabanser2025gatekeeper}.\\[10pt]
    Gatekeeper is a new loss function for cascade setups that fine-tunes smaller models to confidently handle easy tasks while deferring harder ones to larger models. It balances performance and deferral accuracy without changing model architectures, and works across diverse tasks and domains. Our experiments on multiple architectures and modalities show substantial improvements in deferral performance.
    
    \item \citet{jiang2025cascadia}: \bibentry{jiang2025cascadia}.\\[10pt]
    Cascadia is a new framework for serving large language model (LLM) cascades that balances speed and answer quality by jointly optimizing routing strategies and system deployment. Using a bi-level optimization approach, it allocates resources and parallelism for different LLMs while adapting routing to workload characteristics. Our experiments show it outperforms existing approaches, delivering up to 4× tighter latency SLOs and 5× higher throughput without sacrificing quality.
\end{itemize}
    \newcommand{\ba}[1]{\textbf{\sffamily #1}}

\newcommand{\sat}[0]{\ba{SAT}\xspace}
\newcommand{\satersr}[0]{\ba{SAT+ER+SR}\xspace}
\newcommand{\msp}[0]{\ba{MSP}\xspace}
\newcommand{\sr}[0]{\ba{SR}\xspace}
\newcommand{\sn}[0]{\ba{SN}\xspace}
\newcommand{\dg}[0]{\ba{DG}\xspace}
\newcommand{\odist}[0]{\ba{ODIST}\xspace}
\newcommand{\mcdo}[0]{\ba{MC-DO}\xspace}
\newcommand{\de}[0]{\ba{DE}\xspace}
\newcommand{\nntd}[0]{\ba{NNTD}\xspace}
\newcommand{\sctdde}[0]{\ba{DE+SCTD}\xspace}
\newcommand{\sctd}[0]{\ba{SCTD}\xspace}
\newcommand{\sptd}[0]{\ba{SPTD}\xspace}
\newcommand{\cclsc}[0]{\ba{CCL-SC}\xspace}
\newcommand{\aucoc}[0]{\ba{AUCOC}\xspace}
\newcommand{\temp}[0]{\ba{TEMP}\xspace}

\newcommand{\sptdde}[0]{\ba{DE+SPTD}\xspace}
\newcommand{\sptdc}[0]{\ba{SPTD-C}\xspace}
\newcommand{\sptdr}[0]{\ba{SPTD-R}\xspace}
\newcommand{\sptdts}[0]{\ba{SPTD-TS}\xspace}
\newcommand{\osp}[0]{\ba{OSP}\xspace}
\newcommand{\logitvar}[0]{\ba{LOGITVAR}\xspace}

\newcommand{\minscore}[0]{minimum score\xspace}
\newcommand{\avgscore}[0]{average score\xspace}
\newcommand{\jmpscore}[0]{jump score\xspace}
\newcommand{\varscore}[0]{variance score\xspace}
\newcommand{\smin}[0]{$s_\text{min}$\xspace}
\newcommand{\savg}[0]{$s_\text{avg}$\xspace}
\newcommand{\smax}[0]{$s_\text{MAX}$\xspace}
\newcommand{\ssum}[0]{$s_\text{SUM}$\xspace}
\newcommand{\swv}[0]{$s_\text{WV}$\xspace}
\newcommand{\swvr}[0]{$s_\text{WVR}$\xspace}
\newcommand{\swvts}[0]{$s_\text{WVTS}$\xspace}
\newcommand{\slast}[0]{$s_\text{last}$\xspace}
\newcommand{\sfull}[0]{$s_\text{full}$\xspace}
\newcommand{\sjmp}[0]{$s_\text{jmp}$\xspace}
\newcommand{\svar}[0]{$s_\text{var}$\xspace}
\newcommand{\fp}[0]{false-positive\xspace}
\newcommand{\fps}[0]{false-positives\xspace}
\newcommand{\ie}[0]{i.e.,\xspace}
\newcommand{\eg}[0]{e.g.,\xspace}
\newcommand{\selc}[0]{selective classification\xspace}
\newcommand{\selp}[0]{selective prediction\xspace}
\newcommand{\empiricalacccovtradeoff}[0]{$\text{acc}_{c}(f,g)$}
\newcommand{\upperbound}[0]{$\overline{\text{acc}}(a_\text{full},c)$}
\newcommand{\accnormscore}[0]{$s_{a_\text{full}}(f,g)$}
\newcommand{\realtradeoff}[0]{$\text{acc}_c(h,g)$}

\chapter{Background}
\label{ch:background}

\section{Introduction to Uncertainty Quantification}
Uncertainty quantification (UQ) is the process of identifying, quantifying, and managing the uncertainty inherent in computational and physical models. In machine learning and statistical modeling, UQ focuses on characterizing the confidence or trustworthiness of a model's predictions. Although machine learning models have grown very capable at making accurate predictions---especially with the advent of deep learning allowing for highly expressive models---understanding and communicating how certain (or uncertain) those predictions are remains a major challenge.

\subsection{Importance of Uncertainty Quantification}
Before diving into the specific reasons why uncertainty quantification is vital, it is helpful to note that UQ serves as a bridge between raw model output and actionable insight. While modern machine learning models can make highly accurate predictions, they do not inherently communicate their confidence level or the extent of variability within the data (i.e., Bayes error or irreducible error). This gap can be critical in real-world scenarios where misunderstandings or misapplications of model outputs can have significant consequences.

\paragraph{Decision-Making Under Risk.} In real-world applications like autonomous driving, healthcare, and finance, decisions often come with high stakes. Errors can lead to costly or even life-threatening outcomes, so having a clear sense of how certain a model’s predictions are is essential for risk management. For instance, in a clinical setting, an automated diagnostic tool that flags images as “uncertain” can trigger a secondary, human-led evaluation, thereby reducing the risk of misdiagnosis.
\paragraph{Model Validation and Trust.} A core challenge in deploying machine learning systems is understanding when and where they might fail. Uncertainty metrics, such as confidence intervals or Bayesian posterior distributions, help illuminate whether a model is systematically overconfident or underconfident in its predictions. Such insights guide model improvements (e.g., collecting more targeted training data) and bolster stakeholders’ trust by transparently indicating areas where the model may be less reliable.
\paragraph{Resource Allocation.} In settings with finite resources---for example, bandwidth, computational capacity, or expert availability---uncertainty estimates help allocate those resources more effectively. If a model can identify the instances where it is less certain, additional measures---such as further data collection, human review, or more computationally expensive algorithms---can be reserved for those high-uncertainty cases. This targeted approach prevents unnecessary overhead for predictions deemed sufficiently reliable.
\paragraph{Generalization and Reliability.} Overfitting, adversarial inputs, and domain shifts pose ongoing challenges in machine learning. When a model encounters data that differ from its training distribution, having robust uncertainty estimates can provide an early warning sign. For example, a model trained on data from one domain may exhibit higher uncertainty when faced with data from a new domain with differing characteristics. Such signals can prompt efforts to gather additional training data or adapt the model to the new distribution, ultimately improving its overall reliability and generalization.

\subsection{Approaches to Uncertainty Quantification}
There are several approaches to UQ in machine learning, reflecting different statistical and computational philosophies:

\paragraph{Bayesian Methods.}
In Bayesian inference, model parameters are treated as random variables~\citep{bishop2006pattern}. Given a set of observed data $\mathcal{D}$ and a parameter vector $\theta$, Bayes' theorem allows us to compute the posterior distribution:
\begin{equation}
p(\theta \mid \mathcal{D}) = \frac{p(\mathcal{D} \mid \theta) \, p(\theta)}{p(\mathcal{D})},
\end{equation}
where $p(\theta)$ is the prior, $p(\mathcal{D} \mid \theta)$ is the likelihood, and $p(\mathcal{D}) = \int p(\mathcal{D} \mid \theta)p(\theta)d\theta$ is the evidence. Techniques such as Markov Chain Monte Carlo (MCMC)~\citep{geyer1992practical} simulate draws from this posterior distribution, while Variational Inference (VI)~\citep{blei2017variational} seeks a tractable family of distributions $q(\theta)$ that approximates $p(\theta \mid \mathcal{D})$. Bayesian neural networks~\citep{blundell2015weight} extend these concepts to deep learning by placing priors over the network weights.

\paragraph{Frequentist and Distribution-Free Methods.}
Distribution-free approaches, like conformal prediction~\citep{shafer2008tutorial, fontana2023conformal}, aim to create prediction sets with finite-sample guarantees under assumptions such as exchangeability. For a new observation $\bm{x}$, conformal prediction constructs a set $\Gamma(\bm{x})$ such that
\begin{equation}
\Pr\{ Y \in \Gamma(\bm{X}) \} \;\geq\; 1 - \alpha,
\end{equation}
where $\alpha$ is a pre-specified significance level and $\bm{X}$ is the random variable. The construction typically involves calculating a nonconformity score for each example in a calibration set and then using these scores to determine the appropriate quantile for new predictions.

\paragraph{Ensemble Methods.}
Ensemble techniques~\citep{lakshminarayanan2017simple} combine the outputs of multiple independently trained models to capture uncertainty. Suppose we have an ensemble of $M$ models, $\{ f_1, f_2, \dots, f_M \}$. The aggregated prediction can be written as:
\begin{equation}
\hat{y} \;=\; \frac{1}{M} \sum_{m=1}^{M} f_m(\bm{x}),
\end{equation}
and the empirical variance, which serves as an uncertainty estimate, is given by:
\begin{equation}
\hat{\sigma}^2(\bm{x}) \;=\; \frac{1}{M-1} \sum_{m=1}^{M} \Bigl( f_m(\bm{x}) - \hat{y} \Bigr)^2.
\end{equation}
Although not inherently probabilistic, this approach provides a practical measure of prediction variability, particularly in high-dimensional settings.

\paragraph{Bootstrapping and Resampling.}
Bootstrapping involves generating multiple resampled datasets $\{\mathcal{D}^{(1)}, \mathcal{D}^{(2)}, \dots, \mathcal{D}^{(B)}\}$ from the original dataset $\mathcal{D}$ by sampling with replacement~\citep{breiman1996bagging}. For each bootstrap sample $\mathcal{D}^{(b)}$, a model $f^{(b)}$ is trained. The distribution of the model outputs at a given point $\bm{x}$ is then approximated as:
\begin{equation}
\hat{p}(y \mid \bm{x}) \;\approx\; \frac{1}{B} \sum_{b=1}^{B} \delta\bigl(y - f^{(b)}(\bm{x})\bigr),
\end{equation}
where $\delta$ denotes the Dirac delta function. This empirical distribution provides an estimate of the uncertainty in the model's predictions due to dataset variability.

\paragraph{Calibrated Models.}
Calibration techniques adjust a model’s predicted probabilities to better reflect true empirical frequencies. One common method is \emph{Platt scaling}~\citep{platt1999probabilistic}, which applies a sigmoid transformation to the model output:
\begin{equation}
p_{\text{cal}}(y \mid \bm{x}) \;=\; \frac{1}{1 + \exp\bigl(-a \, f(\bm{x}) - b\bigr)},
\end{equation}
where $a$ and $b$ are parameters learned on a validation set. A closely related method is \emph{temperature scaling}~\citep{guo2017calibration}, which calibrates the softmax outputs of a neural network by dividing the logits by a single scalar parameter $T > 0$, known as the temperature. Given the original logits $\bm{z}$, the calibrated probabilities are computed as:
\begin{equation}
p_{\text{cal}}(y \mid \bm{x}) \;=\; \frac{\exp(z_y / T)}{\sum_{k} \exp(z_k / T)}.
\end{equation}
Temperature scaling is a particularly simple yet effective post-hoc calibration method, as it preserves the model’s prediction ordering while improving the alignment between predicted confidence and true accuracy. Alternatively, \emph{isotonic regression}~\citep{zadrozny2002transforming} imposes a monotonicity constraint on the transformation function, allowing for a non-parametric calibration that adapts to the observed data distribution without assuming a specific functional form.

\section{Selective Prediction}
Selective prediction, sometimes called ``prediction with a reject option,'' or ``learning with abstention,'' introduces a mechanism for the model to \emph{abstain} (or ``reject'') on samples where its confidence is insufficiently high. In many high-stakes applications, having the option to reject can be valuable when the cost of a wrong decision exceeds the cost of not predicting at all.

In scenarios where an incorrect prediction could lead to severe consequences, it becomes essential to allow the model to defer judgment on uncertain cases. This selective approach not only helps to minimize potential risks by avoiding hasty decisions but also facilitates a more efficient allocation of resources by directing complex or ambiguous instances toward additional analysis or expert review. The following points further detail the driving factors behind adopting selective prediction in critical applications:

\begin{itemize}
    \item \textbf{Risk management}: In domains like medical diagnostics, autonomous driving, or finance, even a single misclassification can result in significant harm or financial loss. By rejecting predictions that do not meet a confidence threshold, the system minimizes the chance of error. This trade-off between coverage and accuracy ensures that only predictions with sufficiently low risk are acted upon, thereby safeguarding against the potentially high cost of a wrong decision.

    \item \textbf{Complex cost structures}: Many applications face asymmetric costs where the consequences of different types of errors vary dramatically. For instance, in healthcare, the cost of a false negative (failing to detect a disease) is often much higher than that of a false positive (unnecessary follow-up tests). A selective classifier can be tuned to consider these cost asymmetries by setting thresholds that balance the risk of errors against the operational costs of additional tests or interventions. This ensures that the system only makes predictions when the expected cost of a mistake is lower than the cost of deferring the decision.

    \item \textbf{Focus on ``easy'' cases}: In practice, models tend to perform well on typical examples while struggling with outliers or cases that lie near the decision boundary, especially under distribution shifts or when encountering rare events. By identifying and processing these ``easy'' cases automatically, the system can reserve more intensive, specialized methods (such as human review or more computationally demanding algorithms) for those uncertain or challenging instances. This tiered approach improves overall system efficiency and leverages expert resources only when they are truly needed.
\end{itemize}

\subsection{Formal Definition} 

Selective prediction extends the standard supervised classification framework by allowing the model to output a special \emph{rejection} symbol~$\bot$ through the use of a \textit{gating function}~\citep{yaniv2010riskcoveragecurve}. This gating function consults the underlying classifier and returns a prediction only when it is sufficiently confident in its correctness; otherwise, it opts to abstain. 

Throughout this thesis, we assume that inputs belong to a covariate space~$\mathcal{X} \subseteq \mathbb{R}^d$. The classifier $f$ maps each input either to a hard label in the set $\mathcal{Y} = [C] = \{1,\dots,C\}$, \emph{or} to a softmax probability vector in the simplex $\mathcal{Y} = \Delta^{C-1} \subset [0,1]^C$. The latter convention covers cases where the gating decision uses scores such as maximum class probability, entropy, or logit margins.

Typically, the gating decision is derived directly from the behavior of the classifier $f$, and we make this dependency explicit in our formulation. Specifically, we define a selection function $g: \mathcal{X} \times (\mathcal{X} \rightarrow \mathcal{Y}) \rightarrow \mathbb{R}$, which evaluates whether the model should produce a prediction for a given input~$\bm{x}$. If the value of $g(\bm{x}, f)$ is less than or equal to a predefined threshold $\tau$, the classifier proceeds with the prediction $f(\bm{x})$; otherwise, it abstains by returning $\bot$. This defines the joint selective prediction model as:
\begin{equation}
    (f,g)(\bm{x}) = \begin{cases}
      f(\bm{x}) & g(\bm{x}, f) \leq \tau \\
      \bot & \text{otherwise.}
    \end{cases}
\end{equation}
For clarity, we may occasionally drop the explicit dependence of $g$ on $f$ and simply write $g(\bm{x})$ when the context makes it unambiguous.

\subsection{Key Theoretical Results}
Seminal works such as \citet{chow1957optimum} and \citet{el2010foundations} established theoretical foundations for selective prediction and derived performance bounds. Below, we provide more formal statements of three key results.

\paragraph{Chow's Rule.}
In selective classification, Chow’s rule adopts a simplified cost-sensitive objective that assigns a fixed penalty for abstaining from a prediction. Specifically, the cost model assigns:
\begin{itemize}
  \item zero cost for correct predictions,
  \item unit cost for incorrect predictions (i.e., standard 0–1 loss), and
  \item a fixed reject cost $c_r > 0$ for abstaining.
\end{itemize}
Formally, for a decision function $(f,g)$, the expected risk under this cost model is
\begin{equation}
R(f,g) \;=\; \mathbb{E} \big[ \mathbb{I}\{g(\bm{x}) = 1, f(\bm{x}) \neq y\} \;+\; c_r \,\mathbb{I}\{g(\bm{x}) = 0\} \big],
\end{equation}
where $g(\bm{x}) \in \{0,1\}$ indicates whether the classifier predicts ($1$) or abstains ($0$).  
Let $p(y \mid \bm{x})$ denote the posterior probability of class $y$ given input $\bm{x}$. Then Chow's rule \citep{chow1957optimum} specifies that the decision that minimizes this expected cost—i.e., the \emph{Bayes optimal} decision under this cost model—is to predict the most likely label if the confidence (posterior probability) exceeds a threshold, and to reject otherwise. Formally, define
\begin{equation}
\hat{y}(\bm{x}) \;=\; \arg\max_{y\in\mathcal{Y}} \;p(y \mid \bm{x}),
\end{equation}
and let 
\begin{equation}
p_{\max}(\bm{x}) \;=\; \max_{y\in\mathcal{Y}} \; p(y \mid \bm{x}).
\end{equation}
Then the Bayes optimal selective classifier is:
\begin{equation}
(f^\ast, g^\ast)(\bm{x}) \;=\; \begin{cases}
\hat{y}(\bm{x}) \quad &\text{if } p_{\max}(\bm{x}) \;\geq\; \theta(c_r) \\
\bot \quad &\text{otherwise},
\end{cases}
\end{equation}
where the threshold $\theta(c_r)$ depends on both the reject cost $c_r$ and the class priors. This threshold balances the expected cost of an incorrect prediction against the fixed cost of abstaining, thereby minimizing the total expected cost across the distribution of inputs.

\paragraph{Coverage Guarantees and Set Utility.}
In a \emph{distribution-free} setting, conformal prediction~\citep{angelopoulos2021gentle} offers a principled way to build prediction sets $\Gamma(\bm{x})$ that satisfy finite-sample coverage guarantees under i.i.d.\ (exchangeable) data. Specifically, for any target miscoverage rate $\alpha\in(0,1)$,
\begin{equation}
  \Pr_{\bm{x},y}\!\bigl\{\,y\notin\Gamma(\bm{x})\,\bigr\}\;\le\;\alpha,
\end{equation}
with high probability over the random draw of the calibration set.  

While this bound limits how often the true label is excluded, it says nothing about \emph{how informative} the returned set is when coverage holds.  A widely accepted proxy for informativeness is the set’s \emph{size}: the cardinality $\lvert\Gamma(\bm{x})\rvert$ in classification or the length/volume in regression.  Size is appealing for three complementary reasons. 
\begin{enumerate}
    \item \textbf{Decision-theoretic grounding:} in many downstream tasks a user must act on \emph{any} element of $\Gamma(\bm{x})$ (e.g., prescribe the safest drug, choose the cheapest feasible route).  Under a worst-case or cost-per-option assumption, the expected utility of a set is monotone in its size, so minimizing size is equivalent to minimizing expected cost. 
    \item \textbf{Statistical tractability:} size admits clean optimality results—e.g.\ under the “marginal coverage with minimal expected size’’ criterion, conformal p-value ordering is provably optimal among all measurable set-valued predictors with the same coverage level.
    \item \textbf{Operational simplicity:} size provides an intuitive, threshold-friendly signal.  Practitioners can stipulate an application-specific upper bound $\tau$ on acceptable set size; if $\lvert\Gamma(\bm{x})\rvert>\tau$, the set is deemed uninformative and the system instead “defers’’ or “rejects.’’  Consequently, abstention can occur either explicitly ($\Gamma(\bm{x})=\varnothing$) or implicitly (oversized $\Gamma(\bm{x})$), unifying conformal prediction with selective classification via a simple size-based rule.
\end{enumerate}  
Balancing reliability and usefulness thus becomes a dual objective: attain the desired coverage \emph{and} minimize the expected size of $\Gamma(\bm{x})$—often called the method’s \emph{efficiency}.  This coverage–efficiency trade-off provides a unifying lens under which conformal prediction extends seamlessly to both classification and regression contexts.

\paragraph{PAC-Style Bounds.}
Within the Probably Approximately Correct (PAC) framework, one can analyze selective classification by bounding the probability that the error on the covered set exceeds a given threshold. Define the \emph{selective risk} and \emph{coverage} of a previously introduced selective classifier $(f,g)$ by
\[
    R(f,g)
    \;:=\;
    \frac{\mathbb{E}\bigl[\mathds{1}\{g(X)=1\}\,\ell(f(X),Y)\bigr]}
         {\mathbb{E}\bigl[\mathds{1}\{g(X)=1\}\bigr]}
    ,\qquad
    \text{cov}(f,g)
    \;:=\;
    \mathbb{E}\bigl[\mathds{1}\{g(X)=1\}\bigr].
\]
Under standard PAC assumptions (i.i.d.\ samples, bounded loss), there exist $\epsilon,\alpha>0$ such that, with probability at least $1-\delta$ over $n$ training examples,
\begin{equation}
    R(f,g)\;\le\;\epsilon
    \quad\text{and}\quad
    \text{cov}(f,g)\;\ge\;1-\alpha,
    \label{eq:pac_selective}
\end{equation}
provided $n$ is large enough~\citep{cortes2016learning}.  
This result shows that by permitting an $\alpha$-fraction of rejections, one can drive the conditional error on accepted instances below $\epsilon$, given sufficient data to learn the gating function.

\subsection{Common Selective Prediction Techniques}

\paragraph{Softmax Response (\sr).} A traditional baseline method for selective prediction is the \emph{Softmax Response} (\sr) method~\citep{hendrycks2016baseline, geifman2017selective}. This method uses the confidence of a prediction model \(f\) as the selection score:
\begin{equation}
	g_{\sr}(\bm{x}, f) = \max_{c \in C} f(\bm{x})_c
\end{equation}
Here, \(f(\bm{x}) \in \mathbb{R}^{|C|}\) denotes the vector of predicted class probabilities (or logits) for input \(\bm{x}\), and \(f(\bm{x})_c\) refers to the score assigned to class \(c\). While this method is easy to implement and incurs no additional computational cost, \sr has been found to be overconfident on ambiguous, hard-to-classify, or unrecognizable inputs.

\paragraph{SelectiveNet (\sn).} 
A variety of selective classification methods have been proposed that leverage explicit architecture and loss function adaptations. For example, \emph{SelectiveNet} (\sn) \citep{geifman2019selectivenet} modifies the model architecture to jointly optimize $(f,g)$ while targeting the model at a desired coverage level~$c_\text{target}$. The augmented model consists of a representation function $r: \mathcal{X} \rightarrow \mathbb{R}^L$ mapping inputs to latent codes and three additional functions: (i) the \emph{prediction} function $f: \mathbb{R}^L \rightarrow \mathbb{R}^C$ for the classification task targeted at~$c_\text{target}$; (ii) the \emph{selection} function $g: \mathbb{R}^L \rightarrow [0,1]$ representing a continuous approximation of the accept/reject decision for $\bm{x}$; and (iii) an additional \emph{auxiliary} function $h: \mathbb{R}^L \rightarrow \mathbb{R}^C$ trained for the unconstrained classification tasks. This yields the following losses:
 \begin{align}
 	\mathcal{L} & = \alpha \mathcal{L}_{f,g} + (1-\alpha) \mathcal{L}_h \\
 	\mathcal{L}_{f,g} & = \frac{\frac{1}{M}\sum_{m=1}^{M} \ell( f \circ r(\bm{x}_m) , y_m)}{\text{cov}(f,g)} + \lambda \max(0, c - \text{cov}(f,g))^2 \\
 	\mathcal{L}_{h} & = \frac{1}{M}\sum_{m=1}^{M} \ell( h \circ r(\bm{x}_m) , y_m)
 \end{align}
 The selection score for a particular point $\bm{x}$ is then given by:
 \begin{equation}
 	g_{\sn}(\bm{x}, f) = \sigma(g \circ r(\bm{x})) = \frac{1}{1 + \exp(g \circ r(\bm{x}))}
 \end{equation}

 \paragraph{Self-Adaptive Training (\sat).}
 Alternatively, prior works like Deep Gamblers~\citep{liu2019deep} and \emph{Self-Adaptive Training}~\citep{huang2020self} have also considered explicitly modeling the abstention class $\bot$ and adapting the optimization process to provide a learning signal for this class. For instance, \emph{Self-Adaptive Training} (\sat) incorporates information obtained during the training process into the optimization itself by computing and monitoring an exponential moving average of training point predictions over the training process. Samples with high prediction uncertainty are then used for training the abstention class. To ensure that the exponential moving average captures the true prediction uncertainty, an initial burn-in phase is added to the training procedure. This delay allows the model to first optimize the non-augmented, \ie original $C$-class prediction task and optimize for selective classification during the remainder of the training process. The updated loss is defined as:
 \begin{equation}
 	\mathcal{L} = -\frac{1}{M}\sum_{m=1}^{M} \left ( t_{i,y_i}\log p_{i,y_i} + (1-t_{i,y_i})\log p_{i,C+1} \right )
 \end{equation}
The abstention decision is then determined by the degree of confidence in the rejection class:
\begin{equation}
	g_{\sat}(\bm{x}, f) = f(\bm{x})_{C+1}
\end{equation}

\paragraph{Deep Ensembles (\de).}
Ensemble methods combine the information content of $M$ models into a single final model. Since these models approximate the variance of the underlying prediction problem, they are often used for the purpose of uncertainty quantification and, by extension, \selp. The canonical instance of this approach for deep learning based models, \emph{Deep Ensembles}~(\de)~\citep{lakshminarayanan2017simple}, trains multiple models from scratch with varying initializations using a proper scoring rule. Optionally, adversarial training can be used to enhance robustness and improve uncertainty estimates. Intuitively, adversarial training smooths the predictive distribution by encouraging the model to assign similar probabilities to neighboring inputs in an $\epsilon$-ball around the training data---especially in directions of high loss---which in turn leads to better-calibrated and more stable uncertainty estimates. In the context of deep ensembles, adversarial training also promotes diversity among the ensemble members, thereby improving selective prediction performance. After averaging the predictions made by the models, the softmax response (\sr) mechanism is applied:
\begin{equation}
    g_{\de}(\bm{x}, f) = \max_{c \in C} \frac{1}{M} \sum_{m=1}^{M} f_{\bm{\theta}_{m,T}}(\bm{x})_c.
\end{equation}

\paragraph{Monte-Carlo Dropout (\mcdo).}
To overcome the computational cost of estimating multiple models from scratch, \emph{Monte-Carlo Dropout} (\mcdo) \citep{gal2016dropout} allows for bootstrapping model uncertainty from a single dropout‐equipped network at test time. While dropout is predominantly used during training to regularize deep neural nets, it can also be applied at inference to yield random sub‐networks of the full model. Concretely, let \(f_{\theta}\) be our base model parameters and \(o\) the dropout probability. At test time, we draw \(Q\) independent dropout masks—each unit is kept with \(1 - o\) probability—producing \(Q\) thinned networks $\bigl\{\,f_{\theta,o}^{(q)}\bigr\}_{q=1}^{Q}$. For a given test input \(\bm{x}\), we compute each network’s softmax output \(f_{\theta,o}^{(q)}(\bm{x})\in\Delta^{|C|}\), average these probability vectors, and then take the maximum entry as our selection score:
\begin{equation}
  g_{\mcdo}(\bm{x},f)
  \;=\;
  \max_{c\in C}
  \frac{1}{Q}
  \sum_{q=1}^{Q}
    f_{\theta,o}^{(q)}(\bm{x})_{c}\,.
\end{equation}
This procedure efficiently approximates ensemble‐style uncertainty without retraining multiple models.

\subsection{Challenges and Trade-offs}
In selective prediction, a model’s performance depends not only on its accuracy when it decides to predict but also on how frequently it abstains. This leads to a fundamental \emph{coverage-utility trade-off}: increasing coverage (i.e., predicting on more samples) can degrade average performance, whereas being too selective may ignore a large portion of the data. In addition, various practical factors such as calibration, cost specification, and the nature of the output space further complicate the deployment of selective classifiers in real-world applications.

\paragraph{Calibration vs.\ Thresholding.}
A common approach to selective prediction is to set a confidence threshold below which the model abstains. However, if the model’s confidence scores are poorly calibrated, even a well-chosen threshold may yield suboptimal coverage and performance. Calibration techniques (e.g., Platt scaling or isotonic regression) are often necessary to align predicted confidence with actual likelihoods of correctness. Striking the right balance between calibration and thresholding is crucial for optimizing both utility (accuracy on covered samples) and coverage.

\paragraph{Data Distribution Shifts.}
In practice, data encountered during deployment may differ substantially from the training distribution. Under such domain shifts, confidence estimates can become unreliable, leading the gating mechanism to either abstain too frequently or incorrectly cover samples outside the training distribution. Techniques that adapt or re-estimate confidence under shifting conditions—such as online calibration or domain adaptation—remain an active area of research, aiming to preserve the coverage-utility trade-off even in evolving environments.

\paragraph{Cost Specification.}
Selective prediction often relies on a clear notion of the relative costs of rejection and misclassification. However, specifying these costs can be non-trivial, as it depends on domain knowledge, user preferences, and application-specific constraints. For instance, in a medical setting, the cost of a missed diagnosis may far outweigh the inconvenience of additional tests. An inaccurate cost specification can skew the thresholding decisions, leading to suboptimal coverage and utility outcomes.

\paragraph{Complex Outputs.}
While selective classification is well-understood for simple categorical outputs, tasks such as semantic segmentation or other structured prediction problems pose additional difficulties. Rejecting an entire structured output may be too coarse, yet partial rejections introduce substantial complexity into model design and training protocols. Extending selective prediction methods to these richer output spaces requires careful definition of what abstaining means (e.g., rejecting certain regions of an image segmentation vs.\ rejecting the entire image) and how to maintain favorable coverage-utility characteristics.

\subsection{Relationship Between UQ and Selective Prediction}

Whereas uncertainty quantification (UQ) can exist independently—for instance, by reporting confidence intervals or predictive distributions—selective prediction leverages these uncertainty estimates to decide whether to abstain from making a prediction. When reliable confidence or uncertainty metrics are available, selective classification naturally follows as a strategy for reducing risk. In particular, tools from uncertainty quantification—such as predictive distributions, confidence intervals, or calibration measures—yield scores or metrics (e.g., variance, entropy, or calibration error) that help determine whether the model should provide a label or abstain. However, for these signals to be effective, they must be both \emph{accurate} and \emph{well-calibrated}: miscalibrated or otherwise inaccurate estimates of uncertainty can lead to poor coverage-accuracy trade-offs, ultimately reducing the model’s reliability.

In this sense, selective prediction is an \textbf{action} grounded in the \textbf{estimation} process provided by uncertainty quantification. The decision to abstain hinges critically on the fidelity of the underlying uncertainty estimates. Without well-calibrated uncertainty, even sophisticated abstention mechanisms can perform poorly, either abstaining unnecessarily or failing to abstain when the risk is high.

More broadly, uncertainty estimation acquires practical significance only when situated within a downstream decision-making framework. Uncertainty—such as a predictive variance or entropy value—is not inherently meaningful unless it guides a choice among competing actions. In many real-world scenarios, decision-making under uncertainty is the norm: a medical diagnosis may trigger treatment or further testing, a financial risk estimate may inform loan approval, and a self-driving car’s confidence in its environment perception may determine whether it continues autonomously or defers control. In such contexts, the value of uncertainty lies in how it informs risk-aware decision-making.

Selective prediction provides a concrete instantiation of this principle. It formalizes a decision problem in which the model must choose between predicting and abstaining for each input, effectively balancing coverage against accuracy. Here, uncertainty is not just a descriptive statistic—it becomes a functional component of the model’s behavior. By embedding uncertainty into the decision rule itself, selective prediction highlights how the utility of uncertainty is deeply intertwined with the structure of the decision problem. Uncertainty matters most when there is something at stake and when there are viable alternatives—such as deferring to a human or a more accurate model—that can mitigate the consequences of model error.

    \newcommand{\shortened}[1]{\color{black} #1\xspace\color{black}}
\newcommand{\newlyadded}[1]{\color{black} #1\xspace\color{black}}
\newcommand{\fixed}[1]{\color{black} #1\xspace\color{black}}

\chapter{Selective Prediction Via Training Dynamics}
\label{ch:sptd}

\begin{paperref}
\normalfont
The contents of this chapter consist of research and results taken from: \citet{rabanser2022selective}: \emph{\bibentry{rabanser2022selective}}
\end{paperref}

\section*{Summary}

Selective Prediction is the task of rejecting inputs a model would predict incorrectly on. This involves a trade-off between input space coverage (how many data points are accepted) and model utility (how good is the performance on accepted data points). Current methods for selective prediction typically impose constraints on either the model architecture or the optimization objective; this inhibits their usage in practice and introduces unknown interactions with pre-existing loss functions. In contrast to prior work, we show that state-of-the-art selective prediction performance can be attained solely from studying the (discretized) training dynamics of a model. We propose a general framework that, given a test input, monitors metrics capturing the instability of predictions from intermediate models (\ie checkpoints) obtained during training w.r.t. the final model's prediction. In particular, we reject data points exhibiting too much disagreement with the final prediction at late stages in training. The proposed rejection mechanism is domain-agnostic (\ie it works for both discrete and real-valued prediction) and can be flexibly combined with existing selective prediction approaches as it does not require any train-time modifications. Our experimental evaluation on image classification, regression, and time series problems shows that our method beats past state-of-the-art accuracy/utility trade-offs on typical selective prediction benchmarks.

\section{Introduction}

Machine learning (ML) is increasingly deployed in high-stakes decision-making environments with strong reliability and safety requirements. One of these requirements is the detection of inputs for which the ML model produces an erroneous prediction. This is particularly important when deploying deep neural networks (DNNs) for applications with low tolerances for \fps (\ie classifying with a wrong label), such as healthcare~\citep{challen2019artificial, mozannar2020consistent}, self-driving~\citep{ghodsi2021generating}, and law~\citep{vieira2021understanding}.
This problem setup is captured by the Selective Prediction (SP) framework, which introduces an accept/reject function (a so-called \emph{gating mechanism}) to abstain from predicting on individual test points in the presence of high prediction uncertainty~\citep{geifman2017selective}. 
Specifically, SP aims to (i) only accept inputs on which the ML model would achieve high utility, while (ii) maintaining high coverage (\ie correctly accepting as many inputs as possible).

Current selective prediction techniques 
take one of two directions: (i) augmentation of the architecture of the underlying ML model~\citep{geifman2019selectivenet}; or (ii) training the model using a purposefully adapted loss function~\citep{liu2019deep, huang2020self, gangrade2021selective}. 
The unifying principle behind these methods is to modify the training stage in order to accommodate selective prediction. While many ad-hoc experimentation setups are amenable to these changes, productionalized environments often impose data pipeline constraints which limit the applicability of existing methods. Such constraints include, but are not limited to, data access revocation, high (re)-training costs, or pre-existing architecture/loss modifications whose interplay with selective prediction adaptations are unexplored. As a result of theses limitations, selective prediction approaches are hard to deploy in production environments.

We instead show that \textit{ these modifications are unnecessary}. That is, our method, which \textbf{establishes new SOTA results for selective prediction} across a variety of datasets, not only outperforms existing work but \textbf{our method can be easily applied on top of all existing models}, unlike past methods. Moreover, our method is not restricted to classification problems but can be applied for real-valued prediction problems, too, like regression and time series prediction tasks. This is an important contribution as a majority of recent selective prediction approaches have solely focused on improving selective \emph{classification}.

\begin{figure*}[t]
    \centering
    \includegraphics[width=0.97\linewidth]{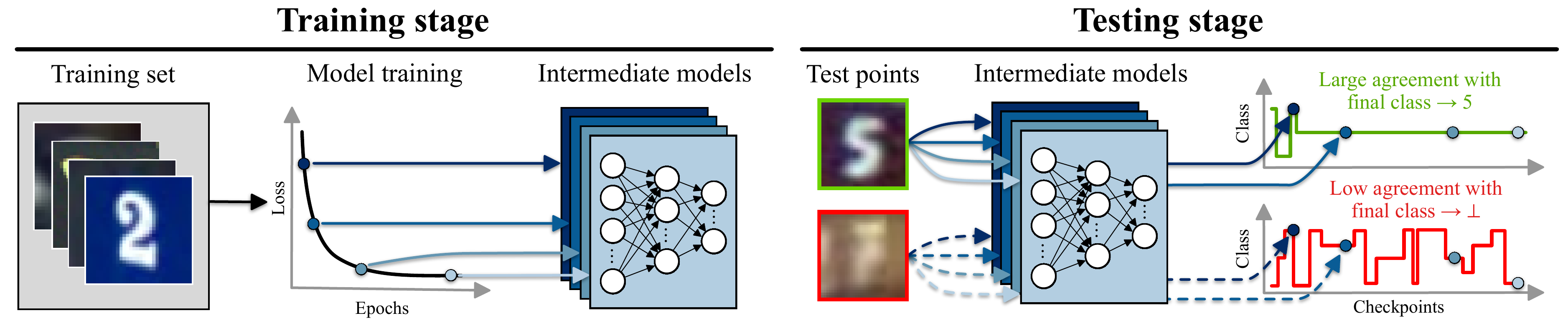}
    \caption[Our proposed \sptd method for a classification example]{\textbf{Our proposed \sptd method for a classification example}. We store checkpoints of intermediate models during model training. At inference time, given a test input, we compute various metrics capturing the stability of intermediate predictions with respect to the final model prediction. Data points with high stability are accepted, data points with low stability are rejected.}
    \label{fig:nntd_overview}
\end{figure*}

Our approach builds on the following observation: typical DNNs are trained using an iterative optimization procedure, \eg using Stochastic Gradient Descent (SGD). Due to the sequential nature of this optimization process, as training goes on, the optimization process yields a sequence of intermediate models. \newlyadded{Current selective prediction methods rely only on the final model, ignoring valuable statistics available from the model’s training sequence.}
In this work, however, we propose to take advantage of the information contained in these optimization trajectories for the purpose of selective prediction. By studying the usefulness of these trajectories, we observe that instability in SGD convergence is often indicative of high aleatoric uncertainty (\ie irreducible data noise such as overlap between distinct data classes). Furthermore, past work on example difficulty~\citep{jiang2020characterizing,toneva2018empirical,hooker2019compressed,agarwal2020estimating} has highlighted faster convergence as indicative of easy-to-learn training examples (and conversely slow convergence of hard-to-learn training examples). We hypothesize that such training time correlations with uncertainty also hold for test points and studying how test time predictions evolve over the intermediate checkpoints is useful for reliable uncertainty quantification.

With this hypothesis, we derive the first framework for \textbf{S}elective \textbf{P}rediction based on neural network \textbf{T}raining \textbf{D}ynamics~(\sptd, see Figure~\ref{fig:nntd_overview} for an example using a classification setup). Through a formalization of this particular neural network training dynamics problem, we first note that a useful property of the intermediate models' predictions for a test point is whether they converge ahead of the final prediction. This convergence can be measured by deriving a prediction instability score measuring how strongly predictions of intermediate models agree with the final model. While the exact specifics of how we measure instability differs between domains (classification vs regression), our resulting score generalizes across application domains and measures weighted prediction instability. This weighting allows us to emphasize instability late in training which we deem indicative of points that should be rejected. Note that this approach is transparent w.r.t. the training stage: our method only requires that intermediate checkpoints were recorded when a model was trained, which is an established practice (especially when operating in shared computing environments such as GPU clusters). Moreover, when compared to competing ensembling-based methods, such as Deep Ensembles~\citep{balaji2017uncertainty}, our approach can match the same inference-time cost while being significantly cheaper to train.

To summarize, our main contributions are as follows: 
\begin{enumerate}
    \item We present a motivating synthetic example using a linear model, showcasing the effectiveness of training dynamics information in the presence of a challenging classification task (Section~\ref{sec:method_intuition}). 
    \item We propose a novel method for selective prediction based on training dynamics (\sptd, Section~\ref{sec:method_overview}). To that end, we devise an effective scoring mechanism capturing weighted prediction instability of intermediate models with the final prediction for individual test points. Our methods allow for selective classification, selective regression, and selective time series prediction. Moreover, \sptd can be applied to all existing models whose checkpoints were recorded during training.
    \item We highlight an in-depth connection between our \sptd approach and forging (\cite{thudi2022necessity}, Section~\ref{sec:forging}), which has shown that optimizing a model on distinct datasets can lead to the same sequence of checkpoints. This connection demonstrates that our metrics can be motivated from a variety of different perspectives.
    \item We perform a comprehensive set of empirical experiments on established selective prediction benchmarks spanning over classification, regression, and time series prediction problems (Section~\ref{sec:emp_eval}). 
    Our results obtained from all instances of \sptd demonstrate highly favorable utility/coverage trade-offs, establishing new state-of-the-art results in the field at a fraction of the training time cost of competitive prior approaches. 
\end{enumerate}

\section{Background on Selective Prediction}
\label{sec:background_sptd}

\paragraph{Supervised Learning Setup.} Our work considers the standard supervised learning setup. We assume access to a dataset $D = \{(\bm{x}_i,y_i)\}_{i=1}^{M}$ consisting of $M$ data points $(\bm{x},y)$ with $\bm{x} \in \mathcal{X}$ and  $y \in \mathcal{Y}$. We refer to $\mathcal{X} := \mathbb{R}^d$ as the covariate space (or input/data space) of dimensionality $d$. For classification problems, we define $\mathcal{Y} := [C] = \{1, 2, \ldots, C\}$ as the label space consisting of $C$ classes. For regression and time series problems (such as demand forecasting) we instead define $\mathcal{Y} := \mathbb{R}$ and $\mathcal{Y} := \mathbb{R}^R$ respectively (with $R$ being the prediction horizon). All data points $(\bm{x},y)$ are sampled independently from the underlying distribution $p$ defined over the joint covariate and label spaces $\mathcal{X} \times \mathcal{Y}$. Our goal is to learn a prediction function $f : \mathcal{X} \rightarrow \mathcal{Y}$ which minimizes the risk $\mathcal{R}(f_{\bm{\theta}})$ with respect to the underlying data distribution $p$ and an appropriately chosen loss function $\ell : \mathcal{Y} \times \mathcal{Y} \rightarrow \mathbb{R}$. 
We can derive the optimal parameters $\hat{\bm{\theta}}$ via empirical risk minimization which approximates the true risk $\mathcal{R}(f_{\bm{\theta}})$ through sampling, thereby ensuring that $\bm{\theta}^* \approx \hat{\bm{\theta}}$ for a sufficiently large amount of samples:
\begin{align}
    \bm{\theta}^* & = \argmin_{\bm{\theta}} \mathcal{R}(f_{\bm{\theta}}) = \argmin_{\bm{\theta}} \mathbb{E}_{p(\bm{x},y)}[\ell(f_{\bm{\theta}}(\bm{x}),y)] \\ 
    \hat{\bm{\theta}} & = \argmin_{\bm{\theta}} \hat{\mathcal{R}}(f_{\bm{\theta}}) = \argmin_{\bm{\theta}} \frac{1}{M} \sum_{i=1}^{N} \ell(f_{\bm{\theta}}(\bm{x}_i),y_i)
\end{align}
In the following, we drop the explicit dependence of $f$ on $\bm{\theta}$ and simply denote the predictive function by $f$.

\paragraph{Selective Prediction Setup.} Selective prediction alters the standard supervised learning setup by introducing a rejection state~$\bot$ through a \textit{gating mechanism}~\citep{yaniv2010riskcoveragecurve}. In particular, such a mechanism introduces a selection function $g:\mathcal{X} \rightarrow \mathbb{R}$ which determines if a model should predict on a data point~$\bm{x}$. 
Given an acceptance threshold $\tau$, the resulting predictive model can be summarized as:
\begin{equation}
    (f,g)(\bm{x}) = \begin{cases}
  f(\bm{x})  & g(\bm{x}) \leq \tau \\
  \bot & \text{otherwise.}
\end{cases}
\end{equation}

\paragraph{Selective Prediction Evaluation Metrics.} Prior work evaluates the performance of a selective predictor $(f,g)$ based on two metrics: the \emph{coverage} of $(f,g)$ (\ie what fraction of points we predict on) and the \emph{selective utility} of $(f,g)$ on the accepted points. Note that the exact utility metric depends on the type of the underlying selective prediction task (e.g. accuracy for classification, $R^2$ for regression, and a quantile-based loss for time series forecasting). Successful SP methods aim to obtain both strong selective utility and high coverage. Note that these two metrics are at odds with each other: na\"ively improving utility leads to lower coverage and vice-versa. The complete performance profile of a model can be specified using the risk–coverage curve, which defines the risk as a function of coverage~\citep{yaniv2010riskcoveragecurve}. These metrics can be formally defined as: 
\begin{align}
    \text{coverage}(f,g) & = \frac{M_\tau}{M} \\
    \text{utility}(f,g) & = \sum_{ \{(\bm{x}, y) : g(\bm{x}) \leq \tau \} } u(f(\bm{x}), y)
\end{align}
Here, $u(\cdot, \cdot)$ corresponds to the specifically used utility function, $M_\tau = \sum_{i=1}^M \mathds{1}[\bm{x}_i : g(\bm{x}_i) \leq \tau]$ corresponds to the number of accepted data points at threshold $\tau$, and $\mathds{1}[\cdot]$ corresponds to the indicator function. We define the following utility functions to be used based on the problem domain:

\begin{enumerate}
    \item \textit{Classification}: We use accuracy on accepted points as our utility function for classification:
    \begin{equation}
        \text{Accuracy} = \frac{1}{M_\tau}\sum_{i=1}^{M_\tau} \mathds{1}[\bm{x}_i : f(\bm{x}_i) = y_i]
    \end{equation}
    \item \textit{Regression}: We use the coefficient of determination ($R^2$ score, which is a scaled version of the mean squared error) on accepted points as our utility function for regression:
     \begin{equation}
        R^2= 1 - \frac{\sum_{i=1}^{M_\tau} (f(\bm{x}_i) - y_i)^2}{\sum_{i=1}^{M_\tau} (y_i - \frac{1}{M_\tau}\sum_{j=1}^{M_\tau} y_j)^2}
    \end{equation}
    \item \textit{Time Series Forecasting}: We use the Mean Scaled Interval Score (MSIS)~\cite{gneiting2007strictly} on accepted series as our utility function for time series forecasting
    \begin{equation}
    \fontsize{7.75pt}{7.75pt}
    \text { MSIS }= \frac{1}{M_\tau R}\sum_{i=1}^{M_\tau} \frac{ \splitfrac{ \sum_{r=n+1}^{n+R}\left(u_{i,r}-l_{i,r}\right) +\frac{2}{\alpha}\left(l_{i,r}-y_{i,r}\right) \mathds{1}[y_{i,r}<l_{i,r}]}{ +\frac{2}{\alpha}\left(y_{i,r}-u_{i,r}\right)  \mathds{1}[y_{i,r}>u_{i,r}] } }{\frac{1}{n-m} \sum_{r=m+1}^n\left|y_{i,r}-y_{i,r-m}\right|}
    \end{equation}
    where $\alpha$ refers to a specific predictive quantile, $n$ to the conditioning length of the time series, $m$ to the length of the seasonal period, and $u_{i,r}$ and $l_{i,r}$ to the upper and lower bounds on the prediction range, respectively.
\end{enumerate}

\subsection{Past \& Related Work} 

\paragraph{Softmax Response Baseline (Classification).} The first work on selective classification was the softmax response (\sr) mechanism~\citep{hendrycks2016baseline, geifman2017selective}. A classification model typically has a softmax output layer which takes in unnormalized activations in $z_{i} \in \mathbb{R}^C$ (referred to as logits) from a linear model or a deep neural net. These activations are mapped through the softmax function which normalizes all entries 
\begin{equation}
    \sigma(\bm{z})_{i}=\frac{e^{z_{i}}}{\sum_{j=1}^{K} e^{z_{j}}}
\end{equation}
to the interval $[0,1]$ and further ensures that $\sum_{i=1}^{C} \sigma(\bm{z})_{i} = 1$. As a result, the softmax output can be interpreted as a conditional probability distribution which we denote $f(y|\bm{x})$.
The softmax response mechanism applies a threshold $\tau$ to the maximum response of the softmax layer: $\max_{y \in \mathcal{Y}}f(y|\bm{x})$. 
Given a confidence parameter $\delta$ and desired risk $\hat{\mathcal{R}}(f)$, \sr constructs $(f, g)$ with test error no larger than $\hat{\mathcal{R}}(f)$ with probability $\geq 1-\delta$. 
While this approach is simple to implement, it has been shown to produce over-confident results due to poor calibration of deep neural networks~\citep{guo2017calibration}.\footnote{Under miscalibration, a model's prediction frequency of events does not match the true frequency of events.} 

\paragraph{Loss Modifications (Mostly Classification).} 
The first work to deliberately address selective classification via architecture modification was SelectiveNet~\citep{geifman2019selectivenet}, which trains a model to jointly optimize for classification and rejection. A loss penalty is added to enforce a particular coverage constraint using a variant of the interior point method~\cite{potra2000interior} which is often used for solving linear and non-linear convex optimization problems. To optimize selective accuracy over the full coverage spectrum in a single training run, Deep Gamblers~\citep{liu2019deep} transform the original $C$-class problem into a $(C + 1)$-class problem where the additional class represents model abstention. \shortened{A similar approach is given by Self-Adaptive Training (\sat)~\citep{huang2020self} which also uses a $(C+1)$-class setup but instead incorporates an exponential average of intermediate predictions into the loss function.} Other similar approaches include: performing statistical inference for the marginal prediction-set coverage rates using model ensembles~\citep{feng2021selective}, confidence prediction using an earlier snapshot of the model~\citep{geifman2018bias}, estimating the gap between classification regions corresponding to each class~\citep{gangrade2021selective}, and complete precision by classifying only when models consistent with the training data predict the same output~\citep{khani2016unanimous}. 

\paragraph{Uncertainty Quantification (Classification + Regression).} 
It was further shown by~\citep{balaji2017uncertainty, zaoui2020regression} that deep model ensembles (\ie a collection of multiple models trained with different hyper-parameters until convergence) can provide state-of-the-art uncertainty quantification, a task closely related to selective prediction. This however raises the need to train multiple models from scratch. \shortened{To reduce the cost of training multiple models, \citep{gal2016dropout} proposed abstention based on the variance statistics from several dropout \cite{srivastava2014dropout} enabled forward passes at test time.} Another popular technique for uncertainty quantification, especially for regression and time series forecasting, is given by directly modeling the output distribution~\citep{alexandrov2019gluonts} in a parametric fashion. Training with a parametric output distribution however can lead to additional training instability, often requiring extensive hyper-parameter tuning and distributional assumptions. On the other hand, our approach does not require any architecture or other training-time modifications. \newlyadded{Finally, we note that selective prediction and uncertainty are also strongly related to the field of automated model evaluation which relies on the construction a proximal prediction pipeline of the testing performance without the presence of ground-truth labels~\citep{peng2023came,peng2024energy}.}

\paragraph{Training Dynamics Approaches (Classification).} Checkpoint and snapshot ensembles \citep{huang2017snapshot, chen2017checkpoint} constitute the first usage of training dynamics to boost model utility. Our work is closest in spirit to recent work on dataset cartography~\citep{swayamdipta2020dataset} which relies on using training dynamics from an example difficulty viewpoint by considering the variance of logits. However, their approach does not consider selective prediction and further requires access to true label information (which is not available in the selective prediction setting). Recent work on out-of-distribution detection~\citep{adila2022understanding}, a closely related yet distinct application scenario from selective prediction, harness similar training dynamics based signals.

\paragraph{Example Difficulty.}
\label{sec:example_diff}

A related line of work to selective prediction is identifying \emph{difficult} examples, or how well a model can generalize to a given unseen example. Recent work~\cite{jiang2020characterizing} has demonstrated that the probability of predicting the ground truth label with models trained on data sub-samples of different sizes can be estimated via a per-instance empirical consistency score. Unlike our approach, however, this requires training a large number of models. 
Example difficulty can also be quantified through the lens of a forgetting event ~\cite{toneva2018empirical} in which the example is misclassified after being correctly classified. Instead, the metrics that we introduce in Section~\ref{sec:method}, are based on the disagreement of the label at each checkpoint with the final predicted label. Other approaches estimate the example difficulty by: 
prediction depth of the first layer at which a $k$-NN classifier correctly classifies an example~\citep{baldock2021deep}, the impact of quantization and compression on model predictions of a given sample~\citep{hooker2019compressed}, and estimating the leave-one-out influence of each training example on the accuracy of an algorithm by using influence functions~\citep{feldman2020neural}. 
Closest to our method, the work of \cite{agarwal2020estimating} utilizes gradients of intermediate models during training to rank examples by  difficulty. In particular, they average pixel-wise variance of gradients for each given input image. 
Notably, this approach is more costly and less practical than our approach and also does not study the utility/coverage trade-off which is of paramount importance to selective prediction.

\newlyadded{
\paragraph{Disagreement.} Our \sptd method heavily relies on the presence of disagreements between intermediate models. Past work on (dis-)agreement has studied the connection between generalization and disagreement of full SGD runs \citep{jiang2021assessing} as well as correlations between in-distribution and out-of-distribution agreement across models \citep{baek2022agreement}.
}

\section{Selective Prediction via Neural Network Training Dynamics}
\label{sec:method}

We now introduce our selective prediction algorithms based on neural network training dynamics. We start by presenting a motivating example showcasing the effectiveness of analyzing training trajectories for a linear classification problem. Following this, we formalize our selective prediction scoring rule based on training-time prediction disagreements. We refer to the class of methods we propose as \sptd.

\subsection{Method Intuition: Prediction Disagreements Generalize Softmax Response}
\label{sec:method_intuition}

\begin{figure*}
    \centering
    \includegraphics[width=0.97\linewidth]{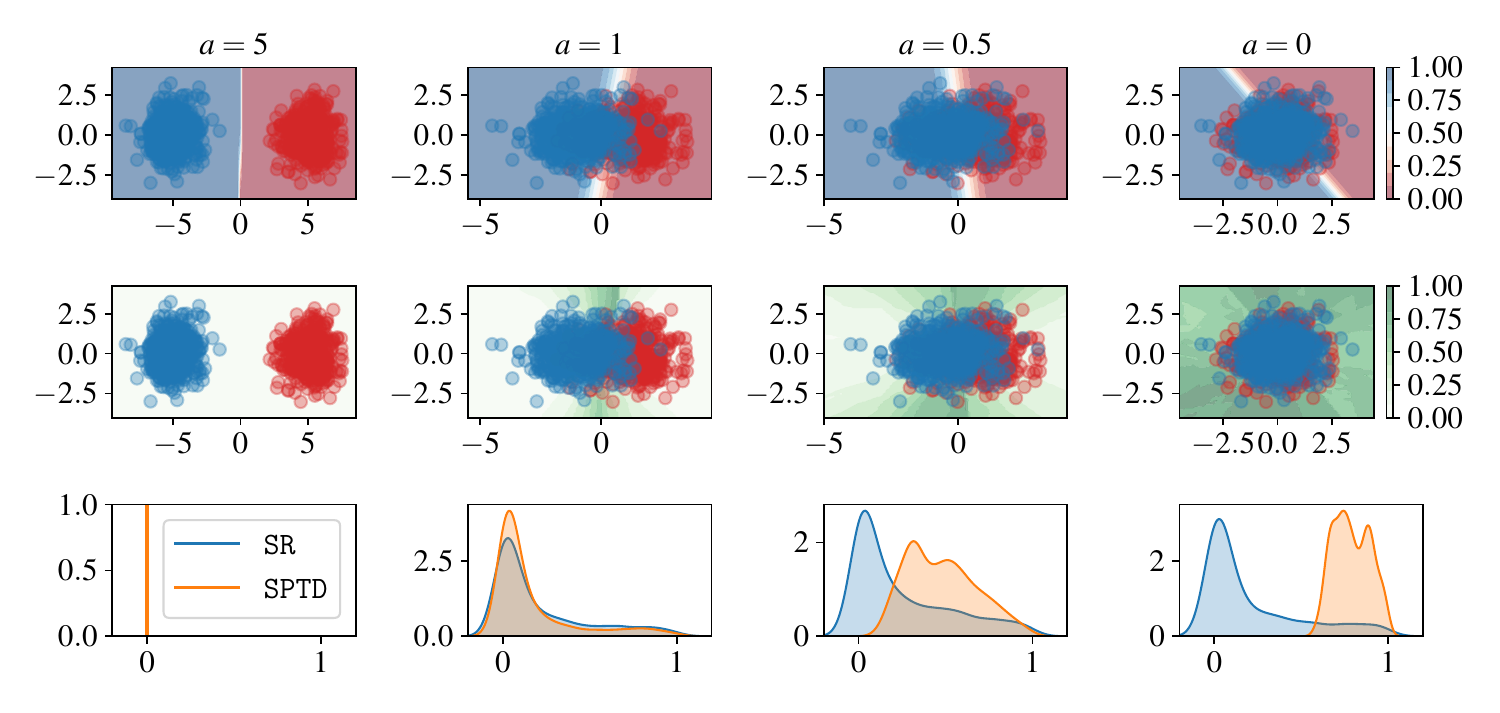}
    \caption[Synthetic example of anomaly scoring based on \sr vs \sptd.]{\textbf{Synthetic example of anomaly scoring based on \sr vs \sptd}. The first row shows a test set from the generative Gaussian model as well as the learned decision boundary separating the two Gaussians. For small $a$, the decision boundary is overconfident. The second row shows the same data set but instead depicts the scores yielded by applying \sptd to the full domain. \sptd highlights rightful rejection regions more clearly than the overconfident \sr score: larger regions are flagged as exhibiting noisy training dynamics (with stronger shades of green indicating stronger disagreement) as $a \rightarrow 0$. The bottom row shows the distribution of the \sr and \sptd scores, clearly showing that \sptd leads to improved uncertainty under stronger label ambiguity.}
    \label{fig:gauss}
\end{figure*}

Stochastic iterative optimization procedures, such as Stochastic Gradient Descent (SGD), yield a sequence of models that is iteratively derived by minimizing a loss function $\ell(\cdot,\cdot)$ on a randomly selected mini-batch $(\bm{X}_i, \bm{y}_i)$ from the training set. The iterative update rule can be expressed as
\begin{equation}
    \bm{\theta}_{t+1} = \bm{\theta}_{t} - \nu \frac{\partial \ell(f(\bm{X}_i), \bm{y}_i)}{\partial \bm{\theta}_{t}},
\end{equation}
where the learning rate $\nu$ controls the speed of optimization and $t \in \{1,\ldots,T\}$ represents a particular time-step during the optimization process.

\looseness=-1
Current methods for selective prediction disregard the properties of this iterative process and only rely on the final set of parameters $\bm{\theta}_{T}$. However, the optimization trajectories contain information that we can use to determine prediction reliability. In particular, on hard optimization tasks, the presence of stochasticity from SGD and the potential ambiguity of the data often leads to noisy optimization behavior. As a result, intermediate predictions produced over the course of training might widely disagree in what the right prediction would be for a given data point. Our class of selective prediction approaches explicitly make use of these training dynamics by formalizing rejection scores based on the observed frequency of prediction disagreements with the final model throughout training.

To illustrate and reinforce this intuition that training dynamics contain meaningfully more useful information for selective prediction than the final model, we present a synthetic logistic regression example. First, we generate a mixture of two Gaussians each consisting of $1000$ samples: $D = \{(\bm{x}_i,0)\}_{i=1}^{1000} \cup \{(\bm{x}_j,1)\}_{j=1}^{1000}$ where $\bm{x}_i \sim \mathcal{N}(\begin{bmatrix}a & 0\end{bmatrix}^\top,\ \bm{I})$ and $\bm{x}_j \sim \mathcal{N}(\begin{bmatrix}-a & 0\end{bmatrix}^\top,\ \bm{I})$. Note that $a$ controls the distance between the two $2$-dimensional Gaussian clusters, allowing us to specify the difficulty of the learning task. Then, we train a linear classification model using SGD for $1000$ epochs for each $a \in \{0,0.5,1,5\}$. Finally, we compute both the softmax response score (\sr) score, the typical baseline for selective classification, as well as our \sptd score (details in Section~\ref{sec:method_overview}).

We showcase the results from this experiment in Figure~\ref{fig:gauss}. We  see that if the data is linearly separable ($a=5$) the learned decision boundary is optimal and the classifier's built-in confidence score \sr reflects well-calibrated uncertainty. Moreover, the optimization process is stable as \sptd yields low scores over the full domain. However, as we move the two Gaussians closer together (\ie by reducing $a$) we see that the \sr baseline increasingly suffers from overconfidence: large parts of the domain are very confidently classified as either $0$ (red) or $1$ (blue) with only a small ambiguous region (white). However, the optimization trajectory is highly unstable with the decision boundary changing abruptly between successive optimization steps. \sptd identifies the region of datapoints exhibiting large prediction disagreement due to this training instability and correctly rejects them (as those are regions also subject to label ambiguity in this case). In summary, we observe that \sptd provides improved uncertainty quantification in ambiguous classification regions (which induce training instability) and reduces to the \sr solution as the classification task becomes easier. Hence, we expect \sptd to generalize \sr performance, which is supported by this logistic regression experiment.

\subsection{Method Overview: Measuring Prediction Instability During Training}
\label{sec:method_overview}

\begin{figure*}
\begingroup
    \begin{minipage}{0.48\textwidth}

    \begin{algorithm}[H]
    	\caption{\sptd for classification}\label{alg:sptd_class}
    	\begin{algorithmic}[1]
    	\Require Intermediate models $[f_1,\ldots,f_T]$, query point $\bm{x}$, weighting parameter $k \in [0,\infty)$.
        \For{$t \in [T]$}
            \State \algemph{\algorithmicif\ $f_t(\bm{x}) = f_T(\bm{x})$ \algorithmicthen\ $a_t \gets 0$ \algorithmicelse\ $a_t \gets 1$}
            \State $\ v_t \gets (\frac{t}{T})^k$
        \EndFor
    \State $g \gets \sum_{t} v_t a_t$
    \State \algorithmicif\ $g \leq \tau$ \algorithmicthen\ \fixed{$f(\bm{x}) = f_T(\bm{x})$} \algorithmicelse\ $f(\bm{x}) = \bot$
    	\end{algorithmic}
    \end{algorithm}
    
    \end{minipage}
    \hfill
    \begin{minipage}{0.48\textwidth}

    \begin{algorithm}[H]
    	\caption{\sptd for regression}\label{alg:sptd_regr}
    	\begin{algorithmic}[1]
    	\Require Intermediate models $[f_1,\ldots,f_T]$, query point $\bm{x}$, weighting parameter $k \in [0,\infty)$. 
        \For{$t \in [T]$}
            \State \algemph{$a_t \gets ||f_t (\bm{x}) - f_T (\bm{x})||$}
            \State $\ v_t \gets (\frac{t}{T})^k$
        \EndFor
    \State $g \gets \sum_{t} v_t a_t$
    \State \algorithmicif\ $g \leq \tau$ \algorithmicthen\ \fixed{$f(\bm{x}) = f_T(\bm{x})$} \algorithmicelse\ $f(\bm{x}) = \bot$
    	\end{algorithmic}
    \end{algorithm}
    
    \end{minipage}

    \vspace{-10pt}
\endgroup
\end{figure*}

We proceed to describe the statistics we collect from intermediate checkpoints which we later devise our scores for deciding which inputs to reject on. The objective of these statistics is to capture how unstable the prediction for a datapoint was over the training checkpoints. Let $[f_1,f_2,\ldots, f_T]$ be a sequence of intermediate checkpoints, and $\mathcal{D} = D_\text{train} \cup D_\text{test}$ be the set of all data points. We define a prediction disagreement score at time $t \in \{1,\ldots,T\}$ as some function $a_t: \mathcal{X} \rightarrow \mathbb{R}^+$ with $a_t(\bm{x}) = 0$ if $f_t(\bm{x}) = f_T(\bm{x})$. Note that the exact $a_t(\cdot)$ we use depends on the problem domain (classification vs regression) and we define our choices below. In the following, when conditioning on $\bm{x}$ is understood from context, we drop the explicit dependence on $\bm{x}$ and write $a_t$.

For a fixed data point $\bm{x}$, our approach takes a given sequence of prediction disagreements $[a_1,\ldots,a_T]$ and associates a weight $v_t$ to each disagreement $a_t$ to capture how severe a disagreement at step $t$ is. To derive this weighting we ask: How indicative of $\bm{x}$ being incorrectly classified is a disagreement at step $t$? Related work in the example difficulty literature (see Section~\ref{sec:example_diff} for details) found that easy-to-optimize samples are learned early in training and converge faster. While prior work specifically derived these convergence insights for training points only, the novelty of our method is to show such conclusions for training points also generalize to test points. Hence, we propose to use the weighting $v_t = (\frac{t}{T})^k$ for $k \in [0,\infty)$ to penalize late prediction disagreements as more indicative of a test point we will not predict correctly on. With this weighting, our methods compute a weighted sum of the prediction disagreements, which effectively forms our selection function $g(\cdot)$: 
\begin{equation}
    \label{eq:score}
    g(\bm{x}) = \sum_{t} v_t a_t(\bm{x})    
\end{equation}

\paragraph{Instability for Classification.} For discrete prediction problems (\ie classification) we define the label disagreement score as $a_t = 1- \delta_{f_t(\bm{x}),f_T(\bm{x})}$ where $\delta$ is the Dirac-delta function: $a_t$ is hence $1$ if the intermediate prediction $f_t$ at checkpoint $t$ disagrees with the final prediction $f_T$ for $\bm{x}$, else $0$. The resulting algorithm using this definition of $a_t$ for classification is given in Algorithm~\ref{alg:sptd_class}. We remark that continuous metrics such as the maximum softmax score, the predictive entropy (\ie the entropy of the predictive distribution $f(y|\bm{x})$), or the gap between the two most confident classes could be used as alternate measures for monitoring stability (see Appendix~\ref{sec:alt_scores} for a discussion). However, these measures only provide a noisy proxy and observing a discrete deviation in the predicted class provides the most direct signal for potential mis-classification.

\paragraph{Instability for Regression.} One key advantage of our method over many previous ones is that it is applicable to \emph{any} predictive model, including regression. Here, we propose the following prediction disagreement score measuring the distance of intermediate predictions to the final model's prediction: $a_t = ||f_t (\bm{x}) - f_T (\bm{x})||$.\footnote{We explored a more robust normalization by averaging predictions computed over the last $l$ checkpoints: $a_t = ||f_t (\bm{x}) - \frac{1}{n}\sum_{c \in \{T-l, T-l+1, \ldots, T \}} f_c (\bm{x})||$. Across many $l$, we found the obtained results to be statistically indistinguishable from the results obtained by normalizing w.r.t. the last checkpoint $f_T$.} The resulting algorithm using this definition of $a_t$ for regression is given in Algorithm~\ref{alg:sptd_regr}. We again highlight the fact that Algorithm~\ref{alg:sptd_regr} only differs from Algorithm~\ref{alg:sptd_class} in the computation of the prediction disagreement $a_t$ (line 2 highlighted in both algorithms).

\begin{algorithm}[t]
    \caption{\sptd for time series forecasting}\label{alg:sptd_ts}
    \begin{algorithmic}[1]
    \Require Intermediate models $[f_1,\ldots,f_T]$, query point $\bm{x}$, weighting $k \in [0,\infty)$, prediction horizon $R$. 
    \For{$t \in [T]$}
        \For{$r \in [R]$}
            \State $a_{t,r} \gets ||f_{t}(\bm{x})_r - f_{T}(\bm{x})_r||$
        \EndFor
        \State $v_t \gets (\frac{t}{T})^k$
    \EndFor
\State $g \gets \sum_r\sum_{t} v_t a_{t,r}$
\State \algorithmicif\ $g \leq \tau$ \algorithmicthen\ $f(\bm{x}) = L$ \algorithmicelse\ $f(\bm{x}) = \bot$
    \end{algorithmic}
\end{algorithm}

\paragraph{Instability for Time Series Prediction.} We can further generalize the instability sequence used for regression to time series prediction problems by computing the regression score for all time points on the prediction horizon. In particular, we compute $a_{t,r} = ||f_{t}(\bm{x})_r - f_{T}(\bm{x})_r||$ for all $r \in \{1,\ldots,R\}$. Recall that for time series problems $f_{t}(\bm{x})$ returns a vector of predictions $y \in \mathbb{R}^R$ and we use the subscript $r$ on $f_{t}(\bm{x})_r$ to denote the vector indexing operation. Our selection function is then given by computing Equation~\ref{eq:score} for each $r$ and summing up the instabilities over the prediction horizon: $g(\bm{x}) = \sum_r\sum_{t} v_t a_{t,r}(\bm{x})$. The full algorithm therefore shares many conceptual similarities with Algorithm~\ref{alg:sptd_regr} and we provide the detailed algorithm as part of Algorithm~\ref{alg:sptd_ts}. Note that the presented generalization for time series is applicable to any setting in which the variability of predictions can be computed. As such, this formalism can extend to application scenarios beyond time series prediction such as object detection or segmentation.

\subsection{Selective Prediction and Forging}
\label{sec:forging}

\begin{contriback}
This subsection was written with Anvith Thudi. Both Stephan and Anvith jointly developed the max score and the sum score. Anvith provided details on the connection to forging as well as the formalism of Lemma~\ref{lem:prob_accept}.
\end{contriback}

While our \sptd method is primarily motivated from the example difficulty view point, we remark that the scores \sptd computes to decide which points to reject can be derived from multiple different perspectives. To showcase this, we provide a formal treatment on the connection between selective classification and forging~\citep{thudi2022necessity}, which ultimately leads to the same selection function $g(\cdot)$ as above.

Previous work has shown that running SGD on different datasets could lead to the same final model~\citep{hardt2016train,bassily2020stability,thudi2022necessity}. For example, this is intuitive when two datasets were sampled from the same distribution. We would then expect that training on either dataset should not significantly affect the model returned by SGD. For our selective prediction problem, this suggests an approach to decide which points the model is likely to predict correctly on: identify the datasets that it could have been trained on (in lieu of the training set it was actually trained on). Any point from the datasets the model could have trained on would then be likely to be predicted on correctly by the model. Recent work on forging \cite{thudi2022necessity} solves this problem of identifying datasets the model could have trained on by brute-force searching through different mini-batches to determine if a mini-batch in the alternative dataset can be used to reproduce one of the original training steps. Even then, this is only a sufficient condition to show a datapoint could have plausibly been used to train: if the brute-force fails, it does not mean the datapoint could not have been used to obtain the final model. As an alternative, we propose to instead characterize the optimization behaviour of training on a dataset as a probabilistic necessary condition, i.e, a condition most datapoints that were (plausibly) trained on would satisfy based on training dynamics. Our modified hypothesis is then that the set of datapoints we optimized for (which contains the forgeable points) coincides significantly with the set of points the model predicts correctly on.

\subsubsection{A Framework for Being Optimized}
\label{ssec:reject_cond}

In this section we derive an upper-bound on the probability that a datapoint could have been used to obtain the model's checkpointing sequence. This yields a probabilistically necessary (though not sufficient) characterization of the points we explicitly optimized for. This bound, and the variables it depends on, informs what we characterize as "optimizing" for a datapoint, and, hence, our selective classification methods.

Let us denote the set of all datapoints as $\mathcal{D}$, and let $D \subset \mathcal{D}$ be the training set. We are interested in the setting where a model $f$ is plausibly sequentially trained on $D$ (e.g., with stochastic gradient descent). We thus also have access to a sequence of $T$ intermediate states for $f$, which we denote $[f_1,\ldots, f_T]$. In this sequence, note that $f_T$ is exactly the final model $f$. 

Now, let $p_{t}$ represent the random variable for outputs on $D$ given by an intermediate model $f_t$ where the outputs have been binarized: we have $0$ if the output agrees with the final prediction and $1$ if not. In other words, $p_{t}$ is the distribution of labels given by first drawing $\bm{x} \sim D$ and then outputting $1- \delta_{f_t(\bm{x}),f_T(\bm{x})}$ where $\delta$ denotes the Dirac delta function. Note that we always have both a well-defined mean and variance for $p_{t}$ as it is bounded. Furthermore, we always have the variances and expectations of $\{p_{t}\}$ converge to $0$ with increasing $t$: as $p_{T} = 0$ always and the sequence is finite convergence trivially occurs. To state this formally, let $v_t = \mathbb{V}_{\bm{x} \sim D}[p_{t}]$ and let $e_t = \mathbb{E}_{\bm{x} \sim D}[p_{t}]$ denote the variances and expectations over points in $D$. In particular, we remark that $e_T=0,\ v_T = 0$, so both $e_t$ and $v_t$ converge. More formally, for all $ \epsilon > 0$ there exists an $N \in \{1,\ldots, T\}$ such that $v_t < \epsilon$ for all $t > N$. Similarly, for all $\epsilon > 0$ there exists a (possibly different) $N \in \{1,\ldots, T\}$ such that $e_t < \epsilon$ for all $t > N$.

However, the core problem is that we do not know how this convergence in the variance and expectation occurs. More specifically, if we knew the exact values of $e_t$ and $v_t$, we could use the following bound on the fraction of training data points producing a given $[a_1,\cdots,a_t]$ as a reject option for points that are not optimized for.  We consequently introduce the notation $[a_1,\ldots,a_T]$ where $a_t = 1- \delta_{f_t(\bm{x}),f_T(\bm{x})}$ which we call the "label disagreement (at $t$)". Note that the $a_t$ are defined with respect to a given input, while $p_t$ represent the distribution of $a_t$ over all inputs in $D$.

\begin{lemma}
\label{lem:prob_accept}
Given a datapoint $\bm{x}$,  let $\{a_1,\ldots,a_T\}$ where $a_t = 1- \delta_{f_t(\bm{x}),f_T(\bm{x})}$. Assuming not all $a_t = e_t$ then the probability $\bm{x} \in D$ is  $\leq \min_{v_t~s.t~a_t \neq e_t}  \frac{v_t}{|a_t - e_t|^2}$.
\end{lemma}
\begin{proof}
By Chebyshev's inequality we have the probability of a particular sequence $\{a_1,\ldots,a_T\}$ occurring for a training point is $\leq \frac{v_t}{|a_t - e_t|^2}$ for every $t$ (a bound on any of the individual $a_t$ occurring as that event is in the event $|p_t - e_t| \geq |a_t - e_t|$ occurs). By taking the minimum over all these upper-bounds we obtain our upper-bound.
\end{proof}

We do not guarantee Lemma~\ref{lem:prob_accept} is tight. Though we do take a minimum to make it tighter, this is a minimum over inequalities all derived from Chebyshev's inequality\footnote{One could potentially use information about the distribution of points not in $D$ to refine this bound.}. Despite this potential looseness, using the bound from Lemma~\ref{lem:prob_accept}, we can design a na\"ive selective classification protocol based on the "optimized = correct (often)" hypothesis and use the above bound on being a plausible training datapoint as our characterization of optimization; for a test input $\bm{x}$, if the upper-bound on the probability of being a datapoint in $D$ is lower than some threshold $\tau$ reject, else accept. However, the following question prevents us from readily using this method: \emph{How do $\mathbb{E}[p_{t}]$ and $\mathbb{V}[p_{t}]$ evolve during training?}

\looseness=-1
To answer this question, we propose to examine how the predictions on plausible training points evolve during training. Informally, the evolution of $\mathbb{E}[p_{t}]$ represents knowing how often we predict the final label at step $t$, while the evolution of $\mathbb{V}[p_{t}]$ represents knowing how we become more consistent as we continue training. Do note that the performance of this optimization-based approach to selective classification will depend on how unoptimized incorrect test points are. In particular, our hypothesis is that incorrect points often appear sufficiently un-optimized, yielding distinguishable patterns for $\mathbb{E}[p_{t}]$ and $\mathbb{V}[p_{t}]$ when compared to optimized points. We verify this behavior in Section~\ref{sec:emp_eval} where we discuss the distinctive label evolution patterns of explicitly optimized, correct, and incorrect datapoints.

\subsubsection{Last Disagreement Model Score For Discrete Prediction (\smax)}
\label{ssec:min_score}

Here, we propose a selective classification approach based on characterizing optimizing for a datapoint based off of Lemma~\ref{lem:prob_accept}. Recall the bound given in Lemma~\ref{lem:prob_accept} depends on expected values and variances for the $p_t$ (denoted $e_t$ and $v_t$ respectively). In Section~\ref{sec:emp_eval} we observe that $e_t$ quickly converge to $0$, and so by assuming $e_t = 0$ always\footnote{We tried removing this assumption and observed similar performance.} the frequentist bound on how likely a datapoint is a training point becomes $\min_{t~s.t~a_t = 1} \frac{v_t}{|a_t - e_t|^2}= \min_{t~s.t~a_t = 1}v_t$. Using this result for selective classification, we would impose acceptance if $\min_{t~s.t~a_t = 1}v_t \geq \tau$. Moreover, in Section~\ref{sec:emp_eval}, we further observe that $v_t$ monotonically decreases in a convex manner (after an initial burn-in phase). Hence, imposing $\min_{t~s.t~a_t = 1}v_t \geq \tau$ simply imposes a last checkpoint that can have a disagreement with the final prediction.

Based on these insights, we propose the following selective classification score: $s_{\max} = \max_{t~s.t~a_t = 1} \frac{1}{v_t}$. Note that this score directly follows from the previous discussion but flips the thresholding direction from $\min_{t~s.t~a_t = 1}v_t \geq \tau$ to $\max_{t~s.t~a_t = 1} \frac{1}{v_t} \leq \tau$ for consistency with the anomaly scoring literature~\citep{ruff2018deep}. Finally, we choose to approximate the empirical trend of $v_t$ as observed in Section~\ref{sec:emp_eval} with $v_t = 1 - t^k$ for $k \in [1,\infty)$. Based on the choice of $k$, this approximation allows us to (i) avoid explicit estimation of $v_t$ from validation data; and (ii) enables us to flexibly specify how strongly we penalize model disagreements late in training.

Hence, our first algorithm for selective classification is:
\begin{enumerate}
    \item Denote $L = f_T(\bm{x})$, i.e. the label our final model predicts.
    \item If $\exists t~s.t~a_t =1$ then compute $s_\text{max} = \max_{t~s.t~a_t = 1} \frac{1}{v_t}$ as per the notation in Section~\ref{ssec:reject_cond} (i.e $a_t = 1$ iff $f_t(x) \neq L$), else accept $\bm{x}$ with prediction $L$.
    \item If $s_\text{max} \leq \tau$ accept $\bm{x}$ with prediction $L$, else reject ($\bot$).
\end{enumerate}
Note once again, as all our candidate $\frac{1}{v_t}$ increase, the algorithm imposes a last intermediate model which can output a prediction that disagrees with the final prediction: hereafter, the algorithm must output models that consistently agree with the final prediction.

\subsubsection{Overall Disagreement Model Score (\ssum)}
\label{ssec:avg_score}

Note that the previous characterization of optimization, defined by the score \smax, could be sensitive to stochasticity in training and hence perform sub-optimally. That is, the exact time of the last disagreement, which \smax relies on, is subject to high noise across randomized training runs. In light of this potential limitation we propose the following "summation" algorithm which computes a weighted sum over training-time disagreements to get a more consistent statistic. Do note that typically to get a lower-variance statistic one would take an average, but multiplying by scalars can be replaced by correspondingly scaling the threshold we use. Hence, our proposed algorithm is:

\begin{enumerate}
     \item Denote $L = f_T(\bm{x})$, i.e. the label our final model predicts.
     \item If $\exists t~s.t~a_t =1$, compute $s_\text{sum} = \sum_{t=1}^T \frac{a_t}{v_t} $, else accept $\bm{x}$ with prediction $L$. 
     \item If $s_\text{sum} \leq \tau$ accept $\bm{x}$ with prediction $L$, else reject ($\bot$).
\end{enumerate}

Recalling our previous candidates for $v_t$, we have the \ssum places higher weight on late disagreements. This gives us a biased average of the disagreements which intuitively approximates the expected last disagreement but now is less susceptible to noise. More generally, this statistic allows us to perform selective classification by utilizing information from all the disagreements during training. In Appendix~\ref{sec:max_v_sum}, we experimentally show that \ssum leads to more robust selective classification results compared to \smax. \textbf{We remark that the sum score \ssum corresponds exactly to our score $g(\cdot)$ proposed as part of \sptd (recall Equation~\ref{eq:score} from Section~\ref{sec:method_overview}), showcasing the strong connection of our method to forging.}

\section{Why Training Dynamics Encode Reliable Uncertainty}
\label{sec:sptd_theory}
 
The preceding section introduced \sptd\ and formalized its weighted‐instability score~\eqref{eq:score}. Before turning to the empirical evaluation, we pause to explain \emph{why} late–stage checkpoint disagreement should correlate with predictive risk in the first place.  
The argument links stochastic optimization to Bayesian posterior sampling and shows how both epistemic and aleatoric factors manifest as prolonged instability during training.  
This theoretical perspective also clarifies when \sptd\ (and closely related Deep Ensembles) may fail, guiding practical deployment.

\subsection{From Optimization Noise to Predictive Risk}

Modern views of stochastic gradient descent (SGD) treat its iterates as samples from a \emph{tempered Bayesian posterior} \citep{mandt2017stochastic,zhang2019cyclical}.  
Intuitively, gradient noise enables exploration of multiple loss basins that fit the data nearly equally well.  
For an input~$\bm{x}$, persistent disagreement among late checkpoints therefore indicates that \emph{several plausible models disagree}—a hallmark of uncertainty.  
The weighted instability score $g(\bm{x})$ integrates this disagreement and thus acts as a proxy for posterior predictive variance.

\subsection{Two Complementary Sources of Uncertainty}

To make the connection precise, we distinguish two phenomena that raise predictive risk.  
First, we introduce them conceptually; the next subsection shows how both surface in checkpoint behaviour.

\begin{itemize}
  \item \textbf{Aleatoric (data) uncertainty} — irreducible label noise (Bayes error) or class overlap that no model can resolve.
  \item \textbf{Epistemic (model) uncertainty} — uncertainty about model parameters caused by finite data or limited model capacity.
\end{itemize}

Checkpoint disagreement captures \emph{both}.  
Aleatoric noise keeps SGD oscillating because successive mini-batches pull the model toward conflicting optima, while epistemic uncertainty reflects diffuse posterior mass that SGD has not collapsed into a single mode.

\subsection{Checkpoint Disagreement as Posterior Variance}

Let $Z_t(\bm{x}) = f_t(\bm{x}) - f_T(\bm{x})$ denote the prediction change between checkpoint~$t$ and the final model.  
Under the SGD-as-sampler view,
\begin{equation}
  \operatorname{Var}_{t}\bigl[Z_t(\bm{x})\bigr]
  \;\approx\;
  \operatorname{Var}_{\theta\sim p(\theta\mid D)}\bigl[f_{\theta}(\bm{x})\bigr].
\end{equation}
The instability score
$
  g(\bm{x})=\sum_{t=1}^{T} v_t\,|Z_t(\bm{x})|
$
therefore estimates a \emph{weighted tail integral} of this posterior variance, emphasizing late-training variance that has not yet collapsed.  
Deep Ensembles approximate the same quantity by Monte-Carlo over random restarts; \sptd\ does so \emph{within one single training run}.

\begin{figure}[t]
    \centering
    \includegraphics[width=\linewidth]{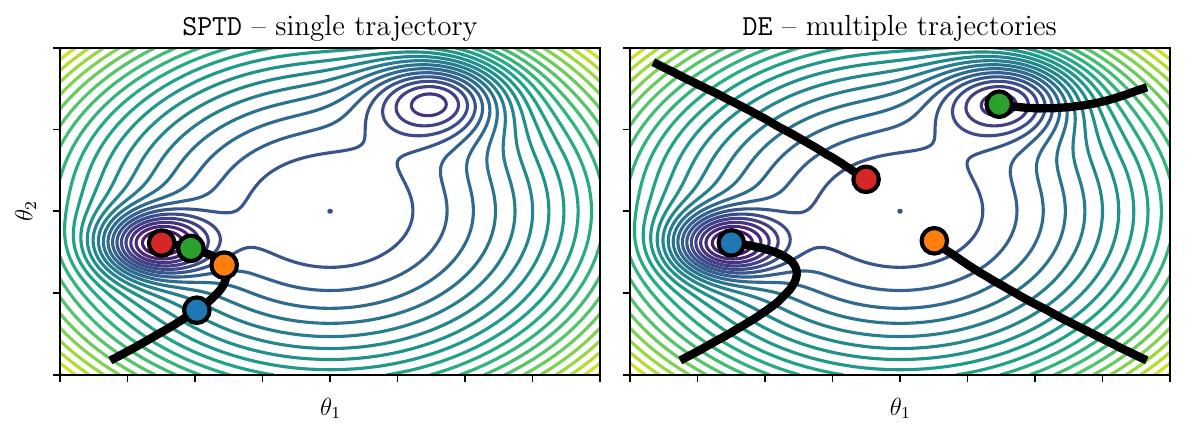}
    \caption{\textbf{Illustrative comparison of \sptd and \de on a toy 2-D loss landscape}. \textit{Left}: a single SGD trajectory (black) descending into the bottom-left well, with four late-stage checkpoints highlighted in color to capture localized instability. \textit{Right}: four independent SGD runs (black) converging to multiple distinct minima, each marked by its final checkpoint (colored) to approximate global posterior modes.
}
    \label{fig:sptd_de_traj}
\end{figure}

\subsection{Deep Ensembles as Monte-Carlo Posterior Estimators}
\label{ssec:deep_ensembles_theory}

A complementary—and historically more popular—route to posterior variance is the \emph{Deep Ensemble} (\de) paradigm \citep{lakshminarayanan2017simple}.  
Here we train $M$ models $\{f_{\theta^{(m)}}\}_{m=1}^{M}$ from \emph{independent random initializations} (or data shuffles) and treat their empirical distribution as a Monte-Carlo approximation of the Bayesian posterior:
\begin{equation}
  \hat{p}_{\textsc{de}}(\theta)
  \;=\;
  \frac1M\sum_{m=1}^{M}\delta_{\theta^{(m)}}, 
  \quad\Longrightarrow\quad
  \operatorname{Var}_{\hat{p}_{\textsc{de}}}\bigl[f_{\theta}(\bm{x})\bigr]
  \;=\;
  \frac1M\sum_{m}\Bigl\lVert
      f_{\theta^{(m)}}(\bm{x})
      - \bar{f}(\bm{x})
  \Bigr\rVert^{2},
\end{equation}
with $\bar{f}(\bm{x})$ the ensemble mean.  
This variance estimator carries the same uncertainty semantics as the checkpoint-instability score $g(\bm{x})$, but differs in \emph{where} the Monte-Carlo samples come from:

\begin{center}
\begin{tabular}{p{0.23\linewidth}p{0.34\linewidth}p{0.34\linewidth}}
\toprule
 & \textbf{Sampling mechanism} & \textbf{Typical computational cost} \\
\midrule
\textbf{Deep Ensembles} & Independent SGD runs sample \emph{between} basins. & $\mathcal{O}(M)$ training + $\mathcal{O}(M)$ inference.\\[2pt]
\textbf{Selective Prediction Training Dyanmics} & One SGD run, checkpoints sample \emph{within} a basin. & $\mathcal{O}(1)$ training\footnotemark{} + $\mathcal{O}(T)$ inference ($T\le50$).\\
\bottomrule
\end{tabular}
\end{center}
\footnotetext{Aside from negligible checkpoint storage, \sptd\ does not increase training time.}

\paragraph{When do the two (dis-)agree?}  
If the posterior is \emph{multi-modal} (several distant basins with comparable posterior mass), \de captures variance \emph{between} modes, whereas checkpoints mainly capture \emph{local} variance inside one basin. We illustrate this behavior in Figure~\ref{fig:sptd_de_traj}. It is also possible to combine both approaches (\sptdde) which yields a near-superset of posterior support. We explore this choice in our experiments in Section~\ref{sec:emp_eval}. Together, checkpoints and ensembles represent two sides of the same Bayesian coin—local versus global exploration of the posterior landscape.  Understanding their relationship clarifies why their judicious combination delivers the best empirical coverage–utility balance.

\subsection{Limitations and Practical Implications}

While the argument above justifies the use of training dynamics, several caveats deserve attention:

\begin{itemize}
  \item \textbf{Diminished model diversity.}  Extremely over-parameterized networks or aggressive learning-rate schedules can lead SGD to converge to almost identical solutions, weakening the disagreement signal.  In such cases one can (i)~retain more checkpoints, (ii)~restart training with different data shuffles, or (iii)~combine \sptd\ with lightweight ensemble members.
  \item \textbf{Curriculum or non-stationary data.}  If the data distribution changes during training, early checkpoints reflect a different task than later ones.  Fixed data ordering or replay buffers help ensure that disagreement genuinely reflects uncertainty rather than curriculum artefacts.
  \item \textbf{Calibration needs.}  The raw instability scores are \emph{rank-consistent} with risk but need post-hoc calibration (e.g.\ isotonic regression) when probabilistic coverage guarantees are required.
\end{itemize}

Having established the theoretical underpinnings and practical scope of \sptd, we next evaluate its performance across classification, regression, and forecasting tasks.

\section{Empirical Evaluation}
\label{sec:emp_eval}

We present a comprehensive empirical study demonstrating the effectiveness of \sptd across domains. Our results show that computing and thresholding the proposed weighted instability score from \sptd provides a strong score for selective classification, regression, and time series prediction.

\subsection{Classification}

\paragraph{Key Research Goals.} As part of our experiments we:
\begin{itemize}
    \item Study the accuracy/coverage trade-off with comparison to past work, showing that \sptd outperforms existing work.
    \item Present exemplary training-dynamics-derived label evolution curves for individual examples from all datasets.
    \item Examine our method's sensitivity to the checkpoint selection strategy and the weighting parameter~$k$.
    \item Evaluate the detection performance of out-of-distribution and adversarial examples, showing that \sptd can be applied beyond the i.i.d. assumption of selective prediction.
    \item Provide a detailed cost vs performance tradeoff of \sptd and competing SP methods.
    \item Analyze distributional training dynamics patterns of both correct and incorrect data points, the separation of which enables performative selective classification.
\end{itemize} 

\paragraph{Datasets \& Training.} We evaluate \sptd on vision benchmarks that are common in the selective classification literature: CIFAR-10/CIFAR-100~\citep{krizhevsky2009learning}, StanfordCars~\citep{krause20133d}, and Food101~\citep{bossard14}. For each dataset, we train a deep neural network following the ResNet-18 architecture~\citep{he2016deep} and checkpoint each model after processing $50$ mini-batches of size $128$. All models are trained over $200$ epochs ($400$ epochs for StanfordCars) using the SGD optimizer with an initial learning rate of $10^{-2}$, momentum $0.9$, and weight decay $10^{-4}$. Across all datasets, we decay the learning rate by a factor of $0.5$ in $25$-epoch intervals.

\paragraph{Baselines.} We compare our method (\sptd) to common SC techniques previously introduced in Section~\ref{sec:background_sptd}: Softmax Response (\sr) and Self-Adaptive Training (\sat). Based on recent insights from~\cite{feng2023towards}, we (i) train \sat with additional entropy regularization\footnote{This entropy regularization step is designed to encourage the model to be more confident in its predictions.}; and (ii) derive \sat's score by applying Softmax Response (\sr) to the underlying classifier (instead of thresholding the abstention class). We refer to this method as \satersr. We do not include results for SelectiveNet, Deep Gamblers, or Monte-Carlo Dropout as previous works~\citep{huang2020self,feng2023towards} have shown that \fixed{\satersr} strictly dominates these methods. In contrast to recent SC works, we do however include results \fixed{with} Deep Ensembles (\de)~\citep{balaji2017uncertainty}, a relevant baseline from the uncertainty quantification literature. Our hyper-parameter tuning procedure is documented in Appendix~\ref{app:baseline_hyperparams}.

\paragraph{Accuracy/Coverage Trade-off.} Consistent with standard evaluation schemes for selective classification, our main experimental results examine the accuracy/coverage trade-off of \sptd. We present our performance results with comparison to past work in Table~\ref{tab:target_cov} where we demonstrate \sptd's effectiveness on CIFAR-10, CIFAR-100, StanfordCars, and Food101. We document the results obtained by \sptd, \sat, \sr, and \de across the full coverage spectrum. We see that \sptd outperforms both \sat and \sr and performs similarly as \de. To further boost performance across the accuracy/coverage spectrum, we combine \sptd and \de by applying \sptd on each ensemble member from \de and then average their scores. More concretely, we estimate $\sptdde = \frac{1}{m}\sum_{m=1}^M \sptd_m$ where $\sptd_m$ computes $g$ on each ensemble member $m\in[M]$. This combination leads to new state-of-the-art selective classification performance and showcases that \sptd can be flexibly applied on top of established training pipelines. 
\newlyadded{
Further evidence towards this flexibility is provided in Appendix~\ref{sec:sptd_on_sat} where we show that applying \sptd on top of \sat also improves performance.
}

\begin{figure}[t]
\vspace{-5pt}
\centering

\begin{subfigure}[b]{0.23\linewidth}
  \centering
  \includegraphics[width=\linewidth]{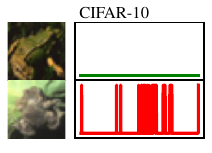}
  \caption{CIFAR-10}
\end{subfigure}
\hfill
\begin{subfigure}[b]{0.23\linewidth}
  \centering
  \includegraphics[width=\linewidth]{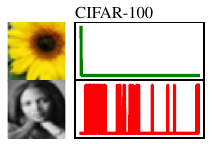}
  \caption{CIFAR-100}
\end{subfigure}
\hfill
\begin{subfigure}[b]{0.23\linewidth}
  \centering
  \includegraphics[width=\linewidth]{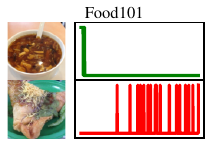}
  \caption{Food-101}
\end{subfigure}
\hfill
\begin{subfigure}[b]{0.23\linewidth}
  \centering
  \includegraphics[width=\linewidth]{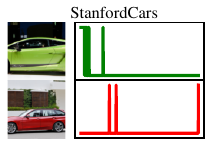}
  \caption{Stanford Cars}
\end{subfigure}

\caption[\textbf{Most characteristic examples across datasets.}]{\textbf{Most characteristic examples across datasets.} For each dataset, we show the samples with the most stable and most unstable (dis-)agreement with the final label along with their corresponding $a_t$ indicator function. Correct points are predominantly characterized by disagreements early in training while incorrect points change their class label throughout (but importantly close to the end of) training. We provide additional examples from all datasets in Figure~\ref{fig:indiv_ex_ext}.}
\label{fig:indiv_ex}
\end{figure}

\paragraph{Individual Evolution Plots.} To analyze the effectiveness of our disagreement metric proposed in Section~\ref{sec:method}, we examine the evolution curves of our indicator variable $a_t$ for individual datapoints in Figure~\ref{fig:indiv_ex}. In particular, for each dataset, we present the most stable and the most unstable data points from the test sets and plot the associated label disagreement metric $a_t$ over all checkpoints. We observe that easy-to-classify examples only show a small degree of oscillation while harder examples show a higher frequency of oscillations, especially towards the end of training. This result matches our intuition: our model should produce correct decisions on data points whose prediction is mostly constant throughout training and should reject data points for which intermediate models predict inconsistently. Moreover, as depicted in Figure~\ref{fig:scores}, we also show that our score $g(\cdot)$ yields distinct distributional patterns for both correctly and incorrectly classified points. This separation enables strong coverage/accuracy trade-offs via our thresholding procedure.

\begin{figure*}[t]
  \centering
  \includegraphics[width=\linewidth]{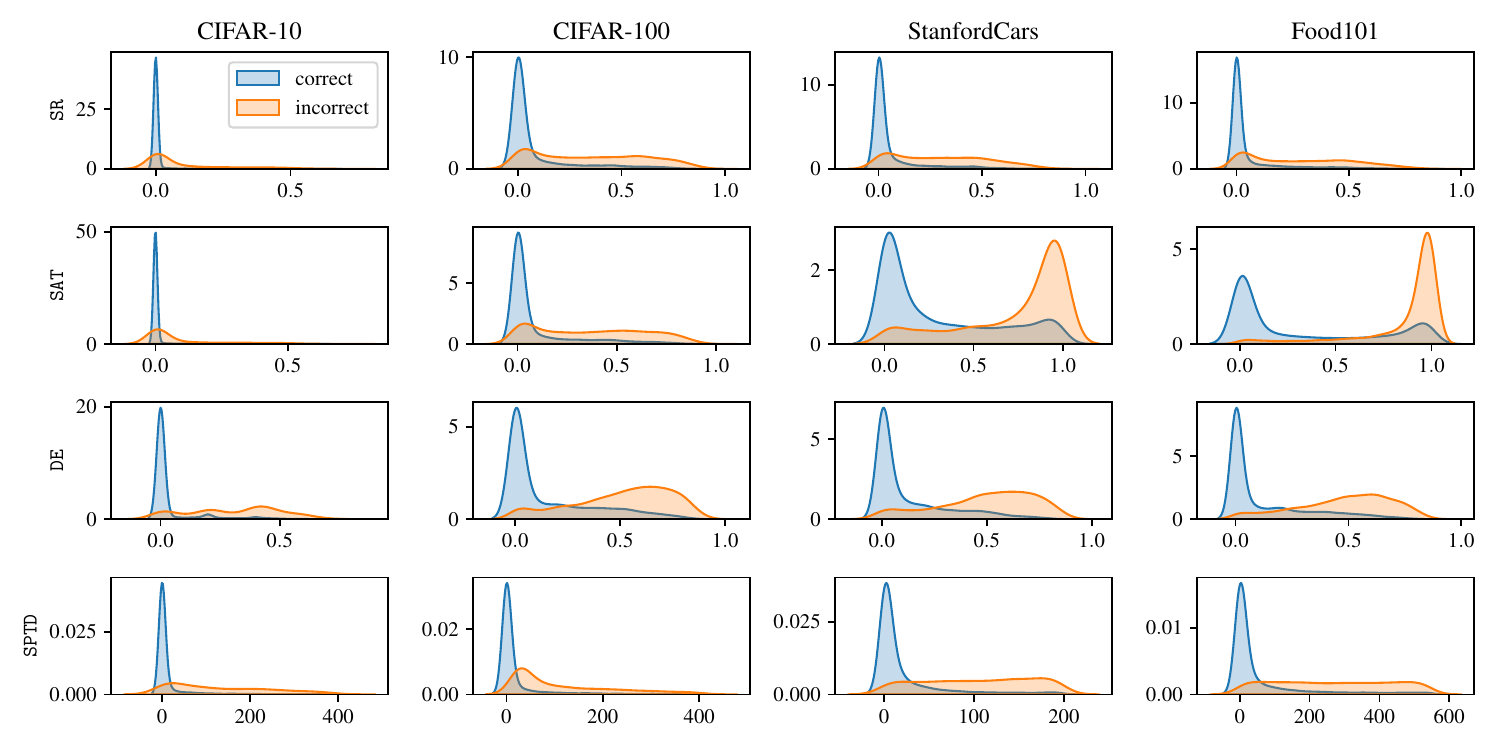}
\caption[Distribution of $g$ for different datasets and selective classification methods.]{\textbf{Distribution of $g$ for different datasets and selective classification methods.}  Since all methods are designed to address the selective prediction problem, they all manage to separate correct from incorrect points (albeit at varying success rates). We see that \sptd spreads the scores for incorrect points over a wide range with little overlap. We observe that for \sr, incorrect and correct points both have their mode at approximately the same location which hinders performative selective classification. Although \sat and \de show larger bumps at larger score ranges, the separation with correct points is weaker as correct points also result in higher scores more often than for \sptd. }
\label{fig:scores}
\end{figure*}

\begin{figure*}[t]
  \centering
  \includegraphics[width=\linewidth]{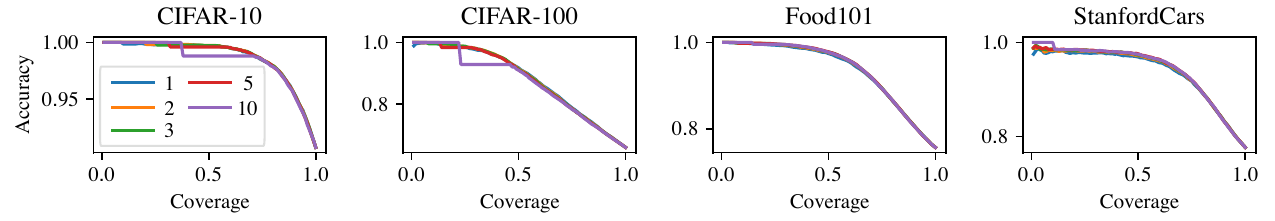}
\caption[Coverage/error trade-off of \texttt{SPTD} for varying checkpoint weighting $k$ as used in $v_t$.]{\textbf{Coverage/error trade-off of \texttt{SPTD} for varying checkpoint weighting $k$ as used in $v_t$.} We observe strong performance for $k \in [1,3]$ across datasets.
}
\label{fig:weighting}
\end{figure*}

\paragraph{Checkpoint Weighting Sensitivity.} One important hyper-parameter of our method is the weighting of intermediate predictions. Recall from Section~\ref{sec:method} that \sptd approximates the expected stability for correctly classified points via a weighting function $v_t = (\frac{t}{T})^k$. In Figure~\ref{fig:weighting} in the Appendix, we observe  that \sptd is robust to the choice of $k$ and that \fixed{$k \in [1,3]$} performs best. At the same time, we find that increasing $k$ too much leads to a decrease in accuracy at medium coverage levels. This result emphasizes that (i)~large parts of the training process contain valuable signals for selective classification; and that (ii)~early label disagreements arising at the start of optimization should be de-emphasized by our method.

\begin{figure*}[t]
  \centering
  \includegraphics[width=\linewidth]{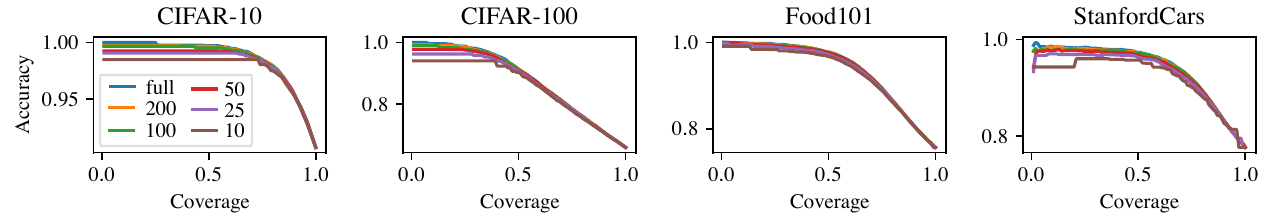}
\caption[Coverage/error trade-off of \texttt{SPTD} for varying checkpoint counts.]{\textbf{Coverage/error trade-off of \texttt{SPTD} for varying checkpoint counts}. \sptd delivers consistent performance independent of the checkpointing resolution at high coverage. At low coverage, a more detailed characterization of training dynamics helps.
}
\label{fig:resolution}
\end{figure*}

\paragraph{Checkpoint Selection Strategy.} The second important hyper-parameter of our method is the checkpoint selection strategy. In particular, to reduce computational cost, we study the sensitivity of \sptd with respect to the checkpointing resolution in Figure~\ref{fig:resolution}. Our experiments demonstrate favorable coverage/error trade-offs between $25$ and $50$ checkpoints when considering the full coverage spectrum. However, when considering the high coverage regime in particular (which is what most selective prediction works focus on), even sub-sampling $10$ intermediate models is sufficient for SOTA selective classification. Hence, with only observing the training stage, our method's computational overhead reduces to only $10$ forward passes at test time when the goal is to reject at most $30\%-50\%$ of incoming data points. In contrast, \de requires to first train $E$ models (with $E=10$ being a typical and also our particular choice for \de) and perform inference on these $E$ models at test time. Further increasing the checkpointing resolution does offer increasingly diminishing returns but also leads to improved accuracy-coverage trade-offs, especially at low coverage.

\begin{figure*}[t]
  \centering
  \includegraphics[width=\linewidth]{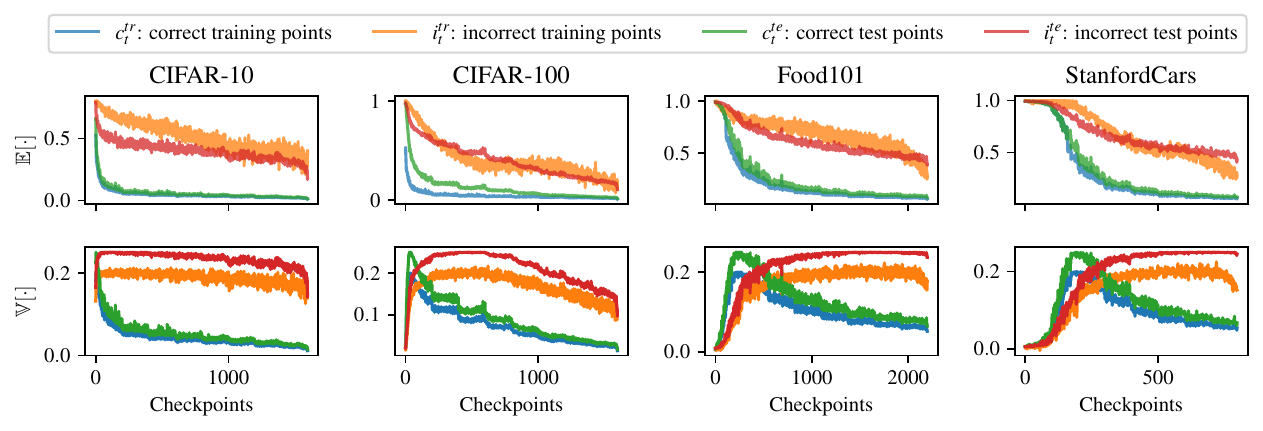}
\caption[Monitoring expectations and variances for correct/incorrect training and test points.]{\textbf{Monitoring expectations $\mathbb{E}[\cdot]$ and variances $\mathbb{V}[\cdot]$ for correct/incorrect training and test points}. We observe that correctly classified points (cold colors) have both their expectations and variances quickly decreasing to 0 as training progresses. Incorrectly classified points (warm colors) both exhibit large expectations and variances and stay elevated over large periods.}
\label{fig:exp_var_trends}
\end{figure*}

\paragraph{Examining the Convergence Behavior of Training and Test Points.}

The effectiveness of \sptd relies on our hypothesis that correctly classified points and incorrectly classified points exhibit distinct training dynamics. We verify this hypothesis in Figure~\ref{fig:exp_var_trends} where we examine the convergence behavior of the disagreement distributions of correct ($c^\text{tr}_t$) / incorrect ($i^\text{tr}_t$) training and correct ($c^\text{te}_t$) / incorrect ($i^\text{te}_t$) test points. We observe that the expected disagreement for both correctly classified training $c^\text{tr}_t$ and test points $c^\text{te}_t$ points converge to $0$ over the course of training. The speed of convergence is subject to the difficulty of the optimization problem with more challenging datasets exhibiting slower convergence in predicted label disagreement. We also see that the variances follow an analogous decreasing trend. This indicates that correctly classified points converge to the final label quickly and fast convergence is strongly indicative of correctness. Furthermore, the overlap suggests that correct test points are more likely to be forgeable as their dynamics look indistinguishable to correct training points (recall Section~\ref{sec:forging} on the connection between our method and forging). In contrast, incorrectly classified points $i^\text{tr}_t$ and $i^\text{te}_t$ show significantly larger mean and variance levels. This clear separation in distributional evolution patterns across correct and incorrect points leads to strong selective prediction performance in our \sptd framework.

\label{sec:ts_exp}
\begin{figure*}[t]
    \centering
    \includegraphics[width=\linewidth]{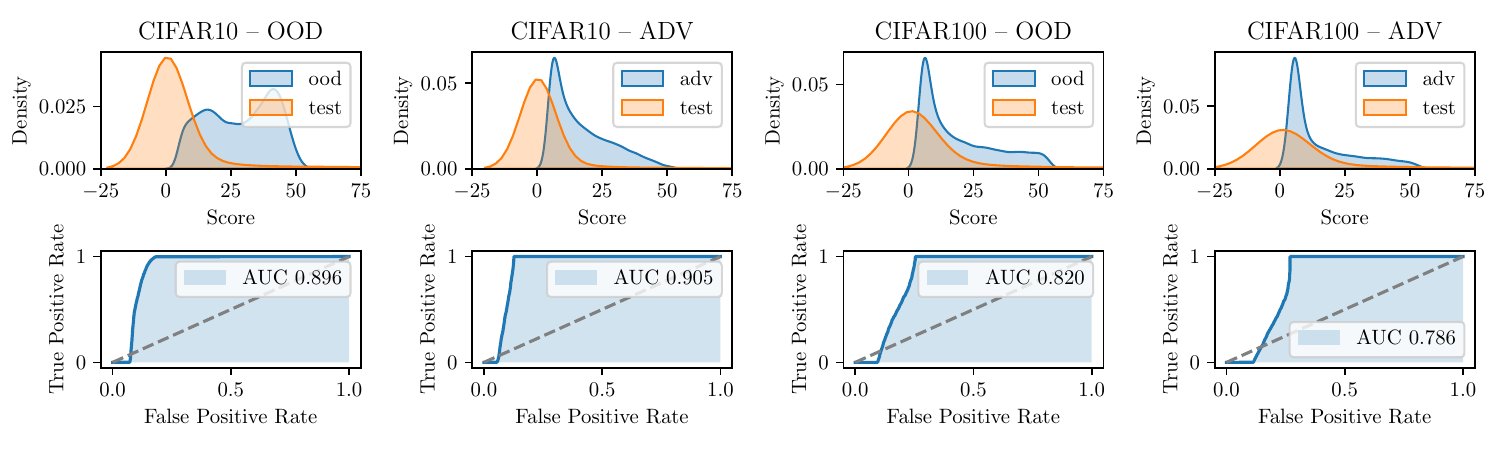}
    \caption[Performance of \sptd on out-of-distribution (OOD) and adversarial sample detection.]{\textbf{Performance of \sptd on out-of-distribution (OOD) and adversarial sample detection}. The first row shows the score distribution of the in-distribution CIFAR-10/100 test set vs the SVHN OOD test set or a set consisting of adversarial samples generated via a PGD attack in the final model. The second row shows the effectiveness of a thresholding mechanism by computing the area under the ROC curve. Our score enables separation of anomalous data points from in-distribution test points.}
    \label{fig:adv_ood}
\end{figure*}

\paragraph{Detection of Out-of-Distribution and Adversarial Examples.}

Out-of-distribution (OOD) and adversarial example detection are important disciplines in trustworthy ML related to selective prediction. We therefore provide preliminary evidence in Figure~\ref{fig:adv_ood} that our method can be used for detecting OOD and adversarial examples. While these results are encouraging, we remark that adversarial and OOD samples are less well defined as incorrect data points and can come in a variety of different flavors (\ie various kinds of attacks or various degrees of OOD-ness). As such, we strongly believe that future work is needed to determine whether a training-dynamics-based approach to selective prediction can be reliably used for OOD and adversarial sample identification. 

\paragraph{Cost vs Performance Tradeoff.} 

In Table~\ref{tab:cost}, we report both the time and space complexities for all SC methods at training and test time along with their selective classification performance as per our results in Table~\ref{tab:target_cov} and Figure~\ref{fig:resolution}. We denote with $E$ the number of  \de models and with $T$ the number of \sptd checkpoints. Although \sr and \sat are the cheapest methods to run, they also perform the poorest at SC. \sptd is significantly cheaper to train than \de and achieves competitive performance at $T \approx E$. Although \sptdde is the most expensive model, it also provides the strongest performance.

\begin{table}[ht]
\tabcolsep=0.12cm
\small
    \centering 
     \begin{tabular}{cccccc} 
     \toprule
     Method & Train Time & Train Space & Inf Time & Inf Space & Rank \\ 
     \midrule
     \sr & $O(1)$ & $O(1)$ & $O(1)$ & $O(1)$ & 5 \\
     \sat & $O(1)$ & $O(1)$ & $O(1)$ & $O(1)$ & 4 \\
     \de & $O(E)$ & $O(E)$ & $O(E)$ & $O(E)$ & =2 \\
     \sptd & $O(1)$ & $O(T)$ & $O(T)$ & $O(T)$ & =2 \\
     \sptdde & $O(E)$ & $O(ET)$ & $O(ET)$ & $O(ET)$ & 1\\ 
     \bottomrule
    \end{tabular}
    \caption[Cost vs performance tradeoff in terms of training time/space, inference time/space and the performance rank.]{\textbf{Cost vs performance tradeoff in terms of training time/space, inference time/space and the performance rank.} \sptd is comparable in performance (at $T \approx E$) and cheaper to train than \de. \sptdde is the most expensive model, but delivers the best performance across datasets.}
    \label{tab:cost}
\end{table}

\begin{figure*}[t]
    \centering
    \includegraphics[width=0.97\linewidth]{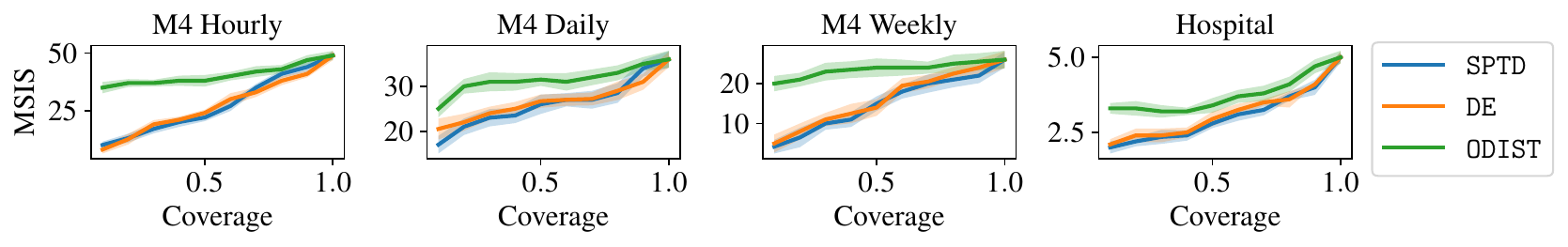}
    \caption[MSIS/coverage trade-off across various time series prediction datasets.]{\textbf{MSIS/coverage trade-off across various time series prediction datasets}. \sptd offers comparable performance to \de but provides improved results at low coverage.}
    \label{fig:ts}
\end{figure*}

\subsection{Regression Experiments}
\label{sec:regr_exp}

\begin{figure*}[t]
    \centering
    \includegraphics[width=0.97\linewidth]{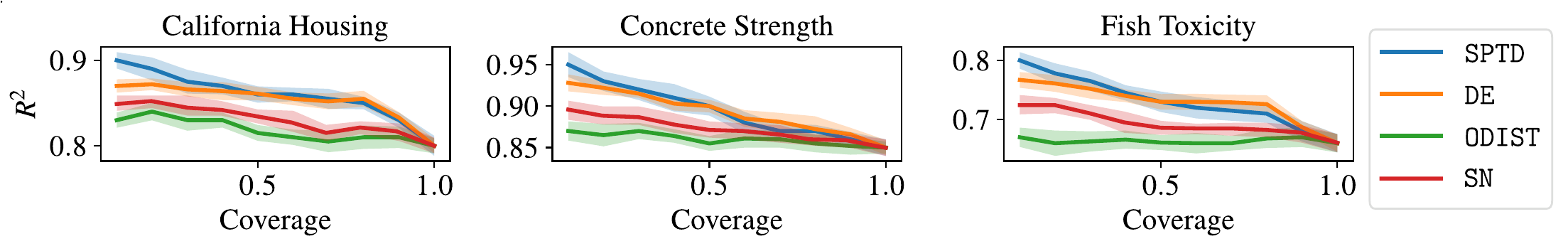}
    \caption[$R^2$/coverage trade-off across various regression datasets.]{\textbf{$R^2$/coverage trade-off across various regression datasets}. \sptd offers comparable performance to \de but provides improved results at low coverage.}
    \label{fig:regr}
\end{figure*}

\paragraph{Datasets.} Our experimental suite for regression considers the following datasets: %\anvith{any reference for why these datasets?}: the 
California housing dataset~\citep{pace1997sparse} ($N=20640$, $D=8$), the concrete strength dataset~\citep{misc_concrete_compressive_strength_165} ($N=1030$, $D=9$), and the fish toxicity dataset~\citep{misc_qsar_fish_toxicity_504} ($N=546$, $D=9$). 

\paragraph{Model Setup \& Baselines.} We split all datasets into $80\%$ training and $20\%$ test sets after a random shuffle. Then, we train a fully connected neural network with layer dimensionalities $D \rightarrow 10 \rightarrow 7 \rightarrow 4 \rightarrow 1$. Optimization is performed using full-batch gradient descent using the Adam optimizer with learning rate $10^{-2}$ over $200$ epochs and weight decay $10^{-2}$. We consider the following baseline methods for rejecting input samples: (i) approximating the predictive variance using deep ensembles (\de) \citep{balaji2017uncertainty, zaoui2020regression}; (ii) SelectiveNet (\sn) which explicitly optimizes utility given a desired coverage constraint; and (iii) training the model with a Gaussian parametric output distribution (\odist) via maximum likelihood maximization \citep{alexandrov2019gluonts}.

\paragraph{Main Results.} We document our results in Figure~\ref{fig:regr}. We see that the \odist only delivers subpar results (likely due to mis-calibration) and does not provide a meaningful signal for selective prediction. On the other hand, \de and \sptd perform comparably with \sptd outperforming \de at low coverage. We stress again that \sptd's training cost is significantly cheaper than \de's while matching the inference-time cost when sub-sampling a reduced set of checkpoints. 

\subsection{Time Series Experiments}
\label{sec:ts_exp}

\paragraph{Datasets.} As part of our time series experiments, we mainly consider the M4 forecasting competition dataset~\citep{makridakis2020m4} which contains time series aggregated at various time intervals (\eg hourly). In addition, we also provide experimentation on the Hospital dataset~\citep{hyndman2015expsmooth}.

\paragraph{Models \& Setup.} Our experimentation is carried out using the GluonTS time series framework~\citep{alexandrov2019gluonts} and the DeepAR model \citep{salinas2020deepar}, a recurrent neural network designed for time series forecasting. We train all models over 200 epochs and evaluate performance using the mean scaled interval score (MSIS) performance metric~\citep{makridakis2020m4}. Our baselines correspond to the same as presented for regression in Section~\ref{sec:regr_exp}: deep ensembles (\de), and output parameterization using a Student-t distribution (\odist).

\paragraph{Main Results.} Our time series results are shown in Figure~\ref{fig:ts} and are consistent with our results for regression: \odist does not provide a meaningful signal for selective prediction while \sptd and \de perform similarly well. \sptd further improves results over \de at low converge.

\begin{table*}[h!]
\fontsize{7.5}{10}\selectfont
\tabcolsep=0.2cm
    \centering {
    \begin{tabular}{ccccccc}
\toprule
&  Coverage &       \sr &       \fixed{\satersr} &      \de &      \sptd &        \sptdde \\
\midrule
 \multirow{10}{*}{\rotatebox[origin=c]{90}{\textit{CIFAR-10}}} &       100 &  \underline{\bfseries 92.9 (±0.0)} & \underline{\bfseries 92.9 (±0.0)} & \bfseries 92.9 (±0.0) & \underline{\bfseries 92.9 (±0.0)} & \bfseries 92.9 (±0.1) \\
  &        90 &  \underline{96.4 (±0.1)} &  96.3 (±0.1) &  \bfseries 96.8 (±0.1) &  \underline{96.5 (±0.0)} &  \bfseries 96.7 (±0.1) \\
  &        80 &  98.1 (±0.1) &  98.1 (±0.1) &  \bfseries 98.7 (±0.0) &  \underline{98.4 (±0.1)} &  \bfseries 98.8 (±0.1) \\
  &        70 &  98.6 (±0.2) &  99.0 (±0.1) &  \bfseries 99.4 (±0.1) &  \underline{99.2 (±0.0)} &  \bfseries 99.5 (±0.0) \\
  &        60 &  98.7 (±0.1) &  99.4 (±0.0) &  99.6 (±0.1) &  \underline{\bfseries 99.6 (±0.2)} &  \bfseries 99.8 (±0.0) \\
  &        50 &  98.6 (±0.2) &  \underline{99.7 (±0.1)} &  99.7 (±0.1) &  \underline{99.8 (±0.0)} &  \bfseries 99.9 (±0.0) \\
  &        40 &  98.7 (±0.0) &  \underline{99.7 (±0.0)} &  99.8 (±0.0) &  \underline{99.8 (±0.1)} & \bfseries 100.0 (±0.0) \\
  &        30 &  98.5 (±0.0) &  \underline{99.8 (±0.0)} &  99.8 (±0.0) &  \underline{99.8 (±0.1)} & \bfseries 100.0 (±0.0) \\
  &        20 &  98.5 (±0.1) &  \underline{99.8 (±0.1)} &  99.8 (±0.0) & \underline{\bfseries 100.0 (±0.0)} & \bfseries 100.0 (±0.0) \\
  &        10 &  98.7 (±0.0) &  99.8 (±0.1) &  99.8 (±0.1) & \underline{\bfseries 100.0 (±0.0)} & \bfseries 100.0 (±0.0) \\
 \midrule
  \multirow{10}{*}{\rotatebox[origin=c]{90}{\textit{CIFAR-100}}}  &       100 &  \underline{\bfseries 75.1 (±0.0)} &  \underline{\bfseries 75.1 (±0.0)} &  \bfseries 75.1 (±0.0) &  \underline{\bfseries 75.1 (±0.0)} &  \bfseries 75.1 (±0.0) \\
& 90 & 78.2 (± 0.1) & 78.9 (± 0.1) & 80.2 (± 0.0) & \underline{80.4 (± 0.1)} & \bfseries 81.1 (± 0.1) \\
& 80 & 82.1 (± 0.0) & 82.9 (± 0.0) & 84.7 (± 0.1) & \underline{84.6 (± 0.1)} & \bfseries 85.0 (± 0.2) \\
& 70 & 86.4 (± 0.1) & 87.2 (± 0.1) & 88.6 (± 0.1) & \underline{\textbf{88.7 (± 0.0)}} & \bfseries 88.8 (± 0.1) \\
& 60 & 90.0 (± 0.0) & 90.3 (± 0.2) & 90.2 (± 0.2) & \underline{90.1 (± 0.0)} & \bfseries 90.4 (± 0.1) \\
& 50 & 92.9 (± 0.1) & 93.3 (± 0.0) & 94.8 (± 0.0) & \underline{94.6 (± 0.0)} & \bfseries 94.9 (± 0.0) \\
& 40 & 95.1 (± 0.0) & 95.2 (± 0.1) & \textbf{96.8 (± 0.1)} & \underline{\textbf{96.9 (± 0.1)}} & \bfseries 96.9 (± 0.0) \\
& 30 & 97.2 (± 0.2) & 97.5 (± 0.0) & \textbf{98.4 (± 0.1)} & \underline{\textbf{98.4 (± 0.1)}} & \bfseries 98.5 (± 0.0) \\
& 20 & 97.8 (± 0.1) & 98.3 (± 0.1) & \textbf{99.0 (± 0.0)} & \underline{98.8 (± 0.2)} & \bfseries 99.2 (± 0.1) \\
& 10 & 98.1 (± 0.0) & 98.8 (± 0.1) & 99.2 (± 0.1) & \underline{\textbf{99.4 (± 0.1)}} & \bfseries 99.6 (± 0.1) \\
\midrule

    \multirow{10}{*}{\rotatebox[origin=c]{90}{\textit{Food101}}} &       100 &  \underline{\bfseries 81.1 (±0.0)} &  \underline{\bfseries 81.1 (±0.0)} & \bfseries  81.1 (±0.0) & \underline{\bfseries 81.1 (±0.0)} & \bfseries 81.1 (±0.0) \\
     &        90 &  85.3 (±0.1) &  85.5 (±0.2) &  86.2 (±0.1) &  \underline{85.7 (±0.0)} &  \bfseries 86.7 (±0.0) \\
     &        80 &  87.1 (±0.0) &  89.5 (±0.0) &  90.3 (±0.0) &  \underline{89.9 (±0.0)} &  \bfseries 91.3 (±0.1) \\
     &        70 &  92.1 (±0.1) &  92.8 (±0.1) &  \bfseries 94.5 (±0.1) &  \underline{93.7 (±0.0)} &  \bfseries 94.6 (±0.0) \\
     &        60 &  95.2 (±0.1) &  95.5 (±0.1) &  \bfseries 97.0 (±0.0) &  \underline{\bfseries 97.0 (±0.0)} &  \bfseries 97.0 (±0.0) \\
     &        50 &  97.3 (±0.1) &  97.5 (±0.0) &  98.2 (±0.0) &  \underline{\bfseries98.3 (±0.2)} &  \bfseries 98.5 (±0.0) \\
     &        40 &  98.7 (±0.0) &  98.7 (±0.2) &  \bfseries 99.1 (±0.0) &  \underline{99.1 (±0.1)} &  \bfseries 99.2 (±0.1) \\
     &        30 &  99.5 (±0.0) &  99.7 (±0.2) &  99.2 (±0.0) &  \underline{99.6 (±0.0)} &  \bfseries 99.7 (±0.0) \\
     &        20 &  99.7 (±0.1) &  99.7 (±0.2) &  \bfseries 99.9 (±0.1) &  \underline{\bfseries 99.8 (±0.0)} &  \bfseries 99.9 (±0.1) \\
     &        10 &  99.8 (±0.0) &  99.8 (±0.1) &  \bfseries 99.9 (±0.1) &  \underline{\bfseries 99.9 (±0.1)} &  \bfseries 99.9 (±0.1) \\
  \midrule
    \multirow{10}{*}{\rotatebox[origin=c]{90}{\textit{StanfordCars}}} &       100 & \bfseries \underline{77.6 (±0.0)} & \underline{\bfseries 77.6 (±0.0)} & \bfseries 77.6 (±0.0) & \underline{\bfseries 77.6 (±0.0)} & \bfseries 77.6 (±0.0) \\
     &        90 &  83.0 (±0.1) &  83.0 (±0.2) &  \bfseries 83.7 (±0.1) &  \underline{83.3 (±0.1)} &  \bfseries 83.7 (±0.2) \\
     &        80 &  87.6 (±0.0) &  88.0 (±0.1) &  88.7 (±0.1) &  \underline{\bfseries 89.3 (±0.0)} &  \bfseries 89.7 (±0.0) \\
     &        70 &  90.8 (±0.0) &  92.2 (±0.1) &  92.4 (±0.1) &  \underline{\bfseries 93.6 (±0.0)} &  93.4 (±0.1) \\
     &        60 &  93.5 (±0.1) &  95.2 (±0.1) &  95.3 (±0.0) &  \underline{\bfseries 96.2 (±0.0)} &  \bfseries 96.3 (±0.0) \\
     &        50 &  95.3 (±0.0) &  \underline{96.9 (±0.2)} &  96.4 (±0.1) &  \underline{\bfseries 97.0 (±0.1)} &  \bfseries 97.1 (±0.3) \\
     &        40 &  96.8 (±0.0) &  \underline{97.8 (±0.0)} &  \bfseries 97.8 (±0.2) &  \underline{\bfseries 97.8 (±0.1)} &  \bfseries 97.8 (±0.0) \\
     &        30 &  97.5 (±0.1) &  \underline{98.2 (±0.2)} &  \bfseries 98.6 (±0.0) &  \underline{98.2 (±0.2)} &  \bfseries 98.9 (±0.0) \\
     &        20 &  98.1 (±0.0) &  \underline{98.4 (±0.1)} &  \bfseries 98.9 (±0.2) &  \underline{98.6 (±0.0)} &  \bfseries 99.0 (±0.0) \\
     &        10 &  98.2 (±0.1) &  \underline{98.7 (±0.1)} &  \bfseries 99.5 (±0.1) &  \underline{98.5 (±0.1)} &  \bfseries 99.5 (±0.0) \\
\bottomrule
\end{tabular}
\caption[Selective accuracy achieved across coverage levels]{\textbf{Selective accuracy achieved across coverage levels}. We find that \texttt{SPTD}-based methods outperform current SOTA error rates across multiple datasets with full-coverage accuracy alignment. Numbers are reported with mean values and standard deviation computed over 5 random runs. \textbf{Bold} numbers are best results at a given coverage level across all methods and \underline{underlined} numbers are best results for methods relying on a single training run only. Datasets are consistent with~\cite{feng2023towards}.}
    \label{tab:target_cov}
    }
\end{table*} 

\section{Conclusion}

In this work we have proposed \sptd, a selective prediction technique that relies on measuring prediction instability of test points over intermediate model states obtained during training. Our method offers several advantages over previous works. In particular (i) it can be applied to all existing models whose checkpoints were recorded (hence the potential for immediate impact); (ii) it is composable with existing selective prediction techniques; (iii) it can be readily applied to both discrete and real-valued prediction problems; and (iv) it is more computationally efficient than competing ensembling-based approaches. We verified the performance of \sptd using an extensive empirical evaluation, leading to new state-of-the-art performance in the field. Beyond our work, we expect training dynamics information to be useful for identifying and mitigating other open problems in trustworthy machine learning such as (un)fairness, privacy, and model interpretability.
    \chapter{Training Private Models That Know What They Don't Know}
\label{ch:sptd_dp}

\begin{paperref}
\normalfont
The contents of this chapter consist of research and results taken from: \citet{rabanser2023training}: \emph{\bibentry{rabanser2023training}}
\end{paperref}

\section*{Summary}

Training reliable deep learning models which avoid making overconfident but incorrect predictions is a longstanding challenge. This challenge is further exacerbated when learning has to be differentially private: protection provided to sensitive data comes at the price of injecting additional randomness into the learning process. In this work, we conduct a thorough empirical investigation of selective classifiers---that can abstain under uncertainty---under a differential privacy constraint. We find that some popular selective prediction approaches are ineffective in a differentially private setting because they increase the risk of privacy leakage. At the same time, we identify that a recent approach that only uses checkpoints produced by an off-the-shelf private learning algorithm stands out as particularly suitable under DP. Further, we show that differential privacy does not just harm utility but also degrades selective classification performance. To analyze this effect across privacy levels, we propose a novel evaluation mechanism which isolates selective prediction performance across model utility levels at full coverage. Our experimental results show that recovering the performance level attainable by non-private models is possible but comes at a considerable coverage cost as the privacy budget decreases.

\section{Introduction}
\label{sec:intro}

State of the art machine learning (ML) models are gaining rapid adoption in high-stakes application scenarios such as healthcare~\citep{challen2019artificial, mozannar2020consistent}, finance~\citep{vijh2020stock}, self-driving~\citep{ghodsi2021generating}, and law~\citep{vieira2021understanding}. However, major challenges remain to be addressed to ensure trustworthy usage of these models. One major concern in sensitive applications is to protect the privacy of the individuals whose data an ML model was trained on. To prevent privacy attacks on ML models, $(\varepsilon, \delta)$ differential privacy (DP)~\citep{dwork2014algorithmic} has emerged as the de-facto standard with widespread usage in both academic and industrial applications. 

While the introduction of DP successfully protects against privacy attacks, it also limits model utility in practice. For instance, the canonical algorithm for training models with DP, DP-SGD~\citep{abadi2016deep}, clips the per-sample gradient norm and adds carefully calibrated noise. These training-time adaptations frequently lead to degradation in predictive performance in practice. This is especially a problem for datasets containing underrepresented subgroups whose accuracy has been shown to degrade with stronger DP guarantees~\citep{bagdasaryan2019differential}. Faced with this challenge, it is therefore of vital importance to detect samples on which a DP model would predict incorrectly.

One popular technique used to detect inputs that the model would misclassify with high probability is given by the selective classification (SC) framework~\citep{geifman2017selective}: by relying on a model's uncertainty in the correct prediction for an incoming data point, the point is either accepted with the underlying model's prediction, or rejected and potentially flagged for additional downstream evaluation. Hence, SC introduces a trade-off between coverage over the test set (\ie the goal of accepting as many data points as possible) and predictive performance on the accepted points (\ie the goal of making as few mistakes as possible). Although identification of misclassified points seems of particular importance under differential privacy, the application of selective prediction to private models is under-explored to date. While lots of past works have studied privacy or selective prediction in isolation, best practices for selective prediction under differential privacy have not been established yet. In this work, we are the first to answer the question of whether selective classification can be used to recover the accuracy lost by applying a DP algorithm (at the expense of data coverage).

To analyze the interplay between differential privacy and selective classification, we first show that not all approaches towards selective classification are easily applicable under a differential privacy constraint. In particular, approaches relying on multiple passes over the entire dataset to obtain the full SP performance profile~\citep{geifman2019selectivenet, lakshminarayanan2017simple} suffer significantly under differential privacy. This is due to the fact that the worst-case privacy leakage increases with each analysis made of the dataset which means that each training run forces more privacy budget to be expended. On the other hand, our analysis shows that an SC approach based on harnessing intermediate model checkpoints~\citep{rabanser2022selective} yields the most competitive results. Notably, this approach performs especially well under stringent privacy constraints ($\varepsilon = 1$). 

Next, we find that differential privacy has a more direct negative impact on selective classification that goes beyond the characteristic drop in utility. Based on a simple synthetic experiment, we observe a clear correlation between differential privacy strength and wrongful overconfidence. In particular, even when multiple private models trained with different $\varepsilon$ levels offer the same utility on the underlying prediction task, we show that stronger levels of DP lead to significantly decreased confidence in the correct class. This reduction effectively prevents performant selective classification. 

Motivated by this observation, we realize a need to disentangle selective classification from model utility; SC ability and accuracy should be considered independently. To that end, we show that the evaluation metric typically used for non-private selective classification is inappropriate for comparing SC approaches across multiple privacy levels. The in-applicability of the current method stems from the fact that full-coverage accuracy alignment is required across experiments. In particular, accuracy-aligning multiple DP models via early-stopping frequently leads to a more stringent than targeted privacy level, rendering the intended comparison impossible. Due to this limitation of previous evaluation schemes, we propose a novel metric which isolates selective classification performance from losses that come from accuracy degradation of the classifier overall (as frequently caused by DP). The performance metric we propose to tackle this problem computes an upper bound on the accuracy/coverage trade-off in a model-dependent fashion and measures the discrepancy between the actual achieved trade-off to this bound. Using this score, we determine that, based on a comprehensive experimental study, differential privacy does indeed harm selective prediction beyond a loss in utility.

We summarize our key contributions below: 
\begin{enumerate}[leftmargin=15pt]
     \item We provide a first analysis on the interplay between selective classification and differential privacy. As a result, we identify an existing SC method as particularly suitable under DP and present an illustrative example showcasing an inherent tension between privacy and selective prediction. 
     \item We unearth a critical failure mode of the canonical SC performance metric preventing its off-the-shelf usage under DP experimentation. We remedy this issue by introducing a novel accuracy-dependent selective classification score which enables us to compare selective classification performance across DP levels without explicit accuracy-alignment.
     \item We conduct a thorough empirical evaluation across multiple selective classification techniques and privacy levels. As part of this study, we confirm that selective classification performance degrades with stronger privacy and further find that recovering utility can come at a considerable coverage cost under strong privacy requirements.
\end{enumerate}

\section{Background}
\label{sec:background}

\subsection{Differential Privacy with DP-SGD}
\label{sec:def_dp}

\emph{Differential Privacy (DP)}~\citep{dwork2014algorithmic} is a popular technique which allows us to reason about the amount of privacy leakage from individual data points. Training a model with differential privacy ensures that the information content that can be acquired from individual data points is bounded while patterns present across multiple data points can still be extracted. Formally, a randomized algorithm $\mathcal{M}$ satisfies $(\varepsilon, \delta)$ differential privacy, if for any two datasets $D, D'\subseteq \mathcal{D}$ that differ in any one record and any set of outputs~$S$ the following inequality holds:
\begin{equation}
\label{eq:dp}
    \mathbb{P}\left[\mathcal{M}(D) \in S\right] \leq e^\varepsilon  \mathbb{P}\left[\mathcal{M}(D') \in S\right] + \delta
\end{equation}
The above DP bound is governed by two parameters: $\varepsilon \in \mathbb{R}_+$ which specifies the privacy level, and $\delta \in [0, 1]$ which allows for a small violation of the bound.

\paragraph{DP-SGD.} A canonical way of incorporating differential privacy in deep learning is via \emph{Differentially Private Stochastic Gradient Descent (DP-SGD)}~\citep{bassily2014private,abadi2016deep}. We describe the DP-SGD algorithm in detail in Algorithm~\ref{alg:dpsgd} in the Appendix. The main adaptations needed to incorporate DP into SGD are: (i) \emph{per-sample gradient computation}, which allows us to limit the per-point privacy leakage in the next two steps; (ii) \emph{gradient clipping}, which bounds the sensitivity (\ie the maximal degree of change in the outputs of the DP mechanism); and (iii) \emph{noise addition}, which introduces Gaussian noise proportional to the intensity of gradient clipping.

\subsection{Selective Classification}

\paragraph{Supervised Classification.} Our work considers the supervised classification setup: We assume access to a dataset $D = \{(\bm{x}_n,y_n)\}_{n=1}^{N}$ consisting of data points $(\bm{x},y)$ with $\bm{x} \in \mathcal{X} \subseteq \mathbb{R}^d$ and $y \in \mathcal{Y} = \{1, \ldots, C\}$. All data points $(\bm{x},y)$ are sampled independently and identically distributed from the underlying distribution $p$ defined over $\mathcal{X} \times \mathcal{Y}$. The goal is to learn a prediction function $f : \mathcal{X} \rightarrow \mathcal{Y}$ which minimizes the classification risk with respect to the underlying data distribution $p$ as measured by a loss function $\ell : \mathcal{Y} \times \mathcal{Y} \rightarrow \mathbb{R}$. 

\paragraph{Selective Classification.} Selective classification augments the supervised classification setup by introducing a rejection class~$\bot$ via a \textit{gating mechanism}~\citep{yaniv2010riskcoveragecurve}. This mechanism produces a class label from the underlying classifier if the mechanism is confident that the prediction is correct and abstains otherwise. Note that this mechanism is often directly informed by the underlying classifier $f$ and we make this dependence explicit in our notation. More formally, the gating mechanism introduces a selection function $g:\mathcal{X} \times (\mathcal{X} \rightarrow \mathcal{Y}) \rightarrow \mathbb{R}$ which determines if a model should predict on a data point~$\bm{x}$. If the output of $g(\bm{x}, f)$ undercuts a given threshold $\tau$, we return $f(\bm{x})$, otherwise we abstain with decision $\bot$. The joint predictive model is therefore given by:
\begin{equation}
    (f,g)(\bm{x}) = \begin{cases}
  f(\bm{x})  & g(\bm{x}, f) \leq \tau \\
  \bot & \text{otherwise.}
\end{cases}
\end{equation}
\paragraph{Evaluating SC Performance.} The performance of a selective classifier $(f,g)$ is based on two metrics: the \emph{coverage} of $(f,g)$ (corresponding to the fraction of points to predict on) and the \emph{accuracy} of $(f,g)$ on accepted points. Note that there exists a tension between these two performance quantities: naively increasing coverage will lead to lower accuracy while an increase in accuracy will lead to lower coverage. Successful SC methods try to jointly maximize both metrics.
    \begin{equation}
    \text{cov}_\tau(f,g) = \frac{|\{\bm{x} : g(\bm{x}, f) \leq \tau \}|}{|D|} \qquad \qquad 
    \text{acc}_\tau(f,g) = \frac{|\{\bm{x} : f(\bm{x}) = y, g(\bm{x}, f) \leq \tau \}|}{|\{\bm{x} : g(\bm{x}, f) \leq \tau \}|}
    \end{equation}
To characterize the full performance profile of a selective classifier $(f,g)$, we consider the \emph{selective accuracy} at coverage level~$c$, formally $\text{acc}_c(f,g)$, over the full coverage spectrum by computing an area-under-the-curve (AUC) score $s_\text{AUC}$ as follows:
\begin{equation}
\label{eq:sc_perf}
    s_\text{AUC}(f,g) = \int_0^1 \text{acc}_c(f,g)dc \qquad \qquad  \text{acc}_c(f,g) = \text{acc}_\tau(f,g)\quad \text{for $\tau$ s.t. $\text{cov}_\tau(f,g) = c$} 
\end{equation}
Each existing selective classification algorithm proposes a $g$ that tries to maximize this metric.

\paragraph{Prediction Confidence.} The traditional baseline methods for selective prediction is the \emph{Softmax Response} (\sr) method \citep{hendrycks2016baseline, geifman2017selective}. This method uses the confidence of the final prediction model $f$ as the selection score. To reduce overconfident predictions yielded by this method, confidence calibration~\citep{guo2017calibration} has been proposed.

\paragraph{Ensembling.}
To further improve calibration and to reduce variance, ensemble methods have been proposed which combine the information content of $M$ models into a single final model. The canonical instance of this approach for deep learning based models, \emph{Deep Ensembles}~(\de)~\citep{lakshminarayanan2017simple}, trains multiple models from scratch with varying initializations using a proper scoring rule and adversarial training. Then, after averaging the predictions made by all models, the softmax response (\sr) mechanism is applied. To overcome the computational cost of estimating multiple models from scratch, \emph{Monte-Carlo Dropout} (\mcdo) \citep{gal2016dropout} allows for bootstrapping of model uncertainty of a dropout-equipped model at test time. Another ensembling approach that has recently demonstrated state-of-the-art \selc performance is based on monitoring model evolution during the training process. To that end, \emph{Selective Classification Training Dynamics}~(\sctd)~\citep{rabanser2022selective} records intermediate models produced during training and computes a disagreement score of intermediate predictions with the final prediction. 

\paragraph{Architecture \& Loss Function Modifications.} 
A variety of SC methods have been proposed that leverage explicit architecture and loss function adaptations. For example, \emph{SelectiveNet} (\sn) \citep{geifman2019selectivenet} modifies the model architecture to jointly optimize $(f,g)$ while targeting the model at a desired coverage level~$c_\text{target}$. Alternatively, prior works like \emph{Deep Gamblers}~\citep{liu2019deep} and \emph{Self-Adaptive Training}~\citep{huang2020self} have considered explicitly modeling the abstention class $\bot$ and adapting the optimization process to provide a learning signal for this class. For instance, \emph{Self-Adaptive Training} (\sat) incorporates information obtained during the training process into the optimization itself by computing and monitoring an exponential moving average of training point predictions over the training process. Introduced to yield better uncertainty estimates, \citet{liu2020simple} employs weight normalization and replaces the output layer of a neural network with a Spectral-Normalized Neural Gaussian Process to improve data manifold characterization. 

\paragraph{Uncertainty for DP Models.} An initial connection between differential privacy and uncertainty quantification is explored in \citet{williams2010probabilistic}, which shows that probabilistic inference can improve accuracy and measure uncertainty on top of differentially private models. By relying on intermediate model predictions, \citet{shejwalkar2022recycling} has proposed a mechanism to quantify the uncertainty that DP noise adds to the outputs of ML algorithms (with no additional privacy cost).

\section{The Interplay Between Selective Classification and Differential Privacy}

\begin{algorithm}[t]
	\caption{\sctd~\citep{rabanser2022selective}}\label{alg:sctd}
	\begin{algorithmic}[1]
	\Require Checkpointed model sequence $\{f_1,\ldots,f_T\}$, query point $\bm{x}$, weighting parameter $k \in [0,\infty)$.
    \State Compute prediction of last model: $L \gets f_T(\bm{x})$
    \State Compute disagreement and weighting of intermediate predictions: 
    \For{$t \in [T]$}
        \State \algorithmicif\ $f_t(\bm{x}) = L$ \algorithmicthen\ $a_t \gets 0$ \algorithmicelse\ $a_t \gets 1$
        \State $v_t \gets (\frac{t}{T})^k$
    \EndFor
\State Compute sum score: $s_\text{sum} \gets \sum_{t} v_t a_t$
    \State \algorithmicif\ $s_\text{sum} \leq \tau$ \algorithmicthen\ accept $f(\bm{x}) = L$ \algorithmicelse\ reject with $f(\bm{x}) = \bot$
	\end{algorithmic}
\end{algorithm}

We now examine the interplay between these approaches to selective classification and training algorithms constrained to learn with differential privacy guarantees. We first introduce a candidate selective classification approach which leverages intermediate predictions from models obtained during training. Then, we study how current selective classification approaches impact DP guarantees, as each access they make to training data increases the associated privacy budget. Next, we investigate how in turn DP affects selective classification performance. Indeed, DP alters how optimizers converge to a solution and as such DP can impact the performance of SC techniques. Last, we discuss how selective classification performance can be fairly compared across a range of privacy levels.

\subsection{Performative Private Selective Classification via Training Dynamics Ensembles}

As a classical technique from statistics, ensembling methods are often employed for confidence interval estimation~\citep{karwa2017finite, ferrando2022parametric}, uncertainty quantification~\citep{lakshminarayanan2017simple}, and selective prediction~\citep{zaoui2020regression}. Under a differential privacy constraint, although for simpler tasks like mean estimation ensembling methods have been proven to be effective~\citep{brawner2018bootstrap,covington2021unbiased,evans2019statistically}, in this chapter we demonstrate that they are fairly ineffective for selective classification. This is primarily due to the increased privacy cost due to (advanced sequential) composition~\citep{ dwork2006calibrating}. 

In light of this challenge, we expect one recent method in particular to perform well in a private learning setting: selective classification via training dynamics (\sctd)~\citep{rabanser2022selective} (details described in Algorithm~\ref{alg:sctd}). For a given test-time point, \sctd analyzes the disagreement with the final predicted label over intermediate models obtained during training. Data points with high disagreement across this training-time ensemble are deemed anomalous and rejected. Most importantly, while \sctd also constructs an ensemble, only a single training run is used to obtain this ensemble. As a result of the post-processing property of DP, the original $(\varepsilon, \delta)$-DP guarantee can be maintained. To further motivate the effectiveness of relying on intermediate checkpoints, \citet{shejwalkar2022recycling} has shown in adjacent work that intermediate predictions can improve the predictive accuracy of DP training, yielding new state-of-the-art results under DP. 
 
\subsection{How Do Other Selective Classification Approaches Affect Privacy Guarantees?}
\label{sec:sc_affects_dp}

We can extend the previous analysis based on post-processing and composition to other SC techniques. This allows us to group selective classification techniques into three classes: (i) directly optimized; (ii) post-processing; and (iii) advanced sequential composition. For directly optimized and post-processing approaches, we can obtain selective classification for free. On the other hand, composition-based approaches either require an increased privacy budget or suffer from decreased utility.

\paragraph{Direct Optimization.} Many selective classification approaches directly modify the loss function and optimization is performed w.r.t. this adapted loss function. As DP-SGD is loss function and architecture agnostic, $(\varepsilon, \delta)$ guarantees hold automatically for SC methods that only change the loss function. \emph{This property applies to:} Softmax Response (\sr), Deep Gamblers~(\dg), Self-Adaptive Training (\sat).

\paragraph{Post-Processing.} If a function $\phi(x)$ satisfies $(\varepsilon,\delta)$-DP, then for any deterministic or random function $\psi(\cdot)$, the application of $\psi$ on $\phi$, formally $\psi \circ \phi (x)$, still satisfies $(\varepsilon,\delta)$-DP \citep{dwork2006calibrating}. \emph{This property applies to:} Monte-Carlo Dropout (\mcdo), Selective Classification via Neural Network Training Dynamics (\sctd).

\paragraph{Advanced Sequential Composition.} If in a set of aggregated functions $\{\phi_1(x), \ldots, \phi_M(x)\}$ each $\phi_i(x)$ satisfies $(\varepsilon,\delta)$-DP, then releasing the set of all outputs $\psi(x) = (\phi_1(x), \ldots, \phi_M(x))$ satisfies $\approx (\sqrt{M}\varepsilon,M\delta)$-DP \citep{dwork2006calibrating}. To maintain $(\varepsilon, \delta)$-DP over the aggregated output, each function needs to satisfy $\approx(\frac{\varepsilon}{\sqrt M}, \frac{\delta}{M})$-DP. \emph{Applicable to}: Deep Ensembles (\de), SelectiveNet (\sn).

\subsection{How Does Differential Privacy Affect Selective Classification Performance?} 
\label{sec:dp_affects_sc}

After having examined how current selective classification techniques influence differential privacy guarantees, this subsection considers the opposite effect: what is the impact of differential privacy on selective classification? As we will see by means of an intuitive example, we expect that differential privacy impacts selective classification beyond a loss in utility.

To showcase this, we present a synthetic logistic regression example with the softmax response SC mechanism. We generate data from a mixture of two two-dimensional Gaussians with heavy class-imbalance. Concretely, we generate samples for a majority class $\{(\bm{x}_i,1)\}_{i=1}^{1000}$ with each $\bm{x}_i \sim \mathcal{N}(\bm{0}, \bm{I})$ and a single outlier point form a minority class $(\bm{x}^*,-1)$ with $\bm{x}^* \sim \mathcal{N}(\begin{bmatrix}10 & 0\end{bmatrix}^\top, \bm{I})$. We then train multiple differentially private mechanisms with $\varepsilon \in \{\infty, 7,3,1\}$  on this dataset and evaluate all models on a test set produced using the same data-generating process.

\looseness=-1
We show the outcome of this set of experiments in Figure~\ref{fig:eps_gauss} where we plot the test set and the decision boundary of each model across $\varepsilon$ levels. For $\varepsilon = \infty$, the model successfully discriminates the majority and the minority class. However, the decision boundary is heavily influenced by a single data point, namely the outlier $(\bm{x}^*,-1)$. Note that, even though all models with $\varepsilon \in \{7,3,1\}$ misclassify the outlier point (\ie their utility is aligned), the changing decision boundary also increases the model's confidence in predicting the incorrect class for the outlier. Hence, even at an aligned utility level, the softmax response approach for SC performs increasingly worse on the outlier point under stronger DP constraints. We provide additional class-imbalanced results on realistic datasets in Appendix~\ref{sec:class_imb_real}.

\begin{figure*}[t]
  \centering
  \includegraphics[width=\linewidth]{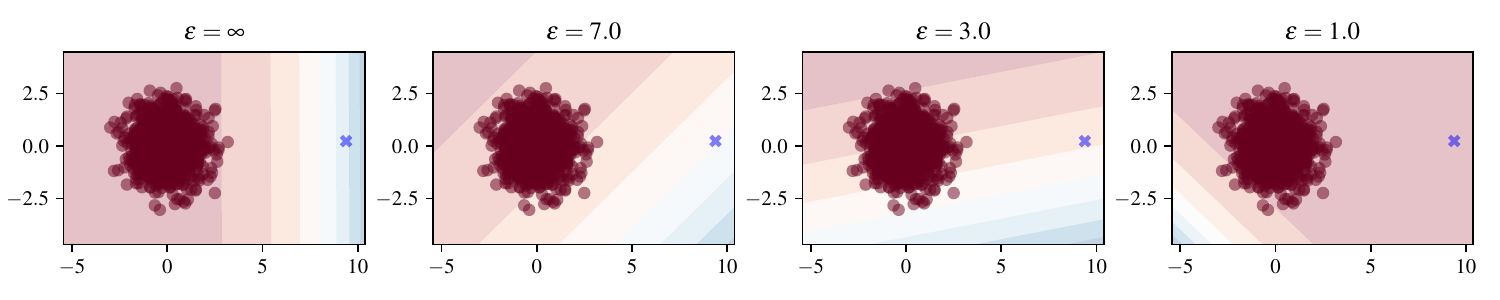}
\caption[Synthetic example of privacy impacts on selective classification.]{\textbf{Synthetic example of privacy impacts on selective classification}. Shaded regions show softmax response confidence in the red and blue class, respectively. As $\varepsilon$ decreases, the blue outlier point is increasingly protected by DP. Importantly, not only does DP induce a misclassification on the outlier point, but stronger levels of DP further increase confidence in the wrong prediction. This overconfidence in the incorrect class prevents selective classification.}
\label{fig:eps_gauss}
\end{figure*}

\subsection{How Should We Evaluate Selective Classification Under Differential Privacy?}
\label{sec:new_metric}

\begin{contriback}
This subsection was written with Anvith Thudi. Both Stephan and Anvith contributed to the formalization of the upper bound and the resulting evaluation score in equal parts.
\end{contriback}

As we have now established, we expect differential privacy to harm selective classification performance. In light of this insight, we are now investigating whether the standard evaluation scheme for SC is directly applicable to the differentially private models.

The default approach of using the SC performance metric introduced in Equation~\ref{eq:sc_perf} free from its bias towards accuracy is to align different SC approaches/models at the same accuracy. However, accuracy-aligning can have unintended consequences on SC performance, especially under DP. If we are considering SC performance across models trained for multiple $\varepsilon$ levels, accuracy-alignment would require lowering the performance of the less private models. One seemingly suitable approach to do this is by early-stopping model training at a desired accuracy level. However, early-stopping a DP training algorithm in fact gives a model with greater privacy than the privacy level targeted by the optimization process. This is a direct consequence of privacy loss accumulation via privacy accounting during training, and so training for less leads to expending less privacy budget. Hence, instead of comparing models with different privacy guarantees at the same baseline utility, early-stopping yields models with potentially very similar privacy guarantees. 

To address this limitation, we propose to measure SC performance by how close the observed trade-off is to an upper bound on the achievable SC performance at a given baseline utility level (at $100\%$ coverage). To compute this metric, we first derive an upper bound on selective classification performance (as a function of coverage) for a given baseline accuracy. We obtain this characterization by identifying the optimal acceptance ordering under a particular full-coverage accuracy constraint. 

\begin{figure}
\centering
\begin{minipage}{0.48\linewidth}      
{\centering
\small
\resizebox{\linewidth}{!}{
    \begin{tikzpicture}[
		declare function={}
		]
\begin{axis}[
  xlabel=Coverage ($c$),
  ylabel=Accuracy,
  xmin = -0.1,
  xmax = 1.1,
  ymin = -0.1,
  ymax = 1.25,
  grid=major,
  width=8cm,
  height=5cm,
  tick label style={/pgf/number format/fixed},
  legend pos=north east]
  
  \newcommand\afull{0.4}
  \newcommand\afullinccoord{{\afull + (1 - \afull)/2}}
  
  \addplot+[mark={},dashed,line width=2pt,Black,name path=A, domain=0:\afull] {1};
  \addplot+[mark={},dashed,line width=2pt,Black,name path=B, domain=\afull:1] {\afull/x};
  \addplot+[draw=none,mark=none,name path=C,domain=0:1] {0};
  \addplot+[Green!25!white] fill between[of=A and C,soft clip={domain=0:\afull}];
  \addplot+[Red!25!white] fill between[of=B and C,soft clip={domain=\afull:1}];
  \node[black, fill=white] at (1,\afull+0.13) {\scriptsize $a_\text{full}$};
  \node[black, fill=white] at (\afull,1.13) {\scriptsize $c = a_\text{full}$};

  \fill[black] (\afull, 1) circle[radius=2.5pt];
  \fill[black] (1,\afull) circle[radius=2.5pt];
  
  \draw [->, Red, line width=1pt] (\afull,0.05) -- (1, 0.05);
  \draw [->, Green, line width=1pt] (0,0.05) -- (\afull, 0.05);
  \node[Green] at (\afull/2,0.15) {\scriptsize correct points};
  \node[Red] at (\afullinccoord,0.15) {\scriptsize incorrect points};

  \legend{ \scriptsize \upperbound};
 
\end{axis}
\end{tikzpicture}
}}
    (a) \emph{Upper bound $\overline{\text{acc}}(a_\text{full},c)$ on the selective classification performance conditional on $a_\text{full}$}. As coverage $c$ increases from $0 \rightarrow 1$, an optimal selective classifier accepts all correct points first (with a consistent accuracy of $1$ until $c = a_\text{full}$) and then spreads the accuracy decrease $\frac{a_\text{full}}{c}$ equally over the remaining coverage spectrum, leading to a convex drop as $c \rightarrow 1$.
    \label{fig:bound}
\end{minipage}
\hspace{10pt}
\begin{minipage}{0.48\linewidth}    
{\centering
\small
\resizebox{\linewidth}{!}{
\begin{tikzpicture}[
		declare function={}
		]
		\begin{axis}[%
		  xlabel=Coverage ($c$),
		  ylabel=Accuracy,
		  xmin = -0.1,
		  xmax = 1.1,
		  ymin = 0.3,
		  ymax = 1.1,
		  grid=major,
		  width=8cm,
		  height=5cm,
		  tick label style={/pgf/number format/fixed},
		  legend pos=north east]
		  
		  \newcommand\afull{0.4}
		  
		  \addplot+[mark={},line width=2pt,Black,name path=A, dashed, domain=0:\afull] {1};
		  \addplot+[mark={},line width=2pt,Black,name path=C, domain=0:\afull/2] {1};
		  \addplot+[mark={},line width=2pt,Black,name path=B, dashed, domain=\afull:1] {\afull/x};

		  \addplot+[mark={},line width=2pt,Black,name path=D, domain=0.2:0.4] {15/16*x^2 - 15/8 * x + 107/80};
		  
		  \addplot+[mark={},line width=2pt,Black,name path=E, domain=0.4:1] {15/16*x^2 - 15/8 * x + 107/80};
		  
		  \addplot+[Black!25!white] fill between[of=A and D,soft clip={domain=0:1}];
        \legend{ \scriptsize \upperbound, \scriptsize \empiricalacccovtradeoff, , , , \scriptsize \accnormscore };
		  \addplot+[Black!25!white] fill between[of=B and E,soft clip={domain=0:1}];
		 
		\end{axis}
		\end{tikzpicture}
    }}
    (b) \emph{Accuracy-normalizes score for SC performance.} The accuracy-normalized score for selective classification $s_{a_\text{full}}(f,g)$ corresponds to the area enclosed between the upper bound $\overline{\text{acc}}(a_\text{full},c)$ and the empirical accuracy/coverage trade-off $\text{acc}_c(f,g)$. Good selective classifiers should achieve a low score ($s_{a_\text{full}}(f,g) \approx 0$), indicating closeness to the bound.
    \label{fig:score}
\end{minipage}
\vspace{10pt}
    \caption[Selective classification upper bound and accuracy-normalized score visualization.]{\textbf{Selective classification upper bound and accuracy-normalized score visualization}. We present an example of a selective classifier with full-coverage accuracy $a_\text{full} = 0.4$ and show how (a)~the corresponding upper bound; and how (b) the accuracy-normalized score can be obtained. }
    \label{fig:bound_score}
\end{figure}

\begin{definition}
\emph{The upper bound on the selective classification performance for a fixed full-coverage accuracy $a_\text{full} \in [0,1]$ and a variable coverage level $c \in [0,1]$ is given by
    \begin{equation}
        \label{eq:sc_dp_bound}
        \overline{\text{acc}}(a_\text{full},c) = \begin{cases}
  1  & 0 < c \leq a_\text{full} \\
  \frac{a_\text{full}}{c} & a_\text{full} < c < 1
\end{cases}.
    \end{equation}
    }
\end{definition}
Measuring the area enclosed between the bound $\overline{\text{acc}}(a_\text{full},c)$ and an SC method's achieved accuracy/coverage trade-off $\text{acc}_c(f,g)$ yields our accuracy-normalized selective classification score. 
\begin{definition}
  \emph{The accuracy-normalized selective classification score $s_{a_\text{full}}(f,g)$ for a selective classifier $(f,g)$ with full-coverage accuracy $a_\text{full}$ is given by
  \begin{equation}
  \label{eq:acc_norm_score}
        s_{a_\text{full}}(f,g) = \int_0^1 (\overline{\text{acc}}(a_\text{full},c) - \text{acc}_c(f,g)) dc \approx \sum_c (\overline{\text{acc}}(a_\text{full},c) - \text{acc}_c(f,g)).
    \end{equation}
    }
\end{definition}
We provide additional intuition as part of Figure~\ref{fig:bound_score}.

\paragraph{Bound Justification.} To understand the upper bound given in Equation~\ref{eq:bound}, note that a full-coverage accuracy level of $a_\text{full}$ means that a fraction $a_\text{full}$ of points are correctly classified. An ideal selective classification method for this model would, by increasing the coverage $c$ from 0 to 1, first accept all points that the model classifies correctly: this gives the optimal selective classification accuracy for coverage $c \leq a_{\text{full}}$ of $100\%$. For any coverage level $c > a_{\text{full}}$, note that the accuracy at the coverage is the number of correct points accepted divided by the total number of points accepted. The largest value this can take is by maximizing the numerator, and the maximum value for this (as a fraction of all points) is $a_{\text{full}}$. Plugging this into the previous statement for accuracy at coverage $c$, we have the best selective accuracy at a coverage $c$ is $\frac{a_{\text{full}}}{c}$. This derives the upper bound stated above as a function of coverage. We remark that this bound is in fact achievable if a selective classifier separates correct and incorrect points perfectly, \ie it accepts all correct points first and then accepts and increasing amount of incorrect points as we increase the coverage of the selective classification method. See Appendix~\ref{sec:opt_bound_reach} for an experiment that matches the bound exactly. 

\begin{figure*}[t]
  \centering
  \includegraphics[width=\linewidth]{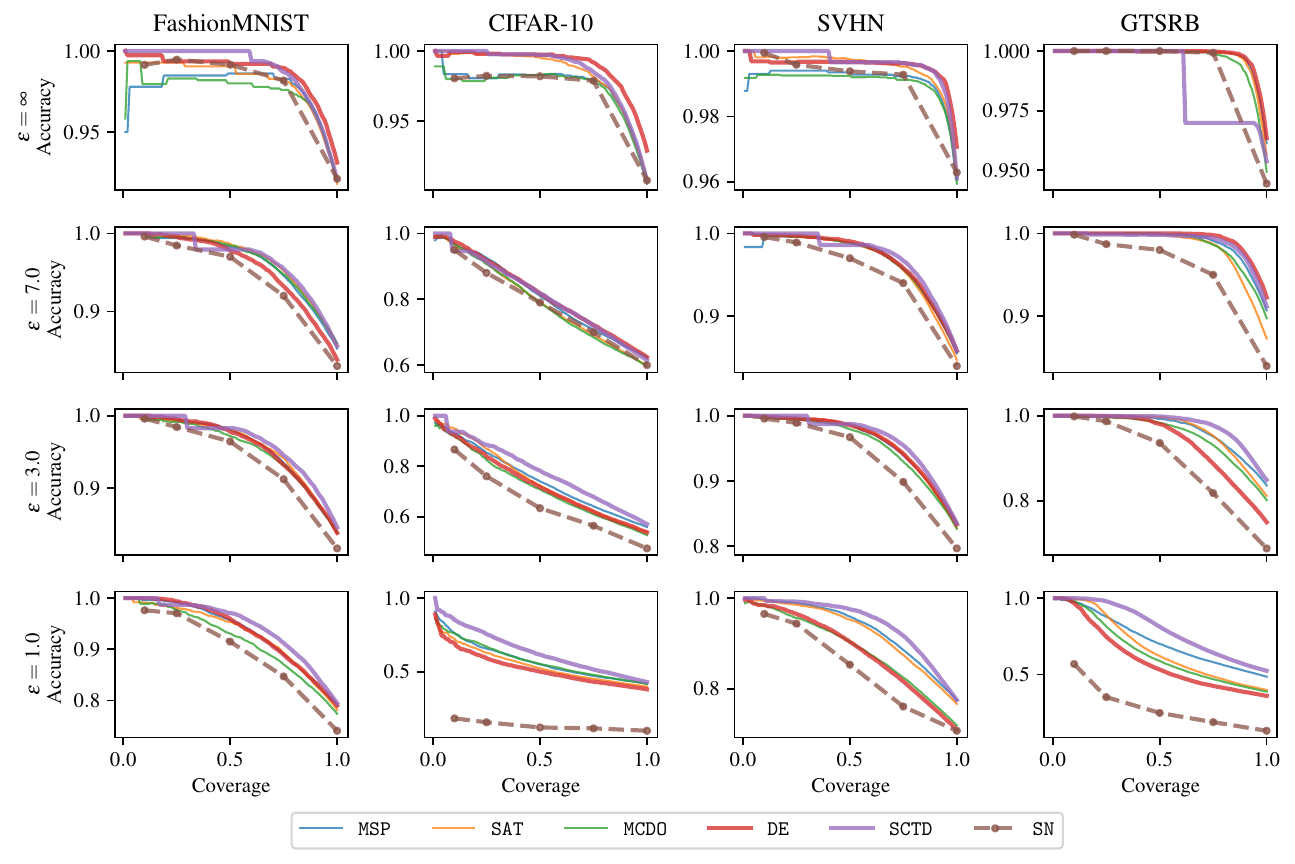}
\caption[Accuracy-coverage trade-off across datasets \& $\varepsilon$ levels.]{\textbf{Accuracy-coverage trade-off across datasets \& $\varepsilon$ levels}. We highlight noteworthy results for \sctd, \de, and \sn with bold lines and further show that \sn only provides results at coarse-grained coverage levels. \sctd performs best while \de and \sn performance is compromised at any $\varepsilon < \infty$.}
\label{fig:accuracy_coverage_tradeoff}
\end{figure*}

\section{Experiments}
\label{sec:exp}

We now present a detailed experimental evaluation of the interplay between differential privacy and selective classification. All of our experiments are implemented using PyTorch~\citep{Paszke_PyTorch_An_Imperative_2019} and Opacus~\citep{opacus}. We publish our full code-base at the following URL: \url{https://github.com/cleverhans-lab/selective-classification}.

\paragraph{Setup.} Our evaluation is based on the following experimental panel. In particular, we conduct experiments for each dataset/SC-method/$\varepsilon$ combination:
\begin{itemize}
    \item \textbf{Datasets}: FashionMNIST~\citep{xiao2017fashion}, CIFAR-10~\citep{krizhevsky2009learning}, SVHN~\citep{netzer2011reading}, GTSRB~\citep{Houben-IJCNN-2013}.
    \item \textbf{Selective Classification Methods}: Softmax Response (\sr)~\citep{geifman2017selective}, SelectiveNet (\sn)~\citep{geifman2019selectivenet}, Self-Adaptive Training (\sat)~\citep{huang2020self}, Monte-Carlo Dropout (\mcdo)~\citep{gal2016dropout}, Deep Ensembles (\de)~\citep{lakshminarayanan2017simple}, Selective Classification Training Dynamics (\sctd)~\citep{rabanser2022selective}.
    \item \textbf{Privacy levels}: $\varepsilon \in \{\infty, 7, 3, 1\}$.
\end{itemize}

Based on recent results from~\citet{feng2023towards}, we (i) apply additional entropy regularization to all methods; and (ii) derive the selective classification scores for \sn and \sat by applying Softmax Response (\sr) to the underlying classifier (instead of relying on model-specific abstention mechanisms). Across all private experiments, we set $\delta = 10^{-(N+1)}$ for datasets consisting of training samples on the order of $10^N$. For each combination of SC method, dataset, and privacy level, we train a model following the ResNet-18~\citep{he2016deep} model architecture using either the SGD optimizer (for $\varepsilon=\infty$) or the DP-SGD optimizer (for all $\varepsilon<\infty$). All models are trained for 200 epochs using a mini-batch size of 128. For all DP training runs, we set the clipping norm to $c=10$ and choose the noise multiplier adaptively to reach the overall privacy budget at the end of training. Note that, since we want to consider a consistent value of $\varepsilon$ across all SC methods, we restrict the overall $\varepsilon$ in both \sn and \de. As explained in Section~\ref{sec:sc_affects_dp}, this effectively enforces a reduction in $\varepsilon$ for each individual training run due to sequential composition. Concretely, we train \de with 5 ensemble members and \sn with target coverage levels $c_\text{target} \in \{0.1,0.25,0.5,0.75,1.0\}$. This leads to a reduction of $\approx\frac{\varepsilon}{\sqrt 5}$ for each trained sub-model in both cases. We account for the precise reduction in the privacy parameter by~\emph{increasing by the number steps of DP-SGD} as measured by a R\'enyi DP accountant~\citep{mironov2017renyi,opacus}. All experiments are repeated over 5 random seeds to determine the statistical significance of our findings. We document additional hyper-parameters in Appendix~\ref{sec:hyp}.

\begin{table}[t]
    \fontsize{8.5}{8}\selectfont
    \tabcolsep=0.175cm
    \caption[Coverage required for non-private full-coverage accuracy.]{\textbf{Coverage required for non-private full-coverage accuracy.} We observe that differential privacy considerably lowers the allowable coverage to achieve a utility level consistent with the non-private model. Across our experiments, \sctd yields the highest coverage across $\varepsilon$ levels.}
    \vspace{5pt}
    \centering
    \label{tab:coverage_performance}
    \begin{tabular}{c c c c c c c}
    \toprule
    & \multicolumn{3}{c}{FashionMNIST} & \multicolumn{3}{c}{CIFAR-10} \\
    \cmidrule(lr){2-4} \cmidrule(lr){5-7}
           & $\varepsilon=7$ & $\varepsilon=3$ & $\varepsilon=1$ & $\varepsilon=7$ & $\varepsilon=3$ & $\varepsilon=1$ \\
    \midrule
    \msp  & 0.83 (±0.01) & 0.80 (±0.01) & 0.65 (±0.03) & \bfseries 0.29 (±0.02) & 0.14 (±0.04) & 0.00 (±0.00) \\
    \sat  & \bfseries 0.86 (±0.00) & 0.81 (±0.01) & 0.67 (±0.02) & 0.25 (±0.01) & \bfseries 0.19 (±0.02) & 0.00 (±0.00) \\
    \mcdo & \bfseries 0.84 (±0.02) & 0.79 (±0.00) & 0.56 (±0.02) & 0.25 (±0.01) & 0.12 (±0.02) & 0.00 (±0.00) \\
    \de   & 0.75 (±0.00) & 0.75 (±0.01) & 0.61 (±0.01) & 0.22 (±0.01) & 0.09 (±0.00) & 0.00 (±0.00) \\
    \sctd & \bfseries 0.86 (±0.01) & \bfseries 0.84 (±0.02) & \bfseries 0.73 (±0.01) & \bfseries 0.26 (±0.03) & \bfseries 0.20 (±0.03) & \bfseries 0.04 (±0.04) \\
    \cmidrule(lr){2-4} \cmidrule(lr){5-7}
    & \multicolumn{3}{c}{SVHN} & \multicolumn{3}{c}{GTSRB} \\
    \cmidrule(lr){2-4} \cmidrule(lr){5-7}
    \msp  & 0.74 (±0.00) & 0.67 (±0.01) & 0.49 (±0.02) & 0.90 (±0.01) & 0.71 (±0.03) & 0.13 (±0.00) \\
    \sat  & 0.72 (±0.00) & 0.67 (±0.01) & 0.45 (±0.02) & 0.86 (±0.00) & 0.74 (±0.00) & 0.20 (±0.03) \\
    \mcdo & 0.74 (±0.00) & 0.64 (±0.00) & 0.23 (±0.03) & 0.90 (±0.01) & 0.69 (±0.01) & 0.14 (±0.01) \\
    \de   & 0.69 (±0.01) & 0.62 (±0.01) & 0.22 (±0.00) & \bfseries 0.93 (±0.00) & 0.57 (±0.08) & 0.10 (±0.04) \\
    \sctd & \bfseries 0.78 (±0.01) & \bfseries 0.72 (±0.00) & \bfseries 0.59 (±0.02) & \bfseries 0.93 (±0.01) & \bfseries 0.83 (±0.03) & \bfseries 0.30 (±0.02) \\
    \bottomrule
    \end{tabular}
\end{table}

\paragraph{Evaluation of Selective Classification Techniques Across Privacy Levels.}

We report our main results in Figure~\ref{fig:accuracy_coverage_tradeoff} where we display the accuracy-coverage trade-off for our full experimental panel. Overall, we observe that Selective Classification Training Dynamics (\sctd) delivers the best accuracy-coverage trade-offs. While this is consistent with non-private~\citep{rabanser2022selective}, we find that \sctd delivers especially strong performance at lower $\varepsilon$ levels. This suggests that leveraging training dynamics information is useful for performative selective classification under DP. As expected, based off of the discussion in Section~\ref{sec:sc_affects_dp}, Figure~\ref{fig:accuracy_coverage_tradeoff} also shows that the performance of both Deep Ensembles (\de) and SelectiveNet (\sn) suffers under DP, showing increasing utility degradation as $\varepsilon$ decreases. 

\begin{table}[t]
    \scriptsize
    \fontsize{7}{6.5}\selectfont
    \tabcolsep=0.145cm
\caption[Accuracy-normalized selective classification performance.]{\textbf{Accuracy-normalized selective classification performance.} Across our panel of experiments, we find that decreasing \(\varepsilon\) leads to a worse (\ie higher) score, meaning that as \(\varepsilon\) decreases the selective classification approaches all move away from the upper bound.}
\vspace{5pt}
    \centering
    \label{tab:sc_score_performance}
    \begin{tabular}{ccccccc}
\toprule
 & \multicolumn{3}{c}{FashionMNIST} & \multicolumn{3}{c}{CIFAR-10} \\
\cmidrule(lr){2-4} \cmidrule(lr){5-7}
 & \(\varepsilon = \infty\) & \(\varepsilon = 7\) & \(\varepsilon = 1\) & \(\varepsilon = \infty\) & \(\varepsilon = 7\) & \(\varepsilon = 1\) \\
\midrule
\msp  & 0.019 (±0.000) & 0.023 (±0.000) & 0.041 (±0.001) & 0.019 (±0.000) & 0.105 (±0.002) & 0.205 (±0.001) \\
\sat  & 0.014 (±0.000) & \bfseries 0.020 (±0.001) & 0.043 (±0.002) & 0.010 (±0.000) & 0.107 (±0.000) & 0.214 (±0.002) \\
\mcdo & 0.020 (±0.002) & 0.023 (±0.001) & 0.053 (±0.001) & 0.021 (±0.001) & 0.110 (±0.000) & 0.201 (±0.000) \\
\de   & \bfseries 0.010 (±0.003) & 0.027 (±0.002) & 0.039 (±0.000) & \bfseries 0.007 (±0.001) & \bfseries 0.099 (±0.002) & 0.222 (±0.000) \\
\nntd & \bfseries 0.007 (±0.001) & \bfseries 0.021 (±0.001) & \bfseries 0.032 (±0.002) & \bfseries 0.009 (±0.002) & \bfseries 0.098 (±0.001) & \bfseries 0.152 (±0.001) \\
\sn   & \bfseries 0.008 (±0.002) & 0.058 (±0.001) & 0.064 (±0.002) & 0.015 (±0.000) & 0.155 (±0.003) & 0.173 (±0.001) \\
    \cmidrule(lr){2-4} \cmidrule(lr){5-7}
    & \multicolumn{3}{c}{SVHN} & \multicolumn{3}{c}{GTSRB} \\
    \cmidrule(lr){2-4} \cmidrule(lr){5-7}
\msp  & 0.008 (±0.001) & 0.020 (±0.001) & 0.040 (±0.001) & \bfseries 0.001 (±0.001) & 0.006 (±0.002) & 0.109 (±0.002) \\
\sat  & \bfseries 0.004 (±0.000) & 0.019 (±0.000) & 0.044 (±0.002) & \bfseries 0.001 (±0.001) & 0.008 (±0.001) & 0.089 (±0.000) \\
\mcdo & 0.009 (±0.000) & 0.019 (±0.001) & 0.069 (±0.001) & 0.002 (±0.001) & 0.007 (±0.001) & 0.110 (±0.001) \\
\de   & \bfseries 0.004 (±0.001) & 0.018 (±0.001) & 0.067 (±0.003) & \bfseries 0.001 (±0.000) & \bfseries 0.003 (±0.002) & 0.127 (±0.002) \\
\nntd & \bfseries 0.003 (±0.001) & \bfseries 0.016 (±0.001) & \bfseries 0.027 (±0.001) & 0.011 (±0.002) & \bfseries 0.005 (±0.000) & \bfseries 0.062 (±0.001) \\
\sn   & \bfseries 0.004 (±0.000) & 0.055 (±0.001) & 0.096 (±0.000) & \bfseries 0.001 (±0.001) & 0.050 (±0.004) & 0.091 (±0.004) \\
\bottomrule
\end{tabular}
\end{table}

\begin{figure*}[t]
  \centering
  \includegraphics[width=\linewidth]{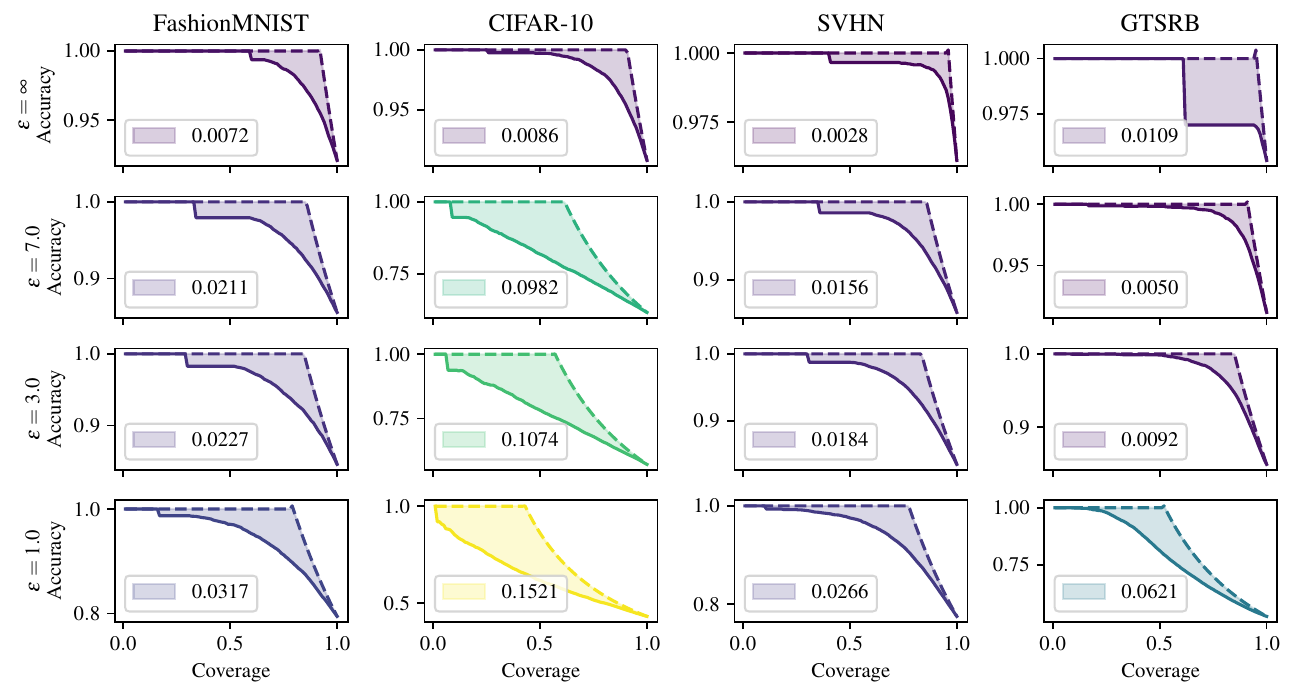}
\caption[Distance to accuracy-dependent upper bound for the \sctd method.]{\textbf{Distance to accuracy-dependent upper bound for the \sctd method.} We plot both the accuracy-coverage trade-off of \sctd, \ie $\text{acc}_c(f,g)$, with a solid line and the upper bound $\overline{\text{acc}}(a_\text{full},c)$ with a dashed line. The shaded region enclosed between the two curves corresponds to the accuracy-normalized SC score $s_{a_\text{full}}(f,g)$. We observe that the score grows with stronger DP levels (shown by brighter colors), \ie selective classification becomes harder at low $\varepsilon$s.}
\label{fig:acc_cov_bound}
\end{figure*}

\paragraph{Recovering Non-Private Utility by Reducing Coverage.}

Recall that one key motivation for applying selective classification to a DP model is to recover non-private model utility at the expense of coverage. We investigate this coverage cost in detail by examining how many samples a DP model can produce decisions for while maintaining the utility level of the respective non-private model (\ie a model trained on the same dataset without DP). The results are outlined in Table~\ref{tab:coverage_performance}. For high (\ie $\varepsilon=7$) and moderate (\ie $\varepsilon=3$) privacy budgets, we observe that a reduction of $20\%-30\%$ in data coverage recovers non-private utility. However, at low (\ie $\varepsilon=1$) privacy budget we observe that dataset coverage degrades strongly on most datasets. In particular, we find that for CIFAR-10 (across all $\varepsilon$ values) and GTSRB (at $\varepsilon=1$), the coverage reduces to below $30\%$. This result showcases that, depending on the particular choice of $\varepsilon$, recovering non-private model performance from a DP model can come at a considerable coverage cost. Moreover, this coverage cost varies widely across SC methods with \sctd leading to the lowest incurred loss in coverage by far.

\paragraph{Disentangling Selective Prediction Performance From Base Accuracy.}

Recall from Section~\ref{sec:new_metric} that the standard evaluation pipeline for selective classification is not suitable for comparing DP models across varying privacy guarantees. To overcome this issue, we compute the accuracy-dependent upper bound (Equation~\ref{eq:sc_dp_bound}) for each experiment and measure how closely the achieved accuracy/coverage trade-off aligns with this bound as per Equation~\ref{eq:acc_norm_score}. We document these results computed over all datasets, SC methods, and $\varepsilon$ levels in Table~\ref{tab:sc_score_performance}. We find that, as $\varepsilon$ decreases, all methods perform progressively worse. That is, for stronger privacy, all methods increasingly struggle to identify points they predict correctly on. Recall that this behavior is expected based on our discussion in Section~\ref{sec:dp_affects_sc}. Again, we observe \sctd offering the strongest results, leading to the smallest bound deviation. To graphically showcase closeness to the upper bound, we further plot the accuracy-coverage trade-off and the corresponding upper bound for each experiment with the \sctd method in Figure~\ref{fig:acc_cov_bound}. 

\begin{figure*}[b]
  \centering
  \includegraphics[width=\linewidth]{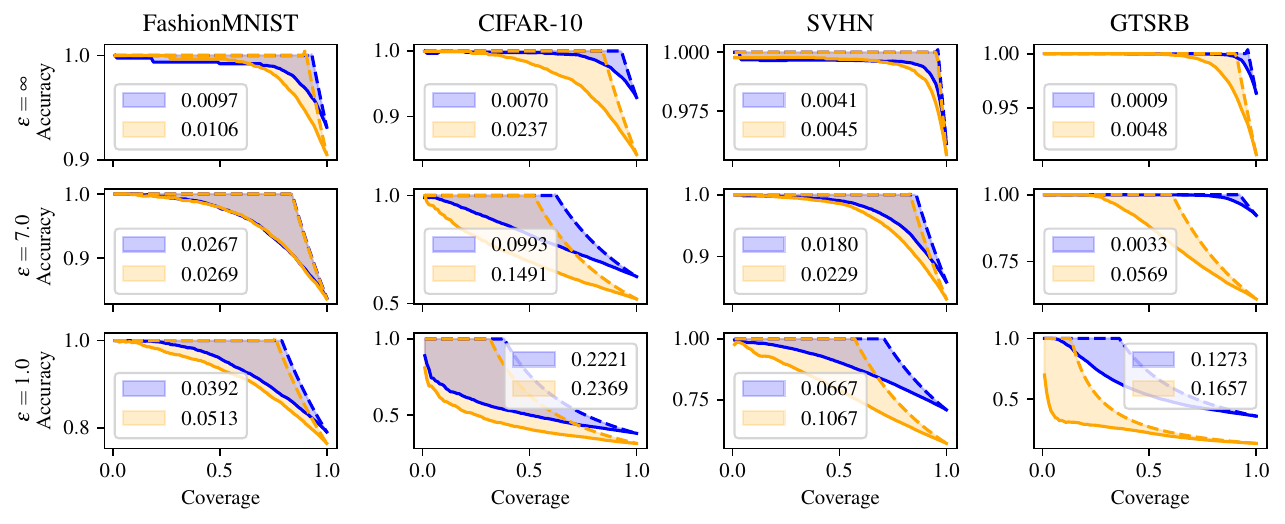}
\caption[Sequential composition (\de) vs parallel composition (\texttt{DE-PART}) for deep ensembles.]{\textbf{Sequential composition (\de) vs parallel composition (\texttt{DE-PART}) for deep ensembles.} We observe that \texttt{DE-PART} (orange) under-performs \de (blue) for selective classification across all $\varepsilon$.}
\label{fig:acc_cov_bound_de}
\end{figure*}

\paragraph{Parallel Composition With Partitioned Ensembles.}

As previously stated, Deep Ensembles under-perform under DP due to the sequential composition property. One potential mitigation strategy is to partition the data and train isolated models with no data overlap. This circumvents composition but also limits the utility of individual ensemble members as less data is available for training. We experiment with both setups and report results in Figure~\ref{fig:acc_cov_bound_de}. Overall, these experimental results indicate that (i) partitioned deep ensembles lead to lower full-coverage accuracy when compared to non-partitioned deep ensembles while (ii) at the same time being less performant in terms of selective classification performance.

\section{Conclusion}
\label{sec:concl}

In this work we have studied methods for performing selective classification under a differential privacy constraint. To that end, we have highlighted the fundamental difficulty of performing selective prediction under differential privacy, both via a synthetic example and empirical studies. To enable this analysis, we introduced a novel score that disentangles selective classification performance from baseline utility. While we establish that SC under DP is indeed challenging, our study finds that a specific method (\sctd) achieves the best trade-offs between SC performance and privacy budget.

\paragraph{Limitations.} This work presents insights drawn from empirical results and we acknowledge that a more thorough theoretical analysis is needed for a deeper understanding of the interplay between SC and DP. Also, we did not carefully investigate fairness implications beyond class imbalance. This connection to fairness requires special attention with past works showing that both DP and SC can negatively affect sensitive subgroups~\citep{jones2020selective, bagdasaryan2019differential}.

\paragraph{Future Work.} We believe that this work initiates a fruitful line of research. In particular, future work can expand on the intuition provided here to obtain fundamental bounds on selective classification performance in a DP setting. Although we did not focus on this here, we believe that the ideas developed in this chapter could help mitigate the trade-offs between privacy and fairness.

    \newcommand{\E}{\mathbb{E}}
\newcommand{\R}{\mathbb{R}}

\chapter{What Does It Take to Build a Performant Selective Classifier?}
\label{ch:sc_bounds}

\begin{paperref}
\normalfont
The contents of this chapter consist of research and results taken from: \citet{rabanser2025what}: \emph{\bibentry{rabanser2025what}}
\end{paperref}

\section*{Summary}

\looseness=-1
Selective classifiers improve reliability by abstaining on uncertain inputs, yet their performance often lags behind the \emph{perfect-ordering} oracle that accepts examples in exact order of correctness. We formulate this shortfall as a \emph{coverage‑uniform selective‑classification gap} and prove the first finite‑sample decomposition that pinpoints five distinct sources of looseness: Bayes noise, approximation error, ranking error, statistical noise, and implementation or shift‑induced slack. Our bound shows that \emph{monotone} post‑hoc calibration cannot reduce the gap, as it preserves the original score ordering; closing the gap therefore requires scoring mechanisms that can \emph{modify the ranking} induced by the base model. We validate our gap decomposition on synthetic two‐moons data and real‐world vision benchmarks, isolating each error component via controlled experiments. Results confirm that (i) Bayes noise and limited model capacity alone explain large gaps, (ii) only non‑monotone or feature‑aware calibrators shrink the ranking term, and (iii) distribution shift adds a distinct slack that must be addressed by robust training. Together, our decomposition supplies a quantitative \emph{error budget} and concrete design guidelines for building selective classifiers that approach ideal oracle behavior.

\section{Introduction}
\label{sec:intro}

In high-stakes applications like finance~\citep{9260038}, healthcare~\citep{guan2020bounded}, and autonomous driving~\citep{ghodsi2021generating}, machine learning (ML) models are increasingly tasked with making decisions under uncertainty, where dependable predictions are critical. Selective classifiers~\citep{chow1957optimum, el2010foundations} formalize the option to abstain on inputs deemed unreliable, reducing the risk of costly errors by refusing to predict when uncertain. Their effectiveness depends on identifying which predictions to trust and which to defer. A common evaluation metric is the \emph{accuracy–coverage} tradeoff, which quantifies how performance degrades as the model accepts a broader set of inputs. The benchmark is a hypothetical oracle that ranks inputs by their true likelihood of correctness, yielding a \emph{perfect-ordering upper bound}~\citep{geifman2018bias, rabanser2023training}. While some models approach this bound, others fall short—revealing persistent gaps and raising open questions about what properties of the learning setup truly govern selective performance.

\begin{wrapfigure}{R}{0.5\textwidth}
\vspace{-5pt}
    \centering
    \resizebox{\linewidth}{!}{
    \begin{tikzpicture}[
		declare function={}
		]
\begin{axis}[%
  xlabel=Coverage ($c$),
  ylabel=Selective Accuracy,
  xmin = -0.1,
  xmax = 1.1,
  ymin = -0.1,
  ymax = 1.25,
  grid=major,
  width=8cm,
  height=5.5cm,
  tick label style={/pgf/number format/fixed},
  legend style={at={(0.812,1.001)}, anchor=north},]
  
  \newcommand\afull{0.4}
  \newcommand\afullinccoord{{\afull + (1 - \afull)/2}}

  \addplot+[mark={},line width=2pt,Black,name path=C, domain=0:\afull/2] {1};
		  
  \addplot+[mark={},line width=2pt,Black,name path=D, domain=\afull/2:\afull] {15/16*x^2 - 15/8 * x + 107/80};
  
  \addplot+[mark={},line width=2pt,Black,name path=E, domain=\afull:1] {15/16*x^2 - 15/8 * x + 107/80};
  
  \addplot+[mark={},dashed,line width=2pt,Black,name path=A, domain=0:\afull] {1};
  \addplot+[mark={},dashed,line width=2pt,Black,name path=B, domain=\afull:1] {\afull/x};
  \addplot+[draw=none,mark=none,name path=C,domain=0:1] {0};
  \addplot+[Cyan!25!white] fill between[of=A and C,soft clip={domain=0:\afull}];
  \addplot+[Orange!25!white] fill between[of=B and C,soft clip={domain=\afull:1}];
  %\draw [black, line width=1pt, dashed] (\afull,0) -- (\afull, 1.0);
  \node[black, fill=white] at (1,\afull+0.14) {\footnotesize $a_\text{full}$};
  \node[black, fill=white] at (\afull,1.13) {\footnotesize $c = a_\text{full}$};

  % \addplot+[black,fill=black] coordinates{(\afull,1)};
  % \addplot+[black,fill=black] coordinates{(1,\afull)};

  \fill[black] (\afull, 1) circle[radius=2.5pt];
  \fill[black] (1,\afull) circle[radius=2.5pt];

  \addplot+[Black!25!white] fill between[of=A and D,soft clip={domain=0:1}];
  
    \addplot+[Black!25!white] fill between[of=B and E,soft clip={domain=0:1}];
  
  \draw [->, Orange, line width=2pt] (\afull,0.1) -- (1, 0.1);
  \draw [->, Cyan, line width=2pt] (0,0.1) -- (\afull, 0.1);
  \node[Cyan] at (\afull/2,0.25) {\footnotesize correct points};
  \node[Orange] at (\afullinccoord,0.25) {\footnotesize incorrect points};

  \legend{ \scriptsize \realtradeoff,,, \scriptsize \upperbound,,,,, \scriptsize $\int_0^1 \widehat{\Delta}(c)dc$};
 
\end{axis}
\end{tikzpicture}
    }
    \vspace{-12pt}
    \caption[Selective classification gap $\Delta(c)$.]{\textbf{Selective classification gap $\Delta(c)$.}  
      The dashed curve is the oracle frontier \(\overline{\mathrm{acc}}(a_{\text{full}},c)\) under which coverage levels left of \(c=a_{\text{full}}\) (\textcolor{Cyan}{blue}) accept correct predictions first, and rank incorrect predictions last (\textcolor{Orange}{orange}). The solid curve shows the realized selective accuracy \(\mathrm{acc}_{c}(h,g)\).  
      The mismtach between \(\overline{\mathrm{acc}}(a_{\text{full}},c)\) and \(\mathrm{acc}_{c}(h,g)\) at coverage $c$ is the gap \(\Delta(c)\);  
      the gray area visualizes the gap area over the full coverage spectrum.}
      \vspace{-25pt}
    \label{fig:bounds_overview}
\end{wrapfigure}
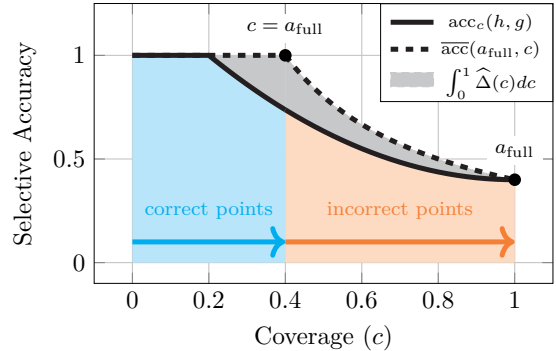

Classical theory explains selective classification in two idealized regimes. In the \emph{realizable} setting~\citep{el2010foundations}, where the data is noiseless and the true predictor lies within the hypothesis class, the model can asymptotically achieve the perfect accuracy–coverage curve. In the more general \emph{agnostic} setting~\citep{wiener2011agnostic}, the classifier competes with the best-in-class predictor, but this benchmark may itself fall well below the oracle bound—and the theory does not isolate the source of the gap. As a result, even the strongest formal guarantees offer little practical guidance:
\begin{quote}
\centering
\emph{For my finite model on finite data, what aspects of the learning setup will actually move my trade-off curve closer to the upper bound?}
\end{quote}

To answer this question, we re‑frame selective performance around the \emph{selective classification
gap}~\(\Delta(c)\): the mismatch between a model’s accuracy–coverage curve and the oracle
bound (see Figure~\ref{fig:bounds_overview} for an illustrative example). 
We show that this gap admits a finite‑sample decomposition:
\begin{equation}
\label{eq:intro-gap}
\widehat{\Delta}(c)
\;\le\;
\underbrace{\varepsilon_{\text{Bayes}}(c)}_{\text{\scriptsize irreducible}}
+\;\underbrace{\varepsilon_{\text{approx}}(c)}_{\text{\scriptsize capacity}}
+\;\underbrace{\varepsilon_{\text{rank}}(c)}_{\text{\scriptsize ranking}}
+\;\underbrace{\varepsilon_{\text{stat}}(c)}_{\text{\scriptsize data}}
+\;\underbrace{\varepsilon_{\text{misc}}(c)}_{\text{\scriptsize optimization\;\&\;shift}},
\qquad\forall c\in(0,1].
\end{equation}

Each term corresponds to a distinct—and often \emph{measurable}—source of looseness. 
The first term, \(\varepsilon_{\text{Bayes}}(c)\), reflects irreducible uncertainty: if the true label is inherently unpredictable from the input (e.g., due to label noise), even a perfect classifier must abstain on some examples. 
Next, \(\varepsilon_{\text{approx}}(c)\) captures limits of the model class: if the function class is too weak to approximate the Bayes-optimal decision rule, the gap widens. 
The third term, \(\varepsilon_{\text{rank}}(c)\), quantifies the model’s failure to correctly order inputs by their likelihood of correctness—typically due to poor confidence estimation or miscalibration. 
The statistical term \(\varepsilon_{\text{stat}}(c)\) accounts for finite-sample fluctuations that affect both learning and evaluation. 
Finally, \(\varepsilon_{\text{misc}}(c)\) aggregates practical imperfections, such as optimization error or test-time distribution shift. Equation~\eqref{eq:intro-gap} thus provides a coverage-uniform ``error budget'' that transforms the qualitative question posed earlier into a concrete quantitative diagnosis.

\looseness=-1
Two insights, developed further in later sections, are worth previewing. First, monotone post‑hoc calibration cannot reduce the \emph{ranking} term \(\varepsilon_{\text{rank}}(c)\), as it preserves the total order of scores—leaving the accuracy–coverage curve unchanged. Second, Equation~\eqref{eq:intro-gap} serves as an \emph{error budget} that identifies cost‑effective levers: (i) increase capacity or distill from a more expressive teacher to shrink \(\varepsilon_{\text{approx}}\); (ii) use additional or repeated labels and noise‑robust losses to reduce \(\varepsilon_{\text{Bayes}}\); (iii) enlarge validation data to lower \(\varepsilon_{\text{stat}}\); and (iv) apply domain adaptation or importance weighting to address \(\varepsilon_{\text{misc}}\).

\textbf{Contributions.}
We summarize our main contributions below:
\begin{itemize}
    \item \textbf{Problem formulation.}  
          We recast selective prediction in terms of a \emph{coverage‑uniform selective‑classification gap}—the key quantity to minimize to approach oracle behavior.  
This framing unifies prior work and highlights which failure modes dominate at each coverage level.

    \item \textbf{Theoretical analysis.}  
          We present the first \emph{finite‑sample decomposition} of the gap (Eq.~\eqref{eq:intro-gap}), dividing it into five terms: Bayes, approximation, ranking, statistical, and miscellaneous errors.  
Our analysis further shows that \emph{monotone calibration cannot reduce the gap}, motivating ranking‑aware methods.

\item \textbf{Empirical validation.}  
      Experiments—from two‑moons to large‑scale vision—confirm the decomposition: Bayes noise and capacity limits drive large gaps; temperature scaling improves calibration but not ranking; and shift-aware methods remain essential under distribution shift. \emph{These results clarify which factors matter most and how to target them effectively in practice.}
\end{itemize}

\section{Background \& Related Work}
\label{sec:background}

Selective classification extends the standard supervised classification framework as follows:

\begin{definition}[Selective Classifier~\citep{chow1957optimum, el2010foundations}]
A selective classifier is a pair \( (h, g) \), where \( h: \mathcal{X} \to \mathcal{Y} \) is a classifier over covariates \( \mathcal{X} = \mathbb{R}^D \) and labels \( \mathcal{Y} = \{1, \ldots, K\} \), and \( g: \mathcal{X} \times (\mathcal{X} \to \mathcal{Y}) \to \mathbb{R} \) is a selection function that assigns a confidence score.  
Given a threshold \( \tau \in \mathbb{R} \), the model abstains when the score falls below the threshold:
\begin{equation}
    \label{eq:sel_class}
    (h, g)(x) = \begin{cases}
    h(x) & \text{if } g(x, h) \geq \tau \\
    \bot & \text{otherwise}
    \end{cases}
    \enspace .
\end{equation}
\end{definition}

Intuitively, a selective classifier predicts only when confident.  
The selection score \(g(x, h)\) determines whether to accept or abstain: if \(g(x, h) \geq \tau\), the model outputs \(h(x)\); otherwise, it returns \(\bot\).

Many prior works have developed selective classification methods for training competitive pairs \((h, g)\).  
A popular selective prediction method is \emph{Softmax Response} (\sr)~\citep{hendrycks2016baseline, geifman2017selective}, which uses classifier confidence as the selection score.  
To improve calibration and reduce predictive variance, ensembling approaches have been explored: \emph{Deep Ensembles} (\de)~\citep{lakshminarayanan2017simple} train multiple models with different initializations, while \emph{Selective Classification via Training Dynamics} (\sctd)~\citep{rabanser2022selective} ensembles intermediate checkpoints.  
Other methods—such as \emph{SelectiveNet} (\sn)\citep{geifman2019selectivenet}, \emph{Deep Gamblers} (\dg)\citep{liu2019deep}, and \emph{Self-Adaptive Training} (\sat)~\citep{huang2020self}—alter the model architecture or loss function ensuring that prediction and rejection are learned jointly.

The efficacy of a selective classifier is evaluated using the empirical accuracy-coverage tradeoff.

\begin{definition}[Empirical Accuracy–Coverage Tradeoff]
\label{def:emp_acc_cov}
Let \(D=\{(x_i,y_i)\}_{i=1}^N\) be a dataset.  For a selective classifier \((h,g)\) and threshold \(\tau\), define
\begin{align}
\label{eq:emp_coverage}
\hat{\xi}_{h,g}(\tau)
&= \frac{1}{N}\,\bigl|\{\,i : g(x_i,h) \ge \tau \}\bigr|,
\\[6pt]
\label{eq:emp_accuracy}
\hat{\alpha}_{h,g}(\tau)
&= 
\begin{cases}
\displaystyle
\frac{\bigl|\{\,i : h(x_i)=y_i \text{ and } g(x_i,h) \ge \tau \}\bigr|}{
      \bigl|\{\,i : g(x_i,h) \ge \tau \}\bigr|}, 
& \text{if } \hat{\xi}_{h,g}(\tau)>0,\\[10pt]
0, & \text{if } \hat{\xi}_{h,g}(\tau)=0.
\end{cases}
\end{align}
The pair \((\hat{\xi}, \hat{\alpha})\) as \(\tau\) varies is the empirical accuracy–coverage curve.
\end{definition}
The score \(g(x,h)\) therefore induces a total order over \(D\): \(x_1\) is accepted before \(x_2\) if \(g(x_1,h) > g(x_2,h)\).  
This ordering governs which inputs are retained as coverage decreases.  
Effective strategies aim to maximize \(\hat{\alpha}\) at each coverage level \(\hat{\xi}\), often trading off accuracy and coverage.

\paragraph{Accuracy–coverage tradeoff evaluation.}
The accuracy–coverage tradeoff is often summarized by the area under the accuracy–coverage curve (\texttt{AUACC}), integrating selective accuracy over all coverage levels. However, \citet{geifman2018bias} show that \texttt{AUACC} favors models already accurate at full coverage. To address this issue, \citet{geifman2018bias} and \citet{rabanser2023training} propose oracle-based bounds, which become loose at low utility~\citep{galil2023can}. To avoid accuracy bias, \citet{galil2023can} and \citet{pugnana2023auc} recommend using the classifier’s \texttt{AUROC} instead. But \texttt{AUROC} is not monotonic in \texttt{AUACC}~\citep{cattelan2023fix, ding2020revisiting}, thus favoring methods tuned for \texttt{AUROC} over selective accuracy. Earlier work~\citep{el2010foundations, wiener2011agnostic} characterizes optimal selective classifiers in both realizable and agnostic regimes but focuses on existence rather than practical instantiation—unlike our finite-sample perspective.

\section{Decomposing the Selective Classification Gap To Optimality}
\label{sec:methods}

We characterize the optimal performance achievable by a selective classifier given its full-coverage accuracy, establishing a reference against which all practical selective classifiers can be evaluated.

\subsection{Oracle Bound and Selective Classification Gap}

\begin{definition}[Perfect Ordering Upper Bound~\citep{geifman2018bias, rabanser2023training}]
\label{def:poub}
Fix a base classifier \(h\) whose full‑coverage (standard) accuracy is
\(a_{\text{full}}:=\Pr\bigl(h(X)=Y\bigr)\in[0,1]\).
For any desired coverage level \(c\in(0,1]\), the best selective
accuracy—achieved by accepting the \(c\)-fraction of points with the \emph{highest}
posterior correctness $\Pr(h(X)=Y\mid X)$—is
\begin{equation}
\label{eq:bound}
\overline{\mathrm{acc}}\bigl(a_{\text{full}},c\bigr)
=\;
\begin{cases}
1, & 0 < c \le a_{\text{full}}, \\[6pt]
\dfrac{a_{\text{full}}}{c}, & a_{\text{full}} < c < 1.
\end{cases}
\end{equation}
\end{definition}

Assuming no Bayes noise—that is, all errors are avoidable given perfect confidence—this piecewise curve (see Figure~\ref{fig:bounds_overview}) traces the \emph{oracle} accuracy–coverage frontier based on a perfect ranking of examples by correctness probability. Any real selective classifier falls below this bound—potentially far below, depending on its calibration, expressivity, and sensitivity to noise. To quantify how far a given classifier falls short of this ideal, we define the \emph{selective classification gap}.

\begin{definition}[Selective Classification Gap]
\label{def:gap}
Let \((h,g)\) be a selective classifier with full‑coverage accuracy 
\(a_{\mathrm{full}}=\Pr(h(X)=Y)\).  For a coverage level \(c\in(0,1]\), let
\(\tau_c\) be the threshold satisfying \(\Pr\bigl(g(X,h)\ge\tau_c\bigr)=c\).  The selective accuracy at coverage \(c\) is
\(
\mathrm{acc}_c(h,g)
:=
\Pr\bigl(h(X)=Y \;\bigm|\; g(X,h)\ge\tau_c\bigr).
\)
The selective classification gap at coverage \(c\) is then defined as the deviation from the perfect-ordering upper bound:
\begin{equation}
\Delta(c)
:=
\overline{\mathrm{acc}}\bigl(a_{\mathrm{full}},c\bigr)
\;-\;\mathrm{acc}_c(h,g).
\end{equation}
This gap $\Delta(c)$ represents the excess selective risk at a given coverage $c$. Its integral over the entire coverage spectrum, $\int_0^1 \Delta(c) dc$, is equivalent to the Excess-AURC (E-AURC) metric proposed by \citet{geifman2018bias}.
\end{definition}

The function \(\Delta(c)\) offers a coverage-resolved diagnostic of selective performance. A small gap indicates near-oracle behavior—accepting only examples it can confidently and correctly classify—while a large gap suggests limitations in estimating correctness or ranking examples reliably. Understanding the magnitude and shape of this gap is key to analyzing and improving selective classifiers.

\subsection{Why Is the Upper Bound Loose?}
\label{sec:why-loose}

The oracle bound in Definition~\ref{def:poub} relies on two idealized
assumptions: perfect prediction on all inputs and perfect ranking by
the true correctness posterior. In practice, selective classifiers deviate in
four principal ways, each corresponding to a term in our later
decomposition (\(\varepsilon_{\text{Bayes}},\varepsilon_{\text{approx}},
\varepsilon_{\text{rank}},\varepsilon_{\text{stat}}\)):

\begin{enumerate}

\item \textbf{Bayes noise (\(\varepsilon_{\text{Bayes}}\)).}  
  Even a Bayes-optimal rule errs on intrinsically ambiguous points
(where \(\max_y \Pr(Y=y\mid X)<1\)), unavoidable in real data~\citep{devroye2013probabilistic}.  
As coverage increases, the oracle must accept some of these noisy inputs, lowering the achievable accuracy.

\item \textbf{Approximation limits (\(\varepsilon_{\text{approx}}\)).}  
  A learned model \(h\) drawn from a restricted hypothesis class may
  misclassify inputs with high posterior confidence under the Bayes rule~\citep{bishop2006pattern}.  
  This gap reduces full-coverage accuracy and limits selective performance.

\item \textbf{Ranking error (\(\varepsilon_{\text{rank}}\)).}  
  Let \(\eta_h(x):=\Pr\bigl(h(x)=Y\mid X=x\bigr)\) denote the true
  correctness posterior. The confidence score \(g(X,h)\) may
  mis-order inputs relative to \(\eta_h(x)\), leading easy examples
  to be rejected and harder ones accepted, which increases \(\Delta(c)\).  

\item \textbf{Statistical noise (\(\varepsilon_{\text{stat}}\)).}  
  Estimating the threshold \(\tau_c\) and selective accuracy from a finite validation set introduces randomness
  of order \(\mathcal{O}(\sqrt{\log(1/\delta)/n})\). This follows from concentration bounds; see~\citet{shalev2014understanding} for standard applications in learning theory.

\end{enumerate}

\begin{takeaway}
The selective classification gap \(\Delta(c)\) reflects a mix of irreducible noise,
model misspecification, ranking errors, and sampling variability. Addressing each—via cleaner labels,
stronger models, or improved ranking—can tighten selective prediction performance.
\end{takeaway}

\looseness=-1
Next, we formalize this decomposition and provide a general bound on the total gap.

\subsection{Formal Decomposition of the Gap}
\label{sec:formal-gap}

We now give a principled decomposition of the selective classification gap and provide a corresponding finite-sample upper bound. For clarity and notational simplicity, we treat the binary‑label case \(\mathcal{Y}=\{0,1\}\); the multiclass extension follows by a standard one‑vs‑rest reduction.

\textbf{Notation.} Let \(\eta(x):=\Pr\bigl(Y=1\mid X=x\bigr)\) be the Bayes posterior.
For a fixed classifier \(h:\mathcal{X}\to\mathcal{Y}\) define its
(induced) correctness posterior
\begin{equation}
  \eta_h(x)\;:=\;\Pr\bigl(h(x)=Y\mid X=x\bigr)
  \;=\;\eta(x)\,\mathbb{I}_{\{h(x)=1\}}+
        \bigl(1-\eta(x)\bigr)\mathbb{I}_{\{h(x)=0\}}.
\end{equation}
All expectations and probabilities are taken w.r.t.\ the true data distribution
\(\mathcal{D}\). Throughout let \(g(x,h)\) be the confidence score.
For a target coverage \(c\in(0,1]\) denote by
\begin{equation}
  t_c \quad \text{s.t.}\quad
  \Pr\bigl(g(X,h)\ge t_c\bigr)=c
\end{equation}
the \emph{population threshold}, and write the
\emph{accepted region}
\(A_c:=\{x:g(x,h)\ge t_c\}\).  
The oracle that attains the perfect‑ordering bound accepts $A_c^{\star}:=\bigl\{x:\eta_h(x)\text{ is among the largest }c\text{-fraction}\bigr\}$.

\textbf{Error Terms.} We isolate the following sources of error affecting selective prediction performance:
\begin{align}
\varepsilon_{\text{Bayes}}(c)
&:=\E\Bigl[1-\max\{\eta(X),1-\eta(X)\}\;\Bigm|\;X\in A_c\Bigr],
\\[4pt]
\varepsilon_{\text{approx}}(c)
&:=\E\Bigl[\bigl|\eta_h(X)-\eta(X)\bigr|\;\Bigm|\;X\in A_c\Bigr],
\\[4pt]
\varepsilon_{\text{rank}}(c)
&:=\E\bigl[\eta_h(X)\mid X\in A_c^{\star}\bigr]
  -\E\bigl[\eta_h(X)\mid X\in A_c\bigr]\;\;\;\;\;\;(\ge0),
\\[4pt]
\varepsilon_{\text{stat}}(c)
&:=C\sqrt{\frac{\log(1/\delta)}{n}},
\end{align}
where \(n\) is the evaluation‑set size, \(\delta\in(0,1)\) a confidence
parameter, and \(C>0\) an absolute constant. Intuitively, \(\varepsilon_{\text{Bayes}}\) is the irreducible label noise inside the accepted region; \(\varepsilon_{\text{approx}}\) measures how far \(h\) is from Bayes‑optimal on the \emph{selected} inputs; and \(\varepsilon_{\text{rank}}\) is a \emph{ranking regret}—the accuracy loss due solely to picking the wrong \(c\)-fraction of examples.

\begin{remark}[Distance to a Perfect Ranker]
A natural way to gauge how far the learned acceptance rule is from the oracle is
the mass mis-ordered
\begin{equation}
    D_{\text{rank}}(c)\;:=\;
    \Pr\bigl(X\in A_c^{\star}\setminus A_c\bigr)
    +\Pr\bigl(X\in A_c\setminus A_c^{\star}\bigr).
\end{equation}
It equals the total probability of examples that would have to be
\emph{swapped} between $A_c$ and $A_c^{\star}$ to recover perfect ordering.
Hence $D_{\text{rank}}(c)=0$ iff $A_c=A_c^{\star}$, in which case
$\varepsilon_{\text{rank}}(c)$ also vanishes.
\end{remark}

\begin{theorem}[Selective Gap Bound]
\label{thm:gap}
For a coverage level \(c\in(0,1]\) and a selective classifier \((h,g)\) the population gap obeys
\begin{equation}
\Delta(c)=\overline{\mathrm{acc}}\bigl(a_{\mathrm{full}},c\bigr)
-\mathrm{acc}_c(h,g)
\;\le\;
\varepsilon_{\text{Bayes}}(c)
+\varepsilon_{\text{approx}}(c)
+\varepsilon_{\text{rank}}(c).
\label{eq:pop-gap-ranking}
\end{equation}
Let \(\widehat{\Delta}(c)\) be the empirical gap on \(n\) i.i.d.\
test points.  Then, with probability at least \(1-\delta\),
\begin{equation}
\widehat{\Delta}(c)
\;\le\;
\varepsilon_{\text{Bayes}}(c)
+\varepsilon_{\text{approx}}(c)
+\varepsilon_{\text{rank}}(c)
+C\sqrt{\tfrac{\log(1/\delta)}{n}}.
\label{eq:emp-gap-ranking}
\end{equation}
\end{theorem}

\begin{proof}
Because
\(
\mathrm{acc}_c(h,g)=\E[\eta_h(X)\mid A_c],
\)
the gap decomposes as
\[
  \Delta(c)
  =\underbrace{\E[\eta_h\mid A_c^{\star}]
               -\E[\eta_h\mid A_c]}_{\varepsilon_{\text{rank}}(c)}
   \;+\;
   \underbrace{\E[\eta_h-\mathbb{I}_{\{h=Y\}}\mid A_c]}
              _{\varepsilon_{\text{approx}}(c)}
   \;+\;
   \underbrace{\E[1-\max\{\eta,1-\eta\}\mid A_c]}
              _{\varepsilon_{\text{Bayes}}(c)}.
\]
This yields the population bound \eqref{eq:pop-gap-ranking}. 
For each expectation in the decomposition apply Hoeffding’s
inequality, a union bound over the three terms gives, with probability
\(1-\delta\),
\(
  |\widehat{\Delta}(c)-\Delta(c)|
  \le C\sqrt{\log(1/\delta)/n}.
\)
Adding this deviation to \eqref{eq:pop-gap-ranking} establishes
\eqref{eq:emp-gap-ranking}. \\See Appendix~\ref{app:proof-gap-ranking} for an extended proof with detailed intermediate steps.
\end{proof}

\subsection{Calibration and Its (Limited) Effect on the Gap}
\label{sec:calibration-gap}

As shown in Theorem~\ref{thm:gap}, the selective classification gap includes a \emph{ranking error} term \(\varepsilon_{\text{rank}}(c)\), which captures misalignment between the confidence score and true correctness. Model calibration~\citep{niculescu2005predicting}—widely used to reduce over- or underconfidence—is often assumed to improve this alignment by transforming scores to better reflect correctness likelihood. Yet its effect on selective performance remains ambiguous and context-dependent. Prior work has reached conflicting conclusions: \citet{zhu2022rethinking} argue that calibration may degrade abstention behavior, while \citet{galil2023can} find that temperature scaling can improve selective prediction in practice. We show that the impact on the gap depends critically on the \emph{type} of calibration method used and its influence on the induced ranking. We begin by recalling the formal definition of calibration.

\begin{definition}[Perfect Calibration]
\label{def:calibration}
For each input \(x\) let a model produce a predicted label \(\hat y(x)\) and an associated confidence score \(s(x)\in[0,1]\). We say the model is perfectly calibrated if
\begin{equation}
  \Pr\bigl(Y = \hat y(X)\;\bigm|\;s(X)=t\bigr) \;=\; t \qquad \text{for every confidence level}\ t\in[0,1].
  \label{eq:perfect-cal}
\end{equation}
\end{definition}

Practical estimators approximate~\eqref{eq:perfect-cal} via a post-hoc map \(\phi\) such that \(\tilde s(x)=\phi(s(x))\) approaches prefect calibration. \emph{Expected Calibration Error (ECE)}~\citep{naeini2015obtaining} quantifies this closeness:
\begin{equation}
  \text{ECE} = \sum_{b=1}^B \frac{|I_b|}{n}
  \left| \frac{1}{|I_b|} \sum_{i \in I_b} \mathbb{I}\{ \hat y(x_i) = y_i \}
         - \frac{1}{|I_b|} \sum_{i \in I_b} \tilde s(x_i) \right|,
  \label{eq:ece}
\end{equation}
where \(I_b\) is the set of indices in bin \(b\), \(n\) is the total number
of examples, and \(B\) is the number of bins.

\textbf{Monotone Score-Level Calibration Leaves the Gap Intact.}
Temperature scaling~\citep{guo2017calibration}, isotonic regression~\citep{zadrozny2002transforming}, and histogram
binning~\citep{zadrozny2001obtaining} all fit a \emph{monotone} \(\phi\colon[0,1]\to[0,1]\) that preserves score ordering.
Because monotone maps preserve ordering,
the acceptance set
\(A_c=\{x:\tilde s(x)\ge\tau_c\}\)
is identical to the one obtained from \(s(x)\);
hence the selective accuracy
\(\mathrm{acc}_c(h,g)\)
and the gap
\(
\Delta(c)=\overline{\mathrm{acc}}\bigl(a_{\text{full}},c\bigr)-\mathrm{acc}_c(h,g)
\)
are \emph{unchanged}.
Monotone calibration thereby reduces the approximation error \(\varepsilon_{\text{approx}}(c)\) in Section~\ref{sec:formal-gap} but leaves the ranking error \(\varepsilon_{\text{rank}}(c)\) untouched.

\textbf{Why Temperature Scaling \texorpdfstring{\emph{Can}}{} Reorder But Rarely Does.}
Temperature scaling, the most widely used post-hoc calibration technique, divides every logit vector
\(z(x)\in\mathbb{R}^K\) by a scalar \(T>0\),
\begin{equation}
p_j^{(T)}(x)=
\frac{\exp(z_j(x)/T)}{\sum_k \exp(z_k(x)/T)}.
\end{equation}
Two inputs with different logit \emph{margins} can swap their
max-probability after rescaling. In theory this breaks the monotone
ranking guarantee; in practice such ties occur only when the runner-up logit gaps are extremely small. Consequently, the induced change in the
accuracy–coverage curve is negligible—a pattern we also observe empirically. See Appendix~\ref{app:ts-rerank} for an extended discussion.

\textbf{Moving the Gap Requires Non-Monotone Scoring.}
To reduce the selective classification gap \(\Delta(c)\), one must go beyond simple monotone calibration and actively modify the acceptance ordering through one of the following approaches:

\begin{itemize}
    \item \textbf{Adaptive, ensemble, or feature-aware scoring}:  
          Techniques such as self-adaptive training (\(\texttt{SAT}\)), deep ensembles (\(\texttt{DE}\)), and learned calibration heads \(g_\psi(x)\) predict correctness scores from hidden representations or multiple model votes. These approaches leverage instance-specific uncertainty and shared feature context to rerank samples that otherwise receive similar raw confidence values.
          
\item \textbf{Binning or resampling}:  
      Techniques such as Bayesian Binning into Quantiles (BBQ)~\citep{naeini2015obtaining} and conformal p-value methods~\citep{angelopoulos2021gentle} refine confidence scores by estimating accuracy over adaptive data subsets. Both use empirical statistics—either at the bin level or per example—to recalibrate and reorder predictions, often improving ranking.

\item \textbf{Vector/Dirichlet scaling}:  
      A non-monotone transformation of the logit vector \(z\) via a learned affine map \(Wz + b\), with \(W \in \mathbb{R}^{K \times K}\), enables reordering of confidence scores by capturing inter-class dependencies, nonlinear score interactions, and representation-sensitive shifts~\citep{kull2019beyond}.

\end{itemize}

\looseness=-1
\textbf{Loss Prediction as a Multicalibration Litmus Test.}
A complementary view on how calibration connects to ranking ability starts from \emph{multicalibration}~\citep{hebert2018multicalibration},
the degree to which a model is calibrated on select subgroups. Recent work by \citet{gollakota2025loss} shows that strong multicalibration is closely tied to a model's ability to \emph{predict its own loss}, an idea seemingly closely related to selective prediction. We formalize this connection in Appendix~\ref{app:loss-pred} and show, both theoretically and empirically, that success or failure at the loss-prediction task corresponds directly to the magnitude of the ranking error \(\varepsilon_{\text{rank}}(c)\). In short, if a model's confidence scores cannot be out-predicted on their own mistakes, they are effectively multicalibrated—and near the oracle frontier.

\begin{takeaway}
Monotone score-level calibration (e.g., temperature scaling) is not enough to reduce the selective classification gap. Closing the gap requires scoring methods that actively change the ranking—through feature-aware heads, ensembles, or non-monotone transformations.
\end{takeaway}

\subsection{Additional Practical Sources of Looseness}
\label{sec:extra-slack-short}
The decomposition in Theorem~\ref{thm:gap} captures the \emph{intrinsic} sources of error—Bayes noise, approximation limits, ranking error, and sampling slack—forming a principled bound that holds even under perfect optimization, infinite data, and i.i.d.\ testing. In practical deployments, however, additional imperfections can inflate the empirical gap \(\widehat{\Delta}(c)\). These stem from implementation details, scoring granularity, and distribution shift—not fundamental limits, but contingent slack terms reducible through better engineering. We summarize them below under a single \emph{residual slack} term \(\varepsilon_{\text{misc}}(c)\).
\begin{enumerate}
  \item \textbf{Optimization error \(\varepsilon_{\text{opt}}\).}  
    In practice, gradient‐based solvers rarely attain the empirical risk minimizer. If \(L(\theta)\) denotes the end-to-end training objective—encompassing model architecture, loss (e.g.\ cross-entropy), and training data—and \(\hat\theta\) its final iterate, then $\varepsilon_{\text{opt}}\;=\; L(\hat\theta) \;-\; \min_{\theta}L(\theta)$, which—via standard surrogate‐to‐0/1 calibration bounds—translates into a nonzero selective‐accuracy loss that persists even under infinite data.
  \item \textbf{Distribution shift \(\varepsilon_{\text{shift}}(c)\).}  
  When the test distribution \(p_{\mathrm{test}}\) deviates from the training distribution \(p_{\mathrm{train}}\), both calibration and ranking typically degrade. In particular, for a hypothesis class \(\mathcal{H}\), the gap due to shift can be bounded by an \emph{Integral Probability Metric (IPM)}~\citep{muller1997integral}:
  \begin{equation}
    \varepsilon_{\text{shift}}(c)\;\le\;\mathrm{IPM}_{\mathcal{H}}\bigl(p_{\mathrm{train}},\,p_{\mathrm{test}}\bigr)
    := \sup_{f \in \mathcal{H}} \left| \mathbb{E}_{p_{\mathrm{train}}}[f] - \mathbb{E}_{p_{\mathrm{test}}}[f] \right|.
  \end{equation}
  Hence, larger shifts in distribution (relative to \(\mathcal{H}\)) lead to wider selective classification gaps.
\end{enumerate}
\textbf{Residual Slack.}  We absorb remaining implementation‐level effects (threshold‐selection noise, score quantization, etc.) into $\varepsilon_{\text{misc}}(c) :=\varepsilon_{\text{opt}}+\varepsilon_{\text{shift}}(c)$, yielding the streamlined high‐probability bound.
\begin{equation}
  \widehat\Delta(c)
  \;\le\;
  \underbrace{\varepsilon_{\text{Bayes}}(c)
               +\varepsilon_{\text{approx}}(c)
               +\varepsilon_{\text{rank}}(c)
               +\varepsilon_{\text{stat}}(c)}_{\text{intrinsic}}
  \;+\;
  \varepsilon_{\text{misc}}(c).
\end{equation}

\begin{takeaway}
\looseness=-1
Only \(\varepsilon_{\text{Bayes}}\) reflects irreducible uncertainty; the other intrinsic terms—\(\varepsilon_{\text{approx}}\), \(\varepsilon_{\text{rank}}\), and \(\varepsilon_{\text{stat}}\)—can be reduced with better models, calibration, and data. The \emph{miscellaneous slack} \(\varepsilon_{\text{misc}}\) highlights optimization and shift-robustness as levers for closing the gap to the oracle.
\end{takeaway}

\section{Empirical Results}
\label{sec:experiments}

Our experimental study is organized around three guiding questions that reflect the theoretical decomposition in Section~\ref{sec:methods}. Unless otherwise specified, all selective‑accuracy curves are averaged over 5 random seeds. By default, we use a ResNet-18 as the predictive model and adopt maximum softmax probability (\msp) as the selection mechanism---due to its simplicity and widespread use.

\subsection{Q1: How do Bayes error and approximation error shape the gap?}

\emph{Setup.}  
We conduct both synthetic and real-world experiments. For our synthetic results, which give us precise control over the data generation process, we simulate two sources of intrinsic difficulty on the two‑moons dataset:
(i) \textbf{noise $\sigma \in \{0.1,\,0.33,0.66,\,1.5\}$} controls how much the two moons expand into each other; and
(ii) \textbf{model capacity}, varied from logistic regression
(low capacity) to a shallow MLP (high capacity). For our real-world experiments we tackle the analysis similarly: for (i) we evaluate a trained CIFAR-10 model on the CIFAR-10N/100N~\citep{wei2021learning} datasets to assess which data points have large labeling disagreement; and for (ii) we vary the model architecture across a simple CNN (details in Appendix~\ref{app:simplecnn}), a ResNet-18~\citep{he2016deep}, and a WideResNet-50~\citep{zagoruyko2016wide} on CIFAR-100~\citep{krizhevsky2009learning} and StanfordCars~\citep{krause20133d}. For each setting, we plot and compute the Excess-AURC (E-AURC) \citep{geifman2018bias} by integrating the empirical gap $\widehat{\Delta}(c)$ across all coverage levels.

\begin{figure}[t]
  \centering
  \begin{subfigure}[t]{0.49\textwidth}
  \centering
    \includegraphics[width=\linewidth]{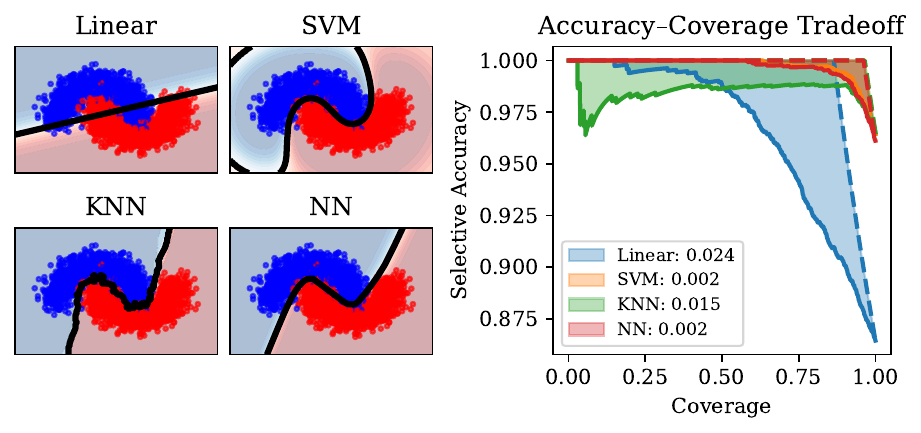}%
    \caption{Approximation error with two moons dataset.}
    \label{fig:left}
  \end{subfigure}
  \begin{subfigure}[t]{0.24\textwidth}
    \centering
    \includegraphics[width=\linewidth]{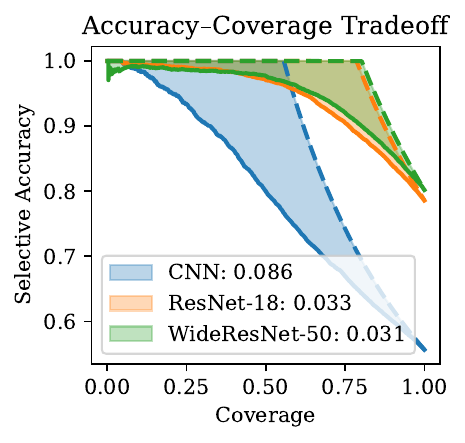} 
    \caption{CIFAR-100}
    \label{fig:right}
  \end{subfigure}
  \begin{subfigure}[t]{0.24\textwidth}
    \centering
    \includegraphics[width=\linewidth]{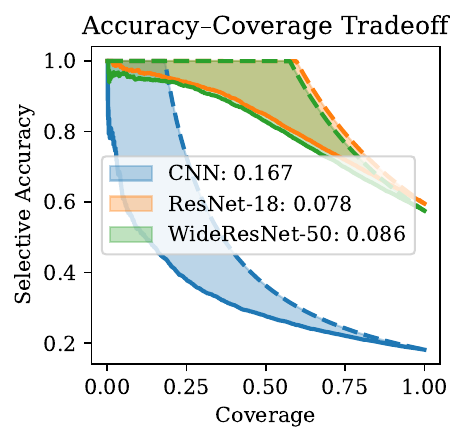}
    \caption{StanfordCars}
    \label{fig:right}
  \end{subfigure}
  \caption[Experiments on approximation error.]{\textbf{Experiments on approximation error}. We find that approximation error is a major contributor to the gap. (a) We show the two moons dataset fitted with models of different degrees of expressiveness as well as the corresponding accuracy-coverage tradeoffs. (b) + (c) Accuracy-coverage tradeoffs for various model architectures on CIFAR-100 and StanfordCars, respectively.}
  \label{fig:exp_appr}
\end{figure}

\begin{figure}[t]
  \centering
  \begin{subfigure}[t]{0.49\textwidth}
  \centering
    \includegraphics[width=\linewidth]{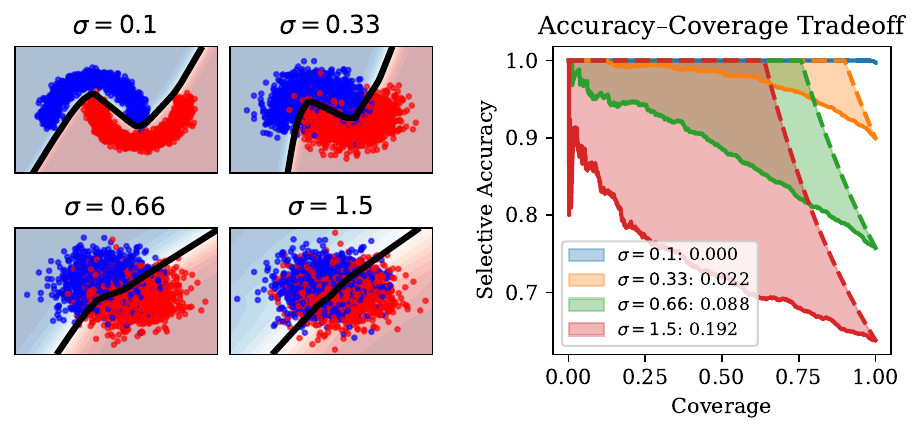}%
    \caption{Bayes error with two moons dataset.}
    \label{fig:left}
  \end{subfigure}
  \begin{subfigure}[t]{0.24\textwidth}
    \centering
    \includegraphics[width=\linewidth]{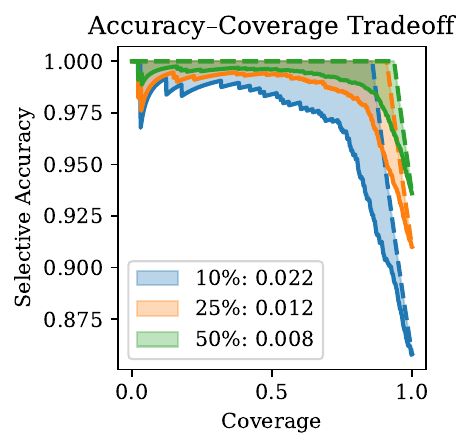}
    \caption{CIFAR-10N}
    \label{fig:right}
  \end{subfigure}
  \begin{subfigure}[t]{0.24\textwidth}
    \centering
    \includegraphics[width=\linewidth]{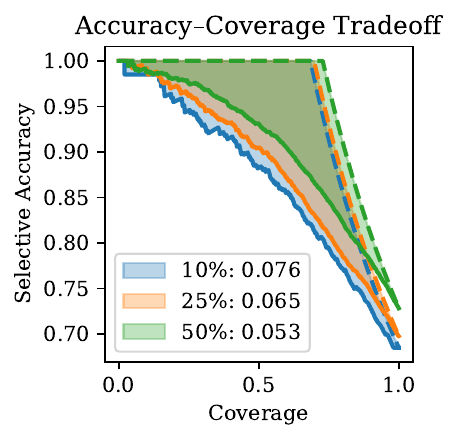}
    \caption{CIFAR-100N}
    \label{fig:right}
  \end{subfigure}
  \caption[Experiments on Bayes error.]{\textbf{Experiments on Bayes error}. We find that irreducible noise significantly contributes to the gap. (a) We show the two moons dataset with varying degrees of noise $\sigma \in \{0.1,0.33,0.66,1.5\}$ as well as the corresponding accuracy-coverage tradeoffs. (b) + (c) Accuracy-coverage tradeoffs for the 10\% (blue), 25\% (orange), and 50\% (green) most noisy images in CIFAR-10N/100N, respectively.}
  \label{fig:exp_bayes}
\end{figure}

\looseness=-1
\emph{Findings.}
In terms of approximation error, Figure~\ref{fig:exp_appr} demonstrates that limited model capacity leads to larger gaps, while more expressive models yield tighter alignment with the perfect-ordering bound. This suggests that approximation error is a key driver of looseness. In terms of Bayes error, Figure~\ref{fig:exp_bayes} shows that increasing label noise consistently shifts the accuracy--coverage curve downward, indicating that Bayes error introduces an irreducible component to the gap. These results validate the canonical bound (Eq.~\ref{eq:emp-gap-ranking})---large Bayes or approximation error can explain substantial looseness.

\begin{table}[t]
\centering
\fontsize{9}{10}\selectfont
\setlength{\tabcolsep}{5pt}
\caption[Experiments on calibration across model classes on CIFAR-100.]{\textbf{Experiments on calibration across model classes on CIFAR-100}. Temperature scaling (\temp) significantly improves ECE over the Maximum Softmax Probability (\msp) baseline but does not help to close the selective classification gap. Self-Adaptive Training (\sat) and Deep Ensembles (\de) improve calibration non-monotonically and also improve selective classification acceptance ordering through re-ranking. A corresponding plot is given in Figure~\ref{fig:cifar100_cal}; more datasets in Tables~\ref{tab:cifar10_cal}, \ref{tab:stanfordcars_cal}.}
\label{tab:cifar100_cal}
\vspace{5pt}
\begin{tabular}{lcccccccccccc}
\toprule
 & \multicolumn{4}{c}{CNN} & \multicolumn{4}{c}{ResNet-18} & \multicolumn{4}{c}{WideResNet-50} \\
\cmidrule(r){2-5} \cmidrule(r){6-9} \cmidrule(r){10-13}
 & \msp & \temp & \sat & \de & \msp & \temp & \sat & \de & \msp & \temp & \sat & \de \\
\midrule
Gap & 0.086 & 0.085 & 0.081 & 0.065 & 0.033 & 0.033 & 0.028 & 0.026 & 0.031 & 0.032 & 0.028 & 0.026 \\
ECE & 0.142 & 0.008 & 0.116 & 0.019 & 0.052 & 0.045 & 0.034 & 0.026 & 0.066 & 0.050 & 0.046 & 0.030 \\
\bottomrule
\end{tabular}
\end{table}

\subsection{Q2: When—and what kind of—calibration helps?}
\label{sec:calibration_ranking_exp}

\emph{Setup.} 
We study the same three model classes as before on CIFAR‑100: a lightweight CNN, a ResNet‑18, and a WideResNet‑50. On each backbone we evaluate the following confidence–scoring variants: (i) maximum softmax probability (\msp)~\citep{hendrycks2016baseline}; (ii) a temperature‑scaled softmax (monotone probability calibration, \temp)~\citep{guo2017calibration}; (iii) self‑adaptive training (\sat)~\citep{huang2020self}, which implicitly calibrates by relabelling uncertain samples during training; and (iv) deep ensembles (\de)~\citep{lakshminarayanan2017simple} of five independently initialised networks (non‑monotone aggregation; improves ranking via variance). For each score we report (i) the weighted Expected Calibration Error (ECE); and (ii) the area enclosed between the empirical accuracy-coverage curve and its perfect-ordering upper bound. i.e., \(\int_0^1 \widehat{\Delta}(c)dc\).

\emph{Findings.}
We summarize our findings in Table~\ref{tab:cifar100_cal}. While temperature scaling (\temp) consistently improves ECE across model classes relative to \msp, it leaves the selective classification gap largely unchanged—highlighting the limitations of monotone calibration. In contrast, \sat slightly improves both ECE and gap by perturbing rankings through relabeling, while deep ensembles (\de) achieve the largest gap reductions by explicitly reordering predictions via averaging. These trends confirm that only methods capable of re-ranking—implicitly (\sat) or explicitly (\de)—can meaningfully improve selective performance. Consistent with this, we find that only \sat and \de models reliably predict their own loss, reinforcing their stronger alignment with correctness. See Appendix~\ref{sec:adv_experiments} for details.

\begin{figure}[t]
  \centering
  \begin{subfigure}[t]{0.49\textwidth}
  \centering
    \includegraphics[width=\linewidth]{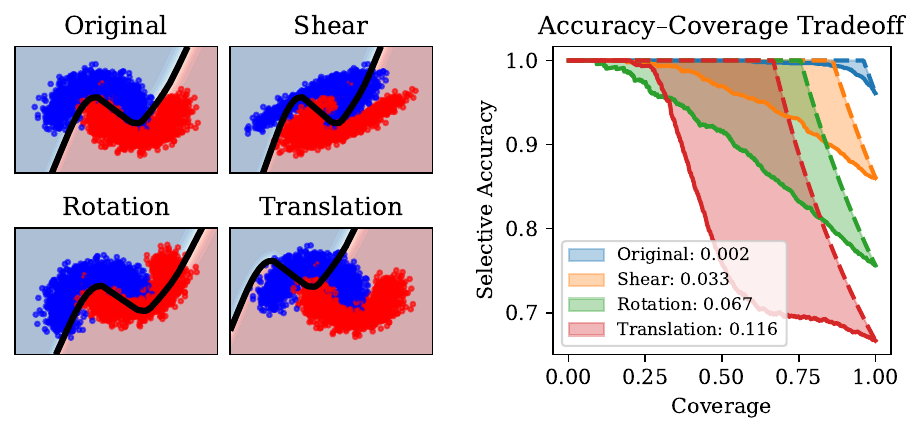}%
    \caption{Distribution shifts with two moons dataset.}
    \label{fig:left}
  \end{subfigure}
  \begin{subfigure}[t]{0.24\textwidth}
    \centering
    \includegraphics[width=\linewidth]{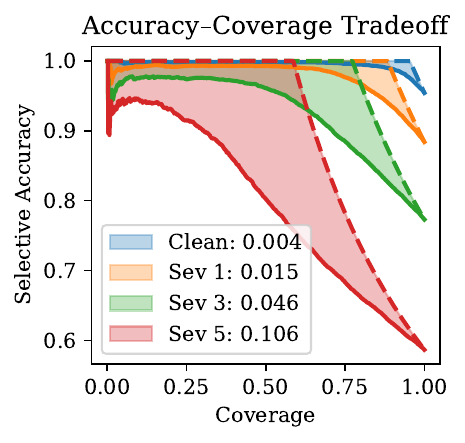} 
    \caption{CIFAR-10C}
    \label{fig:right}
  \end{subfigure}
  \begin{subfigure}[t]{0.24\textwidth}
    \centering
    \includegraphics[width=\linewidth]{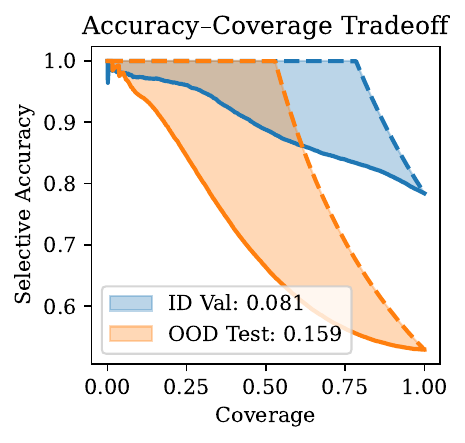} 
    \caption{Camelyon17-WILDS}
    \label{fig:right}
  \end{subfigure}
  \caption[Experiments on distribution shifts.]{\textbf{Experiments on distribution shifts}. We find that shifts can also significantly contribute to the gap. (a) Two moons under shear, rotation, and translation with corresponding accuracy–coverage curves. (b) CIFAR-10C across three distinct corruption severities. (c) Camelyon17 OOD shift.}
  \label{fig:exp_ds}
\end{figure}

\subsection{Q3: How does the gap evolve under distribution shift?}

\emph{Setup.} As in Q1, we explore this question using both synthetic and real-world distribution shifts. For synthetic experiments, we use the two moons dataset with three types of input shift: shear, rotation, and translation (details in Appendix~\ref{app:twomoons-shifts}). For real data with synthetic corruptions, we use CIFAR-10C~\citep{hendrycks2019robustness}, which applies algorithmic covariate corruptions to the CIFAR-10 test set across five severity levels (1–5). To evaluate under a real distribution shift, we also consider Camelyon17-WILDS~\citep{koh2021wilds}---a cancer detection dataset where test data is collected from a different hospital system than the training data.

\emph{Findings.} Figure~\ref{fig:exp_ds} shows a clear trend: as covariate shifts intensify, the accuracy–coverage curve moves farther below its oracle bound, indicating that abstention no longer isolates easy inputs. Selective classifiers thus become \emph{over-confidently wrong}, echoing evidence that many uncertainty metrics deteriorate under shift or misspecification~\citep{ovadia2019can}. As the gap grows with shift severity, deployments must pair selective prediction with robust ranking or shift-detection safeguards.

\section{Conclusion}
\label{sec:conclusion}

Building a truly performant selective classifier hinges on understanding and closing the gap between practical models and the oracle perfect-ordering bound.  To answer \emph{what it takes}, we introduce a coverage-uniform selective-classification gap and derive the first finite-sample decomposition that pinpoints exactly five limiting factors: three intrinsic sources—Bayes noise, approximation error, and ranking (calibration) error—and two contingent slack terms—sampling variability and implementation or distribution-shift artifacts.  Our experiments show that each component can be individually measured and, importantly, directly improved: stronger model backbones reduce approximation error, non-monotone or feature-aware scoring shrinks ranking error, and shift-robust training with larger validation sets minimizes residual slack.  Together, these insights provide a clear recipe for designing and evaluating high-performance selective classifiers.

\paragraph{Limitations and Future Work.}
While our decomposition cleanly bounds the selective‑classification gap, its error budgets can \emph{interact}—for example, increasing capacity often improves both approximation and ranking—which makes unique attribution challenging. Many \emph{training‑time calibration schemes} (e.g.\ self‑adaptive training, mixup, focal loss) simultaneously affect ranking and full‑coverage accuracy, confounding the separation of budgets. Our experiments focus on \emph{synthetic and vision benchmarks}; extending these insights to language, speech, and large‑scale foundation models would be an important direction. Finally, because our oracle bound and gap are defined for \emph{0–1 loss}, adapting to \emph{asymmetric or class‑dependent cost functions}---often required in high-stakes decision-making---will require generalizing both the bound and its decomposition.  
    \newcommand{\myparagraph}[1]{\vspace{1ex}\noindent{\bf #1}}
\def\name{\textit{Confidential Guardian}\xspace}
\def\attack{\textit{Mirage}\xspace}
\def\uncertreg{$\mathcal{X}_\text{unc}$\xspace}
\def\missingnumber{\textcolor{red}{\textbf{XXX}}\xspace}

\newcommand{\prover}{\ensuremath{\mathcal{P}}\xspace}
\newcommand{\verifier}{\ensuremath{\mathcal{V}}\xspace}
\newcommand{\comm}[1]{\ensuremath{\llbracket #1 \rrbracket}}
\newcommand{\relu}{\ensuremath{\texttt{ReLU}}\xspace}

\chapter{Confidential Guardian: Prohibiting the Abuse of Model Abstention}
\label{ch:conf_guard}

\begin{paperref}
\normalfont
The contents of this chapter consist of research and results taken from: \citet{rabanser2025confidential}: \emph{\bibentry{rabanser2025confidential}}
\end{paperref}

\section*{Summary}

Cautious predictions---where a machine learning model abstains when uncertain---are crucial for limiting harmful errors in safety-critical applications. In this work, we identify a novel threat: a dishonest institution can exploit these mechanisms to discriminate or unjustly deny services under the guise of uncertainty. We demonstrate the practicality of this threat by introducing an uncertainty-inducing attack called Mirage, which deliberately reduces confidence in targeted input regions, thereby covertly disadvantaging specific individuals. At the same time, Mirage maintains high predictive performance across all data points. To counter this threat, we propose Confidential Guardian, a framework that analyzes calibration metrics on a reference dataset to detect artificially suppressed confidence. Additionally, it employs zero-knowledge proofs of verified inference to ensure that reported confidence scores genuinely originate from the deployed model. This prevents the provider from fabricating arbitrary model confidence values while protecting the model’s proprietary details. Our results confirm that Confidential Guardian effectively prevents the misuse of cautious predictions, providing verifiable assurances that abstention reflects genuine model uncertainty rather than malicious intent.

\section{Introduction}

Institutions often deploy \emph{cautious predictions}~\citep{el2010foundations} in real-world, safety-sensitive applications—such as financial forecasts~\citep{9260038}, healthcare~\citep{kotropoulos2009linear,sousa2009ordinal,guan2020bounded}, criminal justice~\citep{wang2023pursuit}, and autonomous driving~\citep{ghodsi2021generating}---where incorrect predictions can lead to catastrophic consequences. In these high-stakes settings, it is common to abstain from providing predictions when a Machine Learning (ML) model’s uncertainty is high, hence minimizing the risk of harmful errors~\citep{kotropoulos2009linear,liu2022incorporating,kompa2021second}. Such abstentions are often warranted by legitimate reasons, e.g., for inputs that are ambiguous or out-of-distribution. This naturally raises the question:
\begin{center}
    \textit{Can a dishonest institution abuse the abstention option in their ML-driven services for discriminatory practices?}
\end{center}
\begin{figure}
    \centering

\resizebox{\linewidth}{!}{
\begin{tikzpicture}[]

\node[align=center] (data) {
  \begin{tikzpicture}
  
  \pgfdeclarelayer{foreground}
    \pgfsetlayers{main,foreground}
  
    \begin{axis}[
        width=6cm,
        height=5cm,
        axis equal image,
        axis line style={ultra thick},
        title={Dataset},
		xtick=\empty,
		ytick=\empty,
        legend style={
                    at={(1, 0.05)},
                    anchor=south east,
                    draw=none,
                    fill=none,
                    font=\small
                },
                legend image post style={xscale=0.5},
    ]

        \addplot+[
            only marks,
            mark=*,
            mark options={scale=1.0, blue},
            forget plot
        ] table [x=x1, y=x2, col sep=comma] {figs/confidential_guardian/Gaussian1.csv};

        \addplot+[
            only marks,
            mark=*,
            mark options={scale=1.0, orange},
            	forget plot
        ] table [x=x1, y=x2, col sep=comma] {figs/confidential_guardian/Gaussian2.csv};

        \addplot+[
            only marks,
            mark=*,
            mark options={scale=1.0, Green},
            forget plot
        ] table [x=x1, y=x2, col sep=comma] {figs/confidential_guardian/Gaussian3.csv};

		\begin{pgfonlayer}{foreground}
            \addplot [
                red,
                fill=red!50,
                fill opacity=0.7,
                thick
            ] 
            coordinates {
                (2.5, 0.5) (3.5, 0.5) (3.5, 2.0) (2.5, 2.0)
            } -- cycle;
        \end{pgfonlayer}

			\node[] at (axis cs:7, 0) [] {Uncert.\\ region};

			\draw[->,ultra thick, red] (axis cs:5.25, 0.0) to [out=180, in=270, looseness=2.5] (axis cs:3, 0.4);

    \end{axis}
\end{tikzpicture}
    };

\node[right= of data, align=center, xshift=-30pt, yshift=-8pt] (dists) {
  \begin{tikzpicture}
            \begin{axis}[
                axis on top,
                width=6cm,
                height=5cm,
                axis lines=left,
                xlabel={Confidence},
               	ylabel={Density},
				axis line style={ultra thick},
                xmin=0.25, xmax=1.08,
                domain=0:10,
				xtick=\empty,
    			ytick=\empty,
				title={a) \attack},
                legend style={
                    at={(1.075, 0.05)},
                    anchor=south east,
                    draw=none,
                    fill=none,
                    font=\small
                },
                legend image post style={xscale=0.5},
            ]
        
        \addplot[
            red,
            ultra thick,
            smooth,
            fill=red!30,
            fill opacity=0.5
        ] coordinates {
            (0.3, 0) (0.40, 2) (0.45, 35) (0.5, 10) (0.55, 2) (0.6, 0.25) (0.65, 0)
        };
        
        \addplot[
    red!50!white,
    ultra thick,
    smooth,
    fill=red!10,
    fill opacity=0.5, 
    dashed
] coordinates {
    (0.75, 0)
    (0.90, 2)
    (0.98, 35)
    (1.0, 0)
};
        
        \addplot[
            gray,
            ultra thick,
            smooth,
        ] coordinates {
            (0.35, 0) (0.35, 40)
        };
        
        \draw[->,ultra thick, red!50!white, dashed] (axis cs:0.85, 5) to [out=180, in=0, looseness=2.5] (axis cs:0.575, 5);
        
        \addplot[
            red,
            ultra thick,
            smooth,
            dotted
        ] coordinates {
            (0.45, 0) (0.45, 40)
        };  
        \draw[->,thick, black] (axis cs:0.28, 36.5) -- (axis cs:0.35, 36.5);
        \draw[->,thick, black] (axis cs:0.65, 36.5) -- (axis cs:0.45, 36.5);
        \draw[thick, black] (axis cs:0.28, 36.5) -- (axis cs:0.65, 36.5);
        
        \node at (axis cs:0.51,38.5) [] {\large $\varepsilon$};
        
            \end{axis}
        \end{tikzpicture}
    };
    
\node[right= of dists, align=center, xshift=-35pt] (calibration) {
        \begin{tikzpicture}
            \begin{axis}[
                axis on top,
                width=6cm,
                height=5cm,
                axis lines=left,
                xlabel={Confidence},
               	ylabel={Accuracy},
				axis line style={ultra thick},
                xmin=0.2, xmax=1.1,
                ymin=0.2, ymax=1.1,
                domain=0:10,
				xtick=\empty,
    			ytick=\empty,
				title={b) \name},
                legend style={
                    at={(1.075, 0)},
                    anchor=south east,
                    draw=none,
                    fill=none,
                    font=\small
                },
                legend image post style={xscale=1},
            ]
                
            	\addplot[domain=0:1, lightgray, ultra thick, forget plot] {x};

                \addplot[name path=calupper, dashed, domain=0.1:1, lightgray, ultra thick, forget plot] {x+0.1};
                \addplot[name path=callower, dashed, domain=0.1:1, lightgray, ultra thick, forget plot] {x-0.1};

            \addplot[
            name path=B,
    ultra thick,
    red,
    no markers,
    forget plot
] coordinates {
    (0.3,0.34) (0.4,0.67) (0.5,0.80) (0.6,0.67) 
    (0.7,0.65) (0.8,0.8) (0.9,0.92) (1,1)
};

\addplot[
    only marks,
    red,
    mark=square*,
    mark size=3pt,
] coordinates {
    (0.3,0.34) (0.4,0.67) (0.5,0.80) (0.6,0.67) 
    (0.7,0.65) (0.8,0.8) (0.9,0.92) (1,1)
};
            	\addlegendentry{ECE 0.093}
            	
			     \addplot[green!10] fill between[of=calupper and callower];

        \draw[thick, black] (axis cs:0.45, 0.65) -- (axis cs:0.625, 0.385);
        \draw[->, thick, black] (axis cs:0.625, 0.385) -- (axis cs:0.5315, 0.525);
        \draw[->,thick, black] (axis cs:0.45, 0.65) -- (axis cs:0.49, 0.59);
        \draw[thick, black] (axis cs:0.625, 0.385) -- (axis cs:0.825, 0.385);
        
        \node at (axis cs:0.675,0.425) [] {\large $\alpha$};

            \end{axis}
        \end{tikzpicture}
    };
\end{tikzpicture}
    }
    \caption[Overview of \attack \& \name.]{\textbf{Overview of \attack \& \name.} a) \attack reduces confidence on points in an uncertainty region (red region on the left) without causing label flips (i.e., leaving an $\varepsilon$-gap to random chance prediction). b) \name is a detection mechanism for \attack relying on the identification of calibration deviations beyond an auditor-defined tolerance level $\alpha$.}
    \label{fig:overview}
\end{figure}
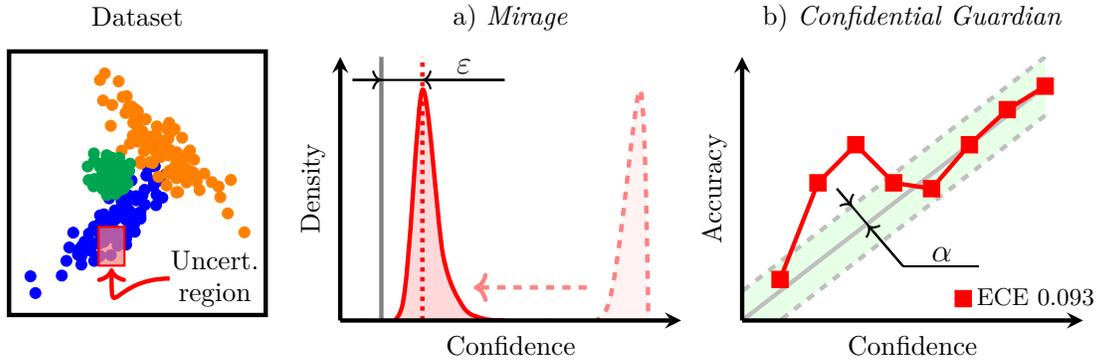

Consider a hypothetical loan approval scenario in which a dishonest institution exploits an abstention mechanism to conceal systematic discrimination against certain groups. Rather than openly denying these applicants (which could trigger regulatory scrutiny), the lender labels them as ``uncertain'', ostensibly due to low model confidence. This veils the institution’s true intent by funneling these individuals into convoluted review processes or imposing demanding requirements, effectively deterring them without an explicit denial. Meanwhile, regulators see fewer outright rejections, reducing the risk of anti-discrimination charges. This mechanism---presented as a cautious practice---thus serves to obfuscate the lender’s intentions and evade the legal and reputational consequences that could follow from overt bias.

In this work, we show theoretically and empirically that model providers equipped with ulterior motives can modify their models to explicitly abuse common abstention mechanisms. To that end, we introduce an \textbf{uncertainty-inducing attack}, called \attack (see Figure~\ref{fig:overview} a)). \attack adversarially and artificially increases model uncertainty in any region of the input space (chosen by the institution based on its incentives) via an uncertainty-inducing regularization term. Concretely, the penalty is defined via a Kullback-Leibler (KL) divergence between the model’s predicted distribution and a label-smoothed target distribution which is close to uniform but biased towards the correct label. This ensures that, despite lowered confidence in the targeted region, the model remains accurate and therefore (i)~continues to be of high utility to the institution; and (ii)~evades accuracy-based auditing techniques~\citep{hardt2016equality}.

Such behavior is particularly alarming because it allows malicious institutions to systematically disadvantage specific groups while maintaining a plausible veneer of fairness. Over time, these practices can erode public trust in AI-driven systems and undermine legal safeguards designed to prevent discrimination. Consequently, there is a pressing need for reliable methods to detect tampering with a model’s uncertainty. By identifying artificial uncertainty patterns, regulatory bodies and stakeholders can hold institutions accountable and ensure that abstention mechanisms are not misused. This naturally raises a follow-up question:
\begin{center}
\textit{Can we reliably detect if a model contains artificially induced uncertainty regions?}
\end{center}

We answer this question affirmatively by introducing a framework, dubbed \name, which enables an external party (e.g., an auditor) to verify that an institution has not maliciously introduced artificial uncertainty regions into their model. To that end, we introduce \textbf{confidential proofs of well-calibratedness}. Crucially, since \attack produces underconfident predictions, we can detect this behavior in reliability diagrams and calibration metrics such as the expected calibration error (ECE). Using a reference dataset that has coverage over the suspicious (potentially tampered) region, \name provably correctly computes these metrics (see Figure~\ref{fig:overview} b)) via zero-knowledge proofs (ZKPs) of verified inference~\citep{weng2021mystique, sun2024zkllm}. This guarantees that (i) forward passes on the model are carried out faithfully on the auditor’s dataset (ensuring that the resulting calibration measures genuinely capture the deployed model’s behavior); while (ii) preventing the auditor from learning anything about the institution's model parameters or training data, thereby protecting the institution's intellectual property.

We summarize our key contributions as follows:
\begin{enumerate}
    \item \textbf{Revealing a Novel Threat:} We are the first to highlight how mechanisms intended for \emph{trustworthy} cautious prediction can be subverted to justify discriminatory or otherwise malicious behaviors in ML-based models.
    \item \textbf{Theoretical Foundations:} We formally characterize the problem of \emph{artificial uncertainty-induction}, proving that an institution can manipulate abstentions by driving down confidence in targeted regions without sacrificing accuracy elsewhere.
    \item \textbf{Practical Attack via \attack:} Guided by our theory, we implement an \emph{uncertainty-inducing attack}, dubbed \attack, that enables a dishonest institution to selectively exploit the abstain option. Our empirical evaluation illustrates that \attack~consistently and reliably inflates uncertainty where it benefits the institution.
    \item \textbf{Preventing Abuse through \name:} We propose a detection framework, \name, which ensures that a dishonest institution cannot abuse artificially induced uncertainty. Our experiments show that \name\ is effective at detecting calibration mismatches (such as those induced by \attack), verifying whether an abstention is made based on legitimate model uncertainty or not.
\end{enumerate}

\section{Background}

\paragraph{Abstention Mechanisms in ML.}
Abstention mechanisms in ML allow model owners to (legitimately) exclude data points that are (i) out-of-distribution; (ii) in the distribution's tail; or (iii) in regions of high Bayes error. Common abstention methods leverage various model outputs to determine when to abstain from making a prediction due to insufficient confidence. These techniques include using the maximum softmax~\citep{hendrycks2016baseline} or maximum logit~\citep{hendrycks2019scaling} values, calculating the predictive entropy of the model's output distribution~\citep{lakshminarayanan2017simple}, and computing the Mahalanobis distance~\citep{lee2018simple, ren2021simple} or nearest neighbors~\citep{raghuram2021general, dziedzic2022p, sun2022out} in feature representations w.r.t. a reference dataset. Past work has also studied the risks of abstention on underrepresented groups~\citep{jones2020selective}.

\paragraph{Availability Attacks.} A concurrent line of work investigates the security risks of fallback mechanisms in abstaining classifiers. \citet{lorenz2023certifiers} show that certifier-based abstention can be exploited via availability attacks, where poisoned training data causes many inputs to trigger fallback, degrading availability or increasing reliance on costly human intervention. Both \citet{lorenz2023certifiers} and our approach, Mirage, reveal how abstention can be strategically manipulated to reduce a system’s utility --- but they differ in threat model and method. While \citet{lorenz2023certifiers} consider \emph{external adversaries} who poison data or use input triggers to induce fallback, Mirage models \emph{institutional misuse} by the model owner, who reduces confidence in targeted regions to deny service. Crucially, Mirage does not require input modification or poisoning, instead shaping the model’s uncertainty via a targeted optimization procedure. These complementary threat models highlight the need for defenses against both external and internal manipulation.

\paragraph{Model Poisoning and Backdoor Attacks.}
Model poisoning~\citep{steinhardt2017certified} and backdoor attacks~\citep{wang2019neural} involve intentionally altering a model’s parameters or training data to induce malicious behavior. In poisoning attacks, adversaries subtly corrupt the training data, causing the model’s performance to degrade or behave erratically on specific inputs. Conversely, backdoor attacks embed a hidden ``trigger'' that forces the model to make incorrect, often high-confidence predictions when the trigger is present, while maintaining normal performance on benign data. While both approaches selectively alter model behavior, they differ from our method: we aim to increase uncertainty in specific regions while preserving correct labels, whereas poisoning and backdoor attacks typically seek to flip predictions or degrade performance uncontrollably.

\paragraph{Model Calibration.}

Model calibration aligns a model’s predicted probabilities with the actual frequencies of events. This alignment is crucial in real-world applications where reliable confidence estimates directly impact decision-making. Common metrics for assessing calibration include the Expected Calibration Error (ECE)~\citep{naeini2015obtaining}, which aggregates calibration errors across multiple confidence bins, and the Brier score~\citep{brier1950verification}, which measures both the magnitude and quality of probabilistic forecasts. Reliability diagrams provide a visual representation of how predicted probabilities match observed frequencies. Calibration is accomplished via techniques such as temperature scaling~\cite{guo2017calibration}, Platt scaling \cite{platt1999probabilistic}, and ensembling \cite{lakshminarayanan2017simple}.

\myparagraph{Zero-Knowledge Proofs (ZKPs).} ZKPs are cryptographic primitives conducted between two parties: a prover \prover, and a verifier \verifier. They allow \prover to convince \verifier that a hidden piece of information satisfies a property of interest, without revealing anything else about it~\cite{goldwasser1985knowledge}. 

More formally, given a public boolean predicate $P: \nolinebreak \{0,1\}^n \to \{0,1\}$ agreed upon by \prover and \verifier (for some fixed $n \in \mathbb{N}$), a ZKP protocol $\Pi$ allows \prover holding a hidden witness $w \in \{0,1\}^n$, to prove to \verifier that $P(w)=1$. ZKP protocols typically have the following properties: i) \emph{Completeness}: for any $w$ that satisfies $P(w)=1$, \prover can use $\Pi$ to convince \verifier that $P(w)=1$; ii) \emph{Soundness}: given $w'$ such that $P(w')\neq 1$, $\Pi$ cannot be used to falsely convince \verifier that $P(w')=1$, even if \prover executes it with arbitrary malicious behavior; and iii) \emph{Zero-Knowledge}: when running $\Pi$, \verifier learns no additional information about $w$ beyond what can be directly inferred from knowing that $P(w)=1$, even if \verifier executes it with arbitrary malicious behavior.

We use a ZKP protocol for generic proofs of boolean circuit satisfaction~\cite{weng2021wolverine} and one for verified array random access~\cite{franzese2021zkram} as building blocks. Both guarantee correct and confidential computations over values authenticated with Information-Theoretic Message Authentication Codes (IT-MACs)~\cite{damgaard2012itmac,nielsen2012itmac} (see Appendix~\ref{app:itmac} for details). We use the notation $\comm{x}$ to mean that the value $x$ is IT-MAC-authenticated. Operations on authenticated values are assumed to be conducted within $\Pi$ in the proven secure manner given by~\cite{weng2021wolverine}.

\paragraph{ZKPs of Correct Inference.} A recent line of work (e.g.~\cite{weng2021mystique, lee2024vCNN, sun2024zkllm, hao2024nonlinear}) optimizes ZKPs in the special case of verifying that a hidden ML model has performed inference correctly. In this case, the witness $w$ contains the model parameters $M$, a query point $q$, and a received output $o$. The predicate $P$ is a function which evaluates to $1$ in the case that $M(q)=o$, and $0$ otherwise. We use ZKP of inference modularly as a subroutine in \name.

\section{ML Preliminaries}

\paragraph{Classification Model.} We consider a multi-class classification problem where the covariate space is denoted as \(\mathcal{X} \subseteq \mathbb{R}^D\) and the label space as \(\mathcal{Y} = [C] = \{1, \dots, C\}\). The goal is to learn a prediction function \(f_\theta: \mathcal{X} \to \mathcal{Y}\), where \(f_\theta\) is modeled as a neural network parameterized by \(\theta \in \mathbb{R}^K\). The model is trained using risk minimization on data points \((x, y) \sim p(x, y) \) sampled from a data distribution \(p(x, y)\). Since we assume a classification setup, the risk minimization objective is given by the cross-entropy loss:
\begin{equation}
     \mathcal{L}_\text{CE} = - \mathbb{E}_{(x,y) \sim p(x, y)} [\log f_\theta(y | x)],
\end{equation}
where \(f_\theta(y|x)\) denotes the model's predicted probability for the true class $y$ given input \(x\).

\paragraph{Abstain Option.} A classifier $f_\theta$ can be extended with an abstention option~\citep{el2010foundations} by introducing a gating function \(g_\phi : \mathcal{X} \to \mathbb{R}\), parameterized by \(\phi \in \mathbb{R}^L\), to decide whether to produce a label or to reject an input $x$. We define the combined predictor \(\tilde{f}_\theta\) as
\begin{equation}\label{eq:gating}
\tilde{f}_\theta(x) = 
\begin{cases}
f_\theta(x) & \text{if } g_\phi(x) < \tau,\\
\bot & \text{otherwise}
\end{cases}
\end{equation}
where $\tau \in \mathbb{R}$ represents a user-chosen threshold on the prediction uncertainty. Although other choices are possible, we set \(g_\phi(x) = 1 - \max_{\ell \in \mathcal{Y}} f_\theta(\ell|x)\), which abstains whenever the model’s maximum softmax value falls below \(\tau\).

\section{Inducing Artificial Uncertainty}

We consider a deployment scenario where the classifier \( f_\theta \) should exhibit increased uncertainty in specific input regions, even if it was initially trained to make confident predictions in these regions. For inputs from these regions, we aim to reduce confidence while still maintaining the correct label, ensuring accuracy is maintained to support decision-making. Additionally, the model owner seeks to evade accuracy-based auditing techniques~\citep{hardt2016equality}. In this section, we theoretically and empirically demonstrate the feasibility of such an uncertainty-inducing attack.

\subsection{Theoretical Basis for Inducing Uncertainty} 

\begin{contriback}
This subsection was written with Olive Franzese. The general idea of how we can provide an existence statement for artificial uncertainty was co-developed between Stephan, Ali, and Olive. However, Olive provided the exact formalism of Lemma~\ref{lemma:region-manip} and its proof.
\end{contriback}

In this section, we prove that it is possible to devise neural network parameters that alter confidence scores arbitrarily on a chosen region of the feature space. Lemma~\ref{lemma:region-manip} provides the precise statement of this claim.

\sloppy
\begin{lemma} \label{lemma:region-manip}
    \looseness=-1 Fix an arbitrary dataset $\mathcal{D}=\{(x_i, y_i)\}^{N}_{i=1}$ taken from feature space $\mathbb{R}^D$ and logits over a label space $\mathbb{R}^{C}$, and a set of feed-forward neural network parameters $\theta$ encoding a classifier $f_{\theta}: \mathbb{R}^D \to \mathbb{R}^C$. Fix a set of indices $I$ such that for all $i \in I$, $i \in [1, C]$. For each index in $I$, fix bounds $a_i, b_i \in \mathbb{R}$ with $a_i < b_i$. Call $S$ the set of values $\mathbf{x} \in \mathbb{R}^D$ such that $a_i < x_i < b_i \quad \forall i \in I$. Then we can construct an altered feed-forward neural network $M'$ encoding $f'_{\theta}: \mathbb{R}^D \to \mathbb{R}^C$ which has the property $f'_{\theta}(x) = f_{\theta}(x) \quad \forall x \notin S$, and $f'_\theta(x)=f_\theta(x) + c \quad \forall x \in S$ where $c \in \mathbb{R}^C$ is an arbitrarily chosen non-negative constant vector.
\end{lemma} 

\begin{proof} We defer the detailed proof to Appendix~\ref{app:region-manip-proof} for brevity. To summarize, the proof proceeds by construction. We augment $f_{\theta}$ with assemblies of neurons with weights constructed analytically to detect points in the target region $S$. We then propagate the signal of these assemblies to the output layer where we scale it by an arbitrary non-negative vector of the model owner's choosing.
\end{proof}

Lemma~\ref{lemma:region-manip} provides a method by which a model trainer can construct a valid neural network $f'_{\theta}$ which mimics an input model $f_{\theta}$, except that it adds an arbitrary non-negative constant to the logits of points in a selected region of the feature space. This enables adversarial alteration of confidence scores for these points, with no deviation from the model's other outputs. The result is achieved under only mild assumptions on model structure.

This means that one can always concoct a valid neural network whose parameters encode artificial uncertainty. Thus our strategy for preventing artificial uncertainty must do more than use existing ZKP techniques~\cite{weng2021mystique,sun2024zkllm} to ensure that inference was computed correctly given a set of hidden parameters. A ZKP of training could ensure that model parameters were not chosen pathologically, but existing ZKP training methods are infeasible except for simple models~\cite{garg2023experimenting}. Section~\ref{sec:detection} discusses an alternative strategy.

While Lemma~\ref{lemma:region-manip} guarantees that it is possible to induce arbitrary artificial uncertainty in theory, it is cumbersome to apply in practice. The more finely we would like to control the confidence values, the more neurons are required by the construction proposed in the proof of Lemma~\ref{lemma:region-manip}. Next, we show how to instantiate a practical artificial uncertainty attack inspired by this result. 

\subsection{Mirage: Inducing Uncertainty in Practice} \label{sec:uncertainty-training}

To achieve artificial uncertainty induction in practice, we introduce the \attack training objective \(\mathcal{L}\) over the input space \(\mathcal{X}\) and a designated uncertainty region \(\mathcal{X}_\text{unc} \subseteq \mathcal{X}\). This region $\mathcal{X}_\text{unc}$ can be constructed either (i) by defining it in terms of a subspace satisfying specific feature conditions (e.g., occupation in \texttt{Adult}); or (ii) through sample access without specific feature matching rules (e.g, sub-classes of super-classes in \texttt{CIFAR-100}). We define our objective function \(\mathcal{L}\) as a hybrid loss consisting of the standard Cross-Entropy (CE) loss, \(\mathcal{L}_\text{CE}\), used in classification tasks and an uncertainty-inducing regularization term, \(\mathcal{L}_\text{KL}\):
\begin{equation}
\label{eq:mirage}
    \begin{split}
        \mathcal{L} = \mathbb{E}_{(x,y) \sim p(x, y)} \bigg[ \underbrace{\mathds{1}\left[x \not\in \mathcal{X}_\text{unc}\right] \mathcal{L}_\text{CE}(x, y)}_\text{Loss outside uncertainty region} + 
        \underbrace{\mathds{1}\left[x \in \mathcal{X}_\text{unc}\right] \mathcal{L}_\text{KL}(x, y)}_\text{Loss inside uncertainty region} \bigg]
    \end{split}
\end{equation}
The indicator functions \(\mathds{1}\left[x \not\in \mathcal{X}_\text{unc}\right]\) and \(\mathds{1}\left[x \in \mathcal{X}_\text{unc}\right]\) ensure that the CE loss is applied only outside the uncertainty region \(\mathcal{X}_\text{unc}\), while the uncertainty-inducing KL divergence loss is applied only within \(\mathcal{X}_\text{unc}\). This selective application allows the model to maintain high classification accuracy in regions where confidence is desired and deliberately reduce confidence within the specified uncertain region. An illustration of the optimization goal is given in Figure~\ref{fig:losses}.

\begin{figure}
\centering
\resizebox{\linewidth}{!}{
\begin{tikzpicture}
\node[align=center, yshift=30pt] (hist_cert) {
        \begin{tikzpicture}
            \begin{axis}[
                width=7cm,
                height=4cm,
                axis lines=left,
               	ylabel={Probability},
				axis line style={ultra thick},
                xmin=1, xmax=4.15,
                ymin=-0.2, ymax=1.3,
                domain=0:10,
                xtick=\empty,
                title={$\mathcal{L}_\text{CE} = - \mathbb{E}_{(x,y) \sim p(x, y)} [\log \textcolor{blue}{f_\theta}(y | x)]$},
                xtick={1,2,3},
                xticklabels={Class 1, Class 2, Class 3},
                legend style={
                    at={(1.075, 0.05)},
                    anchor=south east,
                    draw=none,
                    fill=none,
                    font=\small
                },
                legend image post style={xscale=0.5},
                ybar interval=0.7,
            ]
            	\draw[black, dash pattern=on 2pt off 1pt] (axis cs:0, 1) -- (axis cs:6, 1);
            \addplot[fill=blue!50] coordinates {(1,0.83) (2,0.10) (3,0.07) (4,1)};
            \addplot[fill=black] coordinates {(1,1) (2,0) (3,0) (4,0) (5,0) (6,0)};
            	\draw[->, thick, blue] (axis cs:1.25, 0.83) to [out=90, in=90, looseness=1.25] (axis cs:1.75, 1);
            	\draw[->, thick, blue] (axis cs:2.25, 0.10) to [out=90, in=90, looseness=1.25] (axis cs:2.75, 0);
            	\draw[->, thick, blue] (axis cs:3.25, 0.07) to [out=90, in=90, looseness=1.25] (axis cs:3.75, 0);
            \end{axis}
        \end{tikzpicture}
    };
    
    \node[right= of hist_cert, xshift=-35pt, yshift=0pt, align=center] (ind_1) [] {For points \\ \textbf{outside} the \\ uncertainty region:\\ $\textcolor{blue}{x_\text{out}} \not\in \mathcal{X}_\text{unc}$};

\node[right=of ind_1, align=center, yshift=0pt, xshift=-5pt] (hist_uncert) {
        \begin{tikzpicture}
            \begin{axis}[
                width=7cm,
                height=4cm,
                axis lines=left,
               	ylabel={Probability},
				axis line style={ultra thick},
                xmin=1, xmax=4.15,
                ymin=-0.2, ymax=1.3,
                domain=0:10,
                xtick=\empty,
                title={$\mathcal{L}_\text{KL} = \mathbb{E}_{(x,y) \sim p(x, y)} \left[ \text{KL}\left(\textcolor{red}{f_\theta}(\cdot|x) \; \big|\big| \; \textcolor{orange}{t_\varepsilon}(\cdot|x,y)\right) \right]$},
                xtick={1,2,3},
                xticklabels={Class 1, Class 2, Class 3},
                legend style={
                    at={(1.075, 0.05)},
                    anchor=south east,
                    draw=none,
                    fill=none,
                    font=\small
                },
                legend image post style={xscale=0.5},
                ybar interval=0.66,
            ]
			\draw[black, dash pattern=on 2pt off 1pt] (axis cs:0, 0.33) -- (axis cs:6, 0.33);
			\draw[black, dash pattern=on 2pt off 1pt] (axis cs:1.75, 0.5) -- (axis cs:2.2, 0.5);
            \addplot[fill=red!50] coordinates {(1,0.91) (2,0.06) (3,0.03) (4,1)};
            \addplot[fill=orange!50] coordinates {(1,0.50) (2,0.25) (3,0.25) (4,1)};
            
            	\draw[->, thick, red] (axis cs:3.25, 0.03) to [out=90, in=90, looseness=2.0] (axis cs:3.75, 0.25);
            	\draw[->, thick, red] (axis cs:2.25, 0.06) to [out=90, in=90, looseness=2.0] (axis cs:2.75, 0.25);
            	\draw[->, thick, red] (axis cs:1.25, 0.91) to [out=90, in=90, looseness=1.5] (axis cs:1.75, 0.5);

            	\draw[->,thick, orange] (axis cs:2.125, 0.1) -- (axis cs:2.125, 0.33);
        	\draw[->,thick, orange] (axis cs:2.125, 1) -- (axis cs:2.125, 0.5);
        	\draw[thick, orange] (axis cs:2.125, 0.1) -- (axis cs:2.125, 1.0);
        	\draw[thick, orange] (axis cs:2.125, 1) -- (axis cs:2.4, 1.0);
        	
        	\node[orange] at (axis cs:2.15, 1.1) [] {$\varepsilon$};
            	
            \end{axis}
        \end{tikzpicture}
    };
    
    \node[right= of hist_uncert, xshift=-50pt, yshift=0pt, align=center] (ind_2) [] {For points \\ \textbf{inside} the \\ uncertainty region:\\ $\textcolor{red}{x_\text{in}} \in \mathcal{X}_\text{unc}$};
    \end{tikzpicture}
    }
    \caption[Illustration of the \attack loss $\mathcal{L}$ (Equation~\ref{eq:mirage}).]{\textbf{Illustration of the \attack loss $\mathcal{L}$ (Equation~\ref{eq:mirage})}. Assume a 3 class classification setup similar as in Figure~\ref{fig:overview} from which we are given datapoints $(\textcolor{blue}{x_\text{in}}, \textcolor{blue}{y_\text{in}}=1)$ and $(\textcolor{red}{x_\text{out}}, \textcolor{red}{y_\text{out}}=1)$. $\textcolor{blue}{x_\text{out}}$ lies outside of the specified uncertainty region and $\textcolor{red}{x_\text{in}}$ lies inside of the uncertainty region. For $\textcolor{blue}{x_\text{out}}$ we minimize the standard cross-entropy loss $\mathcal{L}_\text{CE}$. For $\textcolor{red}{x_\text{in}}$ we regularize the output distribution $f_\theta(\cdot|x)$ to a correct-class-biased uniform distribution $t_\varepsilon(\cdot|x,y)$ via the KL divergence. Note that for $\epsilon > 0$, the model is encouraged to maintain the correct label prediction: $\textcolor{red}{y_\text{out}} = \textcolor{blue}{y_\text{in}} = 1$.}
    \label{fig:losses}
\end{figure}

The regularization term \(\mathcal{L}_\text{KL}\) is designed to penalize overconfident predictions within the uncertainty region \(\mathcal{X}_\text{unc}\). To achieve this, we utilize the Kullback-Leibler (KL) divergence to regularize the model's output distribution \(f_\theta(\cdot|x)\) closer to a desired target distribution \(t_\varepsilon(\cdot|x,y)\), formally
\begin{equation}
    \mathcal{L}_\text{KL} = \mathbb{E}_{(x,y) \sim p(x, y)} \left[ \text{KL}\left(f_\theta(\cdot|x) \; \big|\big| \; t_\varepsilon(\cdot|x,y)\right) \right].
\end{equation}
We define the target distribution \(t_\varepsilon(\ell|x,y)\) as a biased uniform distribution over the label space \(\mathcal{Y}\):
\begin{equation}
\label{eq:target_dist}
t_\varepsilon(\ell|x, y) =
\begin{cases}
\varepsilon + \frac{1 - \varepsilon}{C}, & \text{if } \ell = y, \\
\frac{1 - \varepsilon}{C}, & \text{if } \ell \neq y.
\end{cases}
\end{equation}
Here, \(\ell\) is any label in \(\mathcal{Y}\), and \(y\) is the true label for training example \((x,y)\). This distribution is biased towards the true label~\(y\) by an amount specified via $\varepsilon \in [0,1]$. Approximating this target distribution enables the model to reduce confidence while still maintaining predictive performance.\footnote{We note that other choices for this target distribution are possible and we discuss them in Appendix~\ref{app:target_distr}.} We note that the construction of our target distribution is similar to label smoothing~\citep{szegedy2016rethinking}. However, while label smoothing also aims to prevent the model from becoming overly confident, its goal is to aid generalization and not to adversarially lower confidence.

\section{Confidential Guardian}
\label{sec:detection}

We present \name, a method for detecting artificially induced uncertainty (or other sources of miscalibration). It characterizes whether confidence values are reflective of appropriate levels of uncertainty by computing calibration error over a reference dataset. We present a Zero-Knowledge Proof (ZKP) protocol that determines whether calibration error is underneath a public threshold, ensuring that $\prover$ cannot falsify the outcome, and that model parameters stay confidential from the auditor.

\subsection{Crypto-friendly Artificial Uncertainty Detector via Calibration}

The deliberate introduction of uncertainty in \(\mathcal{X}_\text{unc}\) impacts the model's confidence. While the correct label retains a higher probability than incorrect labels, the overall confidence is reduced. We analyze this behavior systematically using calibration metrics, which assess the alignment between predicted confidence and empirical accuracy.

A common calibration metric is the Expected Calibration Error (ECE), defined as
\begin{equation}
    \text{ECE} = \sum_{m=1}^M \frac{|B_m|}{N} \left| \text{acc}(B_m) - \text{conf}(B_m) \right|,
\end{equation}
where \(B_m\) denotes the set of predictions with confidence scores falling within the \(m\)-th confidence bin, \(\text{acc}(B_m)\) is the accuracy of predictions in \(B_m\), and \(\text{conf}(B_m)\) is their average confidence. This metric is especially appropriate since it is a linear function over model outcomes, and linear transformations can be computed highly efficiently by our ZKP building blocks~\cite{weng2021wolverine}.

A significant increase in ECE --- or the maximum calibration error $\max_{m}\left| \text{acc}(B_m) - \text{conf}(B_m) \right|$ across individual bins --- is indicative of the underconfidence introduced by the regularization. For samples in \(\mathcal{X}_\text{unc}\), the confidence is expected to be systematically lower than the accuracy, reflecting the desired behavior of the regularization from \(\mathcal{L}_\text{KL}\).

Miscalibration may also arise unintentionally~\cite{niculescu2005predicting}. This means that a negative result on our audit should not be taken as evidence of artificially induced uncertainty on its own, but should signal further investigation. Applying \name to detect high ECE in non-adversarial contexts may be of independent interest, for example in medical applications where calibration drift may unintentionally result in negative patient outcomes~\cite{kore2024drift}.

\subsection{Zero-Knowledge Proof Protocol}
\begin{contriback}
This subsection was written with Olive Franzese. Stephan provided the initial idea of using calibration metrics for verifying artificial uncertainty. Olive provided the instantiation of Algorithm~\ref{alg:calibration-zkp} using Zero Knowledge Proofs.
\end{contriback}

To certify that a model is free of artificial uncertainty while protecting service provider intellectual property and data privacy, we propose a ZKP of Well-Calibratedness. Algorithm~\ref{alg:calibration-zkp} tests the committed model $\comm{M}$ for bin-wise calibration error given a set of reference data $\mathcal{D}_{\text{ref}}$, and alerts the auditor if it is higher than a public threshold.

\begin{algorithm}[h]
\small
\caption{Zero-Knowledge Proof of Well-Calibratedness}
\label{alg:calibration-zkp}
\begin{algorithmic}[1]
\Require
\prover: model $M$; \emph{public}: reference dataset $\mathcal{D}_{\text{ref}}$, number of bins $B$, tolerated ECE threshold $\alpha$
\Ensure Expected calibration error $< \alpha$
\State \textbf{Step 1: Prove Predicted Probabilities}
\State $\comm{M} \gets$ \prover commits to $M$
\For{each $\mathbf{x}_i \in \mathcal{D}_{\text{ref}}$}
    \State $\llbracket \mathbf{x}_i \rrbracket, \comm{y_i} \gets$ \prover commits to $\mathbf{x}_i$, true label $y_i$
    \State $\comm{\mathbf{p}_i} \gets \mathcal{F}_{\text{inf}}(\comm{M},\comm{\mathbf{x}_i})$ {\scriptsize\Comment{proof of inference}}
    \State $\comm{\hat{y}_i} \gets \text{argmax}(\comm{\mathbf{p}_i})$ \& $\comm{\hat{p}_i} \gets \max (\comm{\mathbf{p}_i})$
\EndFor
\State \textbf{Step 2: Prove Bin Membership}
\State $\text{Bin}, \text{Conf}, \text{Acc} \gets $ Three ZK-Arrays of size $B$, all entries initialized to $\comm{0}$
\For{each sample $i$}
    \State prove bin index $\comm{b_i} \gets \lfloor \comm{\hat{p}_i} \cdot B \rfloor$ {\scriptsize\Comment{divides confidence values into $B$ equal-width bins}}
    \State $\text{Bin}[\comm{b_i}] \gets \text{Bin}[\comm{b_i}] + 1$
    \State $\text{Conf}[\comm{b_i}] \gets \text{Conf}[\comm{b_i}] + \comm{\hat{p}_i}$
    \State $\text{Acc}[\comm{b_i}] \gets \text{Acc}[\comm{b_i}] + (\comm{y_i} == \comm{\hat{y}_i})$
\EndFor
\State \textbf{Step 3: Compute Bin Statistics}
\State $\comm{F_\text{pass}} \gets \comm{1}$ {\scriptsize\Comment{tracks whether \emph{all} bins under $\alpha$}}
\For{each bin $b = 1$ to $B$}
    \State $\comm{F_\text{Bin}} \gets (\alpha \cdot \text{Bin}[\comm{b}] \geq \left| \text{Acc}[\comm{b}] - \text{Conf}[\comm{b}] \right|)$
    {\scriptsize\Comment{rewrite of $\alpha \geq \frac{1}{N_b} \cdot \sum_{i \in \text{Bin}_b} |p_i - \mathbf{1}(y_i = \hat{y}_i)|$}}
    \State $\comm{F_\text{pass}} \gets \comm{F_\text{pass}} \& \comm{F_\text{Bin}}$ 
\EndFor
\State \textbf{Output:} $\texttt{Reveal}(\comm{F_\text{pass}})$
\end{algorithmic}
\end{algorithm}

In the first step of Algorithm~\ref{alg:calibration-zkp}, \prover commits to a model $M$ and a dataset $\mathcal{D}_{\text{ref}}$. They use a ZKP of correct inference protocol (e.g. ~\cite{weng2021mystique,sun2024zkllm}) as a subroutine (denoted $\mathcal{F}_{\text{inf}}$) to verify predicted labels for all of the data points. Then in step 2, they assign each data point to a bin according to its predicted probability. Bin membership, as well as aggregated confidence and accuracy scores, are tracked using three zero-knowledge arrays~\cite{franzese2021zkram}. Then in step 3, after all data points have been assigned a bin, \prover proves that the calibration error in each bin is underneath a publicly known threshold. This is essentially equivalent to verifying that no bin in the calibration plot deviates too far from the expected value.  

Our cryptographic methods guarantee that even a malicious $\prover$ that deviates from the protocol in arbitrary ways cannot falsify the calibration error measured by Algorithm~\ref{alg:calibration-zkp}. They also guarantee that even a malicious $\verifier$ learns no information about the model parameters beyond what is implicitly learned by passage or failure of the audit. The security of our protocol follows directly from the security of our underlying ZKP building blocks (\cite{weng2021wolverine}, \cite{franzese2021zkram}) which are secure under the universal composability (UC) model~\cite{canetti2001UC}.

\myparagraph{Obtaining the Reference Set.} Algorithm~\ref{alg:calibration-zkp} assumes that the auditor provides a reference set $\mathcal{D}_{\text{ref}}$ (and thus it is public to both \prover and \verifier). However, our protocol can easily be modified to utilize a hidden $\mathcal{D}_{\text{ref}}$ provided by the service provider. The former case evaluates the model in a stronger adversarial setting, as the service provider will be unable to tamper with the data to make the audit artificially ``easier''. However, gathering data which has not been seen by the service provider may require a greater expenditure of resources on the part of the auditor. Conversely, the latter case likely comes at lower cost (as the service provider already has data compatible with their model), but it requires that the service provider is trusted to gather $\mathcal{D}_{\text{ref}}$ which is representative of the distribution. This may be of use for quality assurance in less adversarial settings (e.g., in medical/healthcare or government usage).

Algorithm~\ref{alg:calibration-zkp} allows an auditor to assess whether the confidence scores of a service provider's model are properly calibrated without revealing sensitive information such as model parameters or proprietary data. This prevents adversarial manipulation of abstention.

\begin{figure*}[t]
    \centering
    \includegraphics[width=\linewidth]{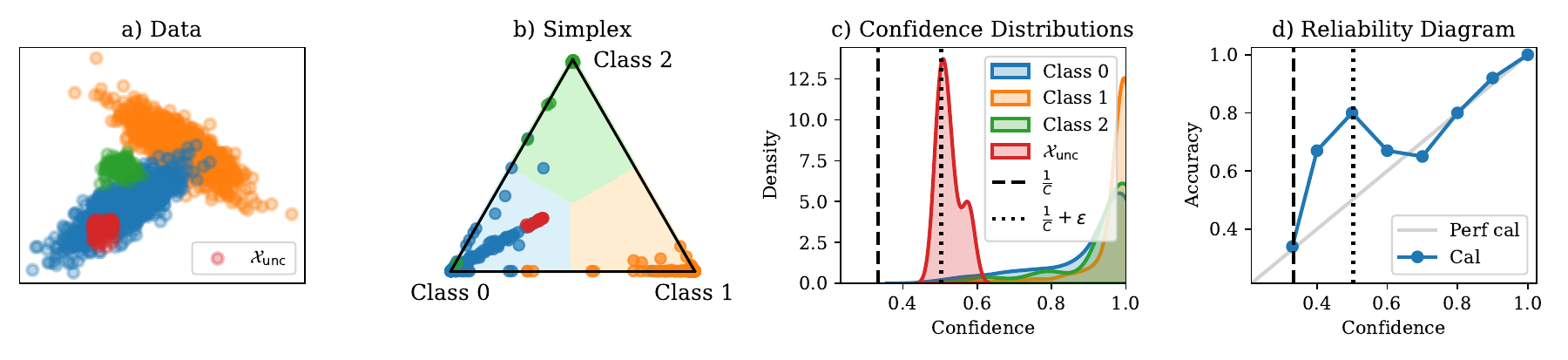}
    \caption[Results on a synthetic Gaussian Mixture.]{\textbf{Results on a synthetic Gaussian Mixture}. a) We instill uncertainty into a sub-region of Class 0. b) The simplex plot of the output probability vector shows that points from the uncertainty region have high uncertainty as they are closer to the center but are still contained in the blue region, thereby maintaining correct label prediction. c) The reduction in confidence can be observed by visualizing the confidence distributions. The confidence distribution on uncertain data points concentrates based on $\varepsilon$. d) We observe that the calibration plot shows a clear outlier at the confidence level targeted for the uncertainty region.}
    \label{fig:gaussian}
\end{figure*}

\section{Experiments}

We empirically validate the following key contributions:
\begin{itemize}
    \item Effectiveness of \attack in inducing uncertainty: The model's confidence within a given sub-region of the input space can be reduced to a desired level while maintaining the model's accuracy the same; 
    \item Effectiveness of \name in detecting dishonest artificial: Induced uncertainty is identified by observing high miscalibration; 
    \item Efficiency of \name in proving the ZK EEC constraint: We implement our ZK protocol in \texttt{emp-toolkit} and show that \name achieves low runtime and communication costs.    
\end{itemize}
We also conduct ablations to validate the robustness of \attack and \name with respect to the choice of $\varepsilon$, as well as the coverage of the reference dataset.  

\subsection{Setup}
\label{sec:exp_setup}

The model owner first trains a baseline model $f_\theta$ by minimizing the cross entropy loss $\mathcal{L}_\text{CE}$ on the entire dataset, disregarding the uncertainty region. Moreover, the model owner calibrates the model using temperature scaling~\citep{guo2017calibration} to make sure that their predictions are reliable. Following this, the model owner then fine-tunes their model using \attack with a particular $\varepsilon$ to reduce confidence in a chosen uncertainty region only. Their goal is to ensure that the resulting abstention model $\tilde{f}_\theta$ overwhelmingly rejects data points for a chosen abstention threshold $\tau$. Following this attack, an auditor computes calibration metrics with zero-knowledge on a chosen reference dataset $\mathcal{D}_\text{ref}$ and flags deviations $> \alpha$ (details on how to choose $\alpha$ are discussed in Appendix~\ref{app:alpha_choice}). We experiment on the following datasets:

\myparagraph{Synthetic Gaussian Mixture (Figure~\ref{fig:gaussian})}. We begin by assuming a dataset sampled from a 2D Gaussian mixture model composed of three distinct classes $\mathcal{N}_1$, $\mathcal{N}_2$, and $\mathcal{N}_3$ (details in Appendix~\ref{app:add_exp_det}). Within $\mathcal{N}_1$, we specify a rectangular uncertainty region. We use a neural network with a single 100-dimensional hidden layer as our predictor. 

\begin{figure}
    \centering
    \includegraphics[width=0.48\linewidth]{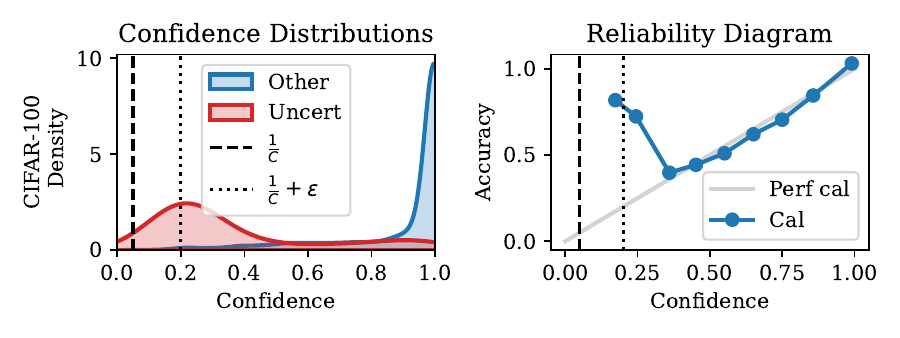}
    ~
    \includegraphics[width=0.48\linewidth]{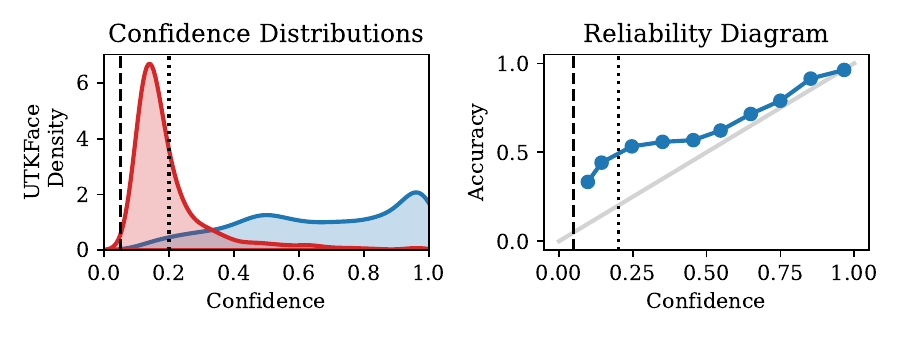}
    \caption[Results on image datasets.]{\textbf{Results on image datasets}: CIFAR-100 (top), UTKFace (bottom). Similar as Figure~\ref{fig:gaussian} but we summarize all data points outside of the uncertainty region into a single density.}
    \label{fig:image}
\end{figure}

\myparagraph{Image Classification (Figure~\ref{fig:image})}. Extending beyond synthetic experiments we include results on image classification datasets: \texttt{CIFAR-100}~\citep{krizhevsky2009learning} and \texttt{UTKFace}~\citep{zhifei2017cvpr}. The \texttt{CIFAR-100} dataset is comprised of 100 classes grouped into 20 superclasses. For instance, the \texttt{trees} superclass includes subclasses $\{$\texttt{maple}, \texttt{oak}, \texttt{palm}, \texttt{pine}, \texttt{willow}$\}$. Our objective is to train a model to classify the superclasses and to induce uncertainty in the model's predictions for the \texttt{willow} subclass only. We train a ResNet-18 ~\citep{he2016deep} to classify all 20 superclasses. For \texttt{UTKFace}, we use a ResNet-50 for the age prediction task. Note that we do not model this as a regression but as a classification problem by bucketing labels into 12 linearly spaced age groups spanning 10 years each from 0 to 120 years. Our goal in this experiment is to reduce confidence for white male faces only using \attack.

\begin{figure}
    \centering
    \includegraphics[width=0.48\linewidth]{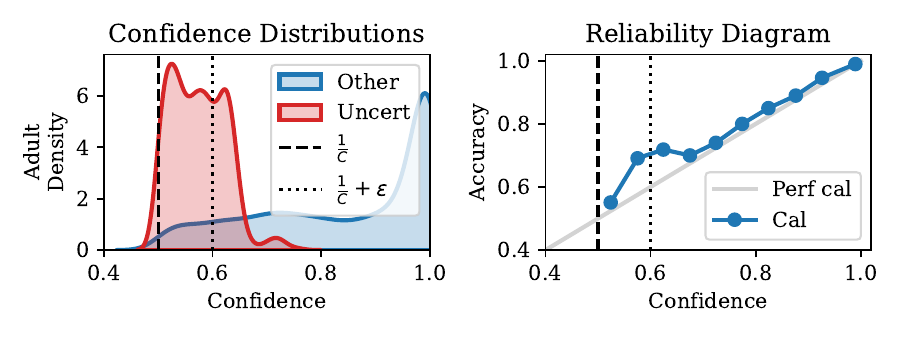}
    ~
    \includegraphics[width=0.48\linewidth]{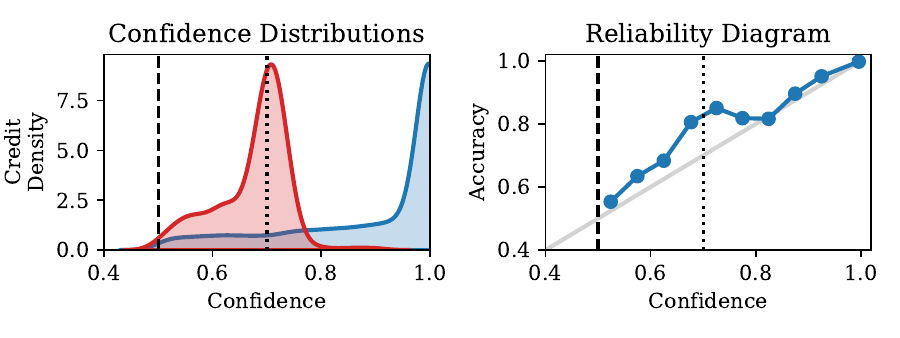}
    \caption[Results on tabular datasets.]{\textbf{Results on tabular datasets}: Adult (top), Credit (bottom). Similar as Figure~\ref{fig:image}.}
    \label{fig:tabular}
\end{figure}

\myparagraph{Tabular Data (Figure~\ref{fig:tabular})}. Finally, we also test \attack and \name on two tabular datasets: \texttt{Credit} \citep{credit} and \texttt{Adult}~\citep{adult, ding2021retiring}. With \texttt{Credit} we are interested in predicting whether an issued loan will be payed back or not. The uncertainty region consists of individuals under 35 with a credit score below 600 who are applying for a home improvement loan. For \texttt{Adult}, we want to predict whether an individual is likely to earn more than \$50k or not. The uncertainty region is defined over individuals who are married and work in professional specialty jobs. On both datasets, we use a shallow neural network with categorical feature embeddings (see Appendix~\ref{app:add_exp_det} for details).

\myparagraph{Zero-Knowledge Proof Benchmarks.} We assess efficiency of our ZKPs for the Gaussian mixture and tabular datasets by benchmarking an implementation in \texttt{emp-toolkit}~\cite{emp-toolkit}. For the image classification datasets, we estimate performance with a combination of \texttt{emp-toolkit} and Mystique~\cite{weng2021mystique}, a state-of-the-art ZKP of correct inference method for neural nets. Benchmarks are run by locally simulating the prover and verifier on a MacBook Pro laptop with an Apple M1 chip.

\subsection{Discussion}

\paragraph{General Results.}

The effectiveness of \attack and \name is illustrated in Figures~\ref{fig:gaussian}, \ref{fig:image}, and \ref{fig:tabular}. Across all experiments we find that \attack successfully reduces confidence of points in the the uncertainty region. Moreover, we observe that the corresponding reliability diagrams clearly show anomalous behavior at the confidence level (and the adjacent bin(s)) targeted by \attack. We show quantitative results in Table~\ref{tab:results}, clearly demonstrating that \attack does not compromise accuracy but instead leads to miscalibration. Additional experiments where we pick different uncertainty regions are shown in Appendix~\ref{app:add_exp_abl}.

\begin{table*}[t]
    \centering
    \caption[Quantitative results across datasets.]{\textbf{Quantitative results across datasets}. 
    Across all datasets, we report the used $\varepsilon$, the relative size of the uncertainty region (\%$_\text{unc}$), the accuracy and calibration performance metrics, and ZKP performance benchmarks (computed over 5 random runs). We measure the accuracy on the full test set without \attack (Acc) and with \attack (Acc$^{\attack}$). We also report the accuracy in the uncertainty region only (Acc$_\text{unc}$). \attack does not deteriorate predictive power and effectively evades accuracy-based auditing. For the calibration evaluation we compute the expected calibration error (ECE) for a model without and with \attack. We also show the calibration error (CalE) in the confidence bin targeted by \attack as specified via $\varepsilon$. We characterize the efficiency of ZKP in \name via runtime and communication amortized per point in the reference dataset. \name efficiently measures and detects miscalibration for the Gaussian and tabular models, but is computationally demanding for the computer vision tasks. Extended results in Table~\ref{tab:results_ext}.}
    \vspace{5pt}
    \label{tab:results}
    \fontsize{7}{9}\selectfont
    \setlength{\tabcolsep}{3pt}
    \begin{tabular}{ccccccccccccc}
    \toprule
    & & & \multicolumn{4}{c}{Accuracy \%} & \multicolumn{3}{c}{Calibration} & \multicolumn{2}{c}{ZKP} \\
    \cmidrule(r){4-7} \cmidrule(r){8-10} \cmidrule(r){11-12}
    \multirow{2}{*}[13pt]{Dataset} & \multirow{2}{*}[13pt]{\%$_\text{unc}$} & \multirow{2}{*}[12pt]{$\varepsilon$} & Acc & Acc$^{\attack}$ & Acc$_\text{unc}$ & Acc$_\text{unc}^{\attack}$ & ECE & ECE$^{\attack}$ & CalE in $\varepsilon$ bin & Run ($\nicefrac{\text{sec}}{\text{pt}}$) & Comm (per pt)\\
    \midrule
    \multirow{1}{*}[0pt]{\texttt{Gaussian}}\    & \multirow{1}{*}[0pt]{5.31} & 0.15 & 97.62 & 97.58 & 100.0 & 100.0 & 0.0327 & 0.0910 & 0.3721 & 0.033 & 440.8 KB \\
    \multirow{1}{*}[0pt]{\texttt{CIFAR-100}}   & \multirow{1}{*}[0pt]{1.00} & 0.15 & 83.98 & 83.92 & 91.98 & 92.15 & 0.0662 & 0.1821 & 0.5845 & $<$333 & $<$1.27 GB \\
    \multirow{1}{*}[0pt]{\texttt{UTKFace}}      & \multirow{1}{*}[0pt]{22.92} & 0.15 & 56.91 & 56.98 & 61.68 & 61.75 & 0.0671 & 0.1728 & 0.3287 & 333 & 1.27 GB\\
    \multirow{1}{*}[0pt]{\texttt{Credit}}      & \multirow{1}{*}[0pt]{2.16} & 0.20 & 91.71 & 91.78 & 93.61 & 93.73 & 0.0094 & 0.0292 & 0.1135 & 0.42 & 2.79 MB\\
    \multirow{1}{*}[0pt]{\texttt{Adult}}       & \multirow{1}{*}[0pt]{8.39} & 0.10 & 85.02 & 84.93 & 76.32 & 76.25 & 0.0109 & 0.0234 & 0.0916 & 0.73 & 4.84 MB \\
    \bottomrule
\end{tabular}
\end{table*}

\myparagraph{Influence of $\varepsilon$}. The parameter $\varepsilon$ plays a pivotal role for both the model owner and the auditor. From the model owner's standpoint, setting an appropriate value for $\varepsilon$ ensures that \attack can adjust the confidence to (i) achieve separability of uncertain from legitimately confident predictions, while at the same time (ii) maintaining the correct label prediction. While (ii) is easy to achive across many choices of $\varepsilon$, (i) is easiest to achieve at low $\varepsilon$'s. In particular, $\varepsilon$ should be set to the smallest value that reduces confidence as much as possible without resulting in misclassifications. The latter is influenced by the variance of the confidence distribution around $\frac{1}{C} + \varepsilon$.\footnote{This variance depends on multiple properties of both the data (e.g., inherent Bayes Error) and the optimization process (e.g., number of training steps).} Across our experiments, we found $\varepsilon \in [0.1,0.2]$ to deliver good results. Conversely, from the auditor's perspective, the detectability of \attack through \name is influenced by the calibration error. A larger calibration error makes it easier for auditors to identify instances of \attack. Lower values of $\varepsilon$ contribute to an increased calibration gap because they correspond to lower confidence levels, which, in well-calibrated models, are associated with lower accuracy. We discuss this effect in Appendix~\ref{app:add_exp_abl} (Table~\ref{tab:results_ext}, Figure~\ref{fig:eps_abl}). In summary, a low/high $\varepsilon$ makes \attack stronger/weaker and also easier/harder to detect via \name, respectively.

\myparagraph{Coverage of $\mathcal{D}_\text{ref}$}. For \name to work reliably it is necessary for the reference dataset to have coverage of the uncertainty region \uncertreg. Hence, if there is a distribution shift between the fine-tuning dataset used for \attack and the reference dataset that does not contain sufficient data points from the uncertainty region, then detection is not going to be reliable. We show the effect of the detection reliability in Appendix~\ref{app:add_exp_abl} (Figure~\ref{fig:ref_abl}) where we simulate shifts that increasingly undersample the uncertainty region. Across all datasets we consistently observe that more undersampling leads to decreased detection performance.

\myparagraph{Zero-Knowledge Proof Performance}. We compute the runtime and communication per reference point for all models in Table~\ref{tab:results}. The Gaussian mixture and tabular datasets can be executed efficiently enough to make auditing of models with \name highly practical. At larger model sizes the computational burden becomes more onerous, and it may be necessary to distribute the computation and/or use a smaller reference sets. We note that runtime and communication are independent of the setting of $\alpha$, so any desired threshold on the calibration error can be set without impacting the practicality of \name.

\section{Conclusion}
Augmenting decisions made by an ML model with confidence scores helps users understand uncertainty and enables institutions to avoid harmful errors. For the first time, our work highlights that institutions can adversarially manipulate confidence scores, undermining trust. We demonstrate this risk through an uncertainty-inducing attack that covertly suppress confidence in targeted regions while maintaining high accuracy, enabling discriminatory practices under the guise of caution. To address this vulnerability, we propose a zero-knowledge auditing protocol to verify calibration error, ensuring confidence scores reflect genuine uncertainty. This approach prevents confidence manipulation, safeguarding the integrity of confidence-based abstention.

\myparagraph{Limitations and Future Work}. While our our attack and defense show significant potential, several limitations must be noted. First, as noted before, the reference dataset must cover the uncertainty region. Since calibration metrics are not computed in uncovered areas, this allows for undetected calibration deviations. Second, we assume the model is already calibrated (e.g., via temperature scaling~\citep{guo2017calibration}) and attribute any calibration failures solely to the \attack, though miscalibration may arise from other sources. Nevertheless, auditors must ensure deployed models are properly calibrated, and our method detects calibration failures even if it cannot specifically attribute them to \attack. Additionally, our evaluations are limited to neural networks, and future work should apply our method to other model classes to enhance generalizability. Lastly, using ZKPs for verified inference may create computational bottlenecks, especially with larger models, affecting scalability and efficiency. Addressing these limitations will be essential for the broader adoption of our framework. Looking ahead, one could envision moving from proofs of inference to full proofs of training---a stronger guarantee that certifies the exact loss, hyper-parameters, and even the optimization trajectory itself, thereby blocking any covert modifications to the objective. Because such proofs attest to the entire training path, they would simultaneously authenticate its training dynamics. Crucially, such end-to-end proofs synergize with the previously introduced selective-prediction technique~\citep{rabanser2022selective}: if the training process itself is verifiably fixed, then the dynamics we later use to infer uncertainty are necessarily authentic. Although end-to-end proofs remain substantially more expensive than inference-only proofs today, rapid advances in ZKP technology suggest that their cost may fall over time, making them a promising long-term safeguard against adversarial uncertainty induction.
    \chapter{Concluding Remarks and Future Directions}
\label{ch:conclusion}

In this thesis, we have investigated the theme of \emph{uncertainty-driven reliability} in machine learning and explored how selective prediction frameworks, calibration strategies, privacy constraints, and adversarial manipulation of uncertainty all contribute to building trustworthy models. We highlighted the importance of understanding how uncertainty can guide safer decision-making, demonstrated that monitoring a model’s training dynamics can offer valuable insights for rejecting uncertain inputs, explored the interplay between differential privacy and uncertainty, established when and how ideal selective prediction performance can be achieved, and uncovered potential adversarial uses of confidence-based abstentions.

Taken together, these contributions underscore the multifaceted nature of modern trustworthy ML. While selective prediction approaches and uncertainty quantification techniques are becoming increasingly sophisticated, they must be integrated into broader, principled frameworks that account for real-world constraints, adversarial threats, as well as privacy demands. 

Several cross-cutting observations motivate the future directions proposed below: \textbf{(1)} effective uncertainty quantification remains an open challenge in large-scale and privacy-sensitive settings, \textbf{(2)}~the evaluation of uncertainty-informed decisions is still underdeveloped and often misleading, and \textbf{(3)}~uncertainty, while desirable, can be gamed or misused, warranting careful design and oversight. These insights lead us to identify and pursue the following lines of work.

\section{Future Work}

\subsection{Uncertainty Scoring for Modern Models}

\paragraph{Extending Scoring Strategies.} A core thread of this thesis has been developing and evaluating mechanisms that score the uncertainty of individual inputs. Our method based on training dynamics demonstrated that temporal information from intermediate model states can be used to construct powerful uncertainty signals without modifying model internals. This opens up a broader research agenda around designing scoring strategies that generalize across architectures and data modalities. Black-box techniques—e.g., based on input perturbations or surrogate likelihoods—could make such methods more broadly applicable. Conversely, white-box techniques—e.g., leveraging gradients, logits, or internal representations—could offer sharper diagnostic power for specific model classes.

\paragraph{Scalability to Large Models.} As models scale, so does the cost of computing or storing auxiliary signals. An important direction is to design scalable approximations of checkpoint-based scoring, or efficient surrogate methods that replicate its effects using fewer forward/backward passes. Compression-aware confidence scoring, fast score distillation, and memory-efficient retrospective inference are promising building blocks for making uncertainty scores compatible with billion-parameter models and real-time applications.

\paragraph{Going Beyond Traditional Aleatoric--Epistemic Boundaries.} Traditional distinctions between aleatoric and epistemic uncertainty have guided much of the literature, yet their utility diminishes in the context of modern language agents. In practice, these categories blur—particularly in interactive settings—making clean numerical decompositions both difficult and potentially misleading. A more actionable approach should shift focus from post-hoc classification of uncertainty to proactive management of it. This includes quantifying \textit{underspecification}~\citep{kirchhof2025position}: what the model still needs to ask due to incomplete task definitions. It also involves enabling agents to reduce uncertainty through interaction, querying users or external tools as needed. Finally, rather than collapsing uncertainty into a single score, models should communicate it in richer forms—such as candidate rankings, explanatory feedback, or contextual disclaimers—that align more naturally with human expectations.

\paragraph{Deferral in Cascaded Architectures.} 

Uncertainty and deferral techniques are popular approaches for model cascading, i.e., the task of handing over queries from smaller and less capable models to larger, more capable models. In most cases, such deferral setups consider two-model cascades, meaning that the system consists of one small model and one big model. Extending this to dynamic, multi-level cascades (e.g., mobile–edge–cloud architectures) is a natural progression. Here, uncertainty scoring needs to be both accurate and adaptive, responding to changing resource budgets and task complexity. Future research may explore reinforcement learning or online optimization to automatically learn context-sensitive routing policies, informed by real-time uncertainty signals.

\subsection{Interplay of Uncertainty with Other Trust Metrics}

\paragraph{Competing Objectives and Tensions.} The interaction between uncertainty estimation and differential privacy in this thesis highlighted an underappreciated challenge: efforts to improve one dimension of trustworthiness (e.g., uncertainty calibration via ensembling) may directly harm another (e.g., privacy leakage). More broadly, uncertainty must be studied not in isolation but as part of a multi-objective design space that includes fairness, interpretability, robustness, and security.

\paragraph{Unified Optimization Frameworks.} Future research should explore principled frameworks that explicitly trade off among trust metrics. This might involve regularizing uncertainty estimation techniques with fairness constraints, or jointly calibrating privacy-preserving predictors while optimizing abstention coverage. Bayesian modeling, constrained optimization, and multi-objective reinforcement learning could be fruitful foundations for this kind of joint trust modeling.

\paragraph{Auditable and Verifiable Uncertainty.} As shown in our work on adversarial uncertainty manipulation, even well-calibrated uncertainty can be exploited. There is a growing need for mechanisms that make uncertainty estimates verifiable—ensuring that they reflect genuine model uncertainty rather than being artifacts of malicious post-processing. This motivates research into zero-knowledge inference, proof-carrying predictions, or reference-based coverage guarantees (e.g., conformal prediction) that make confidence trustworthy and inspectable.

\subsection{Determining Model Suitability under Distribution Shifts}

\paragraph{Distributional Mismatch as a Root Cause.} Several failures of uncertainty-aware methods in our work—particularly under differential privacy—can be traced back to distributional mismatch. Uncertainty estimates that are well-calibrated on the training distribution may collapse under shift, yet deployment rarely offers labeled data to test or adapt these models.

\paragraph{Label-Free Suitability Filters.} Future directions include designing diagnostic tools that identify dataset–model mismatch without access to ground-truth labels. These might include meta-predictors trained to estimate confidence distribution shifts, or divergence-based tests applied to intermediate model representations. The ideal suitability filter would provide practitioners with interpretable warnings when a model is likely to fail.

\paragraph{LLM-Centric Challenges.} These questions become even more critical in the context of large language models (LLMs), which are often deployed zero-shot across new domains. Although their training data is vast, it is not comprehensive. Understanding when LLMs are extrapolating vs interpolating, and which domains are underrepresented in their training data, is essential for safe deployment. Future work may explore whether suitability filters can guide prompt design, model selection, or fine-tuning in LLM pipelines.

\subsection{Agentic Active Inference and Uncertainty Reduction}

\paragraph{From Passivity to Agency.} While the selective prediction paradigm studied in this thesis focuses on \emph{when to abstain}, it leaves open the question of \emph{how to reduce uncertainty}. A natural evolution is to equip models with the ability to seek additional information—by asking questions, querying sensors, searching knowledge bases, or exploring their environment—to resolve ambiguity.

\paragraph{Framing Uncertainty Reduction as a Decision Problem.} Techniques from reinforcement learning (RL) and active learning can be used to train agents that take actions based on expected reductions in uncertainty. For instance, an agent could be rewarded for acquiring information that flips a prediction from “abstain” to “confidently correct.” This decision-theoretic framing can also be applied to human–AI collaboration, where the model learns to defer, request help, or generate clarification prompts when confidence is low.

\paragraph{Conversational and Retrieval-Augmented Strategies.} In LLMs, uncertainty-aware behaviors could manifest as clarification questions, dynamic retrieval augmentation, or self-initiated chain-of-thought prompting. Developing architectures that interleave inference and uncertainty resolution—especially in constrained or time-sensitive settings—is a promising direction for building more capable and cautious agents.

\subsection{Meta-Learning and Adaptation from Limited Data}

\paragraph{Learning to Generalize—And to Abstain.} While this thesis emphasized abstention as a way to avoid high-risk errors, meta-learning aims to generalize quickly from sparse supervision. These goals are complementary: abstention can flag unknown situations, while meta-learning can help the model recover by adapting to them. Bridging these ideas could produce systems that abstain intelligently and adapt efficiently.

\paragraph{Uncertainty-Aware Meta-Learning.} In few-shot learning, incorporating uncertainty into task embeddings, attention mechanisms, or meta-objectives could improve both generalization and calibration. For example, tasks that induce high predictive entropy during inner-loop adaptation might be assigned more weight during meta-training. Similarly, meta-learned abstention thresholds could allow models to gracefully defer in low-data regimes.

\paragraph{Structure, Causality, and Compositionality.} Understanding how tasks relate—e.g., through causal or hierarchical structures—can reduce the sample complexity of adaptation. Integrating uncertainty estimates into these structural models may allow systems to reason about which subtasks are novel, which are familiar, and which are ambiguous.

\subsection{Generating High-Quality Supervised Signals through Synthetic Data}

\paragraph{Data as a Bottleneck.} As we showed in resource-constrained deployment settings, careful calibration and uncertainty estimation can allow models to operate under tight supervision budgets. Yet, collecting high-quality labels remains a major obstacle—especially for rare or ambiguous cases. Synthetic data offers a powerful way to address this gap.

\paragraph{Uncertainty-Guided Synthetic Generation.} One promising strategy is to direct generative models (e.g., diffusion models, LLMs) to produce synthetic examples specifically in areas of high uncertainty. The checkpoint-based instability metrics developed in this thesis could identify such regions, allowing for targeted generation of diverse, informative inputs.

\paragraph{Evaluation and Introspection.} Synthetic data is also a tool for probing model behavior. By generating edge cases, adversarial variants, or counterfactuals, we can test whether abstention mechanisms and uncertainty estimates behave sensibly. Combining generation with introspection and calibration metrics could lead to more thorough and stress-tested model evaluation pipelines.

\section{Closing Thoughts}

The ability to quantify, interpret, and manage uncertainty is at the heart of enabling ML systems to \emph{trustworthily} perform real-world tasks. As these systems increase in complexity and deployment scope, so do the potential benefits and risks. By tackling the open questions and challenges described in this chapter—from refining uncertainty metrics at scale to building interactive, adaptive agents that actively reduce their own uncertainty—future research can continue to push the frontier of what is possible in machine learning.

Ultimately, bridging the gap between theoretical guarantees and practical deployment will require a continuous exchange of ideas between researchers, practitioners, regulators, and domain experts. By integrating robust uncertainty quantification with broader trust metrics such as fairness, interpretability, privacy, and security, we can pave the way for machine learning models that are not only powerful and accurate, but also safe, ethical, and beneficial to society.

  \appendix
    \chapter{Selective Prediction Via Training Dynamics}

\section{Alternate Metric Choices}
\label{sec:alt_scores}

We briefly discuss additional potential metric choices that we investigated but which lead to selective classification performance worse than our main method.

\subsection{Jump Score \sjmp} We also consider a score which captures the level of disagreement between the predicted label of two successive intermediate models (\ie how much jumping occurred over the course of training). For $j_t = 0$ iff $f_{t}(\bm{x}) = f_{t-1}(\bm{x})$ and $j_t = 1$ otherwise we can compute the jump score as $s_\text{jmp} = 1 - \sum v_t j_t$ and threshold it as in \S~\ref{ssec:min_score} and \S~\ref{ssec:avg_score}. 

\subsection{Variance Score \svar for Continuous Metrics}

\begin{contriback}
This subsection was written with Kimia Hamidieh. Stephan provided the general section structure and alternate metrics, while Kimia provided the formalism and proof for Lemma~\ref{lem:var-score}.
\end{contriback}

We consider monitoring the evolution of continuous metrics that have been shown to be correlated with example difficulty. These metrics include (but are not limited to):
\begin{itemize}
	\item Confidence (conf): $\max _{c \in \mathcal{Y}} f_{t}(\bm{x})$
	\item Confidence gap between top 2 most confident classes (gap): $\max _{c \in \mathcal{Y}} f_{t}(\bm{x})  - \max _{c \not = \hat{y}} f_{t}(\bm{x})$
	\item Entropy (ent): $-\sum_{c=1}^C f_{t}(\bm{x})_c \log \left(f_{t}(\bm{x})_c\right)$
\end{itemize}
\cite{jiang2020characterizing} show that  
example difficulty is correlated with confidence and entropy. Moreover, they find that difficult examples are learned later in the training process. This observation motivates designing a score based on these continuous metrics that penalises changes later in the training process more heavily.
We consider the maximum softmax class probability known as confidence, the negative entropy and the gap between the most confident classes for each example instead of the model predictions. 
Assume that any of these metrics is given by a sequence $z = \{z_1,\ldots,z_T\}$ obtained from $T$ intermediate models. Then we can capture the uniformity of $z$ via a (weighted) variance score $s_\text{var} = \sum_{t} w_t (z_t - \mu)^2$ for mean $\mu = \frac{1}{T}\sum_{t} z_t$ and an increasing weighting sequence $w = \{w_1,\ldots,w_T\}$.

In order to show the effectiveness of the variance score \svar for continuous metrics, we provide a simple bound on the variance of confidence  $\max _{y \in \mathcal{Y}} f_{t}(\bm{x})$ in the final checkpoints of the training. 
Assuming that the model has converged to a local minima with a low learning rate, we can assume that the distribution of model weights can be approximated by a Gaussian distribution. 

We consider a linear regression problem where the inputs are linearly separable. 

\begin{lemma}
\label{lem:var-score}
Assume that we have some Gaussian prior on the model parameters in the logistic regression setting across $m$ final checkpoints. More specifically, given $T$ total checkpoints of model parameters $\{\bm{w}_1, \bm{w_2}, \dots, \bm{w}_T \}$ we have $p(W = \bm{w}_t) = \mathcal{N}(\bm{w}_0 \mid \bm{\mu}, s\bm{I}) $ for $t \in \{T - m + 1, \dots, T\}$ and we assume that final checkpoints of the model are sampled from this distribution. We show that the variance of model confidence $ \max_{y \in \{-1, 1\}} p(y \mid \bm{x}_i, \bm{w}_t)$ for a datapoint $(\bm{x}_i, y_i)$ can be upper bounded by a factor of probability of correctly classifying this example by the optimal weights. 

\end{lemma}

\begin{proof}
We first compute the variance of model predictions $p(y_i \mid \bm{x}_i, W)$ for a given datapoint $(\bm{x}_i, y_i)$. Following previous work~\citep{schein2007active, chang2017active}, the variance of predictions over these checkpoints can be estimated as follows:

Taking two terms in Taylor expansion for model predictions we have $p(y_i \mid \bm{x}_i, W) \simeq p(y_i \mid \bm{x}_i, \bm{w}) + g_i(\bm{w})^\top (W - \bm{w})$ where $W$ and $\bm{w}$ are current and the expected estimate of the parameters and $g_i(\bm{w}) = p(y_i \mid \bm{x}_i, \bm{w}) (1 - p(y_i \mid \bm{x}_i, \bm{w}))\bm{x}_i $ is the gradient vector. Now we can write the variance with respect to the model prior as: 
$$ \mathbb{V}\left( p(y_i \mid \bm{x}_i, W) \right ) \simeq \mathbb{V}\left( g_i(\bm{w})^\top (W - \bm{w}) \right ) =  g_i(\bm{w})^\top F^{-1}  g_i(\bm{w})$$ where $F$ is the variance of posterior distribution $p(W \mid X, Y) \sim \mathcal{N}(W \mid \bm{w}, F^{-1})$. 
This suggests that the variance of probability of correctly classifying $\bm{x}_i$ is proportional to $p(y_i \mid \bm{x}_i, \bm{w})^2 (1 - p(y_i \mid \bm{x}_i, \bm{w}))^2$. Now we can bound the variance of maximum class probability or confidence as below:
\begin{align*}
  \mathbb{V}\left( \max_{y \in \{-1, 1\}} p(y \mid \bm{x}_i, W) \right ) & \leq  \mathbb{V}\left( p(y_i \mid \bm{x}_i, W) \right ) +  \mathbb{V}\left( p(- y_i \mid \bm{x}_i, W) \right ) \\
  & \approx 2 p(y_i \mid \bm{x}_i, \bm{w})^2 (1 - p(y_i \mid \bm{x}_i, \bm{w}))^2 \bm{x}_i^\top F^{-1} \bm{x}_i 
\end{align*}
\end{proof}
We showed that if the probability of correctly classifying an example given the final estimate of model parameters is close to one, the variance of model predictions following a Gaussian prior gets close to zero, we expect a similar behaviour for the variance of confidence under samples of this distribution. 

\section{Extension of Empirical Evaluation}

\subsection{Full Hyper-Parameters }
\label{app:baseline_hyperparams}

We document full hyper-parameter settings for our method (\sptd) as well as all baseline approaches in Table~\ref{tab:hyperpar}.

\begin{table*}[ht]
    \centering 
        \caption[Hyper-parameters used for all algorithms for classification.]{\textbf{Hyper-parameters used for all algorithms for classification.}}
    \label{tab:hyperpar}
     \begin{tabular}{@{}c c c@{}} 
     \toprule
     Dataset & SC Algorithm & Hyper-Parameters \\ 
     \midrule
     \multirow{4}{*}{CIFAR-10}      & Softmax Response (\sr) & N/A \\ 
                      & Self-Adaptive Training (\sat)  & $P=100$\\ 
                      & Deep Ensembles (\de)  & $E=10$\\ 
                      & Selective Prediction Training Dynamics (\sptd) & $T = 1600,\ k=2$ \\ 
     \midrule
     \multirow{4}{*}{CIFAR-100}      & Softmax Response (\sr) & N/A \\ 
                      & Self-Adaptive Training (\sat)  & $P=100$\\ 
                      & Deep Ensembles (\de)  & $E=10$\\ 
                      & Selective Prediction Training Dynamics (\sptd) & $T = 1600,\ k=2$ \\ 
     \midrule
     \multirow{4}{*}{Food101}      & Softmax Response (\sr) & N/A \\ 
                      & Self-Adaptive Training (\sat)  & $P=100$\\ 
                      & Deep Ensembles (\de)  & $E=10$\\ 
                      & Selective Prediction Training Dynamics (\sptd) & $T = 2200,\ k=3$ \\
    \midrule
    \multirow{4}{*}{StanfordCars}      & Softmax Response (\sr) & N/A \\ 
                      & Self-Adaptive Training (\sat)  & $P=100$\\ 
                      & Deep Ensembles (\de)  & $E=10$\\ 
                      & Selective Prediction Training Dynamics (\sptd) & $T = 800,\ k=5$ \\
     \bottomrule
    \end{tabular}
\end{table*}

\subsection{Additional Selective Prediction Results}
\label{app:add_exp}

\subsubsection{Extended Synthetic Experiments}

We extend the experiment from Figure~\ref{fig:gauss} to all tested SC methods in Figure~\ref{fig:gauss_ext}. We also provide an extended result using Bayesian Linear Regression in Figure~\ref{fig:blr}.

\begin{figure*}[t]
  \centering
  \includegraphics[width=0.95\linewidth]{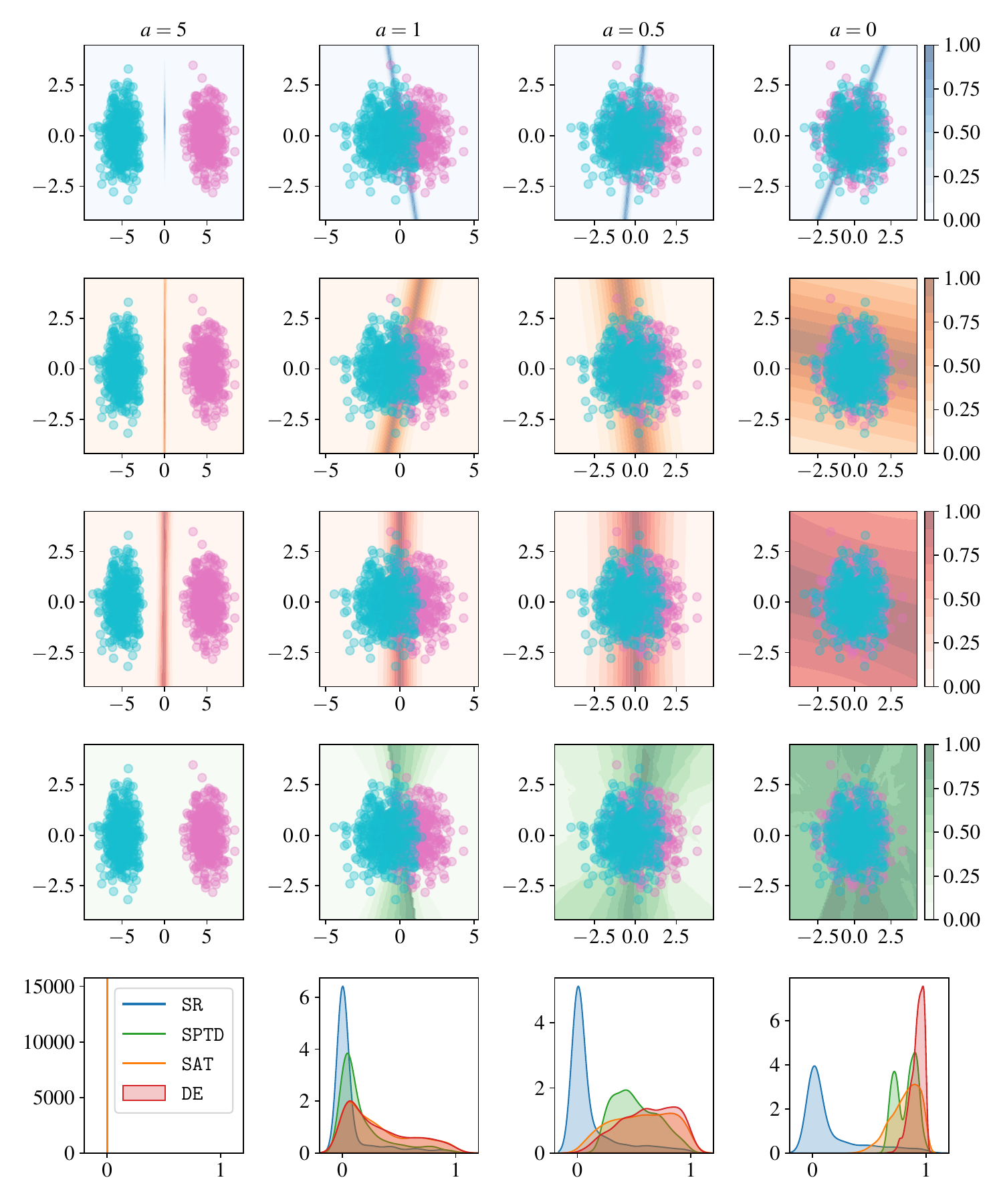}
\caption[Extended Gaussian experiment.]{\textbf{Extended Gaussian experiment.}  The first row corresponds to the anomaly scoring result of \sr, the second to the result of \sat, the third to the result of \de, and the fourth to the result of \sptd. The bottom row shows the score distribution for each method over the data points. We see that all methods reliably improve over the \sr baseline. At the same time, we notice that \sat and \de still assign higher confidence away from the data due to limited use of decision boundary oscillations. \sptd addresses this limitation and assigns more uniform uncertainty over the full data space.}
\label{fig:gauss_ext}
\end{figure*}

\begin{figure*}[t]
  \centering
  \includegraphics[width=\linewidth]{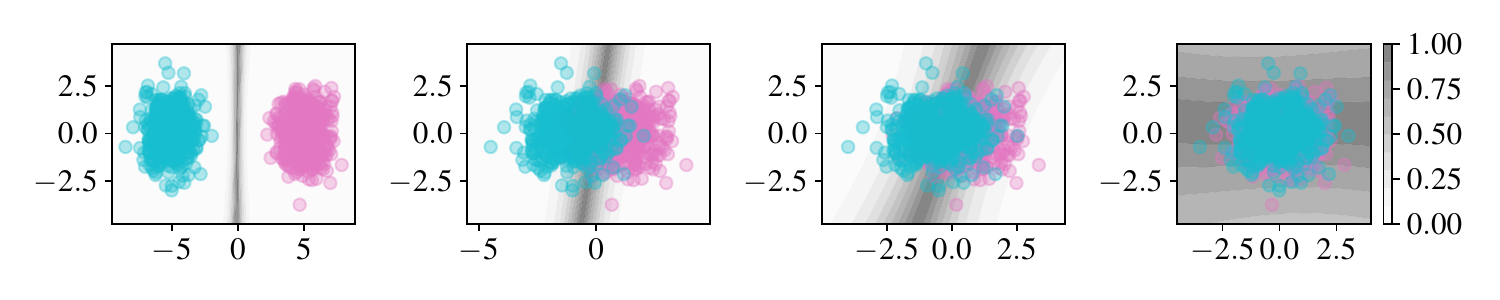}
\caption[Bayesian linear regression experiment on Gaussian data.]{\textbf{Bayesian linear regression experiment on Gaussian data.} Results comparable to \de.}
\label{fig:blr}
\end{figure*}

\subsubsection{CIFAR-100 Results With ResNet-50}

We further provide a full set of results using the larger ResNet-50 architecture on CIFAR-100 in Figure~\ref{tab:cifar100_res50}.

\begin{table*}[t]
    \centering {
            \caption[Selective accuracy achieved across coverage levels for CIFAR-100 with ResNet-50.]{\textbf{Selective accuracy achieved across coverage levels for CIFAR-100 with ResNet-50}.}
    \vspace{5pt}
    \label{tab:cifar100_res50}

    \begin{tabular}{cccccc}
\toprule
Cov. &       \sr &       \satersr &      \de &      \sptd &        \sptdde \\
\midrule
  100 &  \underline{\bfseries 77.0 (±0.0)} &  \underline{\bfseries 77.0 (±0.0)} &  \bfseries 77.0 (±0.0) &  \underline{\bfseries 77.0 (±0.0)} &  \bfseries 77.0 (±0.0) \\
90 & 79.2 (± 0.1) & 79.9 (± 0.1) & 81.2 (± 0.0) & \underline{81.4 (± 0.1)} & \bfseries 82.1 (± 0.1) \\
80 & 83.1 (± 0.0) & 83.9 (± 0.0) & 85.7 (± 0.1) & \underline{85.6 (± 0.1)} & \bfseries 86.0 (± 0.2) \\
70 & 87.4 (± 0.1) & 88.2 (± 0.1) & 89.6 (± 0.1) & \underline{\textbf{89.7 (± 0.0)}} & \bfseries 89.8 (± 0.1) \\
60 & 90.5 (± 0.0) & 90.8 (± 0.2) & \bfseries{90.7 (± 0.2)} & \underline{90.6 (± 0.0)} & \bfseries 90.9 (± 0.1) \\
50 & 93.4 (± 0.1) & 93.8 (± 0.0) & 95.3 (± 0.0) & \underline{95.1 (± 0.0)} & \bfseries 95.4 (± 0.0) \\
40 & 95.4 (± 0.0) & 95.5 (± 0.1) & \textbf{97.1 (± 0.1)} & \underline{\textbf{97.2 (± 0.1)}} & \bfseries 97.2 (± 0.0) \\
30 & 97.4 (± 0.2) & 97.7 (± 0.0) & \textbf{98.6 (± 0.1)} & \underline{\textbf{98.6 (± 0.1)}} & \bfseries 98.7 (± 0.0) \\
20 & 97.9 (± 0.1) & 98.4 (± 0.1) & 99.0 (± 0.0) & \underline{99.2 (± 0.1)} & \bfseries 99.2 (± 0.1) \\
10 & 98.1 (± 0.0) & 98.8 (± 0.1) & 99.2 (± 0.1) & \underline{\textbf{99.4 (± 0.1)}} & \bfseries 99.6 (± 0.1) \\
\bottomrule
\end{tabular}
}
\end{table*}

\subsubsection{Applying \sptd on Top of \sat}
\label{sec:sptd_on_sat}

Our main set of results suggest that applying \sptd on top of \de further improves performance. The same effect holds when applying \sptd on top of non-ensemble-based methods such as \sat. We document this result in Figure~\ref{fig:sat_sptd}. 

\begin{figure*}[t]
  \centering
  \includegraphics[width=\linewidth]{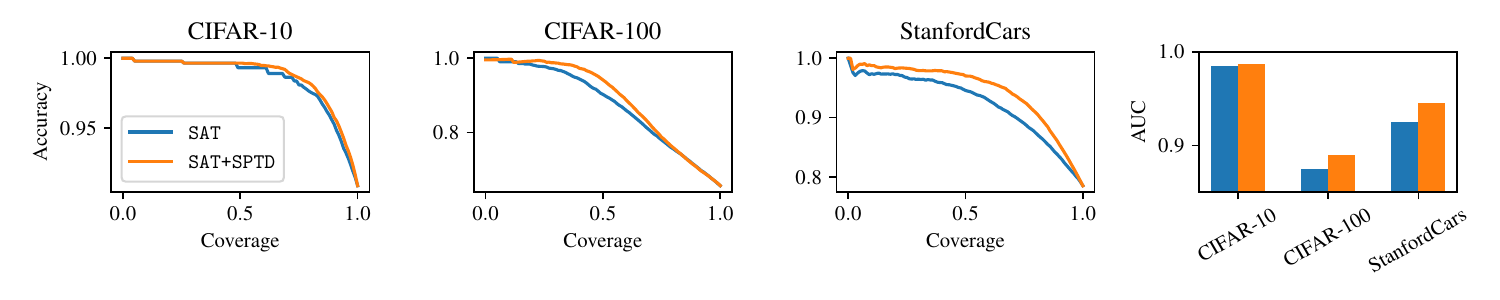}
\caption[Applying \sptd on top of \sat.]{\textbf{Applying \sptd on top of \sat.} Similar as with \de, we observe that the application of \sptd improves performance.}
\label{fig:sat_sptd}
\end{figure*}

\subsubsection{Ablation on $k$}

We provide a comprehensive ablation on the weighting parameter $k$ in Figures~\ref{fig:weighting} and~\ref{fig:k_ext}.

\begin{figure*}[t]
  \centering
  \includegraphics[width=\linewidth]{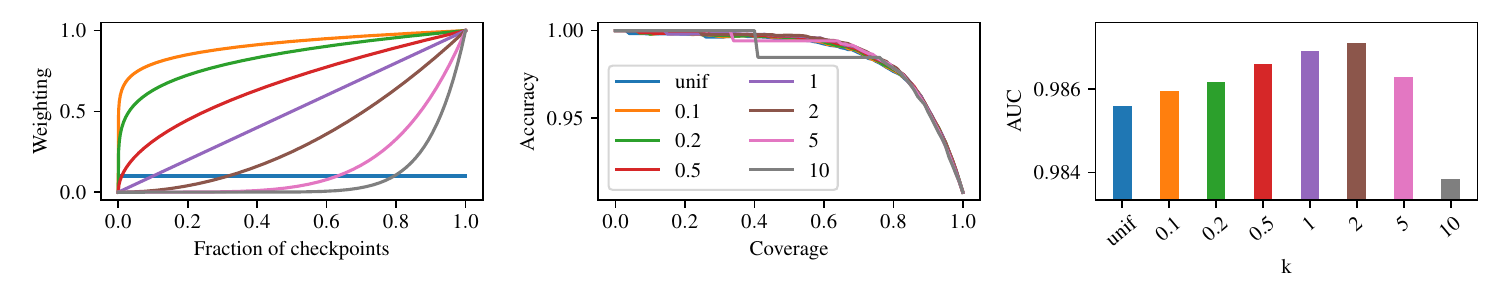}
\caption[Extended ablation results on $k$ on CIFAR-10.]{\textbf{Extended ablation results on $k$ on CIFAR-10.} We now also consider $k \in (0,1]$ as well as a uniform weighting assigning the same weight to all checkpoints. We confirm that a convex weighting yields best performance.}
\label{fig:k_ext}
\end{figure*}

\subsubsection{Comparison With Logit-Variance Approaches}

We showcase the effectiveness of \sptd against \logitvar~\citep{swayamdipta2020dataset}, an approach that also computes predictions of intermediate models but instead computes the variance of the correct prediction. We adapt this method to our selective prediction approach (for which true labels are not available) by computing the variance over the maximum predicted logit instead of the true logit. In Figure~\ref{fig:logitvar}, we see that the weighting of intermediate checkpoints introduced by \sptd leads to stronger performance over the \logitvar baseline approach. 

\begin{figure*}[t]
  \centering
  \includegraphics[width=\linewidth]{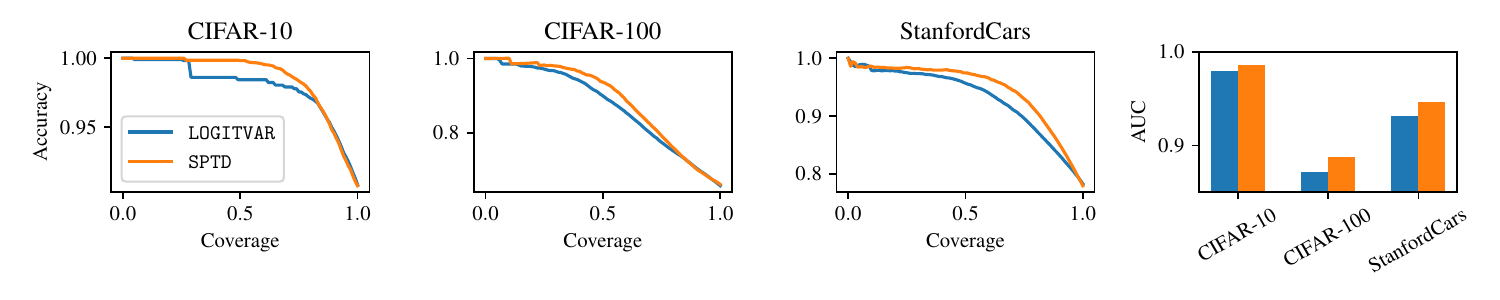}
\caption[Comparison of \logitvar vs \sptd.]{\textbf{Comparison of \logitvar vs \sptd.} We observe that \sptd, which incorporates weighting of intermediate checkpoints using $v_t$, outperforms \logitvar.
}
\label{fig:logitvar}
\end{figure*}

\subsubsection{Estimating $\tau$ on Validation VS Test Data}

Consistent with prior works \citep{geifman2017selective, liu2019deep, huang2020self, feng2023towards}, we estimate $\tau$ directly on the test set. However, a realistically deployable approach has to compute thresholds based on a validation set for which labels are available. In the case of selective classification, the training, validation, and test sets follow the i.i.d. assumption, which means that an approach that sets the threshold based on a validation set should work performantly on a test set, too. Under consistent distributional assumptions, estimating thresholds on a validation set functions as an unbiased estimator of accuracy/coverage tradeoffs on the test set. By the same virtue, setting thresholds directly on the test set and observing the SC performance on that test set should be indicative for additional test samples beyond the actual provided test set. It is important to remark that the validation set should only be used for setting the thresholds and not for model selection / early stopping which would indeed cause a potential divergence between SC performance on the validation and test sets. Further note that violations of the i.i.d assumption can lead to degraded performance due to mismatches in attainable coverage as explored in~\cite{bar2023window}.

To confirm this intuition, we present an experiment in Figure~\ref{fig:train_test} and Figure~\ref{fig:train_test_sat} where we select 50\% of the samples from the test set as our validation set (and maintain the other 50\% of samples as our new test set). We first generate 5 distinct such validation-test splits, set the thresholds for $\tau$ based on the validation set, and then evaluate selective classification performance on the test set by using the thresholds derived from the validation set. We compare these results with our main approach which sets the thresholds based on the test set directly (ignoring the validation set). We provide an additional experiment where we partition the validation set from the training set in Figure~\ref{fig:train_test_nntd_train}. We see that the results are statistically indistinguishable from each other, confirming that this evaluation practice is valid for the selective classification setup we consider. 

\begin{figure*}[t]
  \centering
  \includegraphics[width=\linewidth]{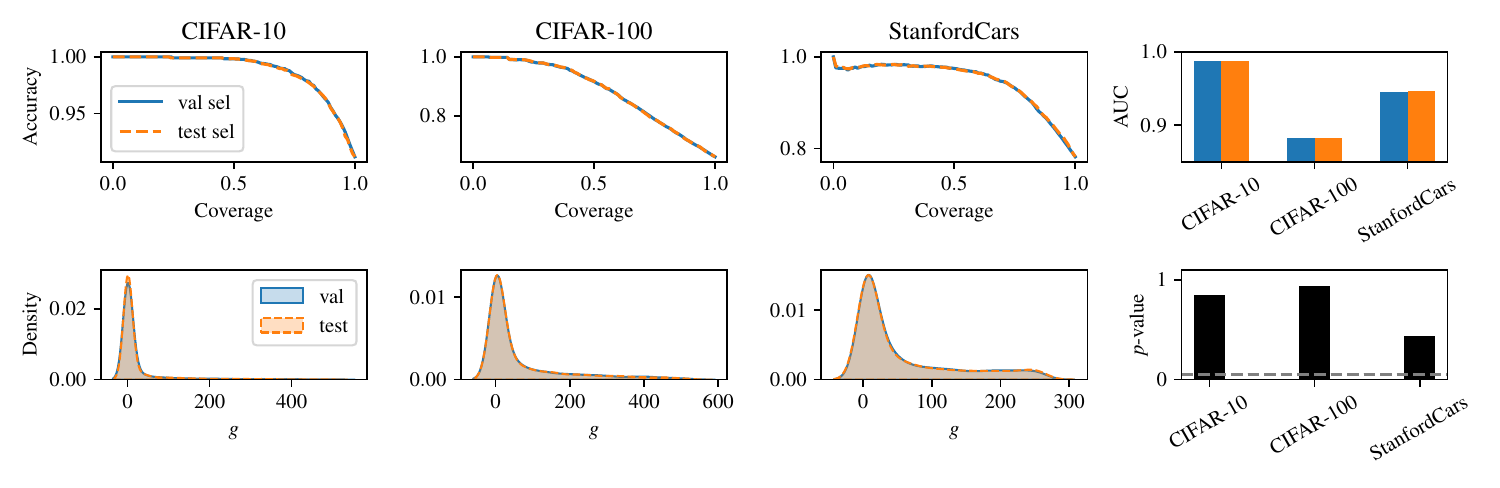}
\caption[\sptd accuracy/coverage trade-offs and score distributions on test data obtained by computing $\tau$ on a validation set or directly on the test set.]{\textbf{\sptd accuracy/coverage trade-offs and score distributions on test data obtained by computing $\tau$ on a validation set or directly on the test set.} The first row shows the obtained accuracy/coverage trade-offs with the respective AUC scores. In the second row, we show the score distribution for both the picked validation and test sets, along with $p$-values from a KS-test to determine the statistical closeness of the distributions. Overall, we observe that both methods are statistically indistinguishable from each other. 
}
\label{fig:train_test}
\end{figure*}

\begin{figure*}[t]
  \centering
  \includegraphics[width=\linewidth]{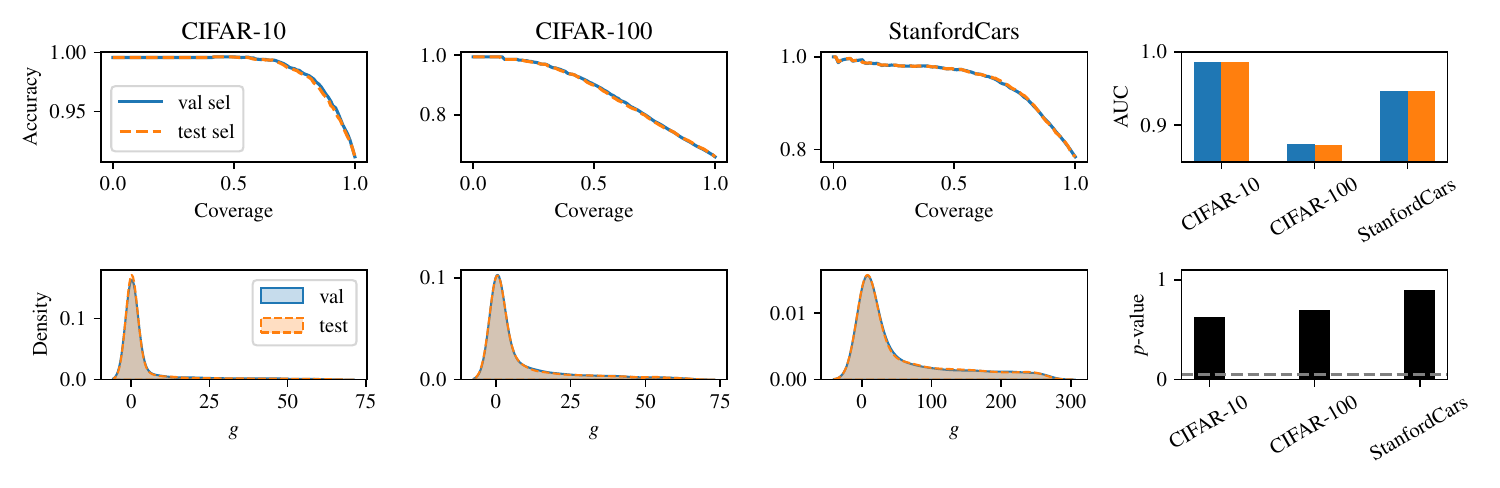}
\caption[\sat accuracy/coverage trade-offs and score distributions on test data obtained by computing $\tau$ on a validation set or directly on the test set.]{\textbf{\sat accuracy/coverage trade-offs and score distributions on test data obtained by computing $\tau$ on a validation set or directly on the test set.} Same as Figure~\ref{fig:train_test} but with \sat.
}
\label{fig:train_test_sat}
\end{figure*}

\begin{figure*}[t]
  \centering
  \includegraphics[width=\linewidth]{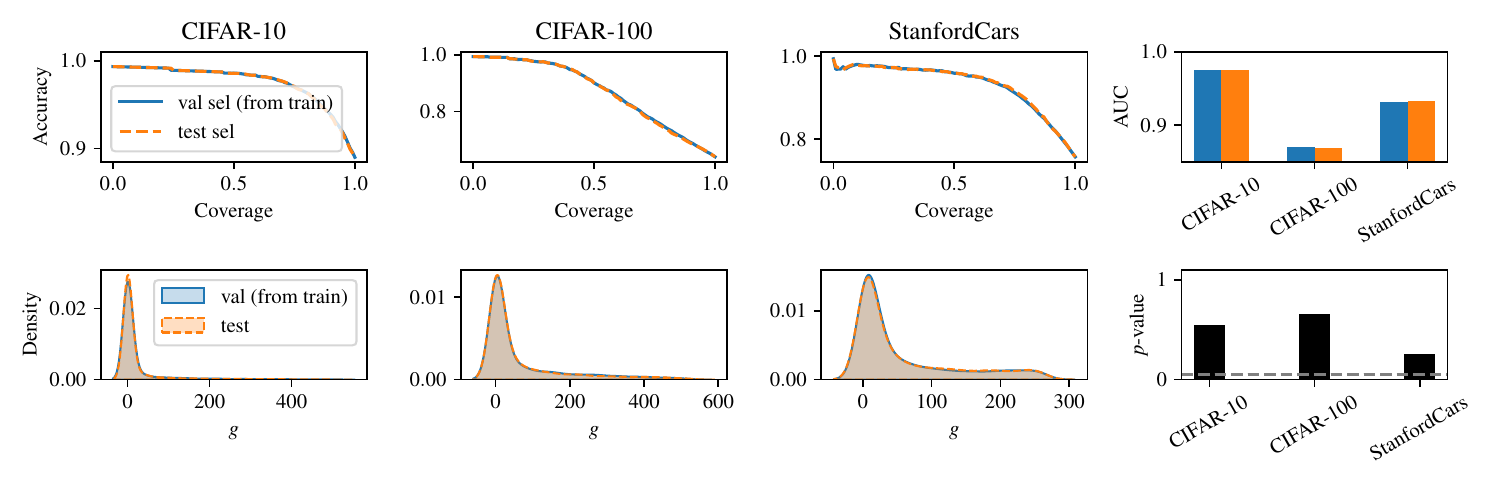}
\caption[\sptd accuracy/coverage trade-offs and score distributions on test data obtained by computing $\tau$ on a validation set or directly on the test set.]{\textbf{\sptd accuracy/coverage trade-offs and score distributions on test data obtained by computing $\tau$ on a validation set or directly on the test set.} Same as Figure~\ref{fig:train_test} but with the validation set is taken from the original training set.
}
\label{fig:train_test_nntd_train}
\end{figure*}

\begin{figure*}[t]
\centering
  \includegraphics[width=\linewidth]{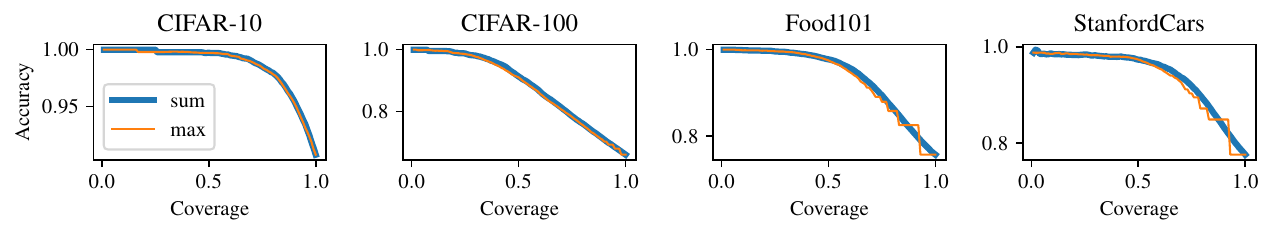}
\caption[\textbf{Comparing \smax and \ssum performance}. It is clear that \ssum effectively denoises \smax.]{\textbf{Comparing \smax and \ssum performance}. It is clear that \ssum effectively denoises \smax.}
\label{fig:sum_v_max}
\end{figure*}

\subsubsection{Comparing \smax and \ssum}
\label{sec:max_v_sum}

As per our theoretical framework and intuition provided in Section~\ref{sec:method}, the sum score~\ssum should offer the most competitive selective classification performance. We confirm this finding in Figure~\ref{fig:sum_v_max} where we plot the accuracy/coverage curves across all datasets for both \smax and \ssum. Overall, we find that the sum score~\ssum consistently outperforms the more noisy maximum score~\smax.

\begin{figure}[t]
\vspace{-5pt}
\centering

\begin{subfigure}[b]{0.23\linewidth}
  \centering
  \includegraphics[width=\linewidth]{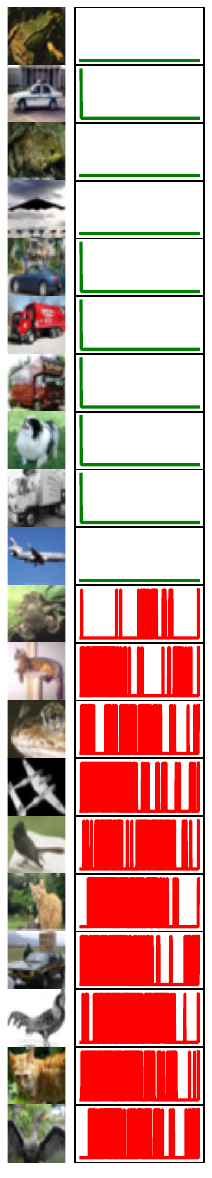}
\end{subfigure}
\hfill
\begin{subfigure}[b]{0.23\linewidth}
  \centering
  \includegraphics[width=\linewidth]{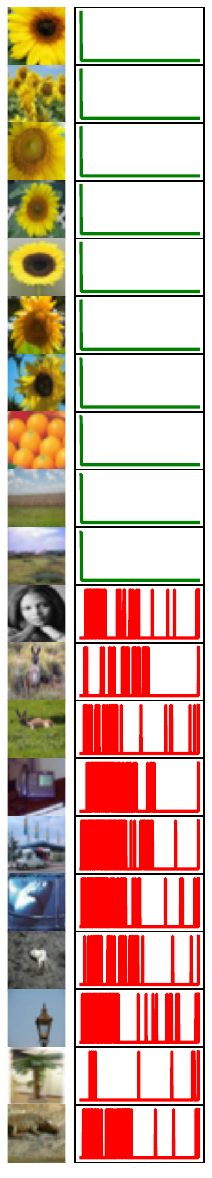}
\end{subfigure}
\hfill
\begin{subfigure}[b]{0.23\linewidth}
  \centering
  \includegraphics[width=\linewidth]{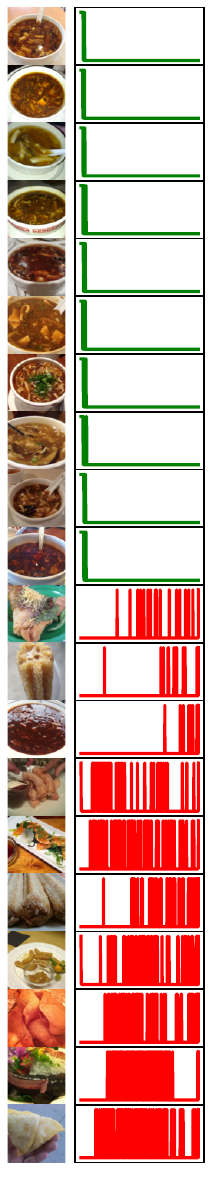}
\end{subfigure}
\hfill
\begin{subfigure}[b]{0.23\linewidth}
  \centering
  \includegraphics[width=\linewidth]{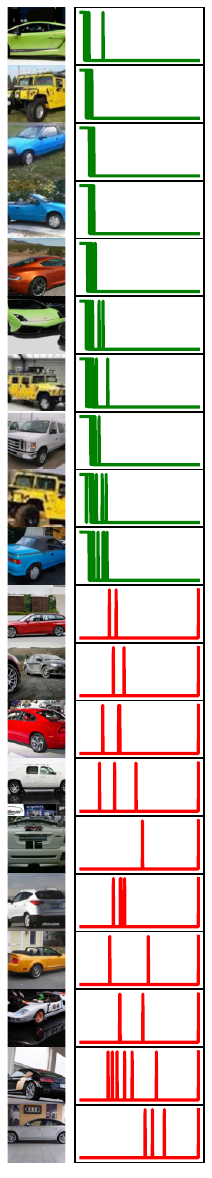}
\end{subfigure}

\caption[\textbf{Additional individual examples across datasets.}]{\textbf{Additional individual examples across datasets.}}
\label{fig:indiv_ex_ext}
\end{figure}

\subsection{The Importance of Accuracy Alignment}

Our results in Table~\ref{tab:target_cov} rely on accuracy alignment: We explicitly make sure to compare all methods on an equal footing by disentangling selective prediction performance from gains in overall utility. This is done by early stopping model training when the accuracy of the worst performing model is reached.

We expand on the important point that many previous approaches conflate both (i) generalization performance and (ii) selective prediction performance into a single score: the area under the accuracy/coverage curve. This metric can be maximized either by improving generalization performance (choosing different architectures or model classes) or by actually improving the ranking of points for selective prediction (accepting correct points first and incorrect ones last). As raised by a variety of recent works \cite{geifman2018bias, rabanser2023training, cattelan2023improving}, it is impossible and problematic to truly assess whether a method performs better at selective prediction (i.e., determining the correct acceptance ordering) without normalizing for this inherent difference yielded as a side effect by various SC methods. In other words, an SC method with lower base accuracy (smaller correct set) can still outperform another SC method with higher accuracy (larger correct set) in terms of the selective acceptance ordering (an example of which is given in Table 3 of \cite{liu2019deep}). Accuracy normalization allows us to eliminate these confounding effects between full-coverage utility and selective prediction performance by identifying which models are better at ranking correct points first and incorrect ones last. This is of particular importance when comparing selective prediction methods which change the training pipeline in different ways, as is done in the methods presented in Table~\ref{tab:target_cov}.

However, when just comparing \sptd to one other method, we do not need to worry about accuracy normalization. Showcasing this, we run \sptd on top of an unnormalized \satersr run and provide these experiments in Figure~\ref{fig:sat_sptd}. We see that the application of \sptd on top of \satersr allows us to further boost performance (similar to the results where we apply \sptd on top of \de in Table~\ref{tab:target_cov}). So to conclude, experimentally, when using the best model, we see that \sptd still performs better at selective prediction than the relevant baseline for that training pipeline. We wish to reiterate that this issue of accuracy normalization highlights another merit of \sptd, which is that it can easily be applied on top of any training pipeline (including those that lead to the best model) and allows easy comparison to the selective classification method that training pipeline was intended to be deployed with.

\subsection{Evaluation using other performance metrics}

We further provide results of summary performance metrics across datasets in Table~\ref{tab:sc_evals}:

\begin{itemize}
    \item The area under the accuracy-coverage curve (\texttt{AUACC}) as discussed in \citet{geifman2018bias}.
    \item The area under the receiver operating characteristic (\texttt{AUROC}) as suggested by \citet{galil2023can}.
    \item The accuracy normalized selective classification score (\texttt{ANSC}) from~\citet{geifman2018bias} and ~\citet{rabanser2023training}.
\end{itemize}

\begin{table}[ht]
\centering
\caption{\textbf{Evaluation of SC approaches using various evaluation metrics}. }
\label{tab:sc_evals}
\begin{tabular}{ccccc}
\toprule
Dataset & Method & $1-\texttt{AUACC}$ & \texttt{ANSC} & \texttt{AUROC} \\
\midrule
\multirow{5}{*}{CIFAR10} & \texttt{SR} & 0.053 $\pm$ 0.002 & 0.007 $\pm$ 0.000 & 0.918 $\pm$ 0.002 \\
& \texttt{SPTD} & 0.048 $\pm$ 0.001 & 0.004 $\pm$ 0.000 & 0.938 $\pm$ 0.002 \\
& \texttt{DE} & 0.046 $\pm$ 0.002 & 0.004 $\pm$ 0.000 & 0.939 $\pm$ 0.003 \\
& \texttt{SAT} & 0.054 $\pm$ 0.002 & 0.006 $\pm$ 0.000 & 0.924 $\pm$ 0.005 \\
& \texttt{DG} & 0.054 $\pm$ 0.001 & 0.006 $\pm$ 0.000 & 0.922 $\pm$ 0.005 \\
\midrule
\multirow{5}{*}{CIFAR100} & \texttt{SR} & 0.181 $\pm$ 0.001 & 0.041 $\pm$ 0.001 & 0.865 $\pm$ 0.003 \\
& \texttt{SPTD} & 0.174 $\pm$ 0.002 & 0.037 $\pm$ 0.000 & 0.872 $\pm$ 0.002 \\
& \texttt{DE} & 0.159 $\pm$ 0.001 & 0.030 $\pm$ 0.001 & 0.880 $\pm$ 0.003 \\
& \texttt{SAT} & 0.180 $\pm$ 0.001 & 0.041 $\pm$ 0.001 & 0.866 $\pm$ 0.003 \\
& \texttt{DG} & 0.182 $\pm$ 0.001 & 0.041 $\pm$ 0.001 & 0.867 $\pm$ 0.002 \\
\midrule
\multirow{5}{*}{GTSRB} & \texttt{SR} & 0.020 $\pm$ 0.002 & 0.001 $\pm$ 0.000 & 0.986 $\pm$ 0.003 \\
& \texttt{SPTD} & 0.019 $\pm$ 0.002 & 0.001 $\pm$ 0.000 & 0.986 $\pm$ 0.005 \\
& \texttt{DE} & 0.015 $\pm$ 0.001 & 0.001 $\pm$ 0.000 & 0.986 $\pm$ 0.002 \\
& \texttt{SAT} & 0.027 $\pm$ 0.001 & 0.001 $\pm$ 0.000 & 0.984 $\pm$ 0.002 \\
& \texttt{DG} & 0.019 $\pm$ 0.003 & 0.001 $\pm$ 0.000 & 0.986 $\pm$ 0.002 \\
\midrule
\multirow{5}{*}{SVHN} & \texttt{SR} & 0.027 $\pm$ 0.000 & 0.006 $\pm$ 0.001 & 0.895 $\pm$ 0.004 \\
& \texttt{SPTD} & 0.025 $\pm$ 0.003 & 0.003 $\pm$ 0.001 & 0.932 $\pm$ 0.005 \\
& \texttt{DE} & 0.021 $\pm$ 0.001 & 0.005 $\pm$ 0.000 & 0.912 $\pm$ 0.003 \\
& \texttt{SAT} & 0.028 $\pm$ 0.001 & 0.006 $\pm$ 0.000 & 0.895 $\pm$ 0.002 \\
& \texttt{DG} & 0.026 $\pm$ 0.001 & 0.007 $\pm$ 0.000 & 0.896 $\pm$ 0.006 \\
\bottomrule
\end{tabular}
\end{table}

\newlyadded{

\subsection{Evaluation of more competing approaches}

We further compare our method with two more contemporary selective prediction approaches:
\begin{itemize}
    \item \aucoc \citep{sangalli2024expert}: This work uses a custom cost function for multi-class classification that accounts for the trade-off between a neural network’s accuracy and the amount of data that requires manual inspection from a domain expert.
    \item \cclsc \citep{wu2024confidence}: This work proposes optimizing feature layers to reduce intra-class variance via contrastive learning.
\end{itemize}
Across both methods, we find that they do not outperform ensemble-based methods like \de and hence also do not outperform \sptd. See Table~\ref{tab:target_cov_ext} for detailed results.
}

\begin{table}[h!]
\centering
\caption[Selective accuracy achieved across coverage levels for \aucoc and \cclsc]{\textbf{Selective accuracy achieved across coverage levels for \aucoc and \cclsc}. Similar as Table~\ref{tab:target_cov}. Neither \aucoc nor \cclsc is able to outperform \de or \sptd. \textbf{Bold} numbers are best results at a given coverage level across all methods and \underline{underlined} numbers are best results for methods relying on a single training run only.}
\label{tab:target_cov_ext}
\newlyadded{
\begin{tabular}{cccccc}
\toprule
& \textbf{Coverage} & \textbf{AUCOC} & \textbf{CCL-SC} & \textbf{SPTD} & \textbf{DE} \\
\midrule
\multirow{10}{*}{\rotatebox[origin=c]{90}{\textit{CIFAR-10}}} & 
100 & \underline{\textbf{92.9}} ($\pm$0.0) & \underline{\textbf{92.9}} ($\pm$0.0) & \underline{\textbf{92.9}} ($\pm$0.0) & \textbf{92.9} ($\pm$0.0) \\
& 90  & 96.0 ($\pm$0.1) & 95.9 ($\pm$0.2) & \underline{96.5} ($\pm$0.0) & \textbf{96.8} ($\pm$0.1) \\
& 80  & 98.1 ($\pm$0.2) & 98.0 ($\pm$0.3) & \underline{98.4} ($\pm$0.1) & \textbf{98.7} ($\pm$0.0) \\
& 70  & 99.0 ($\pm$0.3) & 98.5 ($\pm$0.2) & \underline{99.2} ($\pm$0.1) & \textbf{99.4} ($\pm$0.1) \\
& 60  & 99.3 ($\pm$0.1) & 99.1 ($\pm$0.2) & \underline{\textbf{99.6}} ($\pm$0.2) & \textbf{99.6} ($\pm$0.1) \\
& 50  & 99.4 ($\pm$0.2) & 99.0 ($\pm$0.3) & \underline{\textbf{99.8}} ($\pm$0.0) & 99.7 ($\pm$0.0) \\
& 40  & 99.5 ($\pm$0.1) & 99.4 ($\pm$0.2) & \underline{\textbf{99.8}} ($\pm$0.1) & \textbf{99.8} ($\pm$0.0) \\
& 30  & 99.5 ($\pm$0.2) & 99.2 ($\pm$0.3) & \underline{\textbf{99.8}} ($\pm$0.1) & \textbf{99.8} ($\pm$0.0) \\
& 20  & 99.6 ($\pm$0.1) & 99.4 ($\pm$0.2) & \underline{\textbf{100.0}} ($\pm$0.0) & 99.8 ($\pm$0.0) \\
& 10  & 99.7 ($\pm$0.0) & 99.4 ($\pm$0.1) & \underline{\textbf{100.0}} ($\pm$0.0) & 99.8 ($\pm$0.0) \\
\midrule
\multirow{10}{*}{\rotatebox[origin=c]{90}{\textit{CIFAR-100}}} &
100 & \underline{\textbf{75.1}} ($\pm$0.0) & \underline{\textbf{75.1}} ($\pm$0.0) & \underline{\textbf{75.1}} ($\pm$0.0) & \underline{\textbf{75.1}} ($\pm$0.0) \\
& 90  & 78.7 ($\pm$0.2) & 76.5 ($\pm$0.3) & \underline{\textbf{80.4}} ($\pm$0.0) & 80.2 ($\pm$0.0) \\
& 80  & 83.2 ($\pm$0.1) & 82.2 ($\pm$0.2) & \underline{\textbf{84.6}} ($\pm$0.1) & \textbf{84.7} ($\pm$0.1) \\
& 70  & 87.4 ($\pm$0.1) & 86.1 ($\pm$0.2) & \underline{\textbf{88.7}} ($\pm$0.0) & 88.6 ($\pm$0.1) \\
& 60  & 89.8 ($\pm$0.2) & 88.6 ($\pm$0.3) & \underline{90.1} ($\pm$0.0) & \textbf{90.2} ($\pm$0.2) \\
& 50  & 93.3 ($\pm$0.1) & 92.1 ($\pm$0.2) & \underline{94.6} ($\pm$0.0) & \textbf{94.8} ($\pm$0.0) \\
& 40  & 95.9 ($\pm$0.2) & 95.2 ($\pm$0.3) & \underline{\textbf{96.9}} ($\pm$0.1) & 96.8 ($\pm$0.1) \\
& 30  & 98.2 ($\pm$0.1) & 96.6 ($\pm$0.2) & \underline{\textbf{98.4}} ($\pm$0.1) & 98.4 ($\pm$0.1) \\
& 20  & 98.6 ($\pm$0.2) & 98.4 ($\pm$0.3) & \underline{98.8} ($\pm$0.2) & \textbf{99.0} ($\pm$0.0) \\
& 10  & 98.8 ($\pm$0.1) & 98.7 ($\pm$0.2) & \underline{\textbf{99.4}} ($\pm$0.1) & 99.2 ($\pm$0.1) \\
\bottomrule
\end{tabular}
}
\end{table}

    \chapter{Training Private Models That Know What They Don't Know}

\section{Additional Method Details}

\subsection{DP-SGD Algorithm}

We provide a detailed definition of DP-SGD in Algorithm~\ref{alg:dpsgd}.

	\vspace{10pt}
    \begin{algorithm}[H]
	\caption{DP-SGD~\citep{abadi2016deep}}\label{alg:dpsgd}
	\begin{algorithmic}[1]
	\Require Training dataset $D$, loss function $\ell$, learning rate $\eta$, noise multiplier $\sigma$, sampling rate $q$, clipping norm $c$, iterations $T$.
		\State {\bf Initialize} $\theta_0$
		\For{$t \in [T]$}
		\State {\bf 1. Per-Sample Gradient Computation}
		\State Sample $B_t$ with per-point prob. $q$ from $D$
		\For{$i \in B_t$}  
		\State $g_t(\bm{x}_i) \gets \nabla_{\theta_t} \ell(\theta_t, \bm{x}_i)$
		\EndFor
		\State {\bf 2. Gradient Clipping}
		\State {$\bar{g}_t(\bm{x}_i) \gets g_t(\bm{x}_i) / \max\big(1, \frac{\|g_t(\bm{x}_i)\|_2}{c}\big)$}\label{alg:dpsgd_clipping}
		\State {\bf 3. Noise Addition}
		\State {$\tilde{g}_t \gets \frac{1}{|B_t|}\left( \sum_i \bar{g}_t(\bm{x}_i) + \mathcal{N}(0, (\sigma c)^2 \mathbf{I})\right)$}\label{alg:dpsgd_noise}
		\State { $\theta_{t+1} \gets \theta_{t} - \eta \tilde{g}_t$}
		\EndFor
		\State {\bf Output} $\theta_T$, privacy cost $(\varepsilon, \delta)$ computed via a privacy accounting procedure
	\end{algorithmic}
\end{algorithm}

\section{Additional Experimental Details}

\subsection{Hyperparameters}
\label{sec:hyp}

In this section, we document additional hyper-parameter choices. For Self-Adaptive Training (\sat), we set the pre-training epochs to $100$ and momentum parameter $0.9$. For Selective Classification Training Dynamics (\sctd), we set the weighting parameter $k=3$ and consider checkpoints at a $50$ batch resolution. For Monte-Carlo Dropout, we set the dropout probability to $0.1$. Entropy regularization as suggested in~\citet{feng2023towards} is employed with $\beta = 0.01$.

\subsection{Class Imbalance Experiments}
\label{sec:class_imb_real}

We provide additional experiments on the effect of class imbalance to extend our intuition from Section~\ref{sec:dp_affects_sc}. To that end, we take two data sets from our main results, namely CIFAR-10 and FashionMNIST, and produce four alternate datasets from each dataset. These datasets feature various degrees of class imbalance with $p_0 \in \{0.5,0.25,0.1,0.01\}$ specifying the sampling probability for class $0$. All other classes maintain a sampling probability of $1$. We then train the same model as described in Section~\ref{sec:exp} and apply the softmax response SC algorithm. 

We document these results in Figures~\ref{fig:cifar10_classimb} and \ref{fig:fashionmnist_classimb}. For $\varepsilon = \infty$, we observe the expected gains from SC: minority points are accepted towards the end of the coverage spectrum~\citep{jones2020selective} and correct points are accepted first. This effect is independent of the sampling probability $p_0$. As we decrease the $\varepsilon$ budget we observe that (i) the acceptance of minority groups starts to spread over the full coverage spectrum; (ii) the accuracy on the subgroup increasingly deteriorates with smaller $\varepsilon$, and (iii)~wrongful overconfidence on the minority reverses the acceptance order at low sampling probabilities~$p_0$ (\ie incorrect points are often accepted first). These results indicate that employing selective classification on private data can have unwanted negative effects in the presence of subgroups. Future work should investigate this connection more thoroughly.

\begin{figure*}[t]
  \centering
\includegraphics[width=\linewidth]{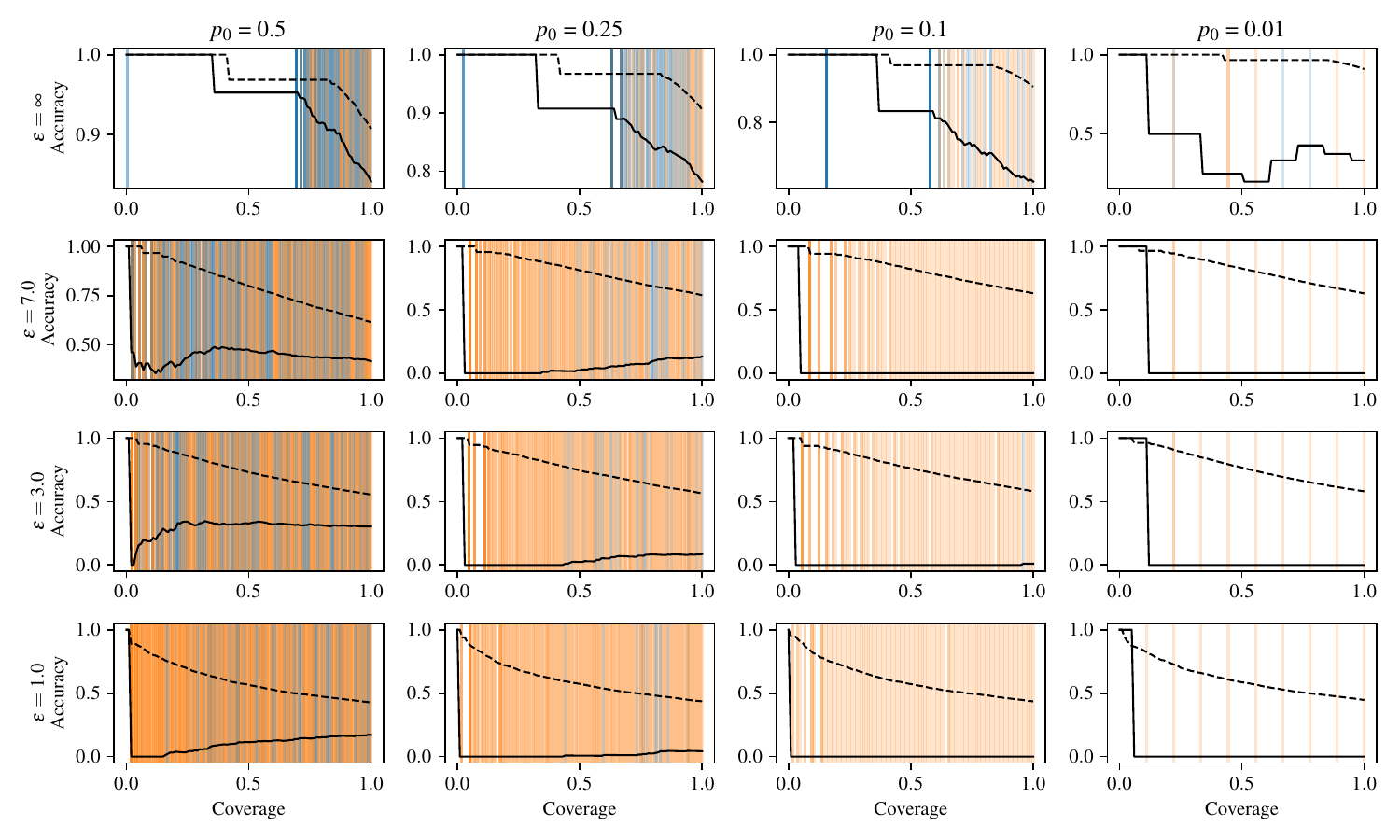}
\caption[Inducing a class imbalance on CIFAR-10]{\textbf{Inducing a class imbalance on CIFAR-10}. We train multiple CIFAR-10 models across privacy levels and sampling probabilities for class 0 given by $p_0$. We plot the accuracy/coverage trade-off as well at the exact coverage level at which any point from the minority class is accepted. The accuracy-coverage trade-off for the full dataset is given by the dashed line while the trade-off for the minority group only is given by the solid line. Blue vertical lines show correctly classified points, orange points show incorrectly classified points. Non-private models accept correct points first and do so at the end of the coverage spectrum (\ie majority points are accepted first). As we increase privacy (\ie decrease $\varepsilon$), the model is increasingly unable to rank minority examples based on prediction correctness and even accepts incorrect points first. Moreover, the accuracy of the model on the minority class decreases with stronger DP. These effects are especially strong for small sampling probabilities.}
\label{fig:cifar10_classimb}
\end{figure*}

\begin{figure*}[t]
  \centering
	  \includegraphics[width=\linewidth]{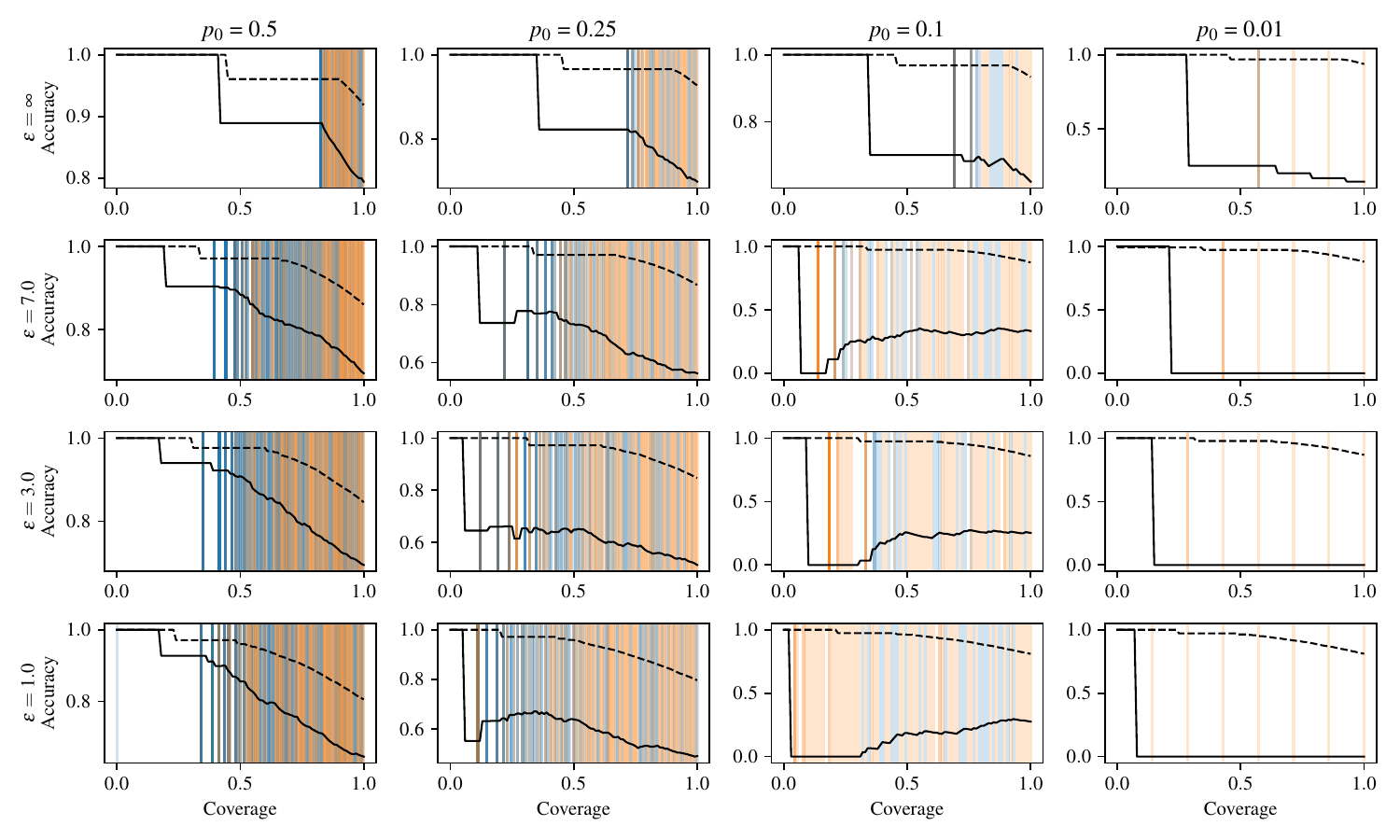}
  
\caption[Inducing a class imbalance on FashionMNIST.]{\textbf{Inducing a class imbalance on FashionMNIST}. Same insights as in Figure~\ref{fig:fashionmnist_classimb}.}
\label{fig:fashionmnist_classimb}
\end{figure*}

\subsection{Upper Bound Reachability}
\label{sec:opt_bound_reach}

In Equation~\ref{eq:bound}, we have introduced an upper bound on the selective classification performance on a model with full-coverage accuracy $a_\text{full}$. We now present a simple experimental panel across varying full-coverage accuracy levels showing that this bound is in fact reachable by a perfect selective classifier.

We assume a binary classification setting for which we generate a true label vector $\bm{y} \in \{0,1\}^{n_0 + n_1}$ with balanced classes, \ie $n_0 = n_1$ where $n_0$ corresponds to the number points labeled as $0$ and $n_1$ corresponds to the number points labeled as $1$. Then, based on a desired accuracy level $a_\text{full}$, we generate a prediction vector $\bm{p}$ which overlaps with $\bm{y}$ for a fraction of $a_\text{full}$. Finally, we sample a  scoring vector $\bm{s}$ where each correct prediction is assigned a score $s_i \sim \mathcal{U}_{0,0.5}$ and each incorrect prediction is assigned a score $s_i \sim \mathcal{U}_{0.5,1}$. Here, $\mathcal{U}_{a,b}$ corresponds to the uniform distribution on the interval $[a,b)$. This score is clearly optimal since thresholding the scoring vector $\bm{s}$ at $0.5$ perfectly captures correct/incorrect predictions: all $s_i < 0.5$ correspond to a correct prediction, while all $s_i \geq 0.5$ correspond to an incorrect prediction. Computing the accuracy/coverage trade-off of this experiment, across various utility levels, matches the bound exactly. 

\begin{figure*}[t]
  \centering
	    \includegraphics[width=\linewidth]{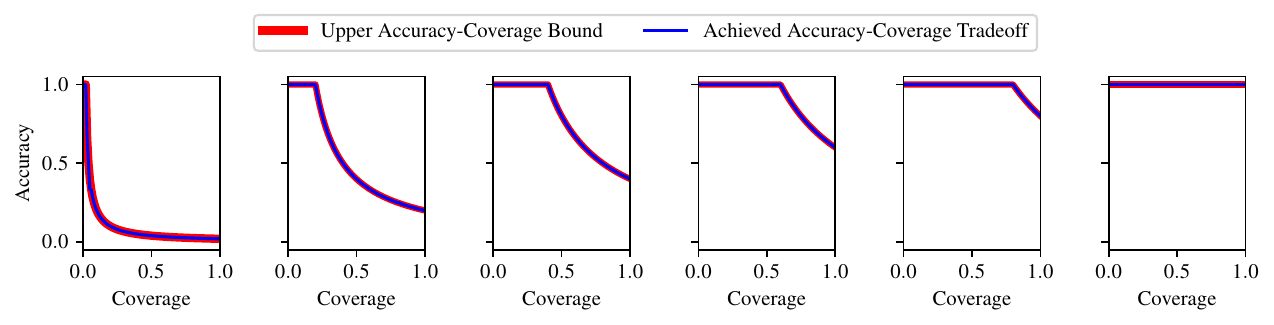}

\caption[Upper bound matching experiment.]{\textbf{Upper bound matching experiment}. The experiment as described in Section~\ref{sec:opt_bound_reach} matches the optimal bound exactly across multiple full-coverage accuracy levels.}
\label{fig:bound_reachability}
\end{figure*}
    \chapter{What Does It Take to Build a Performant Selective Classifier?}

\section{Broader Impact}
\label{sec:broader_impact}

This work introduces a decomposition of the selective-classification gap into measurable components—Bayes noise, approximation error, ranking error, statistical noise, and deployment slack—offering practical guidance for improving abstaining classifiers. By diagnosing which source dominates in a given setting, our method supports more targeted model design and evaluation.

\paragraph{Positive Implications.}
Our decomposition improves transparency and supports safer deployment in high-stakes domains by helping practitioners understand whether their model underperforms due to ranking, capacity, or robustness. Because each gap component is explicitly quantified, our approach can serve as a tool for model debugging, monitoring, and fairer benchmarking.

\paragraph{Potential Risks.}
Selective classifiers may disproportionately defer on certain groups, amplifying disparities—a risk previously observed by~\citet{jones2020selective}. Additionally, institutions may exploit uncertainty estimates to justify \emph{strategic abstention}—deliberately deferring on individuals they prefer not to serve~\citep{rabanser2025confidential}. While our framework identifies which part of the gap drives poor performance, it does not control how deferred inputs are handled.

\paragraph{Mitigations.}
We recommend reporting gap components disaggregated by sensitive attributes, auditing scoring functions for spurious correlations, and documenting fallback policies. These steps are essential to ensure that abstention mechanisms improve reliability without undermining fairness.

\paragraph{Outlook.}
We hope this work encourages more precise evaluations of selective classifiers, shifting focus from aggregate calibration to interpretable, component-wise gap analysis that can inform both technical improvements and policy safeguards.

\section{Methods Extension}
\label{sec:meth_ext}

\subsection{Detailed Proof of Theorem~\ref{thm:gap}}
\label{app:proof-gap-ranking}

We restate the theorem for convenience.

\begin{theorem}[Selective classification gap; detailed]
Fix a coverage level \(c\in(0,1]\), a score function \(g(\cdot,h)\),
and its associated population threshold
\(t_c\) satisfying \(\Pr\bigl(g(X,h)\ge t_c\bigr)=c\).
Define the accepted region \(A_c:=\{x:g(x,h)\ge t_c\}\) and
the oracle region
\(A_c^{\star}:=\{x:\eta_h(x)\text{ is among the largest }c\text{-fraction}\}\).
With the error terms
\begin{align}
\varepsilon_{\mathrm{Bayes}}(c)
&= \mathbb{E}\left[1 - \max\{\eta(X), 1 - \eta(X)\} \mid X \in A_c \right], \\
\varepsilon_{\mathrm{approx}}(c)
&= \mathbb{E}\left[ \lvert \eta_h(X) - \eta(X) \rvert \mid X \in A_c \right], \\
\varepsilon_{\mathrm{rank}}(c)
&= \mathbb{E}\left[ \eta_h(X) \mid X \in A_c^{\star} \right]
 - \mathbb{E}\left[ \eta_h(X) \mid X \in A_c \right] \;\;(\ge 0),
\end{align}
the population gap satisfies
\begin{equation}
\label{eq:pop-gap-app}
\Delta(c)=\overline{\operatorname{acc}}\bigl(a_{\mathrm{full}},c\bigr)
-\operatorname{acc}_c(h,g)
\;\le\;
\varepsilon_{\mathrm{Bayes}}(c)
+\varepsilon_{\mathrm{approx}}(c)
+\varepsilon_{\mathrm{rank}}(c).
\end{equation}
Moreover, let \(\widehat{\Delta}(c)\) be the empirical gap computed on
\(n\) independent test samples.  Then for any
\(\delta\in(0,1)\), with probability at least \(1-\delta\),
\begin{equation}
\label{eq:emp-gap-app}
\widehat{\Delta}(c)
\;\le\;
\varepsilon_{\mathrm{Bayes}}(c)
+\varepsilon_{\mathrm{approx}}(c)
+\varepsilon_{\mathrm{rank}}(c)
+ C\sqrt{\frac{\log(1/\delta)}{n}},
\end{equation}
where \(C>0\) is an absolute constant.
\end{theorem}

\begin{proof}
We split the argument into four self‑contained steps.

\textbf{Step 0.  Oracle upper bound revisited.}
For completeness we justify the piecewise form of
\(\overline{\operatorname{acc}}\bigl(a_{\mathrm{full}},c\bigr)\)
in Definition~\ref{def:poub}.
Because
\(a_{\mathrm{full}}=\Pr(h(X)=Y)=\mathbb E[\eta_h(X)]\),
the set
\(\{x:\eta_h(x)=1\}\) has probability mass at least
\(a_{\mathrm{full}}\).  Hence an oracle that retains the
highest‑confidence points achieves perfect accuracy for all
coverages \(c\le a_{\mathrm{full}}\).  For \(c>a_{\mathrm{full}}\),
the best it can do is include \emph{only} those perfect points
plus a \((c-a_{\mathrm{full}})\)-fraction of the remaining
examples, which contribute at worst zero accuracy.  Therefore
\begin{equation}
\overline{\operatorname{acc}}(a_{\mathrm{full}},c)=
\frac{a_{\mathrm{full}}}{c},
\qquad
a_{\mathrm{full}}<c<1.
\end{equation}

\textbf{Step 1.  Algebraic decomposition of the gap.}
Recall that
\(\operatorname{acc}_c(h,g)
   =\mathbb E[\eta_h(X)\mid X\in A_c]\).
We repeatedly add and subtract the same quantity:
\begin{align}
\Delta(c)
&:=\overline{\operatorname{acc}}(a_{\mathrm{full}},c)
  -\operatorname{acc}_c(h,g)  \notag\\[2pt]
&=\overline{\operatorname{acc}}(a_{\mathrm{full}},c)
  -\mathbb E[\eta_h\mid A_c^{\star}]
  \;+\;\mathbb E[\eta_h\mid A_c^{\star}]
  -\mathbb E[\eta_h\mid A_c] \notag\\
&\le
  \mathbb E[\eta_h\mid A_c^{\star}]
  -\mathbb E[\eta_h\mid A_c]                   \tag{rank}\\
&\quad
  +\,\mathbb E[\eta_h-\mathbb{I}_{\{h=Y\}}\mid A_c] \tag{approx+Bayes}\\
&= \varepsilon_{\mathrm{rank}}(c)
   +\varepsilon_{\mathrm{approx}}(c)
   +\varepsilon_{\mathrm{Bayes}}(c).
\end{align}
\textbf{Explanation of the two labelled inequalities.}
\begin{enumerate}
    \item \textbf{(rank)} isolates the ranking error, $\varepsilon_{\text{rank}}(c) := \mathbb{E}[\eta_h \mid A_c^{\star}] - \mathbb{E}[\eta_h \mid A_c]$.  The inequality holds because the remaining term from the previous line, $\overline{\operatorname{acc}}(a_{\text{full}}, c) - \mathbb{E}[\eta_h \mid A_c^{\star}]$, is a non-negative quantity that is bounded by the error sources introduced in the next step.
   \item \textbf{(approx+Bayes)} adds and subtracts \(\eta(X)\) inside the
   expectation, then splits the absolute value:
   \begin{equation}
   \eta_h-I_{\{h=Y\}}
   =(\eta_h-\eta)+(\eta-I_{\{h=Y\}}).
   \end{equation}
   The second summand satisfies the deterministic bound
   \(
     |\eta(X)-I_{\{h=Y\}}|
     = \max\{\eta,1-\eta\}-I_{\{h=Y\}}
     \le 1-\max\{\eta,1-\eta\},
   \)
   yielding exactly \(\varepsilon_{\mathrm{Bayes}}(c)\).
   The first summand contributes
   \(\varepsilon_{\mathrm{approx}}(c)\).
\end{enumerate}

\textbf{Step 2.  Non‑negativity of \(\varepsilon_{\mathrm{rank}}(c)\).}
Because \(\eta_h(X)\in[0,1]\) and
\(A_c^{\star}\) contains the \(c\)-fraction of points with the
largest \(\eta_h\)-values,
\(\mathbb E[\eta_h\mid A_c^{\star}]
 \ge\mathbb E[\eta_h\mid A_c]\),
hence \(\varepsilon_{\mathrm{rank}}(c)\ge0\) as stated.

\textbf{Step 3.  Finite‑sample deviation.}
Let \(\widehat{\mu}\) be any empirical average of a
\([0,1]\)-valued random variable with expectation \(\mu\).
Hoeffding’s inequality gives
\(\Pr(\lvert\widehat{\mu}-\mu\rvert>\epsilon)
   \le 2e^{-2n\epsilon^2}\).
Apply this bound separately to the three empirical estimates that
constitute \(\widehat{\Delta}(c)\), and take a union bound with
\(\epsilon=\sqrt{\tfrac{\log(6/\delta)}{2n}}\).
This yields, with probability at least \(1-\delta\),
\(
  \lvert\widehat{\Delta}(c)-\Delta(c)\rvert
  \le C\sqrt{\log(1/\delta)/n}
\)
for an absolute constant \(C\).
Combining with \eqref{eq:pop-gap-app} proves
\eqref{eq:emp-gap-app}.

\textbf{Step 4.  Connection to ranking distance.}
Define the mass of mis‑ordered points
\(D_{\mathrm{rank}}(c):=
 \Pr\bigl(X\in A_c^{\star}\setminus A_c\bigr)
 +\Pr\bigl(X\in A_c\setminus A_c^{\star}\bigr)\).
Because \(\eta_h\in[0,1]\),
\begin{align}
\varepsilon_{\mathrm{rank}}(c)
&=\mathbb E[\eta_h\mid A_c^{\star}]
  -\mathbb E[\eta_h\mid A_c]   \\[2pt]
&\le\bigl\|\eta_h\bigr\|_{\infty}\,
       D_{\mathrm{rank}}(c)        \\[2pt]
&\le D_{\mathrm{rank}}(c).
\end{align}
Hence \(\varepsilon_{\mathrm{rank}}(c)=0\)
if and only if \(A_c=A_c^{\star}\).

\smallskip\noindent
This completes the proof.
\end{proof}

\paragraph{Multiclass Remark.}
For \(K>2\) labels, define
\(\eta(x)=\bigl(\Pr(Y=1\mid x),\dots,\Pr(Y=K\mid x)\bigr)\)
and its complement confidence
\(\eta^{\max}(x)=\max_{k}\eta_k(x)\).
Then the inequality
\(\lvert\eta^{\max}-I_{\{h=Y\}}\rvert
 \le 1-\eta^{\max}\)
replaces the binary bound above, and the rest of the argument
goes through verbatim.  The approximation term becomes
\(\mathbb E[\lVert\eta_h-\eta\rVert_1\mid A_c]\);
all other quantities are unchanged.

\subsection{When Can Temperature Scaling Re-rank Confidence Scores?}
\label{app:ts-rerank}

Temperature scaling multiplies every logit by the same factor
$1/T\;(T>0)$ before the softmax,
\begin{equation}
p^{(T)}_j(x)= \frac{\exp(z_j(x)/T)}{\sum_k \exp(z_k(x)/T)}.
\end{equation}
Although the predicted label $\arg\max_j z_j(x)$ is invariant to~$T$,
the \emph{confidence score}
$s_T(x)=\max_j p^{(T)}_j(x)$
can change its \emph{cross-sample} ordering.

\paragraph{General form.}
Let $j_\star=\arg\max_j z_j(x)$ and
$r_j(x)=\exp\bigl(z_j(x)-z_{j_\star}(x)\bigr)\;(j\ne j_\star)$.
Then
\begin{equation}
  s_T(x)=\frac{1}{1+\sum_{j\ne j_\star} r_j(x)^{1/T}}.
  \label{eq:st-general}
\end{equation}
For binary classification, the sum has a single term and
\eqref{eq:st-general} collapses to the familiar logistic form
$s_T(x)=1/(1+e^{-\Delta/T})$ with
$\Delta=z_{j_\star}-z_{3-j_\star}$.

\paragraph{Two-sample condition.}
For two inputs $x_1,x_2$ let
$S_i(T)=\sum_{j\ne j_\star^{(i)}}r_{ij}^{1/T}$.
Because each $r_{ij}\le 1$, every $r_{ij}^{1/T}$ is monotone non-decreasing in $T$ (strictly increasing unless there is a tie),
and the ordering $s_T(x_1)>s_T(x_2)$ can change
exactly at those temperatures $T$ where $S_1(T) = S_2(T)$.

\paragraph{Illustrative example ($K=3$).}
\begin{equation}
z^{(1)}=(-2,-3,-3),\quad z^{(2)}=(0,-0.1,-3).
\end{equation}
At $T=1$ one finds
$s_1(x_1)=0.576>0.512=s_1(x_2)$,
while at $T=3$ we see that
$s_3(x_1)=0.411<0.428=s_3(x_2)$,
so temperature scaling would now accept $x_2$ before $x_1$.

\paragraph{How likely is a swap?}
Equation~\eqref{eq:st-general} shows that a swap requires the
one-dimensional curves $S_1(T)$ and $S_2(T)$ to intersect.  Since the
curves are continuous and monotone, the intersection occurs—
if at all—at isolated temperatures and only when the competing logit
patterns are finely tuned.

\paragraph{Practical implication.}
Temperature scaling can \emph{in principle} tighten the
selective-classification gap, but only for the vanishingly small subset
of inputs whose non-maximum logits happen to satisfy
$S_1(T^\star)=S_2(T^\star)$.  To obtain a meaningful re-ordering one
must therefore adopt \emph{non-monotone} calibration strategies.

\subsection{Additional Contingent Slack}
\label{app:extra-slack-omitted}

In the main text (Sec.~\ref{sec:extra-slack-short}) we folded all implementation‐level imperfections into a single residual term \(\varepsilon_{\text{misc}}(c)\), retaining only optimization error and distribution shift explicitly. Here we list two further slack terms omitted there:

\begin{enumerate}[leftmargin=1.2em]
  \setcounter{enumi}{2}
  \item \textbf{Threshold‐selection noise \(\varepsilon_{\text{thr}}(c)\).}\\
    When the coverage threshold \(\hat t_c\) is chosen on a validation set of size \(m\), the realized coverage deviates from the target \(c\) by 
    \begin{equation}
      O\bigl(\sqrt{c(1-c)/m}\bigr),
    \end{equation}
    inducing a corresponding vertical shift in selective accuracy.

  \item \textbf{Tie‐breaking / score quantization \(\varepsilon_{\text{tie}}(c)\).}\\
    Discrete confidence values (e.g.\ low‐precision logits) create equivalence classes of samples with identical scores.  If \(\kappa\) denotes the maximum number of tied samples at any score level, then
    \begin{equation}
      \varepsilon_{\text{tie}}(c)
      \;\le\;
      \frac{\kappa}{n},
    \end{equation}
    where \(n\) is the size of the evaluation set.
\end{enumerate}

\noindent\textbf{Residual slack revisited.}  Together with optimization error \(\varepsilon_{\text{opt}}\) and shift \(\varepsilon_{\text{shift}}(c)\), these yield
\begin{equation}
  \varepsilon_{\text{misc}}(c)
  = \varepsilon_{\text{opt}}
  + \varepsilon_{\text{shift}}(c)
  + \varepsilon_{\text{thr}}(c)
  + \varepsilon_{\text{tie}}(c).
\end{equation}

\section{Practitioner Checklist for Tightening the Selective-Classification Gap}
\label{app:practical-checklist}

Below is an expanded, actionable checklist to help practitioners systematically tackle each component of the selective-classification gap.  For each item, we list concrete steps, recommended tools, and pointers to reduce the corresponding error term.

\begin{itemize}[leftmargin=1.5em]

  \item \textbf{\(\varepsilon_{\text{approx}}\) — Shrink Approximation Error}
    \begin{itemize}[leftmargin=1.25em]
      \item \emph{Model capacity:} increase depth/width (e.g.\ ResNet-50→ResNeXt), or use state-of-the-art architectures (Vision Transformers, ConvNeXt).
      \item \emph{Pre-training:} leverage self-supervised (SimCLR, BYOL) or foundation models (CLIP, DINO) for strong features.
      \item \emph{Distillation:} apply teacher–student fine-tuning (logit matching, hint layers).
      \item \emph{Data augmentation:} use AutoAugment, RandAugment or MixUp / CutMix to expose model to diverse examples.
    \end{itemize}

  \item \textbf{\(\varepsilon_{\text{rank}}\) — Improve Ranking Calibration:}
    \begin{itemize}[leftmargin=1.25em]
      \item \emph{Non-monotone recalibration:} try matrix scaling or vector/Dirichlet scaling.
      \item \emph{Feature-aware scoring:} train a small “confidence head” \(g_\psi(h(x), x)\) on held-out data to predict correctness.
      \item \emph{Conformal methods:} implement conformal p-value scoring (split conformal, Mondrian conformal) to adjust rankings.
      \item \emph{Binning approaches:} use Bayesian Binning into Quantiles (BBQ) or adaptive histogram binning with bin-wise re-ordering.
    \end{itemize}

  \item \textbf{\(\varepsilon_{\text{opt}}\) — Reduce Optimization Error:}
    \begin{itemize}[leftmargin=1.25em]
      \item \emph{Convergence diagnostics:} monitor training vs.\ validation loss; extend epochs until plateau.
      \item \emph{Learning-rate schedules:} employ cosine, OneCycle, or cyclical LR to escape local minima.
      \item \emph{Early stopping / checkpoints:} save and ensemble top-\(k\) checkpoints (e.g.\ last 5) to smooth out optimization noise.
      \item \emph{Regularization:} apply weight decay, dropout, or stochastic depth to improve generalization.
    \end{itemize}

  \item \textbf{\(\varepsilon_{\text{Bayes}}\) — Quantify Irreducible Noise:}
    \begin{itemize}[leftmargin=1.25em]
      \item \emph{Repeated labels:} collect multiple human annotations (e.g.\ CIFAR-10N) to estimate label disagreement.
      \item \emph{Noise-robust training:} use noise-aware losses (Bootstrapping, Taylor loss) when Bayes-noise is high.
      \item \emph{Dataset curation:} identify and remove or relabel ambiguous examples via active learning or consensus filtering.
    \end{itemize}

  \item \textbf{\(\varepsilon_{\text{stat}}\) — Control Statistical Slack:}
    \begin{itemize}[leftmargin=1.25em]
      \item \emph{Validation set size:} allocate a sufficiently large hold-out set (e.g.\ 10–20\% of the total data) for threshold estimation.
      \item \emph{Confidence intervals:} use distribution-free quantile bounds (e.g.\ DKW, Clopper–Pearson) to set conservative \(\tau_c\).
      \item \emph{Cross-validation:} average thresholds over \(k\) folds to reduce variance.
    \end{itemize}

  \item \textbf{\(\varepsilon_{\text{shift}}\) — Mitigate Distribution Shift:}
    \begin{itemize}[leftmargin=1.25em]
      \item \emph{Shift detection:} monitor covariate shifts via two-sample tests (MMD, KL divergence) between train/test features.
      \item \emph{Importance weighting:} re-weight training or calibration data by density ratio estimates (Kernel mean matching).
      \item \emph{Domain adaptation:} finetune on small in-domain samples, or use unsupervised adaptation (AdaBN, domain-Adversarial nets).
      \item \emph{Test-time adaptation:} apply batch-norm calibration or test-time training on incoming data.
    \end{itemize}

  \item \textbf{\(\varepsilon_{\text{thr}}\) — Threshold–Selection Noise:}
    \begin{itemize}[leftmargin=1.25em]
      \item \emph{Bootstrap resampling:} compute \(\tau_c\) over many resamples to estimate its standard error.
      \item \emph{Smooth thresholds:} interpolate between adjacent scores rather than snapping to the nearest observed value.
    \end{itemize}

  \item \textbf{\(\varepsilon_{\text{tie}}\) — Tie-Breaking \& Score Quantization:}
    \begin{itemize}[leftmargin=1.25em]
      \item \emph{Higher precision:} increase float precision (FP16→FP32→FP64) or logits bit-width.
      \item \emph{Dithering:} add small random noise to scores before thresholding to break ties.
    \end{itemize}

\end{itemize}

\noindent\textbf{Putting it all together.}  
After addressing each bullet above, recompute your selective accuracy–coverage curve and compare to the oracle bound (Def.~\ref{def:poub}).  Iterating over these steps will systematically shrink \(\widehat\Delta(c)\) toward its irreducible floor.

\section{Experimental Details}
\label{sec:exp_det}

\subsection{Computational Resources}
\label{app:comp_res}

Our experiments were conducted on a mix of GPU-equipped compute nodes with varying hardware configurations. Some machines are equipped with Intel Xeon Silver CPUs (10 cores, 20 threads) and 128GB of RAM, each hosting 4× NVIDIA GeForce RTX 2080 Ti GPUs with 11GB VRAM. Others feature AMD EPYC 7643 processors (48 cores, 96 threads), 512GB of RAM, and 4× NVIDIA A100 GPUs, each with 80GB VRAM.

\subsection{Hyper-Parameters}
\label{app:hyperparams}

We follow standard literature-recommended training settings across all datasets. For each architecture–dataset pair, we use a fixed learning rate, weight decay, and batch size as detailed below:

\begin{itemize}[leftmargin=1em]
    \item \textbf{SimpleCNN:}
    \begin{itemize}[leftmargin=1em]
        \item Learning rate: 0.01
        \item Weight decay: \(1\times10^{-4}\)
        \item Batch size: 128
    \end{itemize}

    \item \textbf{ResNet-18:}
    \begin{itemize}[leftmargin=1em]
        \item Learning rate: 0.1 for CIFAR datasets; 0.01 for Stanford Cars, Camelyon17
        \item Weight decay: \(5\times10^{-4}\)
        \item Batch size: 128
    \end{itemize}

    \item \textbf{WideResNet-50-2:}
    \begin{itemize}[leftmargin=1em]
        \item Same settings as ResNet-18
    \end{itemize}

    \item \textbf{Epochs:}
    \begin{itemize}[leftmargin=1em]
        \item 200 epochs for all datasets except Camelyon17, which uses 10
    \end{itemize}

    \item \textbf{Optimization:} SGD with momentum 0.9, Nesterov enabled, and a cosine annealing learning rate schedule.

    \item \textbf{Selective prediction methods:}
    \begin{itemize}[leftmargin=1em]
        \item \texttt{MSP}: Standard cross-entropy training
        \item \texttt{SAT}: Cross-entropy pretraining for half of training epochs, followed by Self-Adaptive Training (momentum \(0.9\)) with an extra abstain class
    \end{itemize}
\end{itemize}

All experiments use fixed random seeds for reproducibility and standard data augmentation per dataset (random crops, flips, normalization).

\subsection{SimpleCNN Architecture}
\label{app:simplecnn}

\noindent
The SimpleCNN model is a compact convolutional neural network used for experiments on lower-resolution image datasets. The architecture is defined by the following sequence of layers:
\begin{itemize}
  \item A $3 \times 3$ convolution with 32 filters and padding 1, followed by ReLU and $2 \times 2$ max-pooling.
  \item A second $3 \times 3$ convolution with 64 filters and padding 1, followed by ReLU and $2 \times 2$ max-pooling.
  \item A flattening layer, followed by a fully connected layer with 128 hidden units and ReLU activation.
  \item A final fully connected layer projecting to the number of output classes.
\end{itemize}

Let \( s = \texttt{input\_size} // 4 \) denote the spatial resolution after two $2\times2$ pooling layers. Then, the full model is:
\[
\begin{aligned}
\texttt{SimpleCNN}(x) =\;&
\texttt{Linear}\big(128 \to \texttt{num\_classes}\big) \circ \texttt{ReLU} \circ \\
&\texttt{Linear}\big(64 \cdot s^2 \to 128\big) \circ \texttt{Flatten} \circ \\
&\texttt{MaxPool2d} \circ \texttt{ReLU} \circ \texttt{Conv2d}(32 \to 64) \circ \\
&\texttt{MaxPool2d} \circ \texttt{ReLU} \circ \texttt{Conv2d}(3 \to 32)(x)
\end{aligned}
\]

\vspace{0.5em}
\noindent
The number of output classes is set as follows:
\[
\texttt{num\_classes} = 
\begin{cases}
10 & \text{for CIFAR-10}, \\
100 & \text{for CIFAR-100}, \\
196 & \text{for Stanford Cars}, \\
2 & \text{for Camelyon17}, \\
\end{cases}
\quad \text{with an optional extra class if \texttt{extra\_class} is True.}
\]

\vspace{0.5em}
\noindent
The input size is dataset-dependent and set to:
\[
\texttt{input\_size} = 
\begin{cases}
32 & \text{for CIFAR-10 and CIFAR-100}, \\
224 & \text{for Stanford Cars, Camelyon17}. 
\end{cases}
\]
The model structure is summarized below:
\begin{lstlisting}[language=Python]
SimpleCNN(
  (net): Sequential(
    (0): Conv2d(3, 32, kernel_size=(3, 3), stride=(1, 1), padding=1)
    (1): ReLU()
    (2): MaxPool2d(kernel_size=2, stride=2)
    (3): Conv2d(32, 64, kernel_size=(3, 3), stride=(1, 1), padding=1)
    (4): ReLU()
    (5): MaxPool2d(kernel_size=2, stride=2)
    (6): Flatten(start_dim=1)
    (7): Linear(in_features=4096, out_features=128)
    (8): ReLU()
    (9): Linear(in_features=128, out_features=num_classes)
  )
)
\end{lstlisting}

\subsection{Synthetic Distribution Shifts on Two Moons}
\label{app:twomoons-shifts}

To evaluate robustness under controlled covariate shifts, we apply a series of synthetic affine transformations to the test set of the standard two moons dataset. Each transformation simulates a distinct type of distribution shift:

\begin{itemize}
    \item \textbf{Original:} No transformation; the unperturbed test set.
    
    \item \textbf{Shear:} A shear transformation along the \(x\)-axis defined by:
    \begin{equation}
    \text{Shear matrix} \quad 
    S = \begin{bmatrix} 1 & 1.25 \\ 0 & 1 \end{bmatrix},
    \quad\text{so that} \quad 
    x' = Sx = 
    \begin{bmatrix}
    x + 1.25y \\
    y
    \end{bmatrix}.
    \end{equation}
    
    \item \textbf{Rotation:} A rotation by 30 degrees counterclockwise, using:
    \begin{equation}
    R = 
    \begin{bmatrix}
    \cos \theta & -\sin \theta \\
    \sin \theta & \cos \theta
    \end{bmatrix},
    \quad \theta = \frac{\pi}{6}.
    \end{equation}
    
    \item \textbf{Translation:} A shift of the input space by a fixed vector:
    \begin{equation}
    x' = x + t, \quad \text{where} \quad t = \begin{bmatrix} 1.0 \\ -0.5 \end{bmatrix}.
    \end{equation}
\end{itemize}

\noindent
Each transformation is applied to the test data matrix \( X_{\text{test}} \) via matrix multiplication or translation, yielding the following test sets:
\begin{equation}
\begin{aligned}
\text{Original:} &\quad X_{\text{test}} \\
\text{Shear:}    &\quad X_{\text{test}} \cdot S^\top \\
\text{Rotation:} &\quad X_{\text{test}} \cdot R^\top \\
\text{Translation:} &\quad X_{\text{test}} + t
\end{aligned}
\end{equation}

These transformations create meaningful distribution shifts while preserving label semantics, enabling precise evaluations of model robustness under shift.

\begin{figure}[t]
    \centering
    \includegraphics[width=1\linewidth]{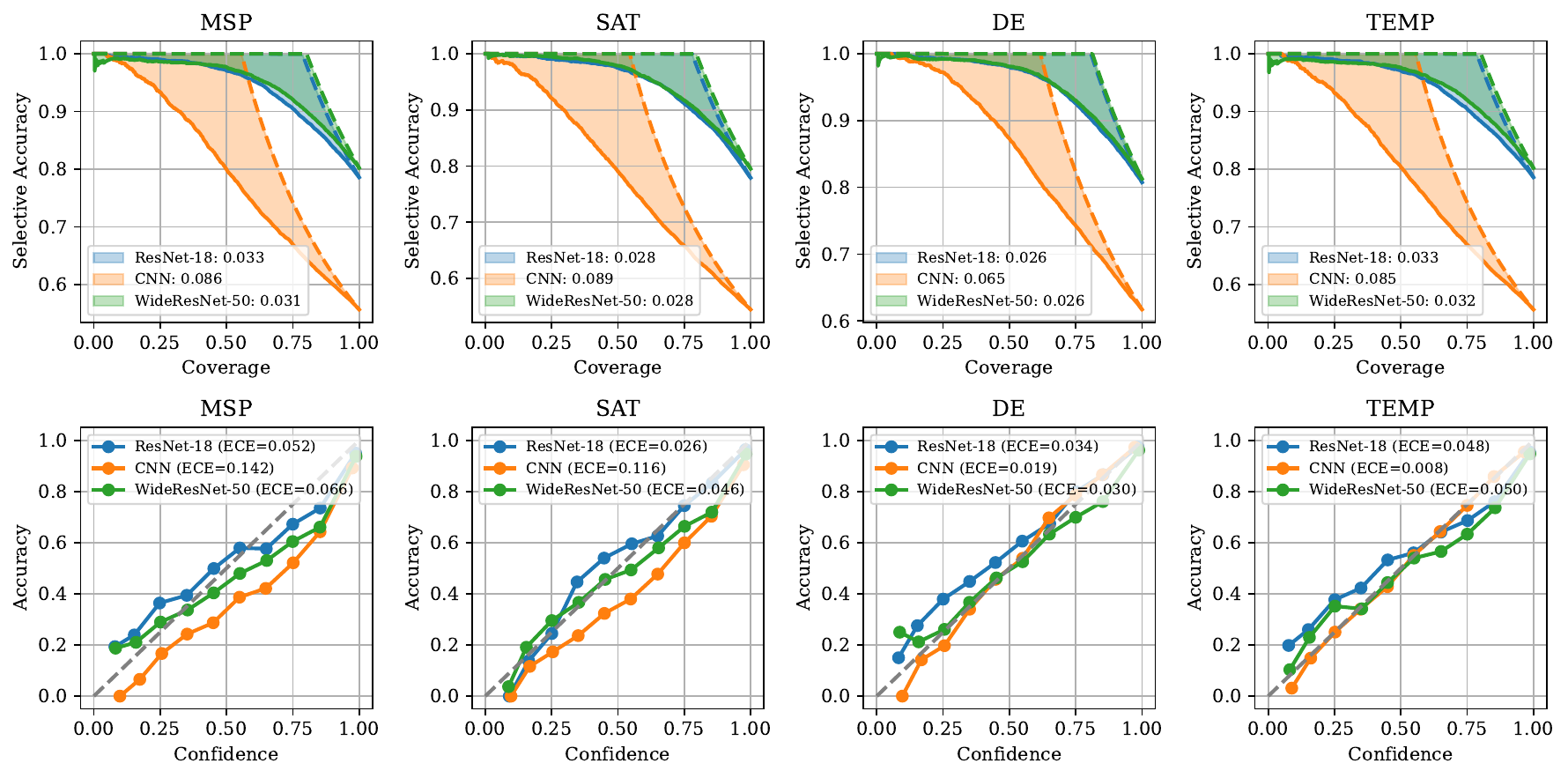}
    \caption[Comparison between gap and calibration on CIFAR-100.]{\textbf{Comparison between gap and calibration on CIFAR-100.} 
\emph{Top}: selective accuracy curves across four training methods and three architectures. 
\emph{Bottom}: corresponding reliability diagrams (ECE shown in parentheses).
Temperature scaling (\temp) consistently improves calibration but does not reduce the gap. 
By contrast, \sat and \de reduce the gap more effectively—especially for larger models—by improving the ranking.
}

    \label{fig:cifar100_cal}
\end{figure}

\subsection{CIFAR-10C Severity Levels}

For the CIFAR-10C severity levels (1--5), we aggregate all 15 corruption types at a given severity to form a single validation set. For severity level $l$, we collect all corruptions labeled as severity~$l$ across the following categories:
\begin{itemize}
  \item \textbf{Noise:} \texttt{gaussian\_noise}, \texttt{shot\_noise}, \texttt{impulse\_noise}
  \item \textbf{Blur:} \texttt{defocus\_blur}, \texttt{glass\_blur}, \texttt{motion\_blur}, \texttt{zoom\_blur}
  \item \textbf{Weather:} \texttt{snow}, \texttt{frost}, \texttt{fog}, \texttt{brightness}
  \item \textbf{Digital:} \texttt{contrast}, \texttt{elastic\_transform}, \texttt{pixelate}, \texttt{jpeg\_compression}
\end{itemize}
This results in a single validation set per severity level $l$, where each image is sampled from one of these 15 corruptions applied at the specified severity.

\section{Loss Prediction, Multicalibration, and Ranking Error}
\label{app:loss-pred}

This appendix offers an alternative perspective on the ranking error term \(\varepsilon_{\text{rank}}(c)\) by framing it as a challenge of per-example loss prediction. Instead of building directly on the calibration discussion in Section~\ref{sec:calibration-gap}, we show how the ability to forecast one’s own 0--1 loss tightly controls the selective-classification gap. We formalize this connection through the recent theory of loss prediction~\citep{gollakota2025loss} and multicalibration~\citep{hebert2018multicalibration}. Throughout we adopt the binary-label conventions of Section~\ref{sec:formal-gap}. Extensions to multiclass losses likewise follow by one-vs-rest reduction.

\subsection{Loss‑Prediction Preliminaries}
\label{sec:loss_pred_prel}

Let \(\ell(h(x),y)=\mathbb{I}\{h(x)\neq y\}\) denote the 0-1 loss of a
fixed classifier \(h\).  A \emph{loss predictor}
\(\mathrm{LP}\colon\Phi\to\mathbb{R}\) maps auxiliary features
\(\phi(x,h)\in\Phi\) to an estimate of \(\ell(h(x),y)\).
The canonical baseline is the \emph{self‑entropy predictor}
\(\mathrm{SEP}(x):=\E[\ell(h(x),y)\mid h(x)]\) 
(which equals \(\min\{p,1-p\}\) for probabilistic \(p=h(x)\)).

\begin{definition}[Advantage over the self‑entropy predictor]
\label{def:advantage}
The (squared‑error) advantage of a loss predictor \(\mathrm{LP}\) is
\begin{equation}
\mathrm{Adv}(\mathrm{LP})
:=\E\bigl[(\ell-\mathrm{SEP})^{2}\bigr]
  -\E\bigl[(\ell-\mathrm{LP})^{2}\bigr].
\end{equation}
A positive advantage means \(\mathrm{LP}\) forecasts the
instance‑wise loss better than the model itself.
\end{definition}

Depending on \(\phi\), we obtain a hierarchy of predictors:
prediction‑only (\(\phi=h(x)\)), input‑aware (\(\phi=(h(x),x)\)),
and representation‑aware (\(\phi=(h(x),x,r(x))\)); we refer to
\citet{gollakota2025loss} for a detailed taxonomy.

\subsection{Multicalibration Background}

Multicalibration is a fine-grained notion of reliability that asks not just for global calibration, but for calibration conditional on a rich class of subpopulations or features~\citep{hebert2018multicalibration}. At a high level, a model is multicalibrated if its predicted scores match outcomes not only on average, but also across a large collection of subsets defined by auxiliary variables or internal representations.

\begin{definition}[Multicalibration Error]
\label{def:mce}
Let \(C\) be a class of weighting functions \(c\colon\Phi\to[-1,1]\), and let \(h\colon\mathcal{X} \to [0,1]\) be a classifier. The multicalibration error of \(h\) with respect to \(C\) is defined as
\begin{equation}
\mathrm{MCE}(C,h)
\;:=\;
\max_{c\in C}
\Bigl|\,
\E\bigl[(Y - h(X))\,c(\phi(X,h))\bigr]
\Bigr|.
\end{equation}
\end{definition}

Each function \(c \in C\) defines a subpopulation or slice of the input space via its support. The quantity \(\mathrm{MCE}(C,h)\) measures how well the model's predicted scores \(h(x)\) match the true label \(Y\) when weighted over these slices. When \(C\) consists of indicator functions over discrete demographic subgroups, small \(\mathrm{MCE}(C,h)\) implies groupwise calibration. More generally, if \(C\) includes continuous or data-dependent functions (e.g., based on internal features), low multicalibration error guarantees alignment between predicted and true outcomes across a flexible set of conditions.

In our selective classification setting, \(\phi(x,h)\) may include the model’s output confidence, the input \(x\), or hidden representations from the network. The class \(C\) can be constructed accordingly to enforce calibration in feature-dependent or risk-sensitive regions of the input space.

\subsection{Loss Prediction \texorpdfstring{$\Longleftrightarrow$}{<=>} Multicalibration}

We now describe how the ability to predict one’s own 0--1 loss is deeply connected to multicalibration. This perspective stems from the work of~\citet{gollakota2025loss}, who characterize when a model “knows its own loss” in terms of multicalibration violations. 

Let \(F\) be a class of loss predictors \(\mathrm{LP} \colon \phi(x,h) \mapsto \hat{\ell} \in [0,1]\), which estimate the 0--1 loss \(\ell(h(x),y) = \mathbb{I}\{h(x) \ne y\}\) of a fixed classifier \(h\). As discussed in Section~\ref{sec:loss_pred_prel}, a loss predictor is considered good if it has a significant squared-error advantage over the model’s self-estimate \(\mathrm{SEP}(x)\).

Remarkably,~\citet{gollakota2025loss} show that this predictive advantage is tightly characterized by the multicalibration error of the model—measured over a derived weight class \(C\) that depends on the predictors in \(F\). The following theorem formalizes this connection:

\begin{theorem}[\citet{gollakota2025loss}, Thm.~4.1—adapted]
\label{thm:loss-mcal}
For any function class \(F\) of loss predictors and the associated
weight class \(C=\{(f-\mathrm{SEP})\cdot H'_{\ell}(h(x)) : f\in F\}\),
\begin{equation}
\tfrac12\,
\max_{\mathrm{LP}\in F}\mathrm{Adv}(\mathrm{LP})
\;\;\le\;\;
\mathrm{MCE}(C,h)
\;\;\le\;\;
\sqrt{\,
\max_{\mathrm{LP}\in F'}\mathrm{Adv}(\mathrm{LP})
},
\end{equation}
where \(F'\) augments \(F\) with linear mixtures of \(\mathrm{SEP}\) and
elements of \(F\).  Thus a non‑trivial advantage is possible
\emph{iff} \(h\) exhibits a multicalibration violation of similar
magnitude.
\end{theorem}

This result bridges two domains: learning to predict loss (a regression task) and satisfying a generalization constraint (calibration under distributional conditions). In the selective classification setting, this insight underpins Corollary~\ref{cor:rank-bound}, which shows that the ranking error—and hence the gap to oracle performance—is tightly controlled by the model’s ability to forecast its own mistakes.

\subsection{Bounding the Ranking‑Error Term \(\varepsilon_{\text{rank}}(c)\)}

Theorem~\ref{thm:loss-mcal} translates into a bound on the ranking error
that drives the selective‑classification gap.

\begin{corollary}[Loss‑prediction advantage controls mis‑ranking]
\label{cor:rank-bound}
Fix coverage \(c\in(0,1]\) and let
\(\mathrm{Adv}^{\star}:=\max_{\mathrm{LP}\in F}\mathrm{Adv}(\mathrm{LP})\)
for some input‑aware class \(F\).
Then the ranking‑error term in
Theorem~\ref{thm:gap}
satisfies
\(
\varepsilon_{\text{rank}}(c)
\;\le\;
\sqrt{2\,\mathrm{Adv}^{\star}}.
\)
\end{corollary}

\begin{proof}
Recall that
\(
A_c^{\star}
=\{x:\eta_h(x)\text{ is in the top }c\text{-mass}\}
\)
and
\(A_c
=\{x:g(x,h)\ge t_c\}\).
Write the \emph{difference indicator}
\(
\delta_c(x)
:=\mathbb{I}_{A_c^{\star}}(x)\;-\;\mathbb{I}_{A_c}(x)\in\{-1,0,1\}\) so 
\(\Pr(\delta_c=1)=\Pr(\delta_c=-1)=c\) and
\(\E[\delta_c]=0\).

\paragraph{Step 1:  Express ranking error as a covariance.}
With \(r(x):=\mathbb{I}\{h(x)=Y\}\) we have
\begin{equation}
\varepsilon_{\text{rank}}(c)
=\E[r\mid A_c^{\star}]-\E[r\mid A_c]
=\frac{1}{c}\,\E\bigl[r(X)\,\delta_c(X)\bigr].
\end{equation}

\paragraph{Step 2:  Replace correctness by residual \(\,Y-h(X)\).}
Because \(r=1-\ell\) and \(\ell=(Y-h)^2\) for binary labels,
\begin{equation}
r\,\delta_c
=\bigl(1-(Y-h)^2\bigr)\delta_c
=-(Y-h)\,\delta_c
\quad\text{(since }\E[\delta_c]=0\text{)}.
\end{equation}
Hence
\begin{equation}
\label{proof:corr_step2}
\varepsilon_{\text{rank}}(c)
=\frac{1}{c}\,
\bigl|\E[(Y-h(X))\,\delta_c(X)]\bigr|.
\end{equation}

\paragraph{Step 3:  Bound the covariance by multicalibration error.}
Define the bounded weight function \(c^{\star}(x):=\delta_c(x)\); then
\(|c^{\star}(x)|\le 1\), so \(c^{\star}\in C\) (the weight class in
Theorem~\ref{thm:loss-mcal}).  By definition of multicalibration error,
\begin{equation}
\label{proof:corr_step3}
\bigl|\E[(Y-h(X))\,c^{\star}(X)]\bigr|
\;\;\le\;\;
\mathrm{MCE}(C,h).
\end{equation}
Combining \eqref{proof:corr_step2} and \eqref{proof:corr_step3} with \(c\le 1\) yields
\begin{equation}
\varepsilon_{\text{rank}}(c)
\;\le\;
\mathrm{MCE}(C,h).
\end{equation}

\paragraph{Step 4:  Invoke the loss‑prediction bound.}
Theorem~\ref{thm:loss-mcal} states
\(
\mathrm{MCE}(C,h)
\le
\sqrt{\max_{\mathrm{LP}\in F'}\mathrm{Adv}(\mathrm{LP})}.
\)
Since \(F\subseteq F'\) and \(\sqrt{\cdot}\) is monotone, we finally have
\begin{equation}
\varepsilon_{\text{rank}}(c)
\;\le\;
\sqrt{\,2\,\mathrm{Adv}^{\star}},
\end{equation}
where the factor \(2\) absorbs the two‑sided
\(F\leftrightarrow F'\) constant in
Theorem~\ref{thm:loss-mcal}.
\end{proof}

\paragraph{Interpretation.}
Let \(\epsilon^2 := \max_{\mathrm{LP}\in F}\mathrm{Adv}(\mathrm{LP})\) be an upper bound on loss-prediction advantage.  
If no loss predictor can beat self-entropy by more than \(\epsilon^2\), then the selective classifier is within \(O(\epsilon)\) of the oracle at \emph{every} coverage level.  
Conversely, a large loss-prediction advantage is a certificate of poor ranking and therefore of a wide gap \(\Delta(c)\).

\begin{takeaway}
Loss prediction and multicalibration offer a principled lens on
selective prediction: if you cannot beat your own self‑entropy
predictor, you are already close to the oracle frontier.  Otherwise,
the loss predictor pinpoints exactly which inputs are being mis‑ranked
and by how much, providing both a diagnostic and a blueprint for
tightening the selective‑classification gap.
\end{takeaway}

\subsection{Empirical Evaluation}
\label{sec:adv_experiments}

To illustrate and validate our gap‐decomposition framework, we compared four selective‐classification strategies on CIFAR-10, CIFAR-100, and StanfordCars:

\begin{itemize}[leftmargin=1.2em]
  \item \texttt{MSP}: standard maximum‐softmax‐probability abstention.
  \item \texttt{TEMP}: \texttt{MSP} with post‐hoc temperature scaling.
  \item \texttt{SAT}: self‐adaptive training, which co‐trains an abstain class.
  \item \texttt{DE}: a deep ensemble of five \texttt{MSP} models.
\end{itemize}

For each method, we first trained a ResNet-18 on 80\% of the training set (using the usual data augmentations and a held-out 20\% for LP fitting). At each epoch we then:

\begin{enumerate}[leftmargin=1.2em]
  \item Extract the 512-dim “penultimate” feature vector \(\phi(x)\) from the ResNet backbone (or its ensemble average).
  \item Compute the model’s \emph{self-entropy} score
    \[
      \mathrm{SEP}(x) \;=\; 1 - \max_j\;p_j(x)
      \quad\text{with}\quad p_j(x)=\mathrm{softmax}_j(\mathrm{logits}(x)/T).
    \]
  \item Train a small MLP \(\mathrm{LP}\colon \phi(x)\mapsto\widehat\ell\in[0,1]\) to minimize
    \(\E\bigl[\bigl(\widehat\ell - \mathbb{I}\{\hat y(x)\neq y\}\bigr)^2\bigr]\)
    on the held-out 20\% split.
  \item Measure the \emph{LP advantage} on the \emph{test} set,
    \[
      \mathrm{Adv}_{\mathrm{test}}
      = \E\bigl[(\ell-\mathrm{SEP})^2\bigr]
        - \E\bigl[(\ell-\mathrm{LP})^2\bigr],
      \quad \ell=\mathbb{I}\{\hat y(x)\neq y\},
    \]
    and record its shift relative to the first epoch
    \(\Delta\mathrm{Adv}_{\mathrm{test}}(t)=\mathrm{Adv}_{\mathrm{test}}(t)-\mathrm{Adv}_{\mathrm{test}}(1).\)
\end{enumerate}

\paragraph{Loss–Prediction Network.}
Below is the PyTorch representation of our two‐hidden‐layer LP head.  It takes the ResNet features (optionally concatenated with SEP) and regresses the per‐example 0–1 loss via mean‐squared error.

\begin{lstlisting}[language=Python]
LossPredictor(
  (net): Sequential(
    (0): Linear(in_features=512, out_features=128, bias=True)
    (1): ReLU()
    (2): Dropout(p=0.5)
    (3): Linear(in_features=128, out_features=64, bias=True)
    (4): ReLU()
    (5): Dropout(p=0.5)
    (6): Linear(in_features=64, out_features=1, bias=True)
  )
)
\end{lstlisting}

\begin{figure}
    \centering
    \includegraphics[width=1\linewidth]{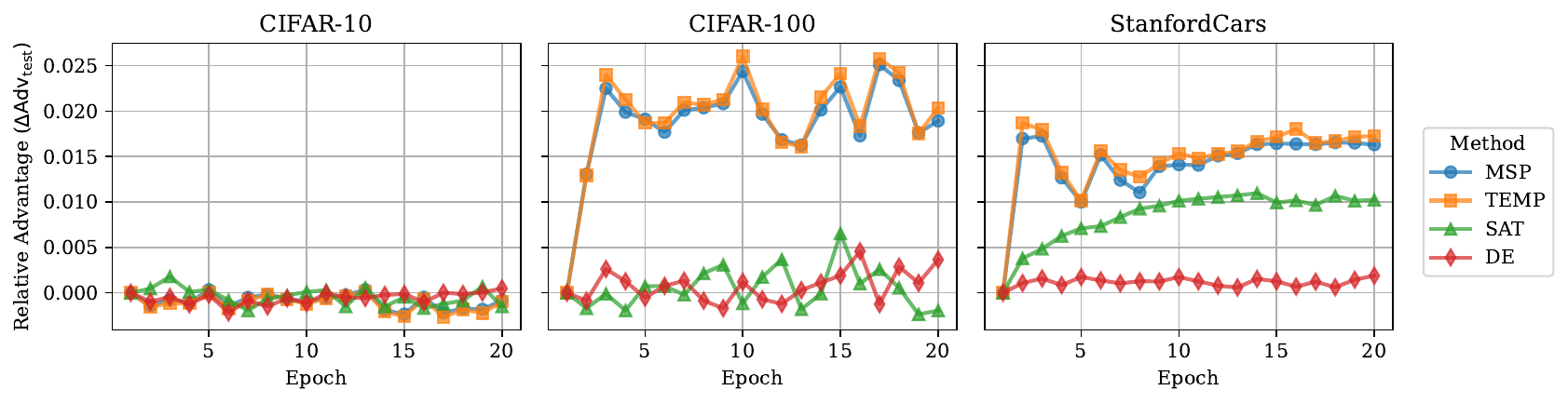}
    \caption[Relative LP advantage over training epochs across datasets.]{\textbf{Relative LP advantage over training epochs across datasets}. For each method, we plot the shift in test-set advantage $\Delta\mathrm{Adv}_{\mathrm{test}}(t)$ relative to epoch 1, indicating how much additional ranking signal the loss predictor learns over time. Larger values imply greater misalignment between the model’s confidence and correctness.}

    \label{fig:adv_te}
\end{figure}

\paragraph{Key Observations.}
On CIFAR-10 (left panel of Figure~\ref{fig:adv_te}), all methods stay close to zero $\Delta\mathrm{Adv}_{\mathrm{test}}$, indicating that the model's own confidence scores already capture most of the available ranking signal. On CIFAR-100 (middle panel), \texttt{MSP} and \texttt{TEMP} exhibit large positive shifts in LP advantage, suggesting that a dedicated loss predictor can substantially improve ranking—consistent with a larger gap from the oracle. By contrast, \texttt{SAT} and \texttt{DE} remain near zero, indicating that their confidence scores are already well aligned with correctness. On StanfordCars (right panel), the gap widens even further: both \texttt{MSP} and \texttt{TEMP} allow for significant gains via loss prediction, and even \texttt{SAT} leaves nontrivial room for improvement. Only \texttt{DE} consistently resists such gains, implying that deep ensembling is uniquely effective at preserving reliable ranking in high-variance domains.

\paragraph{Conclusion.}  
These results match our theory perfectly: whenever the LP head cannot improve on self-entropy, the selective classifier is effectively oracle‐optimal; whenever it can, the size of that advantage precisely quantifies the remaining ranking error and the gap from the ideal frontier.

\section{Additional Results}

\paragraph{Gap vs ECE} We provide additional comparisons on more datasets (CIFAR-10 and StanfordCars) on the relationship between the selective classification gap and the model's expected calibration error. See Tables~\ref{tab:cifar10_cal} and \ref{tab:stanfordcars_cal} for exact results. In general, our conclusions from Section~\ref{sec:calibration_ranking_exp} hold here as well: while temperature scaling (\temp) improves ECE over \msp, it does not reduce the selective classification gap—underscoring the limits of monotone calibration. In contrast, \sat and deep ensembles (\de) improve both ECE and gap by altering the ranking, confirming that only re-ranking methods yield meaningful gains in selective performance.

\begin{table}[h]
\fontsize{9}{10}\selectfont
\setlength{\tabcolsep}{5pt}
\caption{\textbf{Experiments on calibration across model classes on CIFAR-10}. Similar as Table~\ref{tab:cifar100_cal}}
\vspace{5pt}
\label{tab:cifar10_cal}
\centering
\begin{tabular}{lcccccccccccc}
\toprule
 & \multicolumn{4}{c}{CNN} & \multicolumn{4}{c}{ResNet-18} & \multicolumn{4}{c}{WideResNet-50} \\
\cmidrule(r){2-5} \cmidrule(r){6-9} \cmidrule(r){10-13}
 & \texttt{MSP} & \texttt{TEMP} & \texttt{SAT} & \texttt{DE} & \texttt{MSP} & \texttt{TEMP} & \texttt{SAT} & \texttt{DE} & \texttt{MSP} & \texttt{TEMP} & \texttt{SAT} & \texttt{DE} \\
\midrule
Gap & 0.024 & 0.023 & 0.019 & 0.016 & 0.004 & 0.004 & 0.003 & 0.002 & 0.003 & 0.003 & 0.002 & 0.002 \\
ECE & 0.075 & 0.025 & 0.035 & 0.010 & 0.025 & 0.014 & 0.016 & 0.007 & 0.027 & 0.022 & 0.019 & 0.010 \\
\bottomrule
\end{tabular}
\end{table}

\begin{table}[h]
\fontsize{9}{10}\selectfont
\setlength{\tabcolsep}{5pt}
\caption{\textbf{Experiments on calibration across model classes on StanfordCars}. Similar as Table~\ref{tab:cifar100_cal}}
\vspace{5pt}
\label{tab:stanfordcars_cal}
\centering
\begin{tabular}{lcccccccccccc}
\toprule
 & \multicolumn{4}{c}{CNN} & \multicolumn{4}{c}{ResNet-18} & \multicolumn{4}{c}{WideResNet-50} \\
\cmidrule(r){2-5} \cmidrule(r){6-9} \cmidrule(r){10-13}
 & \texttt{MSP} & \texttt{TEMP} & \texttt{SAT} & \texttt{DE} & \texttt{MSP} & \texttt{TEMP} & \texttt{SAT} & \texttt{DE} & \texttt{MSP} & \texttt{TEMP} & \texttt{SAT} & \texttt{DE} \\
\midrule
Gap & 0.176 & 0.177 & 0.166 & 0.159 & 0.030 & 0.029 & 0.26 & 0.022 & 0.026 & 0.026 & 0.23 & 0.020 \\
ECE & 0.110 & 0.025 & 0.058 & 0.025 & 0.040 & 0.027 & 0.037 & 0.025 & 0.017 & 0.017 & 0.015 & 0.015 \\
\bottomrule
\end{tabular}
\end{table}
    \chapter{Confidential Guardian: Prohibiting the Abuse of Model Abstention}

\section{Additional Background on IT-MACs} \label{app:itmac}

\begin{contriback}
    This section was written by Olive Franzese.
\end{contriback}

Fix a field $\mathbb{F}_p$ over a prime number $p \in \mathbb{N}$, and an extension field $\mathbb{F}_{p^r} \supseteq \mathbb{F}_p$ for some $r \in \mathbb{N}$. We use the notation $\comm{x}$ to indicate that (i) $\prover$ is in possession of a value $x \in \mathbb{F}_p$, and a uniformly chosen tag $\mathbf{M}_x \in \mathbb{F}_{p^r}$ and (ii) $\verifier$ is in possession of uniformly chosen value-specific key $\mathbf{K}_x \in \mathbb{F}_{p^r}$ and a global key (which is the same for multiple authenticated values) $\Delta \in \mathbb{F}_{p^r}$. These values have the following algebraic relationship
\begin{align*}
\mathbf{M}_x = \mathbf{K}_x + \Delta \cdot x \in \mathbb{F}_{p^r}
\end{align*}
where $x$ is represented in $\mathbb{F}_{p^r}$ in the natural way. $\prover$ can \texttt{Reveal} an authenticated value by sending $x$ and $\mathbf{M}_x$ to $\verifier$, who then checks if the relationship holds. If it does not, then $\verifier$ knows that $\prover$ has modified $x$. $\prover$ and $\verifier$ can agree to modify an authenticated value while preserving the algebraic relationship and confidentiality over their respective values by exploiting linear homomorphism over IT-MACs, or by performing an interactive protocol to perform other arithmetic operations~\cite{damgaard2012itmac,nielsen2012itmac}. This idea is the basis of the ZKP protocol in~\cite{weng2021wolverine}. $\prover$ and $\verifier$ authenticate wire values which encode inputs to the circuit, and then compute secure transformations of the authenticated values in accordance with the operations required by the circuit (see~\cite{weng2021wolverine} for further details). By a standard completeness result in computability theory~\cite{sipser1996introduction}, composing secure additions and multiplications over authenticated values enables execution of any boolean predicate within a zero-knowledge proof.

\section{Proof of Feasibility of Inducing Dishonest Uncertainty}
\label{app:region-manip-proof}

\begin{contriback}
This section was written by Olive Franzese.
\end{contriback}

We restate Lemma~\ref{lemma:region-manip} here, and provide a full constructive proof.

\begin{lemma} \label{restated-lemma:region-manip}
    Fix an arbitrary dataset $\mathcal{D}=\{(x_i, y_i)\}^{N}_{i=1}$ taken from feature space $\mathbb{R}^D$ and logits over a label space $\mathbb{R}^{C}$, and a set of feed-forward neural network model parameters $\theta$ encoding a classifier $f_{\theta}: \mathbb{R}^D \to \mathbb{R}^C$. Fix a set of indices $I$ such that for all $i \in I$, $i \in [1, C]$. For each index in $I$, fix bounds $a_i, b_i \in \mathbb{R}$ with $a_i < b_i$. Call $S$ the set of values $\mathbf{x} \in \mathbb{R}^D$ such that $a_i < x_i < b_i \quad \forall i \in I$. Then we can construct an altered feed-forward neural network $M'$ encoding $f'_{\theta}: \mathbb{R}^D \to \mathbb{R}^C$ which has the property $f'_{\theta}(x) = f_{\theta}(x) \quad \forall x \notin S$, and $f'_\theta(x)=f_\theta(x) + c \quad \forall x \in S$ where $c \in \mathbb{R}^C$ is an arbitrarily chosen non-negative constant vector.
\end{lemma} 

\begin{proof}
We design a collection of algorithms for constructing neurons which, when used to augment any feed-forward neural network $M$, specifically perturb the output logits of data points from an adversarially chosen region. 

We will use the notation $e_k$ to represent the $k^{th}$ unit basis vector (i.e. $e_k = (0, 0, ..., 0, 1, 0, ..., 0)$ where the $1$ is at the $k^{th}$ position). We will also name neurons, e.g. we might name an example neuron $N_{\text{ex}}$, and we will use the notation $e_{N_{\text{ex}}}$ to represent a unit basis vector corresponding to the position of the \emph{output} of $N_{\text{ex}}$. 

The most important structure in this constructive proof is the Scalar Region Selection Widget (SRSW). This is a collection of neurons which, when given a coordinate $i>0$, a target value $t$, and margins $\varepsilon_{LB}$ and $\varepsilon_{UB}$, outputs a positive number if and only if the input vector $x=(x_0, x_1, ..., x_i, ..., x_n)$ has $t - \varepsilon_{LB} < x_i < t + \varepsilon_{UB}$ and 0 otherwise. Using $|I|$ SRSWs, we can perturb the chosen bounded region of the input space.

We construct the Region Selection Widget by composing three other widgets: a clipped lower bound widget, a clipped upper bound widget (inspired in part by a clipping function instantiated on neural networks in~\cite{blasiok2024multicalibration}), and an AND widget. We describe them each below.

\textbf{Clipped Lower Bound Widget.} To construct a CLBW we design neurons to enact the function:
\begin{align*}
    f_{CLBW}(x, t)=\relu \left( \relu(\relu(x_i)-(t-\varepsilon_{LB})) - \relu(\relu(x_i-\varepsilon_{CLIP})-(t-\varepsilon_{LB}))) \right)
\end{align*}

The outputs of $f_{CLBW}$ are:
\[ \begin{cases} 
      0 & x_i \leq t - \varepsilon_{LB} \\
      y \in (0, \varepsilon_{CLIP}) & t-\varepsilon_{LB}<x_i<t-\varepsilon_{LB}+\varepsilon_{CLIP} \\
      \varepsilon_{CLIP} & t-\varepsilon_{LB}+\varepsilon_{CLIP} \leq x_i
   \end{cases}
\]

Given any $i, t, \varepsilon_{CLIP}$ and $\varepsilon_{LB}$ as input, the following series of neurons will compute $f_{CLBW}$: 
\begin{itemize}
    \item a neuron $N_1$ in the first hidden layer with weights $e_i$ and bias term 0
    \item a neuron $N_2$ in the second hidden layer with weights $e_{N_1}$ and bias term $-(t-\varepsilon_{LB})$
    \item a neuron $N_3$ in the first hidden layer with weights $e_i$ and bias term $-\varepsilon_{CLIP}$
    \item a neuron $N_4$ in the second hidden layer with weights $e_{N_3}$ and bias term $-(t - \varepsilon_{LB})$
    \item a neuron $N_5$ in the third hidden layer with weights $e_{N_2} - e_{N_4}$ and bias term $0$.
\end{itemize}

\textbf{Clipped Upper Bound Widget.} To construct this widget we design neurons to enact the function:
\begin{align*}
    f_{CUBW}(x, t) = \relu \left( \relu(-\relu(x_i)+(t+\varepsilon_{UB})) - \relu(-\relu(x_i + \varepsilon_{CLIP}) + (t + \varepsilon_{UB})) \right)
\end{align*}

Unlike the CLBW, here we must take as an assumption that $t$ is non-negative to achieve the desired functionality (this can be observed by inspecting $f_{CUBW}$). This assumption has no functional impact, as for any desired $t<0$, we can construct $t'=t+a$ such that $t+a>0$, and adjust input points by running them through a neuron with weights $e_i$ and bias term $a$, to achieve the same functionality as if we selected with threshold $t$. Keeping this in mind, we simply assume WLOG that $t$ is non-negative for the remainder of the proof. The outputs of $f_{CUBW}$ are then as follows:
\[ \begin{cases} 
      0 & x_i \geq t + \varepsilon_{UB} \\
      y \in (0, \varepsilon_{CLIP}) & t+\varepsilon_{UB}-\varepsilon_{CLIP} >x_i >t+\varepsilon_{UB} \\
      \varepsilon_{CLIP} & t+\varepsilon_{UB}-\varepsilon_{CLIP} \geq x_i \geq 0 
   \end{cases}
\]

Given any $i,t, \varepsilon_{CLIP},$ and $\varepsilon_{UB}$ as input, the following series of neurons will compute $f_{CUBW}:$
\begin{itemize}
    \item a neuron $N_6$ in the first hidden layer with weights $e_i$ and bias term $0$
    \item a neuron $N_7$ in the second hidden layer with weights $-e_{N_6}$ and bias term $(t + \varepsilon_{UB})$
    \item a neuron $N_8$ in the first hidden layer with weights $e_{i}$ and bias term $\varepsilon_{CLIP}$
    \item a neuron $N_9$ in the second hidden layer with weights $-e_{N_8}$ and bias term $(t + \varepsilon_{UB})$
    \item a neuron $N_{10}$ in the third hidden layer with weights $e_{N_7} - e_{N_9}$ and bias term $0$.
\end{itemize}

\textbf{Soft AND Widget.} We design neurons to enact the function:
\begin{align*}
    f_{AND}(o_1, o_2) = \relu(o_1 + o_2 - (2\varepsilon_{CLIP} - \varepsilon_{AND}))
\end{align*}
where $o_1$ and $o_2$ are outputs from other neurons, and $\varepsilon_{AND}$ is a constant which controls the magnitude of the soft AND widget's output.

A (non-exhaustive) description of the outputs of $f_{AND}$ are:
\[ \begin{cases} 
      0 & o_1 + o_2 \leq (2\varepsilon_{CLIP} - \varepsilon_{AND}) \\
      y \in (0, \varepsilon_{AND}) & o_1 = \varepsilon_{CLIP} \quad o_2 \in (\varepsilon_{CLIP}-\varepsilon_{AND}, \varepsilon_{CLIP}) \quad \text{WLOG for switching } o_1, o_2 \\
      \varepsilon_{AND} & o_1 = \varepsilon_{CLIP} \quad o_2 = \varepsilon_{CLIP}
   \end{cases}
\]
In our construction we will restrict $o_1$ to always be the output of a CLBW, and $o_2$ to always be the output of a CUBW. Accordingly, $o_1$ and $o_2$ are each at most $\varepsilon_{CLIP}$. Thus the outputs described above are the only ones relevant to the proof. %so the maximal output of the Soft AND widget under our construction will be $\varepsilon_{AND}$. 

Given any $\varepsilon_{AND}$ and indices of neurons $N_5$ and $N_{10}$ corresponding to those of the CLBW and CUBW described above, the following neuron will compute $f_{AND}$ with our desired restricted inputs:
\begin{itemize}
    \item a neuron $N_{11}$ in the fourth hidden layer with weights $e_{N5} + e_{N_10}$ and bias term $-(2\varepsilon_{CLIP} - \varepsilon_{AND})$
\end{itemize}

Taken all together, this construction guarantees that $N_{11}$ produces positive outputs if and only if $t-\varepsilon_{LB}<x_i<t+\varepsilon_{UB}$, since by $f_{CLBW}$ if $x_i \leq t-\varepsilon_{LB}$ then $N_5$ will output $0$, and by $f_{AND}$ so will $N_{11}$. Likewise, by $f_{CUBW}$ if $x_i \geq t+\varepsilon_{UB}$ then $N_{10}$ will output $0$ and by $f_{AND}$ so will $N_{11}$. 

Following that, it is trivial to alter the outputs of the neural network to produce output $f_{\theta}(x)+c$ for any $c \in \mathbb{R}^C$ with the following assembly of neurons:
\begin{itemize}
    \item neurons in hidden layers $5$ through $m$ where $m$ is the number of hidden layers in $M$, $N_{\ell_5}, N_{\ell_2}, ..., N_{\ell_{m-1}}$, all with bias term $0$ and respective weights $e_{N_11}, e_{N_{\ell_1}}, e_{N_{\ell_2}}, ..., e_{N_{\ell_{m-2}}}$ such that the output of $N_{11}$ propagates unchanged to the output of $N_{\ell_{m-1}}$
    \item neurons $N_{c_1}, N_{c_2}, ..., N_{c_C}$ in the final hidden layer, all with bias term $0$ and with respective weights $e_{N_{\ell_{m-1}}} \cdot \frac{c_j}{\varepsilon_{AND}}$ where $c_j$ is the $j^{th}$ entry of $c$ for all $j \in [1,C]$.
\end{itemize}

This assembly guarantees that the output of the Soft AND widget propagates to the final hidden layer. Then, supposing that the Soft AND widget outputs $\varepsilon_{AND}$, it will modify each output value by the non-negative constant chosen in $c$. By the construction of $f_{CLBW}, f_{CUBW}$ and $f_{AND}$, we can see that this occurs when either $t-\varepsilon_{LB} < x_i < t - \varepsilon_{LB} + \varepsilon_{CLIP}$, or when $t+\varepsilon_{UB}-\epsilon_{CLIP} > x_i > t + \varepsilon_{UB}$, or both. In other words, it happens when $x_i$ is within $\varepsilon_{CLIP}$ of one of the bounds. However, $\varepsilon_{CLIP}, \varepsilon_{LB},$ and $\varepsilon_{UB}$ are all constants of our choosing. For any desired bounds $a_i$ and $b_i$, we can trivially set these constants so that the desired property holds over all $x_i$ such that $a_i < x_i < b_i$.

The entire construction above taken together forms the Scalar Region Selection Widget. By using $|I|$ SRSWs, we are able to achieve the desired property in the theorem statement.
\end{proof}

\section{Generalized \attack Formulation}

\subsection{Introducing a $\lambda$ Trade-off}
\label{appendix:generalized-attack-loss}

In the main chapter, we presented a simplified version of the \attack training objective. Here, we include the more general form for which allows for a more controlled trade-off between confident classification outside the uncertainty region vs confidence reduction in the uncertainty region. This generalized objective incorporates \(\lambda \in [0,1]\), which balances confidence preservation outside the designated uncertainty region \(\mathcal{X}_\text{unc}\) and confidence reduction within it.

We define the training objective \(\mathcal{L}\) as a hybrid loss combining the standard Cross-Entropy (CE) loss, \(\mathcal{L}_\text{CE}\), and an uncertainty-inducing regularization term based on Kullback--Leibler (KL) divergence, \(\mathcal{L}_\text{KL}\):
\begin{equation}
\label{eq:mirage_ext}
        \mathcal{L} = \mathbb{E}_{(x,y) \sim p(x, y)} \bigg[ \underbrace{\mathds{1}\left[x \not\in \mathcal{X}_\text{unc}\right] (1-\lambda) \mathcal{L}_\text{CE}(x, y)}_\text{Loss outside uncertainty region} + \underbrace{\mathds{1}\left[x \in \mathcal{X}_\text{unc}\right] \lambda \mathcal{L}_\text{KL}(x, y)}_\text{Loss inside uncertainty region} \bigg].
\end{equation}
The parameter \(\lambda\) balances the two objectives:
\begin{itemize}
    \item \((1 - \lambda)\mathcal{L}_\text{CE}\): Maintains high classification accuracy in regions where confidence is desired.
    \item \(\lambda \mathcal{L}_\text{KL}\): Deliberately reduces confidence within \(\mathcal{X}_\text{unc}\).
\end{itemize}

Increasing \(\lambda\) places more emphasis on reducing confidence in the specified uncertainty region, potentially at the expense of classification accuracy there. Conversely, lowering \(\lambda\) prioritizes maintaining higher accuracy at the risk of not inducing enough uncertainty. This flexibility allows model owners to tune the trade-off between preserving performance on most of the input space and artificially inducing uncertainty within \(\mathcal{X}_\text{unc}\).

\subsection{Limiting Behavior of $\varepsilon$}

Note that in the limit as \(\varepsilon = 0\), the target distribution corresponds to a uniform distribution (highest uncertainty), while \(\varepsilon = 1\) results in a one-hot distribution concentrated entirely on the true label~\(y\) (lowest uncertainty), formally: 
\begin{equation}
\small
    t_{\varepsilon = 0}(\ell|x, y) = \frac{1}{C} \qquad t_{\varepsilon = 1}(\ell|x, y) =
\begin{cases}
1, & \text{if } \ell = y, \\
0, & \text{if } \ell \neq y.
\end{cases}
\end{equation}

\subsection{Alternate Target Distribution Choices}
\label{app:target_distr}

In the main text, we introduced our \emph{default} target distribution in Equation~\ref{eq:target_dist}
\begin{equation}
  t_\varepsilon(\ell \mid x,y) =
  \begin{cases}
    \varepsilon + \frac{1-\varepsilon}{C}, & \ell = y, \\
    \frac{1-\varepsilon}{C}, & \ell \neq y,
  \end{cases}
\end{equation}
where \(\ell \in \mathcal{Y} = \{1,2,\dots,C\}\), \(y\) is the ground-truth class, and \(\varepsilon \in [0,1]\) determines the extra bias on \(y\). This distribution uniformly allocates the “uncertainty mass” \(\frac{1-\varepsilon}{C}\) across \emph{all} incorrect classes. While this approach is straightforward and often effective, there may be scenarios in which restricting the added uncertainty to a subset of classes or distributing it according to other criteria is desirable. Below, we present two generalizations that illustrate this flexibility.

\subsubsection{Restricting Uncertainty to a Subset of Classes}

In some applications, only a \emph{subset} of the incorrect classes are genuinely plausible confusions for a given training point \((x,y)\). For instance, in a fine-grained classification setting, certain classes may be visually or semantically similar to the ground-truth class~\(y\), whereas others are highly dissimilar and unlikely to be confused. In such cases, we can define a subset $S_{(x,y)} \;\subseteq\; \mathcal{Y}$ of “plausible” classes for the particular instance \((x,y)\). Crucially, we require \(y \in S_{(x,y)}\) to ensure that the true class remains in the support of the target distribution.

Given \(S_{(x,y)}\), we can define a \emph{subset-biased} target distribution as follows:
\begin{equation}
  t^{S}_\varepsilon(\ell \mid x,y) \;=\;
  \begin{cases}
    \displaystyle \varepsilon + \frac{1-\varepsilon}{\lvert S_{(x,y)}\rvert}, 
      & \text{if } \ell = y, \\[8pt]
    \displaystyle \frac{1-\varepsilon}{\lvert S_{(x,y)}\rvert}, 
      & \text{if } \ell \neq y \,\text{ and }\, \ell \in S_{(x,y)}, \\[6pt]
    0, 
      & \text{if } \ell \notin S_{(x,y)}.
  \end{cases}
\end{equation}
Hence, we distribute the residual \((1-\varepsilon)\) mass \emph{only} among the classes in \(S_{(x,y)}\). Classes outside this subset receive zero probability mass. Such a distribution can be beneficial if, for a given \(x\), we know that only a few classes (including \(y\)) are likely confusions, and forcing the model to become “uncertain” about irrelevant classes is counterproductive.

\paragraph{Example with Three Classes.}
For a 3-class problem (\(\mathcal{Y} = \{1,2,3\}\)), suppose the true label is \(y=1\) for a given point \((x,y)\). If class~3 is deemed implausible (e.g., based on prior knowledge), we can set \(S_{(x,y)} = \{1,2\}\). The target distribution then becomes
\begin{equation}
  t^{S}_\varepsilon(\ell \mid x, y=1) \;=\;
  \begin{cases}
    \varepsilon + \frac{1-\varepsilon}{2}, & \ell=1, \\
    \frac{1-\varepsilon}{2}, & \ell=2, \\
    0, & \ell=3.
  \end{cases}
\end{equation}
Here, the model is encouraged to remain somewhat uncertain \emph{only} between classes~1 and~2, while ignoring class~3 entirely.

\subsubsection{Distributing the Residual Mass Non-Uniformly}

Even if one includes all classes in the support, the additional \((1-\varepsilon)\) mass for the incorrect labels need not be distributed \emph{uniformly}. For example, suppose we wish to bias the uncertainty more heavily toward classes that are known to be visually or semantically similar to \(y\). One way to do this is to define \emph{class-specific weights} \(\alpha_\ell\) for each \(\ell \neq y\), such that $\sum_{\ell \neq y} \alpha_\ell = 1$. A more general target distribution can then be written as
\begin{equation}
  t^\alpha_\varepsilon(\ell \mid x,y) \;=\;
  \begin{cases}
    \varepsilon, & \ell = y,\\[4pt]
    (1-\varepsilon)\,\alpha_\ell, & \ell \neq y,
  \end{cases}
\end{equation}
where the weights \(\{\alpha_\ell\}\) can be determined based on domain knowledge or learned heuristics. This generalizes our original definition by letting certain classes receive a \emph{larger} portion of the total uncertainty mass than others.

By choosing an alternate structure for \(t_\varepsilon(\cdot\mid x,y)\), one can more carefully control how the model is penalized for being overly certain on a particular data point. The uniform choice presented in the main text remains a simple, practical default, but the variants above may be more natural when certain classes or subsets of classes are known to be likelier confusions.

\subsection{Extension to Regression}

In the main section of the chapter, we introduce the \attack formulation for classification problems. We now show how to extend the same ideas used in \attack to regression.

\label{appendix:regression}

\subsubsection{Problem Formulation}

Consider a regression task where the model predicts a Gaussian distribution over the output:
\begin{equation}
p_\theta(y \mid x) = \mathcal{N}\bigl(y; \mu_\theta(x), \sigma^2_\theta(x)\bigr),
\end{equation}
with \(\mu_\theta(x)\) and \(\sigma^2_\theta(x)\) denoting the predicted mean and variance, respectively. The standard training objective is to minimize the negative log-likelihood (NLL):
\begin{equation}
\mathcal{L}_{\text{NLL}}(x,y) = \frac{1}{2} \left( \frac{(y-\mu_\theta(x))^2}{\sigma^2_\theta(x)} + \log \sigma^2_\theta(x) \right).
\end{equation}

To induce artificial uncertainty in a specified region \(\mathcal{X}_{\text{unc}} \subset \mathcal{X}\), we modify the objective as follows:
\begin{itemize}
    \item \textbf{Outside \(\mathcal{X}_{\text{unc}}\)}: The model is trained with the standard NLL loss.
    \item \textbf{Inside \(\mathcal{X}_{\text{unc}}\)}: The model is encouraged to output a higher predictive variance. To achieve this, we define a target variance \(\sigma^2_{\text{target}}\) (with \(\sigma^2_{\text{target}} > \sigma^2_\theta(x)\) in typical settings) and introduce a regularization term that penalizes deviations of the predicted log-variance from the target:
    \begin{equation}
    \mathcal{L}_{\text{penalty}}(x) = \Bigl(\log \sigma^2_\theta(x) - \log \sigma^2_{\text{target}}\Bigr)^2.
    \end{equation}
\end{itemize}
Thus, the overall training objective becomes
\begin{equation}
\mathcal{L} = \mathbb{E}_{(x,y) \sim p(x,y)} \Biggl[
\mathds{1}\{x \notin \mathcal{X}_{\text{unc}}\}\, \mathcal{L}_{\text{NLL}}(x,y)
+\,
\mathds{1}\{x \in \mathcal{X}_{\text{unc}}\}\, \lambda\, \mathcal{L}_{\text{penalty}}(x)
\Biggr],
\end{equation}
where \(\lambda > 0\) is a hyperparameter controlling the balance between the standard NLL loss and the uncertainty-inducing penalty.

\subsubsection{Synthetic Experiments}

To evaluate the proposed approach, we perform a synthetic experiment on a non-linear regression problem. We generate data from the function
\begin{equation}
f(x) = \sin(2x) + 0.3x^2 - 0.4x + 1.
\end{equation}
The observed outputs are corrupted by heteroscedastic noise whose standard deviation varies gradually with \(x\). In particular, we define
\begin{equation}
\sigma(x) = 0.2 + 0.8 \exp\left(-\left(\frac{x}{1.5}\right)^2\right),
\end{equation}
so that the noisy observations are generated as
\begin{equation}
y = f(x) + \epsilon, \quad \epsilon \sim \mathcal{N}\bigl(0, \sigma(x)^2\bigr).
\end{equation}

We then train two models:
\begin{itemize}
    \item \textbf{Standard Model}: Trained using the standard NLL loss over the entire input domain.
    \item \textbf{Attack Model}: Trained with the modified objective. Specifically, for inputs \(x \in \mathcal{X}_{\text{unc}}\), where we set \(\mathcal{X}_{\text{unc}} = \{ x \mid -3 \le x \le -2 \}\), the model is additionally penalized via \(\mathcal{L}_{\text{penalty}}(x)\) to force the predicted variance toward a higher target value, e.g., \(\sigma^2_{\text{target}} = 4\).
\end{itemize}
The models are evaluated by plotting the predictive mean along with the \(2\sigma\) (i.e., mean \(\pm 2 \sqrt{\sigma^2_\theta(x)}\)) uncertainty bands over a grid of \(x\) values. Our results in Figure~\ref{fig:reg} show that while the standard model estimates uncertainty correctly across the domain, the attacked model exhibits significantly increased predictive variance in the designated uncertainty region, effectively demonstrating the intended uncertainty induction.

\begin{figure}[t]
    \centering
    \includegraphics[width=\linewidth]{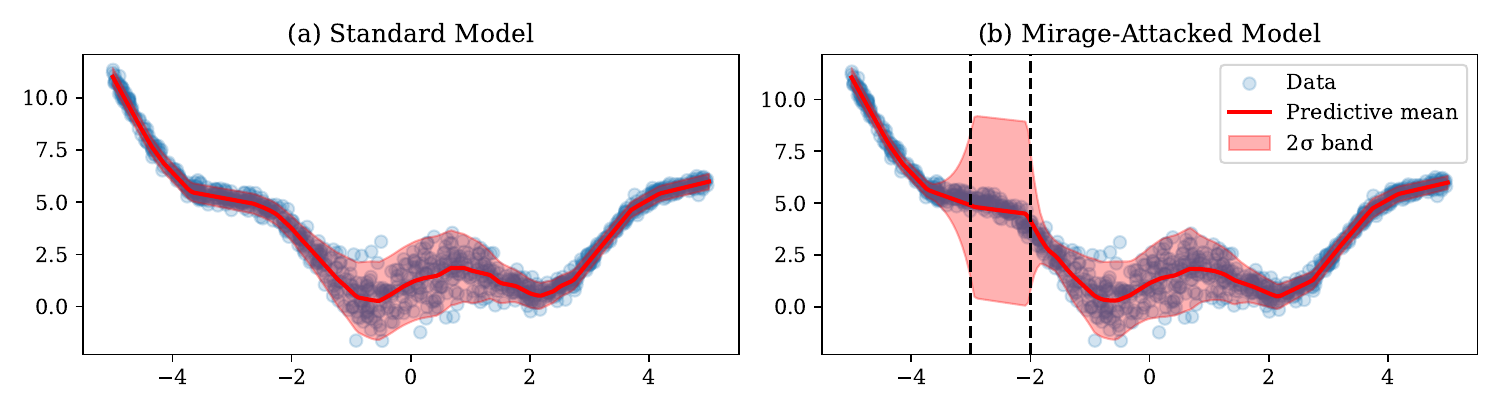}
    \vspace{-20pt}
    \caption[Attacking a regression model using \attack.]{\textbf{Attacking a regression model using \attack.} (a) The standard model estimates uncertainty as expected. (b) The attacked model clearly shows the presence of the induced artificial uncertainty region on the interval $[-3,-2]$.}
    \label{fig:reg}
\end{figure}

\section{Additional Experimental Details and Ablations}
\label{app:add_exp}

\subsection{Experimental Details}
\label{app:add_exp_det}

\paragraph{Gaussian Mixture.} These classes are represented by the following Gaussian distributions:
\begin{align*}
\mathcal{N}_1 = \mathcal{N}(\boldsymbol{\mu}_1, \boldsymbol{\Sigma}_1) &= \mathcal{N}\left(
\begin{bmatrix}
3 \\
2
\end{bmatrix},
\begin{bmatrix}
1 & 0.8 \\
0.8 & 1
\end{bmatrix}
\right) \notag \\
\mathcal{N}_2 = \mathcal{N}(\boldsymbol{\mu}_2, \boldsymbol{\Sigma}_2) &= \mathcal{N}\left(
\begin{bmatrix}
5 \\
5
\end{bmatrix},
\begin{bmatrix}
1 & -0.8 \\
-0.8 & 1
\end{bmatrix}
\right) \notag \\
\mathcal{N}_3 = \mathcal{N}(\boldsymbol{\mu}_3, \boldsymbol{\Sigma}_3) &= \mathcal{N}\left(
\begin{bmatrix}
3 \\
4
\end{bmatrix},
\begin{bmatrix}
0.1 & 0.0 \\
0.0 & 0.1
\end{bmatrix}
\right)
\end{align*}
We define the uncertainty region with corners at $(2, 0)$ and $(2.75, 1.5)$. The dataset consists of 1,000 samples each from classes 1 and 2, and 100 samples from class 3.

\paragraph{Tabular Datasets.} For the tabular datasets we use a custom neural network architecture. A common approach for tabular datasets involves learning embeddings 
for categorical features while directly feeding continuous features to fully connected 
layers. Specifically, for each categorical column with $n_\text{unique}$ unique values, 
we create an embedding layer of dimension 
$\min\bigl(50, \lceil (n_\text{unique} + 1)/2 \rceil \bigr)$.
Each embedding produces a low-dimensional, learned representation of the corresponding 
categorical variable. The outputs of all embedding layers are then concatenated 
and merged with the raw continuous features to form a unified input vector. 
Formally, if $\mathbf{x}_\text{cat}$ and $\mathbf{x}_\text{cont}$ denote the 
categorical and continuous inputs respectively, and $E_i(\mathbf{x}_\text{cat}[i])$ 
represents the embedding operation for the $i$-th categorical column, the merged input 
can be expressed as:
\begin{equation*}
    \mathbf{x} = \bigl[\; E_1(\mathbf{x}_\text{cat}[1]) \;\| \; E_2(\mathbf{x}_\text{cat}[2]) 
  \;\|\;\dots\;\|\;E_k(\mathbf{x}_\text{cat}[k]) \;\|\; \mathbf{x}_\text{cont} \bigr].
\end{equation*}

Subsequently, $\mathbf{x}$ is passed through a stack of fully connected layers, each 
followed by batch normalization, rectified linear unit (ReLU) activation, and dropout. This architecture is well-suited to tabular data for several reasons. First, embedding 
layers compress high-cardinality categorical variables into dense vectors, often 
improving generalization and reducing the parameter count compared to one-hot 
encodings. Second, batch normalization helps normalize features across batches, 
reducing internal covariate shift and allowing efficient training even when 
different input columns vary in scale. Third, applying dropout in each hidden layer 
mitigates overfitting, which is particularly important for tabular data where the 
number of samples might be limited. Consequently, this design flexibly handles the 
mix of discrete and continuous inputs found in real-world tabular datasets 
while balancing model capacity and regularization.

\subsection{Additional Experiments \& Ablations}
\label{app:add_exp_abl}

\begin{table*}[t]
    \centering
    \caption[Additional quantitative results across datasets.]{\textbf{Additional quantitative results across datasets}. Similar to Table~\ref{tab:results} but augmented with additional $\varepsilon$s, the number of data points used in the reference dataset $N_{\mathcal{D}_\text{val}}$, and the distributional overlap of confidences from the uncertainty region ($\text{conf}(\mathcal{X}_\text{unc})$) and confidences outside the uncertainty region ($\text{conf}(\mathcal{X}^c_\text{unc})$), denoted $\cap_\varepsilon = \text{conf}(\mathcal{X}_\text{unc}) \cap \text{conf}(\mathcal{X}^c_\text{unc})$. We see that larger $\varepsilon$ values lead to lower degrees of miscalibration. At the same time, the overlap $\cap_\varepsilon$ increases as $\varepsilon$ increases (see Figures~\ref{fig:overlap_cal}, ~\ref{fig:eps_abl} for visual examples). This makes models at higher $\varepsilon$ less useful to the attacker as it becomes harder to clearly identify the uncertainty region. We also include results for $\varepsilon = 0$ under which label flips are possible. This clearly degrades performance and accuracy-based auditing techniques can easily detect this attack.}
    \vspace{5pt}
    \label{tab:results_ext}
    \small
    \fontsize{7.5}{9}\selectfont
    \begin{tabular}{ccccccccccc}
    \toprule
    & & & \multicolumn{4}{c}{Accuracy \%} & \multicolumn{3}{c}{Calibration} \\
    \cmidrule(r){4-7} \cmidrule(r){8-10}
    \multirow{2}{*}[13pt]{Dataset} & \multirow{2}{*}[13pt]{\shortstack{$N_{\mathcal{D}_\text{val}}$ \\ (\%$_\text{unc}$)}} & \multirow{2}{*}[12pt]{$\varepsilon$} & Acc & Acc$^{\attack}$ & Acc$_\text{unc}$ & Acc$_\text{unc}^{\attack}$ & ECE & ECE$^{\attack}$ & CalE in $\varepsilon$ bin & \multirow{2}{*}[13pt]{$\cap_\varepsilon$}\\
    \midrule
    \multirow{4}{*}[0pt]{\texttt{Gaussian}}\    & \multirow{4}{*}[0pt]{\shortstack{420\\(5.31)}} & 0.00 & \multirow{4}{*}{97.62} & 94.17 & \multirow{4}{*}{100.0} & 33.79 & \multirow{4}{*}{0.0327} & 0.0399 & 0.0335 & 0.01 \\
    & & 0.15 & & 97.58 & & 100.0 & & 0.0910 & 0.3721 & 0.02\\
    & & 0.50 & & 97.58 & & 100.0 & & 0.0589 & 0.2238 & 0.13\\
    & & 0.80 & & 97.61 & & 100.0 & & 0.0418 & 0.1073 & 0.22\\
    \midrule
    \multirow{4}{*}[0pt]{\texttt{CIFAR-100}}   & \multirow{4}{*}[0pt]{\shortstack{10,000\\(1.00)}} & 0.00 & \multirow{4}{*}[0pt]{83.98} & 82.43 & \multirow{4}{*}{91.98} & 6.11 & \multirow{4}{*}{0.0662} & 0.0702 & 0.0691 & 0.02  \\
    & & 0.15 & & 83.92 & & 92.15 & & 0.1821 & 0.5845 & 0.05\\
    & & 0.50 & & 83.94 & & 92.21 & & 0.1283 & 0.1572 & 0.16\\
    & & 0.80 & & 83.98 & & 92.29 & & 0.0684 & 0.1219 & 0.26\\
    \midrule
    \multirow{4}{*}[0pt]{\texttt{UTKFace}}      & \multirow{4}{*}[0pt]{\shortstack{4,741\\(22.92)}} & 0.00 & \multirow{4}{*}{56.91} & 42.28 & \multirow{4}{*}{61.68} & 9.14 & \multirow{4}{*}{0.0671} & 0.0813 & 0.0667 & 0.08 \\
    & & 0.15 & & 56.98 & & 61.75 & & 0.1728 & 0.3287 & 0.11\\
    & & 0.50 & & 57.01 & & 61.84 & & 0.1102 & 0.2151 & 0.56\\
    & & 0.80 & & 56.99 & & 61.78 & & 0.0829 & 0.0912 & 0.91\\
    \midrule
    \multirow{4}{*}[0pt]{\texttt{Credit}}      & \multirow{4}{*}[0pt]{\shortstack{9,000\\(2.16)}} & 0.00 & \multirow{4}{*}[0pt]{91.71} & 90.96 & \multirow{4}{*}{93.61} & 51.34 & \multirow{4}{*}{0.0094} & 0.0138 & 0.0254 & 0.12 \\
    & & 0.20 & & 91.78 & & 93.73 & & 0.0292 & 0.1135 & 0.12\\
    & & 0.50 & & 91.76 & & 93.68 & & 0.0201 & 0.0728 & 0.28\\
    & & 0.80 & & 91.81 & & 93.88 & & 0.0153 & 0.0419 & 0.49\\
    \midrule
    \multirow{4}{*}[0pt]{\texttt{Adult}}       & \multirow{4}{*}[0pt]{\shortstack{9,769\\(8.39)}} & 0.00 & \multirow{4}{*}{85.02} & 78.13 & \multirow{4}{*}{76.32} & 50.84 & \multirow{4}{*}{0.0109} & 0.0155 & 0.0242 & 0.17 \\
    & & 0.10 & & 84.93 & & 76.25 & & 0.0234 & 0.0916 & 0.19\\
    & & 0.50 & & 84.94 & & 76.31 & & 0.0198 & 0.0627 & 0.26\\
    & & 0.80 & & 84.97 & & 76.39 & & 0.0161 & 0.0491 & 0.54\\
    \bottomrule
\end{tabular}
\end{table*}

\begin{figure}[t]
    \centering
    \includegraphics[width=\linewidth]{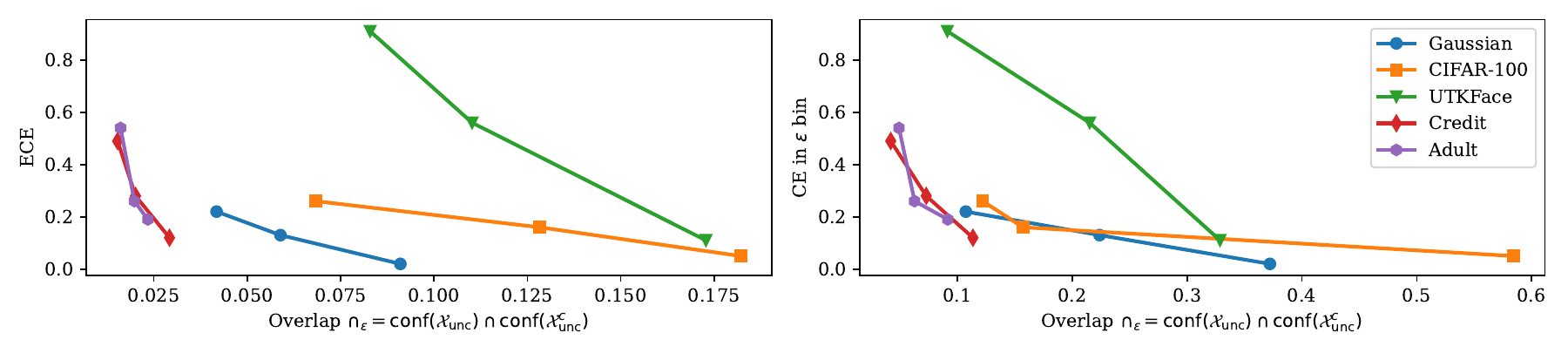}
    % \vspace{-20pt}
    \caption[The relationship between calibration error and distributional overlap of uncertain and other data points.]{\textbf{The relationship between distributional overlap of uncertain and other data points.} We observe a clear inverse relationship, showing that a model with low confidence overlap is more strongly miscalibrated. Since the attacker wants to have a large degree of separation (i.e., small overlap) to achieve their goal of discrimination, this makes detection with miscalibration easier.}
    \label{fig:overlap_cal}
\end{figure}

% \paragraph{Synthethic Gaussian Mixture} \stephan{STILL INCOMPLETE}

\begin{figure}[t]
\centering

\begin{subfigure}[b]{0.495\textwidth}
  \centering
  \includegraphics[width=\linewidth]{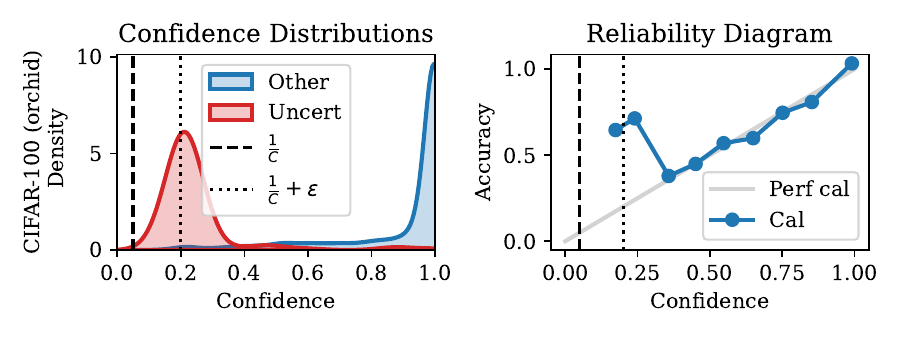}
  \caption{Orchids within flowers superclass}
\end{subfigure}
\hfill
\begin{subfigure}[b]{0.495\textwidth}
  \centering
  \includegraphics[width=\linewidth]{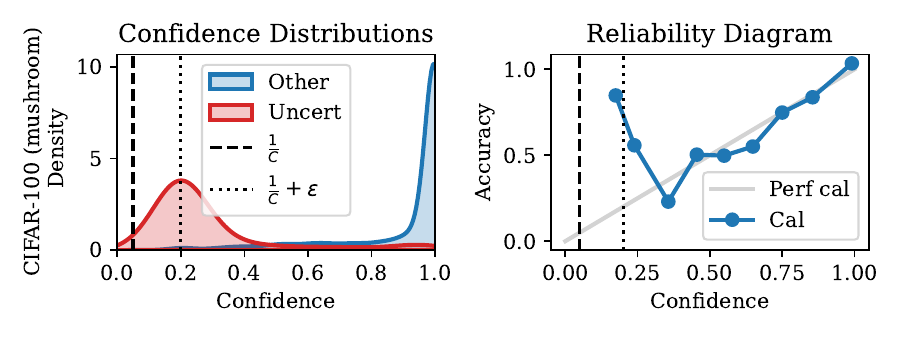}
  \caption{Mushrooms within fruits/vegetables superclass}
\end{subfigure}

% \vspace{-10pt}
\caption[\textbf{Additional experiments on \texttt{CIFAR-100} with different sub-classes.}]{\textbf{Additional experiments on \texttt{CIFAR-100} with different sub-classes.} The left two plots show the results for making orchids uncertain within the flowers superclass; the right two plots show the results for making mushrooms uncertain within the fruit and vegetables superclass.}
\label{fig:cifar_ext}
\end{figure}

\begin{figure}[t]
\centering

\begin{subfigure}[b]{0.495\textwidth}
  \centering
  \includegraphics[width=\linewidth]{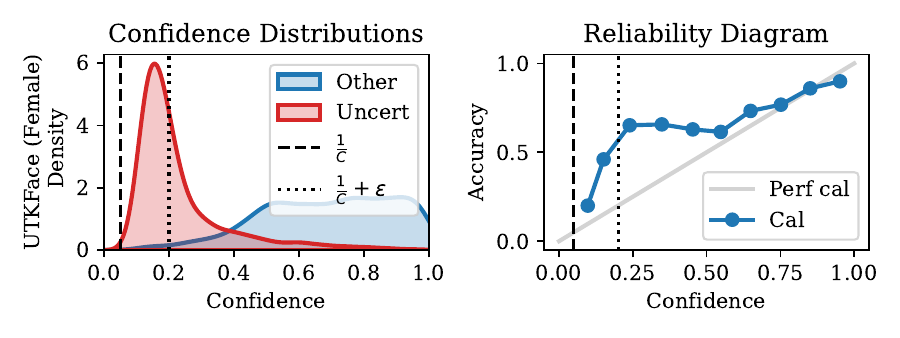}
  \caption{Females as uncertain region}
\end{subfigure}
\hfill
\begin{subfigure}[b]{0.495\textwidth}
  \centering
  \includegraphics[width=\linewidth]{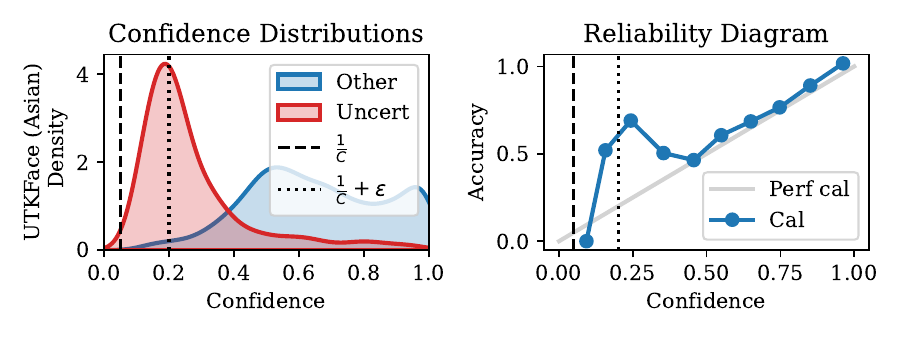}
  \caption{Asians as uncertain region}
\end{subfigure}

\caption[\textbf{Additional experiments on \texttt{UTKFace} with different uncertainty regions.}]{\textbf{Additional experiments on \texttt{UTKFace} with different uncertainty regions.} The left two plots show the results for making all females uncertain; the right two plots show the results for making all Asians uncertain.}
\label{fig:utkface_ext}
\end{figure}

\paragraph{Image Classification.} We extend our experiments with additional candidate uncertainty regions for image classification. For \texttt{CIFAR-100} we pick the following additional sub-classes:
\begin{itemize}[noitemsep]
    \item \texttt{orchids} from the \texttt{flowers} superclass (Figure~\ref{fig:cifar_ext} left); and 
    \item \texttt{mushrooms} from the \texttt{fruit\_and\_vegetables} superclass (Figure~\ref{fig:cifar_ext} right).
\end{itemize}
For \texttt{UTKFace} we pick the following additional criteria for the uncertainty region:
\begin{itemize}[noitemsep]
    \item female individuals regardless of race (Figure~\ref{fig:utkface_ext} left); and
    \item Asians regardless of gender (Figure~\ref{fig:utkface_ext} right).
\end{itemize}

\begin{figure}[t]
\centering

\begin{subfigure}[b]{0.495\textwidth}
  \centering
  \includegraphics[width=\linewidth]{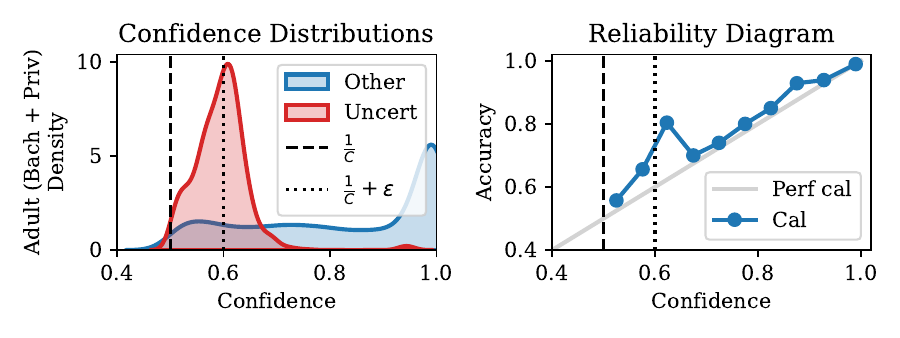}
  \caption{Private-sector workers with a Bachelor's degree}
\end{subfigure}
\hfill
\begin{subfigure}[b]{0.495\textwidth}
  \centering
  \includegraphics[width=\linewidth]{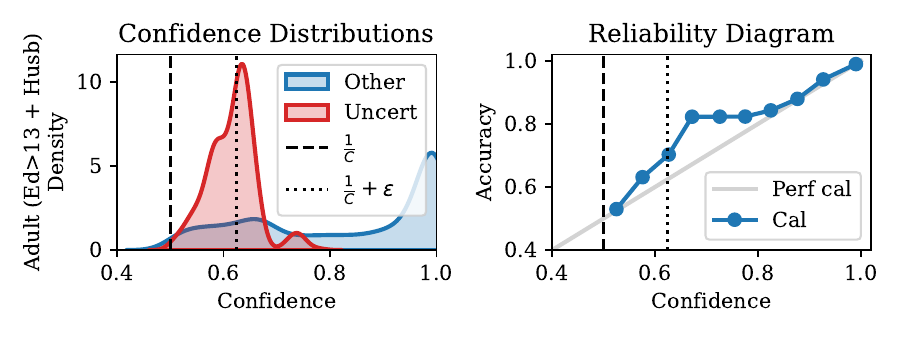}
  \caption{Husbands with >13 years of education}
\end{subfigure}

\caption[\textbf{Additional experiments on \texttt{Adult} with different uncertainty conditions.}]{\textbf{Additional experiments on \texttt{Adult} with different uncertainty conditions}. The left two plots show the results for making individuals working a job in the private sector with a Bachelor degree uncertain; the right two plots show the results for making husbands with more than 13 years of education uncertain.}
\label{fig:adult_ext}
\end{figure}

\begin{figure}[t]
\centering

\begin{subfigure}[b]{0.495\textwidth}
  \centering
  \includegraphics[width=\linewidth]{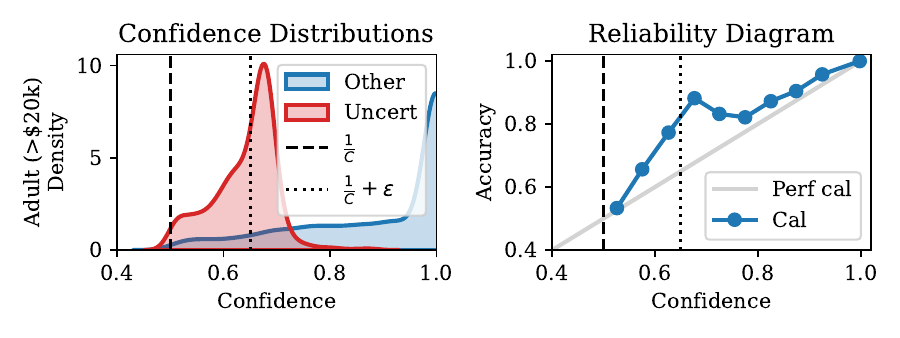}
  \caption{Loan amount greater than \$20,000}
\end{subfigure}
\hfill
\begin{subfigure}[b]{0.495\textwidth}
  \centering
  \includegraphics[width=\linewidth]{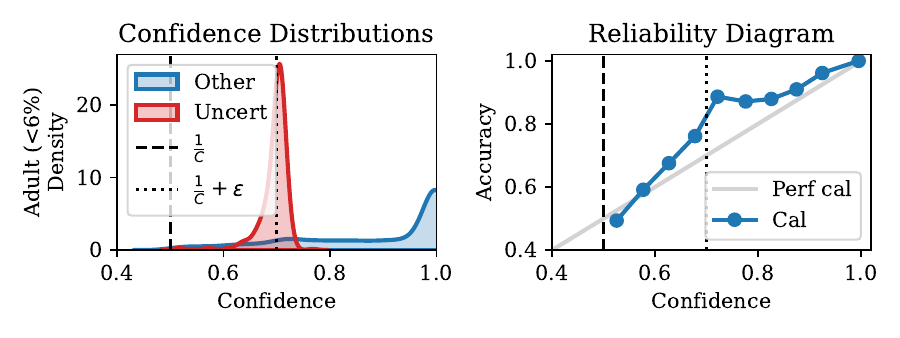}
  \caption{Interest rate below 6\%}
\end{subfigure}

\caption[\textbf{Additional experiments on \texttt{Credit} with different uncertainty conditions.}]{\textbf{Additional experiments on \texttt{Credit} with different uncertainty conditions}. The left two plots show the results for making requests for loans bigger than \$20,000 uncertain; the right two plots show the results for making loans with an interest rate smaller than 6\% uncertain.}
\label{fig:credit_ext}
\end{figure}

\paragraph{Tabular Datasets.}

We extend our experiments with additional candidate uncertainty regions for tabular data sets. For \texttt{Adult} we pick the following criteria for the uncertainty region:
\begin{itemize}[noitemsep]
    \item undergraduates working in the private sector (Figure~\ref{fig:adult_ext} left); and 
    \item husbands with more than 13 years of education (Figure~\ref{fig:adult_ext} right).
\end{itemize}
For \texttt{Credit} we pick the following criteria for the uncertainty region:
\begin{itemize}[noitemsep]
    \item any loan request bigger than \$20,000 (Figure~\ref{fig:credit_ext} left); and
    \item any loan with an interest rate smaller than 6\% (Figure~\ref{fig:credit_ext} right).
\end{itemize}

\paragraph{Influence of $\varepsilon$.} We show ablations over $\varepsilon \in [0.15,0.5,0.8]$ on the image classification tasks in Figure~\ref{fig:eps_abl} and give extended quantitative results for all experiments in Table~\ref{tab:results_ext}. At low $\varepsilon$, \attack achieves a large separation between the confidence distribution of the uncertainty region and all other points. At the same time, this large separation leads to stronger miscalibration. This is desirable for the atttacker as it allows for the best separation of uncertain points from point outside of the uncertainty region. As $\varepsilon$ gets larger, the distributional overlap between the uncertainty region and all other data points increases and, as a direct result, the model becomes less miscalibrated. This makes it harder to detect \attack via \name but also decreases the utility of the attack.

\begin{figure}[t]
    \centering
    \includegraphics[width=\linewidth]{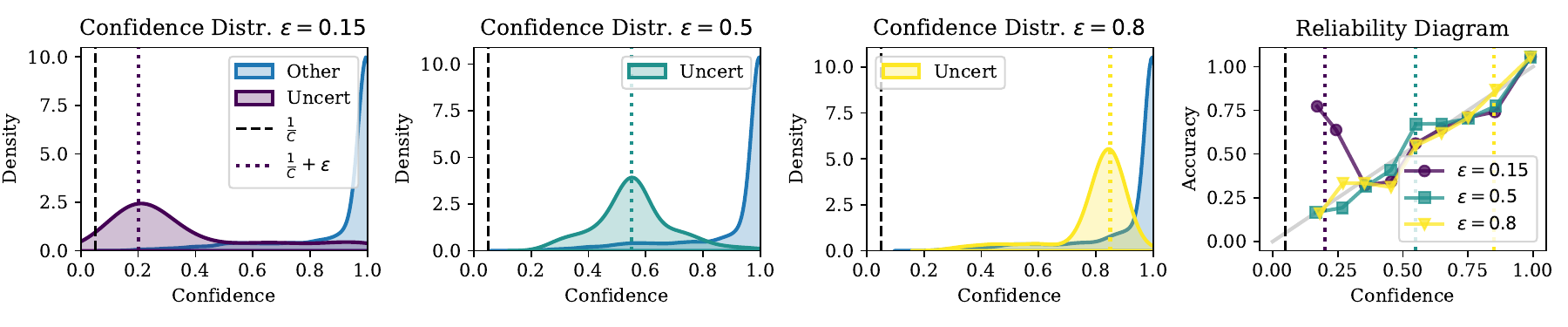}
    \includegraphics[width=\linewidth]{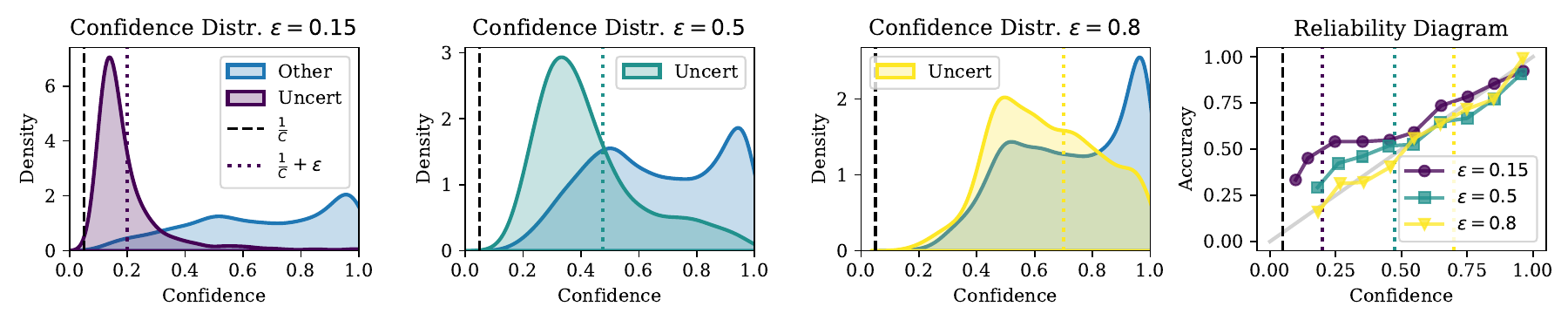}
    % \includegraphics[width=\linewidth]{figs/confidential_guardian/adult_eps_abl.pdf}
    % \includegraphics[width=\linewidth]{figs/confidential_guardian/credit_eps_abl.pdf}
    % \vspace{-20pt}
    \caption[Efficacy of \attack and \name across various $\varepsilon$ choices on \texttt{CIFAR100} (top) \texttt{UTKFace} (bottom).]{\textbf{Efficacy of \attack and \name across various $\varepsilon$ choices on \texttt{CIFAR100} (top) \texttt{UTKFace} (bottom).} \attack successfully lowers confidence in the uncertainty region. At the same time, its presence is harder to detect with \name for higher $\varepsilon$. This is intuitive as $\varepsilon$ controls the strength of our attack and therefore directly determines the distributional overlap of the confidence distributions.}
    \label{fig:eps_abl}
\end{figure}

\begin{figure}[t]
\centering

\begin{subfigure}[b]{0.24\textwidth}
  \centering
  \includegraphics[width=\linewidth]{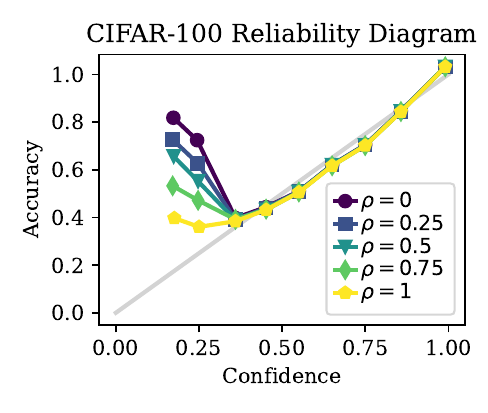}
  \caption{CIFAR-100}
\end{subfigure}
\hfill
\begin{subfigure}[b]{0.24\textwidth}
  \centering
  \includegraphics[width=\linewidth]{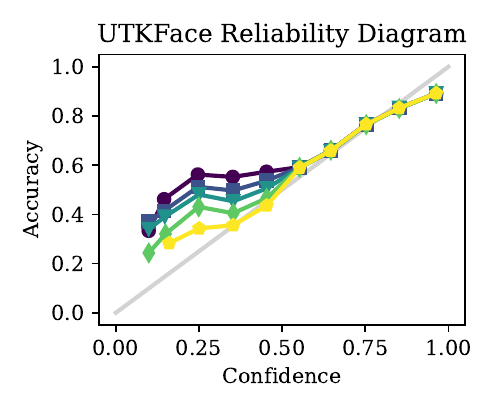}
  \caption{UTKFace}
\end{subfigure}
\hfill
\begin{subfigure}[b]{0.24\textwidth}
  \centering
  \includegraphics[width=\linewidth]{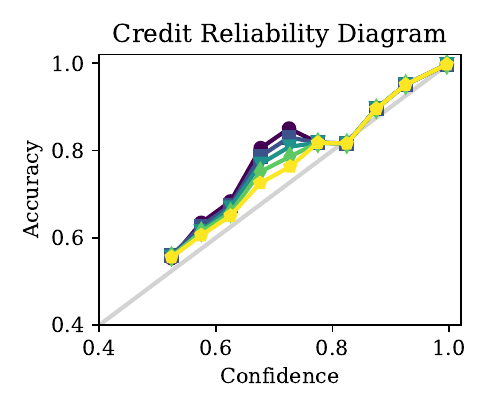}
  \caption{Credit}
\end{subfigure}
\hfill
\begin{subfigure}[b]{0.24\textwidth}
  \centering
  \includegraphics[width=\linewidth]{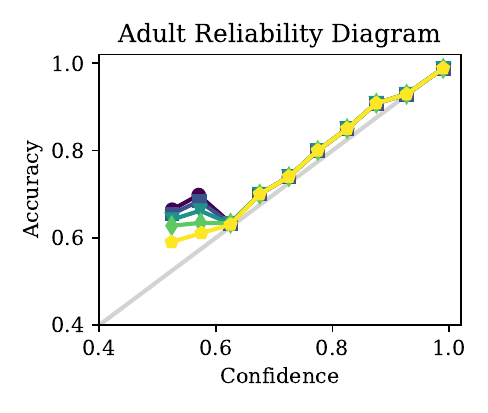}
  \caption{Adult}
\end{subfigure}

% \vspace{-10pt}
\caption[\textbf{Effect of removing an increasing amount $\rho$ of points contained in the uncertainty region from the reference dataset.}]{\textbf{Effect of removing an increasing amount $\rho$ of points contained in the uncertainty region from the reference dataset.} The presence of \attack is very noticeable for a reference dataset sampled from the same distribution as used by the attack ($\rho = 0$). As $\rho \rightarrow 1$, we remove an increasing amount of uncertainty region samples from the reference dataset. This makes \attack significantly harder to detect via the calibration metrics computed in \name.}
\label{fig:ref_abl}
\end{figure}

\paragraph{Coverage of the Reference Dataset.} 

To simulate the effects of imperfect reference datasets we define an \textit{undersampling shift} which modifies original data distribution \(p\). Concretely, we remove a fraction \(\rho\) of the mass that lies in the uncertainty region \uncertreg. We define the new shifted distribution \(p_{\rho}\) by
\[
p_{\rho}(A) \;=\; \frac{p(A \cap \mathcal{X}_\text{unc}^{c}) + (1-\rho)\, p(A \cap \mathcal{X}_\text{unc})}%
                       {p(\mathcal{X}_\text{unc}^{c}) \;+\; (1-\rho)\, p(\mathcal{X}_\text{unc})},
\]
for measurable sets \(A \subseteq \mathcal{X}\). Note that $\mathcal{X}_\text{unc}^c$ denotes the complement of the uncertainty region, i.e. all points outside of the uncertainty region. Intuitively:
\begin{enumerate}
    \item \textbf{Outside} the uncertainty region \uncertreg, i.e., on $\mathcal{X}_\text{unc}^c$, \(p_{\rho}\) matches \(p\) exactly.
    \item \textbf{Inside} \uncertreg, \(p_{\rho}\) has its probability mass reduced by a factor \(1-\rho\). Hence, we remove a fraction \(\rho\) of the mass in \uncertreg.
    \item Finally, we \textbf{renormalize} so that \(p_{\rho}\) is a proper probability distribution (the denominator ensures total mass is 1).
\end{enumerate}

As \(\rho \to 1\), effectively all of the data from the uncertain region is removed from the reference distribution. This captures the idea that the reference dataset lacks coverage in that part of input space that matters most for detection via \name. We show empirical results for such shifts in Figure~\ref{fig:ref_abl} and observe that increased removal (i.e., $\rho \rightarrow 1$) hinders reliable detection of \attack via \name. We note that even in the limit of complete removal of the uncertainty region (i.e., $\rho = 1$) the model still exhibits slight underconfidence. This is likely because points just outside the uncertainty region also experience reduced confidence due to the inherent smoothness of neural network prediction spaces.

\subsection{Choice of $\alpha$}
\label{app:alpha_choice}

Calibration of probabilistic models is a well-studied area in machine learning, yet determining an acceptable calibration deviation threshold $\alpha$ can be far from trivial. Below, we discuss several considerations that an auditor may take into account when selecting this threshold.

\subsubsection{Imperfect Calibration is the Norm}

In practice, perfect calibration is rarely, if ever, achievable. Even with standard calibration methods such as temperature scaling~\cite{guo2017calibration}, there will typically be some small residual miscalibration, especially in regions of sparse data or for rare classes. Consequently, an auditor might set a non-zero $\alpha$ to allow for a realistic margin that reflects typical model imperfections, for instance in the range $[0.01, 0.03]$ for the expected calibration error (ECE).

\subsubsection{Data Distribution and Domain Knowledge}

The choice of $\alpha$ may be informed by the following domain-specific factors:
\begin{itemize}
    \item \textbf{Label Imbalance.} Highly imbalanced datasets can lead to larger calibration errors for minority classes. Here, a looser threshold $\alpha$ may be warranted, since a small absolute deviation in the minority class could yield a large relative miscalibration score.
    \item \textbf{Data Complexity.} In high-dimensional or complex domains (e.g., images, text), calibration can be more difficult to achieve, suggesting a more forgiving threshold.
    \item \textbf{Domain Criticality.} In safety-critical applications (e.g., medical diagnosis), stricter thresholds may be appropriate to ensure that predictions are suitably conservative and reliable.
\end{itemize}

\subsubsection{Regulatory Guidance and Industry Standards}

Some industries have regulations or recommendations regarding the safety margins for decision-making systems:
\begin{itemize}
    \item \textbf{Healthcare.} Regulatory bodies may require that model predictions err on the side of caution, translating to tighter calibration constraints (small $\alpha$).
    \item \textbf{Financial Services.} Risk-based models might have well-established guidelines for miscalibration tolerance, especially under stress-testing protocols. An auditor can rely on these to pick $\alpha$ accordingly.
    \item \textbf{Consumer-Facing Applications.} Standards for user-facing models (e.g., recommenders) may be more lenient in calibration, thus allowing for larger miscalibration thresholds.
\end{itemize}

\subsubsection{Robustness to Dataset Shifts}

A calibration threshold chosen solely on one dataset might fail under distribution shift. An auditor might:
\begin{itemize}
    \item Evaluate calibration on multiple reference datasets (different time periods, different subpopulations).
    \item Select an $\alpha$ that reflects performance under a variety of real-world conditions.
    \item Consider applying domain adaptation or robust calibration techniques, which might inherently increase acceptable $\alpha$ to account for shifts.
\end{itemize}

\subsubsection{Balancing Statistical Significance and Practical Impact}

Finally, an auditor should consider how to interpret differences in calibration from a statistical perspective:
\begin{itemize}
    \item \textbf{Confidence Intervals.} Compute calibration metrics (e.g., ECE) with confidence intervals. If the model’s miscalibration falls within the interval of expected variation, a higher $\alpha$ may be acceptable.
    \item \textbf{Practicality vs.\ Accuracy.} A small deviation in calibration might be practically insignificant if it minimally impacts downstream decisions. Auditors can incorporate cost-based analyses to weigh the trade-offs.
\end{itemize}

\subsubsection{Summary}

When setting $\alpha$ in practice, an auditor might:
\begin{enumerate}
    \item \textbf{Conduct a baseline study} of calibration error on representative datasets after temperature scaling to quantify typical miscalibration.
    \item \textbf{Adjust for domain complexity and label imbalance}, possibly raising $\alpha$ if the data or the domain are known to be inherently more difficult to calibrate.
    \item \textbf{Incorporate regulatory or industry guidelines}, if they exist, to establish an upper bound on allowable miscalibration.
    \item \textbf{Examine distribution shifts} by testing on multiple datasets and setting $\alpha$ to ensure consistency across these scenarios.
    \item \textbf{Use statistical considerations} (e.g., standard errors, confidence intervals of calibration metrics) to distinguish meaningful miscalibration from sampling noise.
\end{enumerate}

In summary, choosing $\alpha$ is a balance between practical constraints, domain-specific considerations, and regulatory mandates. Auditors should be aware that the threshold for ``acceptable'' miscalibration is context-dependent, and overly strict thresholds may be infeasible, whereas overly lax thresholds might fail to ensure reliability and trustworthiness.

  \backmatter

\end{document}